\documentclass{article}

\PassOptionsToPackage{numbers, compress}{natbib}
\PassOptionsToPackage{colorlinks=true,citecolor=blue}{hyperref}

    \usepackage[final]{neurips_2023}

\usepackage[utf8]{inputenc} 
\usepackage[T1]{fontenc}    
\usepackage{hyperref}       
\usepackage{url}            
\usepackage{booktabs}       
\usepackage{amsfonts}       
\usepackage{nicefrac}       
\usepackage{microtype}      
\usepackage{xcolor}         
\usepackage{amsthm,amsmath,amsfonts,amssymb}
\theoremstyle{plain}

\newtheorem{theorem}{Theorem}[section]
\newtheorem{lemma}[theorem]{Lemma}
\usepackage{amssymb}
\usepackage{multirow}
\usepackage{array}
\usepackage{bm}
\usepackage{booktabs}
\usepackage{float}
\usepackage{multirow}
\usepackage{algorithm}
\usepackage{algorithmic}
\usepackage{mathtools}
\usepackage{mathrsfs}
\usepackage{dsfont}
\usepackage{relsize}
\usepackage{rotating}
\usepackage{enumitem}
\usepackage{setspace}
\usepackage{minitoc}
\usepackage{etoc}
\usepackage{listings}
\usepackage{wrapfig}

\def\beq#1\eeq{\begin{equation}#1\end{equation}}
\def\baa#1\eaa{\begin{eqnarray}#1\end{eqnarray}}
\def\bal#1\eal{\begin{align}#1\end{align}}

\DeclareMathOperator*{\argmin}{arg\,min}
\DeclareMathOperator*{\argmax}{arg\,max}

\def\##1\#{\begin{align}#1\end{align}}
\def\$#1\${\begin{align*}#1\end{align*}}

\newcommand{\RNum}[1]{\uppercase\expandafter{\romannumeral #1\relax}}

\newtheorem{thm}{Theorem}
\counterwithin{thm}{section}

\counterwithin{proposition}{section}
\newtheorem{corollary}{Corollary}
\counterwithin{corollary}{section}
\newtheorem{assump}{Assumption}

\makeatletter
\newcommand{\vast}{\bBigg@{3.5}}
\newcommand{\Vast}{\bBigg@{5}}
\makeatother

\counterwithin{theorem}{section}
\newtheorem{definition}[theorem]{Definition}

\title{Bi-Level Offline Policy Optimization with Limited Exploration}

\author{%
  Wenzhuo Zhou\\
  Department of Statistics \\ 
   University of California Irvine\\
  \texttt{wenzhuz3@uci.edu} \\
}

\begin{document}

\maketitle

\begin{abstract}
We study offline reinforcement learning (RL) which seeks to learn a good policy based on a fixed, pre-collected dataset. A fundamental challenge behind this task is the distributional shift due to the dataset lacking sufficient exploration, especially under function approximation. To tackle this issue, we propose a bi-level structured policy optimization algorithm that models a hierarchical interaction between the policy (upper-level) and the value function (lower-level). The lower level focuses on constructing a confidence set of value estimates that maintain sufficiently small weighted average Bellman errors, while controlling uncertainty arising from distribution mismatch. Subsequently, at the upper level, the policy aims to maximize a conservative value estimate from the confidence set formed at the lower level. This novel formulation preserves the maximum flexibility of the implicitly induced exploratory data distribution, enabling the power of model extrapolation. In practice, it can be solved through a computationally efficient, penalized adversarial estimation procedure. Our theoretical regret guarantees do not rely on any data-coverage and completeness-type assumptions, only requiring realizability. These guarantees also demonstrate that the learned policy represents the ``best effort'' among all policies, as no other policies can outperform it. We evaluate our model using a blend of synthetic, benchmark, and real-world datasets for offline RL, showing that it performs competitively with state-of-the-art methods.
\end{abstract}

\section{Introduction}
\vspace{-1mm}

Offline reinforcement learning (RL) is a task to learn a good policy using only a pre-collected, fixed dataset, without further exploration with the environment. This distinctive characteristic positions offline RL as a promising approach for solving real-world sequential decision-making problems in healthcare \citep{murphy2001marginal,zhou2022estimating}, financial marketing \citep{theocharous2020reinforcement}, robotics \citep{thomas2015high} and education \citep{mandel2014offline}, as acquiring diverse or expert-quality data in these fields can be costly or practically unattainable.

Arguably, two of the biggest challenges in offline RL are the distributional shift between the data-generating distribution and those induced by candidate policies, and the stringent requirements on the properties of function approximation \citep{levine2020offline}. It has been observed that, in practice, the distributional mismatch often results in unsatisfactory performance of many existing algorithms, and even amplifying with function approximation \citep{fujimoto2019off,lee2021optidice}. Many prior works \citep{precup2000eligibility,ernst2005tree, antos2008learning,farahmand2016regularized}  crucially rely on a global data-coverage assumption and completeness-type function approximation condition in a technical sense. The former necessitates that the dataset to contain any state-action pair with a lower bounded probability so that the distributional shift can be well calibrated. The latter requires the function class to be closed under Bellman updates. Both assumptions are particularly strong and are likely to be violated in practice \citep{zanette2021provable}. Consequently, algorithms that depend on these assumptions may experience performance degradation and instability \citep{wang2021instabilities}. Therefore, it is crucial to develop novel algorithms that relax these assumptions, offering robust and 
widely applicable solutions for real-world scenarios.

To address the aforementioned challenges in offline settings, one fundamental principle is the concept of pessimism, which aims to maximize rewards in the worst possible MDP consistent with the offline dataset \citep{fujimoto2019off,xie2021bellman}. In practice, these methods have generally been shown to be more robust when coverage assumptions are violated \citep{lee2021optidice}.

Although many such pessimistic algorithms have been developed, very few works can tackle data-coverage and function approximation issues simultaneously, while establishing strong regret guarantees. For instance, deep offline RL algorithms \citep{fujimoto2019off, kumar2020conservative,kostrikov2021offline} exhibit impressive empirical performance, but their theoretical consistency guarantees are limited to tabular Markov decision processes (MDPs). The works of \citep{liu2020provably,rashidinejad2021bridging,jin2021pessimism,xie2021bellman,yin2021near,zanette2021provable, cheng2022adversarially} relax the global coverage to a partial coverage condition, wherein the offline data only covers a single comparator policy. However, all of these methods require Bellman completeness for the function class. The most recent works \citep{chen2022offline,zhan2022offline} take a significant step towards relaxing Bellman completeness to realizability, that the function class can capture the target ground-truth function. Nonetheless, these algorithms are unable to provide a meaningful regret guarantee without any data-coverage assumption (when both global and partial coverage fails), and also empirical evaluations are absent. Even without additional conditions, the learned policies of these algorithms can only compete with the (Bellman flow) optimal policy, resulting in a lack of robustness when the optimal policy is not covered by data, a situation that frequently occurs. Due to page limit, we have only discussed the closest related work here, and the rest is deferred to Appendix.

\textbf{Our contribution.} In this paper, we develop a provably sample-efficient offline RL framework. Our information-theoretic algorithm is designed based on the concept of bi-level (upper and lower level) structured optimization, which leads to a hierarchical interpretation and naturally enjoys learning stability and algorithmic convergence from a game-theoretic perspective. In particular, at the lower level, one component is to construct a confidence set with consistent value estimates regarding the appropriately small weighted average Bellman error, effectively preventing overly pessimistic evaluation. Meanwhile, the second component, which deals with uncertainty control, implicitly enhances the power of model extrapolation. In addition to the information-theoretic algorithm, we also develop a computationally efficient counterpart that is solved by a penalized adversarial estimation algorithm with proximal-mapping updating, allowing both non-linear and linear function approximation. From a theoretical standpoint, we establish a strong regret guarantee for both information-theoretical and practical algorithms under only realizability 
 \textit{without} requiring any data-coverage (neither global nor partial coverage) and completeness-type assumptions. As a special case study, we further refine our developed mixture density ratio-based concentrability coefficient to a relative condition number in linear MDP settings. The sample complexity of our regret bound improves or at least matches the prior results in the fully exploratory or partial coverage settings where the Bellman-completeness holds. Notably, compared with existing works, either focusing on theoretical or empirical development, we provide a comprehensive theoretical analysis of the proposed framework and also conduct synthetic, benchmark, and real data experiments for empirical evaluation.



\vspace{-1.5mm}
\section{Preliminaries and Notations}
\vspace{-2mm}

\textbf{Markov decision process.} We consider an infinite-horizon discounted MDP $\mathcal{M}=\{\mathcal{S}, \mathcal{A}, \mathds{P}, \gamma, r, s^{0}\}
$ \citep{puterman2014markov},  where $\mathcal{S}$ is the state space, $\mathcal{A}$ is the action space, $
\mathds{P}: \mathcal{S} \times \mathcal{A} \rightarrow \Delta(\mathcal{S})
$ is the Markov transition kernel for some probabilistic simplex $\Delta$, $
r: \mathcal{S} \times \mathcal{A} \rightarrow [0, \bar{R}]
$ is the reward function for $\bar{R} \geq 0$, $\gamma \in [0,1)$ is the discounted factor and $s^0$ is the initial state. A policy $\pi: \mathcal{S} \rightarrow  \Delta(\mathcal{A})$  induces a distribution of the trajectory $
s^{0}, a^{0}, r^{0}, s^{1}, \ldots$, where $
a^{t} \sim \pi(\cdot | s^{t}), r^{t}=r(s^{t}, a^{t}), s^{t+1} \sim \mathds{P}(\cdot|s^{t}, a^{t})$ for any $t \geq 0$. The expected discounted return of a policy is defined as $J(\pi) = \mathbb{E}[\sum_{t=0}^{\infty} \gamma^{t} r^{t} | \pi]$. The discounted return when the trajectory starts with $(s, a)$ and all remaining actions are taken according to $\pi$ is called $q$-function $q^{\pi}: \mathcal{S} \times \mathcal{A} \rightarrow [0, \bar{V}]$. The $q^{\pi}$ is the unique fixed point of the Bellman operator $\mathcal{B}^{\pi}
$, satisfying the Bellman equation \citep{sutton2018reinforcement}: 
$\mathcal{B}^{\pi} q(s, a) \coloneqq r(s, a)+\gamma \mathbb{E}_{s^{\prime} \sim \mathds{P}(\cdot | s, a)}[q(s^{\prime}, \pi)].
$
Here $q(s^{\prime}, \pi)$ is denoted as shorthand for $
\mathbb{E}_{a^{\prime} \sim \pi\left(\cdot | s^{\prime}\right)}\left[q\left(s^{\prime}, a^{\prime}\right)\right]
$, and we define $\mathds{P}^{\pi}q(s,a) := \mathbb{E}_{s^{\prime} \sim \mathds{P}(\cdot | s, a)}\left[q\left(s^{\prime}, \pi\right)\right]$.   Additionally, it is helpful to remember that $J(\pi) = q^{\pi}(s^{0},\pi)$.  Another important notion is 
the normalized discounted visitation of $\pi$, defined as 
$
d_{\pi}(s,a) \coloneqq (1-\gamma)\sum_{t=0}^{\infty} \gamma^{t} d_{\pi,t}(s,a)
$,
where $d_{\pi,t}$ is the marginal state-action distribution at the time-step $t$. 

\textbf{Offline RL under function approximation.} In the offline RL setting, there exists an unknown offline data-generating distribution $\mu$ induced by behavior policies.  Despite the unknowns of $\mu$, we can observe a set of transition pairs, as offline dataset $\mathcal{D}_{1:n} \coloneqq \{s_i, a_i, r_i, s^{\prime}_i\}^{n}_{i=1}$ sampling from $\mu$. For a given policy $\pi$, the density-ratio (importance-weight), $\tau_{d_{\pi} / \mu}(s, a)
= d_{\pi}(s, a)/\mu(s, a)$, measures how effectively $\mu$ covers the visitation induced by $\pi$. The primary objective of offline policy optimization is to learn an optimal policy that maximizes the return, $J(\pi)$, using the offline dataset. Under the function approximation setting, we assume access to two function classes $\mathcal{Q}: \mathcal{S} \times \mathcal{A} \rightarrow \mathbb{R}$ and $\Omega: \mathcal{S} \times \mathcal{A} \rightarrow \mathbb{R}$, which are utilized to capture $q^\pi$ and $\tau_{d_{\pi} / \mu}$, respectively.  



\textbf{Exploration and coverage.} In general, when saying an offline dataset is well-explored, it means that a well-designed behavior policy has been executed, allowing for comprehensive exploration of the MDP environment. As a result, the dataset is likely to contain possibly all state-action pairs. This implicitly requires $\mu$ has the global coverage \citep{fujimoto2019off,uehara2022pessimistic}. In this context, the global coverage means that the density ratio-based concentrability coefficient, 
$\sup_{s,a} \{d_{\pi}(s, a) / \mu(s, a)\}$, is upper-bounded by a constant  $c \in \mathbb{R}^{+}$ for all policies $\pi \in \Pi$, where $\Pi$ is some policy class. This condition is frequently employed in offline RL \citep{antos2008learning,chen2019information,duan2020minimax}. However, in practice, this assumption may not hold true, as devising an exploratory policy is a challenging task for large-scale RL problems. Instead, our goal is to learn a good policy with strong theoretical guarantees that can compete against any arbitrarily covered comparator policy under much weaker conditions than the global coverage. 

\vspace{-0.5mm}
\section{Bi-Level Offline Policy Optimization Algorithm}
\vspace{-1mm}

In this section, we introduce our bi-level offline policy optimization framework. The development of the framework consists of three major steps. 

\textbf{Step 1: robust interval learning.} In this step, we aim to provide a robust off-policy interval evaluation. The major advantage of this interval formulation is its robustness to the model-misspecification of the importance-weight class $\Omega$, and the encoding of distributional-shift information in the policy evaluation process. First, we define a detection function $\mathbb{D}(\cdot)$, which is used to measure the degree of the distributional-shift in terms of density ratio.  
\begin{definition}
\label{alpha_def}
For $x,c_1,c_2,C \in \mathbb{R}^{+}$ and $C \geq 1$, 
the detection function $\mathbb{D}(\cdot)$ satisfies the following conditions: (1) \textit{1-minimum}: $\mathbb{D}(1) = 0$. (2) \textit{Non-negativity}: $\mathbb{D}(x) \geq 0$. (3) \textit{Boundedness on first-order derivative}: $|\mathbb{D}^{\prime}(x)| \leq c_2$ if $x \in [0,C]$. (4) \textit{Boundedness on value}: $|\mathbb{D}(x)| \leq c_1$ for $x \in [0,C]$. 
(5) \textit{Strong convexity}: $\mathbb{D}(x)$ is $M$-strongly convex with respect to $x$.
\end{definition}
The family of R\'enyi entropy \citep{renyi1961measures}, Bhattacharyya distance \citep{choi2003feature}, and simple quadratic form functions \citep{zhou2023distributional}all satisfy the conditions outlined in Definition \ref{alpha_def}. Under this definition, it can easily observe that $\mathbb{D}$ has a convex conjugate function \citep{boyd2004convex},
$\mathbb{D}_{*}$ with $\mathbb{D}_{*}\left(x_*\right)=\sup _x\left\{x \cdot x_*-\mathbb{D}(x)\right\}$, that satisfies $\mathbb{D}_*(0)=0$.
It follows from Bellman equation $\mathcal{B}^{\pi}q^{\pi}(s,a) = q^{\pi}(s,a)$ for any $s,a$, then $J(\pi)=q^{\pi}(s^{0},\pi)+\mathbb{E}_{\mu}[\lambda\mathbb{D}_*((\mathcal{B}^{\pi}q^{\pi}(s,a)- q^{\pi}(s,a)/\lambda))/(1-\gamma)]$ for $\lambda \geq 0 $. Applying Fenchel-Legendre transformation \citep{nachum2020reinforcement,jiang2020minimax}, and model $x$ in a restricted importance weight class $\Omega$ for any $s,a$, we obtain
\#
J(\pi)= & q^{\pi}(s^{0},\pi)+ \mathbb{E}_{\mu}[\sup_{x}x\cdot(\mathcal{B}^{\pi}q^{\pi}(s,a)- q^{\pi}(s,a)) - \lambda\mathbb{D}(x)]/(1-\gamma)\\
\geq & q^{\pi}(s^{0},\pi)+ \mathbb{E}_{\mu}[\tau(s,a)(r(s,a)+\gamma q^{\pi}(s^{\prime},\pi) - q^{\pi}(s,a))- \lambda\mathbb{D}(\tau(s,a))]/(1-\gamma). 
\label{conjugate}
\#
Suppose $q^{\pi}$ is well-sepcified, i.e., $q^{\pi} \in \mathcal{Q}$, we can find a lower bound of \eqref{conjugate}, which is valid for any $\tau \in \Omega$, via replacing $q^{\pi}$ with $\inf_{q\in\mathcal{Q}}$ as follows:
\$
J(\pi) \geq& \inf_{q \in \mathcal{Q}}\Big\{\underbrace{\big( \mathbb{E}_{\mu}\left[\tau(s,a)(r(s,a)+\gamma q\left(s^{\prime}, \pi\right)-q(s, a))\right] + q(s^0,\pi)\big)/(1-\gamma)}_{:=H(\tau,q,\pi)} \\
& \qquad  - \underbrace{\lambda/(1-\gamma)\mathbb{E}_{\mu}[\mathbb{D}(\tau(s,a))]}_{:=\lambda\xi(\mathbb{D},\tau)} \Big\}.
\$
After following a similar derivation, we can establish an upper bound for $J(\pi)$ as well, and thus construct a value interval for $J(\pi)$. This interval holds for any $\tau$ and is therefore robust against model-misspecification of $\Omega$. In order to obtain a tighter interval, we can shrink the interval width by maximizing the lower bound and minimizing the upper bound, both with respect to $\tau$. This procedure can be interpreted as searching for some good $\tau \in \Omega$ to minimize the function approximation error.
\#
J(\pi) \in \left[\sup_{\tau\in\Omega}\inf_{q \in \mathcal{Q}}H(\tau,q,\pi)- \lambda\xi(\mathbb{D},\tau),\;  \inf_{\tau \in \Omega}\sup_{q \in \mathcal{Q}}H(\tau,q,\pi) +\lambda\xi(\mathbb{D},\tau)\right],
\label{lower_upper_true}
\#
While the interval offers a robust method for dealing with the bias introduced by function approximation when estimating $J(\pi)$, it lacks a crucial and non-trivial step for handling statistical uncertainty.

\textbf{Step 2: uncertainty quantification.}  In this step, we quantify the uncertainty of the interval \eqref{lower_upper_true}, and establish a non-asymptotic confidence interval (CI) for $J(\pi)$ which integrates bias and uncertainty quantifications in a single interval inspired by \citep{zhou2023distributional}. Given offline data $\mathcal{D}_{1:n}$, our formal result for quantifying sampling uncertainty in order to establish the CI for $J(\pi)$. 

\begin{thm}[Non-asymptotic confidence interval]\label{reg_ci_thm}
For a target policy $\pi$, the return $J(\pi)$ is within a CI for any $\tau \in \Omega$ with probability at least $1-\delta$, i.e., $J(\pi) \in [\widehat{J}^{-}_{n}(\pi;\tau),\widehat{J}^{+}_{n}(\pi;\tau) ]$ for
\#
\widehat{J}^{-}_{n}(\pi;\tau) := & \frac{1}{n}\sum^{n}_{i=1}\frac{r_i\tau(s_i,a_i)}{1-\gamma} - \sup_{q \in \mathcal{Q}}\widehat{M}_{n}(-q,\tau) -  \lambda\xi_n(\mathbb{D},\tau)  - \sigma_n, \notag \\
\widehat{J}^{+}_{n}(\pi;\tau) := & \frac{1}{n}\sum^{n}_{i=1}\frac{r_i\tau(s_i,a_i)}{1-\gamma} +  \sup_{q \in \mathcal{Q}}\widehat{M}_{n}(q,\tau) + \lambda\xi_n(\mathbb{D},\tau) + \sigma_n,
\label{reg_ci}
\#
if the uncertainty deviation $\sigma_n$ satisfies
\$
P\bigg(\sup_{\tau \in \Omega}\Big| \frac{1}{n(1-\gamma)}\sum^{n}_{i=1}\tau(s_i,a_i)\left(r_i +\gamma q^{\pi}(s^{\prime}_i,\pi)-q^{\pi}(s_i,a_i)\right)-\lambda\xi_n(\mathbb{D},\tau)\Big| \leq \sigma_n \bigg) \geq 1-\delta,
\$
where  $\widehat{M}_{n}(q,\tau) := \sum^{n}_{i=1}\tau(s_i,a_i)(\gamma q(s^{\prime}_i,\pi) - q(s_i,a_i))/(1-\gamma)n+ q(s^{0},\pi)$. 
\end{thm}
Similar to the value interval, the CI $ [\widehat{J}^{-}_{n}(\pi;\tau),\widehat{J}^{+}_{n}(\pi;\tau) ]$ also holds for any $\tau \in \Omega$. Therefore, we can optimize the confidence lower and upper bounds in \eqref{reg_ci} over $\tau \in \Omega$ to tighten the CI, and obtain:
\$
P\Big(J(\pi) \in [\sup_{\tau\in\Omega}\widehat{J}^{-}_{n}(\pi;\tau),  \inf_{\tau\in\Omega}\widehat{J}^{+}_{n}(\pi;\tau)] \subseteq 
[\widehat{J}^{-}_{n}(\pi;\tau), \widehat{J}^{+}_{n}(\pi;\tau)\Big) \geq 1-\delta. 
\$

\textbf{Step 3: bridge policy evaluation to policy optimization.} In this step, we aim to formulate a policy optimization based on the derived high-confidence policy evaluation from the previous steps. Given the consistent CI estimation of $J(\pi)$, we can naturally incorporate the pessimism principle, i.e., using the CI lower bounds of $J(\pi)$ as the value estimate of the policy evaluation of $\pi$ \citep{jin2021pessimism}. With such a procedure, our objective is to maximize these lower bounds over some family $\Pi$ of policies: 
\#
\max_{\pi \in \Pi} \left\{\sup_{\tau\in\Omega}\widehat{J}^{-}_{n}(\pi;\tau)\right\}.
\label{dual_est}
\#
Although \eqref{dual_est} is algorithmically feasible for obtaining a policy solver $\widehat{\pi}$, it lacks direct interpretation without taking advantage of the bi-level optimization structure in hindsight. Therefore, we propose to reformulate \eqref{dual_est} via a \textit{dual-to-prime conversion} (shown in Theorem \ref{game}), which naturally lends itself to lower-upper optimization with guaranteed convergence. Specifically, we formulate \eqref{dual_est} as a bi-level framework problem:
\#
& (\text{Upper Level}) ~~~~\min_{\pi \in \Pi}-\underline{q^{\pi}}(s^{0},\pi),~~~~~~~~~~~~~~~~~~~~~~~~~~~~~~~~~~~~~~~~~~~~~~~~~~~~~~~~~~~~~~~~~~~ \label{prime_opt}  \\
& (\text{Lower Level})  ~~~~s.t. \; \underline{q^{\pi}} \in  \argmin_{q\in \mathcal{Q}_{\varepsilon_{n}}}q(s^{0},\pi), \label{prime_perturbed_main}\\
& \underline{\textbf{\textit{Consistency}}}: ~~~~\mathcal{Q}_{\varepsilon_{n}} = \big\{q \in \mathcal{Q}:  \sup_{\tau \in \widetilde{\Omega}_{\widetilde{\sigma}_{n}}}\big|n^{-1}\sum^{n}_{i=1}\tau(s_i,a_i)(r_i +\gamma q(s^{\prime}_i,\pi)-q(s_i,a_i))\big| \leq \varepsilon_{n} \big\}, \notag \\
& \underline{\textbf{\textit{Uncertainty Control}}}:  ~~~~\widetilde{\Omega}_{\widetilde{\sigma}_{n}} =    \left\{\tau_{\circ}/\sup_{\tau_{\circ} \in \Omega}\|\tau_{\circ}\|_{\Omega} \; \text{for} \; \tau_{\circ} \in \Omega: \xi_n(\mathbb{D},\tau_{\circ})) \leq \widetilde{\sigma}_{n}  \right\}. \notag 
\#
At the upper level, the learned policy $\widehat{\pi}$ attempts to maximize the value estimate of $\underline{q^{\pi}}$ over some policy class $\Pi$, while at the lower level, $\underline{q^{\pi}}$ is 
 to seek the $q$-function with the pessimistic policy evaluation value from the confidence set $\mathcal{Q}_{\varepsilon_{n}}$ with consistency guarantee and uncertainty control. For  \textit{consistency}, whenever $q^{\pi}$ or its good approximator is included in
 $\mathcal{Q}$ (realizability for $\mathcal{Q}$ class is satisfied), the set $\mathcal{Q}_{\varepsilon_{n}}$ ensures the estimation consistency of $q^{\pi}$ in terms of ``sufficently small'' weighted average Bellman error. For \textit{uncertainty control}, the constrained set $\widetilde{\Omega}_{\widetilde{\sigma}_{n}}$ attempts to control the uncertainty arising from distributional shift via a user-specific thresholding hyperparameter $\widetilde{\sigma}_{n}$. The feasible (uncertainty controllable) candidates $\tau \in \widetilde{\sigma}_{n}$ are used as weights for the average Bellman error, helping to construct the consistent set $\mathcal{Q}_{\varepsilon_{n}}$. Risk-averse users can specify a lower value for the thresholding hyperparameter or consider a higher $\widetilde{\sigma}_{n}$ to tolerate a larger distribution shift. In other words, the chosen value of $\widetilde{\sigma}_{n}$ depends on the degree of pessimism users want to incorporate in the policy optimization.

\begin{thm}
\label{game}
There must exist some threshold values $\varepsilon_{n}$ and $\widetilde{\sigma}_{n}$, the return policy of \eqref{dual_est} $\widehat{\pi}$ satisfies the minimization problem in \eqref{prime_opt}, indicating the solution of the optimization \eqref{dual_est} and \eqref{prime_opt} is equivalent. 
\end{thm}

Interestingly, the new form in \eqref{prime_opt} characterizes our policy optimization framework as a two-player general-sum game \citep{fudenberg1991game}, which is a sequential game involving two players. Each player aims to maximize their own payoffs while considering the decisions of other players. Our bi-level optimization framework has been demonstrated to improve learning stability and ensure algorithmic convergence, benefiting from the existence of a local equilibrium \citep{von2010market}.

To close this section, we remark that the establishment of \textit{consistency} with respect to the weighted average Bellman error is the key point for us to relax the completeness-type assumptions. In the famous API/AVI-type algorithms \citep{ernst2005tree,farahmand2016regularized,chen2019information}, they target to minimize a squared or minimax Bellman error for finding $q \in \mathcal{Q}$ so that $\|q-\mathcal{B}^{\pi}q \|^{2}_{L_{2}(\mu)} \approx 0$ to obtain $q \approx q^{\pi}$.  
 Unfortunately, even with the infinite amount of data, the empirical estimate of $\|q-\mathcal{B}^{\pi}q \|^{2}_{L_{2}(\mu)}$, i.e., squared empirical Bellman error) is biased due to the appearance of unwanted conditional variance, i.e., the \textit{double sampling} issue \citep{baird1995residual}. The API/AVI-type algorithms need a separate helper function class $\widetilde{\mathcal{Q}}$ for modeling $\mathcal{B}^{\pi}q$, and \cite{chen2019information} has shown that when the class $\widetilde{\mathcal{Q}}$ realizes the Bayes optimal regressor $\mathcal{B}q$ (Bellman-completeness condition), the estimation is consistent and unbiased. In contrast, thanks to not using the squared loss, our weighted average Bellman error can be estimated from an unbiased estimate without concern about the \textit{double sampling} issue, and thus no need for any completeness-type conditions.

\vspace{-1mm}
\section{Information-Theoretic  Results}
\vspace{-1mm}

In this section, we provide theoretical analyses of our algorithm, which reveals the advantages of the proposed policy optimization method from a technical standpoint.

Notably, to the best of our knowledge, Theorem \ref{main_thm_reg} is the first result of regret guarantee under only realizability \textit{without} requiring any data coverage or completeness-type assumptions. Additionally, in contrast to most existing works that assume finite function classes, we carefully quantify the space complexities for infinite function classes (e.g., a class of real-valued functions generated by neural networks) using Pollard's pseudo-dimension \citep{pollard1990empirical}. The formal definition is provided in Appendix. It notices that the pseudo-dimension is a generalization of the well-known VC dimension \citep{vapnik2015uniform}. In the following, we first introduce the necessary assumptions before presenting the guarantees for our algorithm.

 

\begin{assump}[Realizability for $q$-function class]
For any policy $\pi \in \Pi$, we have 
$q^{\pi} \in \mathcal{Q}$. When this assumption holds approximately, we measure violation by
$
\inf _{q \in \mathcal{Q}} \sup _{ \rho}\mathbb{E}_{\rho}[\left(q(s,a)-\mathcal{B}^{\pi} q(s,a)\right)^{2}] \leq \varepsilon_{\mathcal{Q}} 
$,
where $ \varepsilon_{\mathcal{Q}} \geq 0$ and $\rho$ is some data distribution such that $\rho \in \{ d_{\widetilde{\pi}}: \widetilde{\pi} \in \Pi\}$.
\label{reliable_assum}
\vspace{-1.5mm}
\end{assump}

We would like to emphasize that we do not require Bellman-completeness condition \citep{xie2021bellman,zanette2021provable}, which is much
stronger than the realizability condition. In addition, we do not impose realizability on the importance-weight class $\Omega$, thereby allowing model misspecification on $\Omega$. Having stated the major assumptions, we now turn to the routine ones on boundedness before presenting the main results. 

\begin{assump}[Boundedness on $\mathcal{Q}$]
There exists a non-negative constant $\bar{V}  < \infty$, the function  $q(s,a) \in [0, \bar{V}], \, \forall q \in \mathcal{Q}, s \in \mathcal{S}, a \in \mathcal{A}$. 
\label{q_bound}
\end{assump}

\begin{assump}[Boundedness on $\Omega$]
There exists a non-negative constant $1 \leq \mathcal{U}^{\tau}_{\infty} < \infty$, the function $\tau(s,a) \in [0, \mathcal{U}^{\tau}_{\infty}], \, \forall \tau \in \Omega, s \in \mathcal{S}, a \in \mathcal{A}$.
\label{tau_bound}
\end{assump}



\begin{thm}\label{main_thm_reg}
Under Assumptions \ref{reliable_assum}-\ref{tau_bound} and denote supremum of $\mu$-weighted $L_2$ norm of $\Omega$, i.e., $\sup_{\tau \in \Omega}\|\tau(s,a)\|_{L_2(\mu)}$, as $\mathcal{U}^{\tau}_{2}$. Let $\widehat{\pi}$  be the output of solving \eqref{prime_opt} when we set  $\varepsilon_{n} = \widetilde{\mathcal{O}}(n^{-1/2}\mathcal{U}^{\tau}_{2}(\sqrt{\ln \{\operatorname{Vol}({\Theta})/\delta\}} + \mathcal{U}^{\tau}_{\infty}\sqrt{\varepsilon_{\mathcal{Q}}})$
and 
$\widetilde{\sigma}_n =  \widetilde{\mathcal{O}}(n^{-1/2}\mathcal{U}^{\tau}_{2} L\sqrt{\ln \{\operatorname{Vol}({\Theta})/\delta\}}+M(\mathcal{U}^{\tau}_{2}-1)^2)$, then for any policy $\pi \in \Pi$ and some constant $\mathcal{U}^{\star}_{2} \in [1,\mathcal{U}^{\tau}_{2})$, w.p. $\geq 1-\delta$, 
\$
& J(\pi) - J(\widehat{\pi}) \leq \; \frac{1}{1-\gamma}\widetilde{\mathcal{O}}\Bigg(\underbrace{\mathcal{U}^{\star}_{2}\mathfrak{C}_{\bar{V}, L}\sqrt{\frac{\ln\{\operatorname{Vol}({\Theta})/\delta\}}{nM}} }_{\epsilon_{\sigma}}  + \underbrace{\sqrt{\frac{\mathfrak{C}_{\mathcal{U}^{\tau}_{\infty}}}{M}}\max\{(\varepsilon_{\mathcal{Q}})^{1/2},(\varepsilon_{\mathcal{Q}})^{3/4}\}}_{\epsilon_{\text{mis}}} \\
&  + \min_{\left\{\rho: \left\|\frac{\rho}{\mu}\right\|_{L_{2}(\mu)} \leq \mathcal{U}^{\star}_{2}\right\}} \bigg\{ \mathbb{E}_{\left(d_{\pi}-\rho\right)^{+}}\big[\underbrace{\mathds{1}_{\mu=0}(\mathbb{I}-\gamma \mathds{P}^{\pi})\Delta_{\overline{q^{\pi}}-\underline{q^{\pi}}}(s,a)}_{\epsilon_{\text{off}}} + \underbrace{\mathds{1}_{\mu>0}\mathfrak{C}_{\bar{V},\gamma}\sqrt{\frac{\ln\{\operatorname{Vol}({\Theta})/\delta\}}{n}} }_{\epsilon_{b}}\big]\bigg\}\Bigg).
\$
Here $\Delta_{\overline{q^{\pi}}-\underline{q^{\pi}}}(s,a) = \overline{q^{\pi}}(s,a) - \underline{q^{\pi}}(s,a)$ for $\overline{q^{\pi}}:= \argmax_{q\in \mathcal{Q}_{\varepsilon_n}}q(s^{0},\pi)$ and $\underline{q^{\pi}}:= \argmin_{q\in \mathcal{Q}_{\varepsilon_n}}q(s^{0},\pi)$. For  Pollard's pseudo-dimensions $D_{\Omega}, D_{\mathcal{Q}}, D_{\Pi}$, $\operatorname{Vol}({\Theta}) = (e^{D}\max\{D_{\Omega},D_{\mathcal{Q}},D_{\Pi}\}+1)^3(\{1 \vee L\}\mathcal{U}^{\tau}_{2})^{2D}$ with the effective pseudo dimension $D=D_{\Omega}+D_{\mathcal{Q}}+D_{\Pi}$, where $L$ is Lipschitz constant of $M$-strongly convex function $\mathbb{D}(\cdot)$. Moreover, $\mathfrak{C}_{x}$ and $\widetilde{\mathcal{O}}$ denote constant terms depending on $x$, and big-Oh notation ignoring high-order terms, respectively.

\end{thm}

In the upper bound of Theorem \ref{main_thm_reg}, we split the regret into four different parts: the on-support intrinsic uncertainty $\epsilon_{\sigma}$, the on-support bias $\epsilon_{b}$, the violation of realizability $\epsilon_{\text{mis}}$, and the off-support extrapolation error $\epsilon_{\text{off}}$. Recall that  we require $q^{\pi} \in \mathcal{Q}$ as in Assumption \ref{reliable_assum}, in fact, we can further relax the condition to requiring $q^{\pi}$ to be in the linear hull of $\mathcal{Q}$ \citep{uehara2020minimax}, which is more robust to the realizability error $\epsilon_{\text{mis}}$. In the following, we focus on investigating the roles of the on-support and off-support error terms in the regret bound. 

\textbf{On-support errors: bias and uncertainty tradeoff.} The on-support error consists of two terms: $\epsilon_{b}$ and $\epsilon_{\sigma}$. The on-support uncertainty deviation, $\epsilon_{\sigma}$, is scaled by a weighted $L_2$-based concentrability coefficient $\mathcal{U}^{\star}_{2}:=\|\rho/\mu\|_{L_{2}(\mu)}$, which measures the distribution mismatch between the implicit exploratory data distribution and the baseline data distribution $\mu$. Meanwhile, $\epsilon_{b}$ depends on the probability mass of $(d_{\pi}-\rho)^{+}\mathds{1}_{\mu>0}$, and represents the bias weighted by the probability mass difference between $d_{\pi}$ and $\rho$ in the support region of $\mu$. In general, a small value of $\mathcal{U}^{\star}_{2}$ necessitates
choice of the distribution $\rho$ to be closer to $\mu$ which reduces $\epsilon_{\sigma}$, reducing $\epsilon_{\sigma}$ but potentially increasing the on-support bias $\epsilon_{b}$ due to the possible mismatch between $d_{\pi}$ and $\rho$. Consequently, within the on-support region, there is a trade-off between $\epsilon_{\sigma}$ and $\epsilon_{b}$, which is adjusted through $\mathcal{U}^{\star}_{2}$.


\textbf{Off-support error: enhanced model extrapolation.} 
One of our main algorithmic contributions is that the off-support extrapolation error $\epsilon_{\text{off}}$ can be minimized by selecting the best possible $\rho$ \textit{without} worrying about balancing the error trade-off, unlike the on-support scenario. This desirable property is essential for allowing the model to harness its extrapolation capabilities to minimize $\epsilon_{\text{off}}$, while simultaneously achieving a good on-support estimation error. As a result, the model attains a small regret. Recall the bi-level formulation; at the lower level, \eqref{prime_perturbed_main} addresses uncertainty arising from the distributional shift using $L_{2}(\mu)$ control rather than $L_{\infty}$ control. 
 This plays an important role in enhancing the power of the  model extrapolation. In particular,  Specifically, there exists an implicit exploratory data distribution $\rho$ with on-support behavior ($\rho \mathds{1}_{\mu>0}$) close to $\mu$, such that $\|\rho/\mu\|_{L_2(\mu)}$ is small. On the other hand, its off-support behavior ($\rho \mathds{1}_{\mu=0}$) can be arbitrarily flexible, ensuring that $d_{\pi} \mathds{1}_{\mu=0}$ is close to $\rho \mathds{1}_{\mu=0}$. Consequently, $(d_{\pi}-\rho)^{+}\mathds{1}_{\mu=0}$ is small, as is $\epsilon_{\text{off}}$.

When a dataset with partial coverage, as indicated in \citep{uehara2022pessimistic}, it is necessary to provide a guarantee: learn the policy with  ``best efforts'' which is competitive to any policy as long as it is covered. Before we state the near-optimal regret guarantee of our algorithm, we formally define a notion of covered policies according to a newly-defined concentrability coefficient.



\begin{definition}[${\mathcal{U}}^{\tau}_{2}$-covered policy class]
\label{cover_class_def}
Let $\Pi({\mathcal{U}}^{\tau}_{2})$ denote the ${\mathcal{U}}^{\tau}_{2}$-covered policy class of $\mu$ for ${\mathcal{U}}^{\tau}_{2} \geq 1$, defined as 
\$
\Pi({\mathcal{U}}^{\tau}_{2}) := \left\{\pi \in \Pi: \left\|\frac{d_{\pi}(s,a)\mathds{1}_{\mu(s,a)>0}}{\mu(s,a)}\right\|_{L_2(\mu)} \leq {\mathcal{U}}^{\tau}_{2} \; \text{and} \; \sup_{s,a}\frac{d_{\pi}(s,a)\mathds{1}_{\mu(s,a)=0}}{\mu(s,a)} < + \infty \right\}.
\$
\end{definition}

Note that this mixture density ratio concentrability coefficient is always bounded by the $L_{\infty}$-based concentrability coefficient. Thus such single-policy concentrability assumption in terms of the mixture density ratio is weaker than the standard $L_{\infty}$ density ratio-based assumption.


\begin{corollary}[Near-optimal regret]\label{optimal_part_reg}
Under Assumptions \ref{reliable_assum}-\ref{tau_bound} with $\varepsilon_{\mathcal{Q}} \in [0,1)$, and we set $\varepsilon_{n}, \widetilde{\sigma}_n$ as in Theorem \ref{main_thm_reg}, then for any good comparator policy $\pi^{\diamond} \in \Pi({\mathcal{U}}^{\tau}_{2})$ (not necessary the optimal policy $\pi^{*}$),
w.p. $\geq 1-\delta$, the output policy $\widehat{\pi}$ of \eqref{prime_opt} satisfies 
\$
J(\pi^{\diamond}) - J(\widehat{\pi}) \leq \frac{1}{1-\gamma}\widetilde{\mathcal{O}}\Bigg( \mathcal{U}^{\star}_{2}(\bar{V}+L)\sqrt{\frac{\ln\{\operatorname{Vol}({\Theta})/\delta\}}{nM}} +  \sqrt{\left(1 + \mathcal{U}^{\tau}_{\infty} + \mathcal{U}^{\tau}_{\infty}/M\right) \varepsilon_{\mathcal{Q}}} \Bigg).
\$
\end{corollary}

A close prior result to Corollary \ref{main_thm_reg} is that of \citep{chen2022offline}, which develops a pessimistic algorithm based on a nontrivial performance gap condition. Their regret guarantees only hold if the data covers the optimal policy $\pi^{*}$, in particular, requiring a bounded $L_{\infty}$ single-policy concentrability with respect to $\pi^{*}$. In comparison, our guarantee can still provide a meaningful guarantee even when $\pi^{*}$ is not covered by data.  In the following, we include the sample complexity of our algorithm when $\varepsilon_{\mathcal{Q}}=0$.



\begin{corollary}[Polynomial sample complexity]\label{sample_comp_reg}
Under the conditions in Corollary \ref{optimal_part_reg}, the output policy $\widehat{\pi}$ of solving \eqref{prime_opt} satisfies  $
J(\pi^{\diamond}) - J(\widehat{\pi}) \leq \varepsilon $ w.p. $\geq 1-\delta$, if 
\$
n = \mathcal{O}\Bigg(\Big(\frac{(\mathcal{U}^{\star}_{2}(\bar{V}+L)/\sqrt{M})^2}{\varepsilon^{2}(1-\gamma)^2} + \frac{(\mathcal{U}^{\tau}_{2}\bar{V}^2(\bar{V}+L)/M)^{0.67}}{\varepsilon^{1.33}(1-\gamma)^{1.33}} + \frac{\mathcal{U}^{\tau}_{\infty}(\bar{V}+L)}{\varepsilon(1-\gamma)}\Big) \ln\frac{\operatorname{Vol}({\Theta})}{\delta}\Bigg).
\$
\vspace{-3mm}
\end{corollary}

 The sample complexity consists of three terms corresponding to the slow rate $\mathcal{O}(n^{-1/2})$ and the two faster rate $\mathcal{O}(n^{-1})$ and $\mathcal{O}(n^{-3/4})$ terms in Corollary \ref{optimal_part_reg}. When $\mathcal{U}^{\tau}_{2}$ and $ \mathcal{U}^{\tau}_{\infty}$ are not too much larger than $\mathcal{U}^{\star}_2$, the fast rate terms are dominated, and the sample complexity is of order $\mathcal{O}(1/\varepsilon^2)$, which is much faster than $ \mathcal{O}(1/\varepsilon^{6})$ in the close work of \citep{zhan2022offline}. It is worth noting that even in exploratory settings where the global coverage assumption holds, our sample complexity rate matches the fast rate in popular offline RL frameworks with general function approximation \citep{chen2019information, xie2020q, duan2020minimax}. 

In addition to the near-optimal regret guarantee, in safety-critical applications, an offline RL algorithm should consistently improve upon the baseline (behavior) policies that collected the data \citep{ghavamzadeh2016safe, laroche2019safe}. Our algorithm also achieves this improvement guarantee with respect to the baseline policy.

\begin{thm}[Baseline policy improvement]
Under Assumptions \ref{reliable_assum}-\ref{tau_bound} with $\varepsilon_{\mathcal{Q}} = 0$ and set $\varepsilon_{n}, \widetilde{\sigma}_n$ as in Theorem \ref{main_thm_reg}. Suppose $1 \in \Omega$ and the baseline policy $\pi_{b} \in \Pi$ such that $d_{\pi_b} = \mu$, then the regret $(1-\gamma)(J(\pi_{b}) - J(\widehat{\pi}))$ for the output policy $\widehat{\pi}$ of solving \eqref{prime_opt}, w.p. $\geq 1-\delta$, is upper bounded by 
\$
\mathcal{O}\bigg(
\sqrt{\frac{(\bar{V}+L)^2\ln\{\operatorname{Vol}({\Theta})/\delta\}}{nM}} + \sqrt{\frac{(\bar{V}^3+\bar{V}^2L)}{M}}\bigg(\frac{\ln\{\operatorname{Vol}({\Theta})/\delta\}}{n}\bigg)^{\frac{3}{4}} + \frac{(\bar{V}+L)\ln\{\operatorname{Vol}({\Theta})/\delta\}}{n} \bigg).
\$
\label{safe_imp} 
\end{thm}     

The aforementioned information-theoretic results enhance the understanding of the developed algorithm, in terms of the function approximation and coverage conditions, sample complexity, horizon dependency, and bound tightness. In practice, although the information-theoretic algorithm offers a feasible solution to the problem, it is not yet tractable and computationally efficient due to the need to solve constrained optimization. In the following section, we develop a practical algorithm as a computationally efficient counterpart for the information-theoretic algorithm. 





\section{Penalized Adversarial Estimation Algorithm}

Although the information-theoretic algorithm offers a feasible solution to the problem, it is not yet tractable and computationally efficient due to the need to solve constrained optimization.  In this section, we develop an adversarial estimation proximal-mapping algorithm that still adheres to the pessimism principle, but through penalization. Specifically, the adversarial estimation loss is constructed as follows:
$\underset{\tau}{\operatorname{max}} \underset{q}{\operatorname{min}} \mathcal{L}(q,\tau,\pi,c^{*},\lambda)$ for solving 
\$
q (s^{0}, \pi) + \frac{1}{(1-\gamma)n}\left\{c^{*}\Big|\sum^{n}_{i=1}   \tau(s_i,a_i)\left(q(s_i,a_i)-r_i -\gamma q(s^{\prime}_i,\pi)\right)\Big| - \lambda\sum^{n}_{i=1} \mathbb{D}(\tau(s_i,a_i))\right\}.
\$
We observe that the inner minimization for solving $q$ is relatively straightforward, as we can obtain a closed-form global solver using the maximum mean discrepancy principle \citep{gretton2012kernel, shi2022minimax}. In contrast, optimizing $\tau_{\psi}$ is more involved, often requiring a sufficiently expressive non-linear function approximation class, e.g., neural networks. However, concavity typically does not hold for such a class of functions \citep{jiang2020minimax}. 
From this perspective, our problem can be viewed as solving a non-concave maximization problem, conditional on the solved global optimizer $\bar{q}:= \argmin_{q} \mathcal{L}(q,\tau,\pi,c^{*},\lambda)$. At each iteration, we propose to update $\tau$ by solving the proximal mapping \cite{parikh2014proximal} using the Euclidean distance to reduce the computational burden. As a result, the pre-iteration computation is quite low. 

\renewcommand\footnoterule{}
\noindent\begin{minipage}{\textwidth}
\vspace{-4mm}
\begin{algorithm}[H]
\setstretch{1.12}
	\caption{Adversarial proximal-mapping algorithm}
\label{prox_map}
	\begin{algorithmic}[1]
	\STATE \textbf{Input} observed data $\mathcal{D}_{1:n}=\{(s_i,a_i,r_i,s_i^{\prime})\}^n_{i=1}$ and parameters $q^{0},\tau^{0},\pi^{0},c^{*}$, $\lambda$ and $\zeta$. 
				\STATE \textbf{For} $k=1$ to $\bar{K}$:
		\STATE \; Update $\tau^{k}$ and $q^{k}$ by solving
    $\underset{\tau}{\operatorname{max}} \underset{q}{\operatorname{min}} \; \mathcal{L}(q,\tau,\pi^{k-1},c^{*},\lambda)
    $ 
	\STATE  \; Update $\pi^{k}$ by solving 
	$
\pi^{k}(\cdot|s) = \underset{\pi \in \Pi}{\operatorname{argmax}} \; \zeta \left\langle q^{k}(\cdot,s), \pi(\cdot|s) \right\rangle-D_{\text{NegEntropy}}\left(\pi(\cdot|s), \pi^{k}(\cdot|s)\right).
	$	
\STATE \textbf{Return} the policy $\widehat{\pi}$, which randomly selects a policy from the set $\{\pi^{k}\}^{\bar{K}}_{k=1}$.
\end{algorithmic}
\end{algorithm}
\end{minipage}

Once $q$ and $\tau$ are solved, we apply mirror descent in terms of the negative entropy $D_{\text{NegEntropy}}$ \citep{beck2017first}. That is, given a
stochastic gradient direction of $\pi$ we solve the prox-mapping in each iteration as outlined in step 4 of Algorithm \ref{prox_map}. A detailed version of Algorithm \ref{prox_map} with extended discussions on convergence and complexity is provided in Appendix. In the following, we establish the regret guarantee for the policy output by Algorithm \ref{prox_map}. 





\begin{thm}\label{prac_main_thm_reg}
Under Assumptions \ref{reliable_assum}-\ref{tau_bound} with $\varepsilon_{\mathcal{Q}}=0$, we properly choose $\lambda = \lambda({\mathcal{U}}^{\tau}_{2})$, i.e., $\lambda$ well depends on ${\mathcal{U}}^{\tau}_{2}$, and $
c^{*} =  \widetilde{\mathcal{O}}\big(\sqrt{n\bar{V}/(\lambda L{\mathcal{U}}^{\tau}_{2}\ln \{\operatorname{Vol}({\Theta}^{\dagger})/\delta\})}\big) 
$.
After running $\bar{K} \geq \log|\mathcal{A}|$ rounds of Algorithm \ref{prox_map} with the stepsize 
 $\zeta = \sqrt{\log|\mathcal{A}|/(2\bar{V}\bar{K})}$, for any policy $\pi \in \Pi$, the output policy $\widehat{\pi}$ of the algorithm, w.p $\geq 1-\delta$, satisfies,
 \$
& J(\pi) - J(\widehat{\pi}) \leq \; \frac{1}{1-\gamma}\widetilde{\mathcal{O}}\Bigg(\sqrt[\leftroot{-1}\uproot{2}\scriptstyle 4]{\frac{({\mathcal{U}}^{\star}_{2})^2\mathfrak{C}^{1}_{\bar{V},\lambda, L}\ln\{\operatorname{Vol}({\Theta}^{\dagger})/\delta\}}{n}} + \sqrt{\frac{\bar{V}\log|\mathcal{A}|}{\bar{K}}} \\
&  +  \frac{1}{\bar{K}}\sum^{\bar{K}}_{k=1}\min_{\rho_k \in \Delta_{{\mathcal{U}}^{\star}_{2}}} \mathbb{E}_{\left(d_{\pi}-\rho_k\right)^{+}}\bigg[\mathds{1}_{\mu=0}\left(\mathcal{B}^{\pi^{k}}q^{k}(s,a) -q^{k}(s,a)\right) + \mathds{1}_{\mu>0}\sqrt{\frac{\mathfrak{C}^{2}_{\bar{V},\lambda, L}\ln\{\operatorname{Vol}({\Theta}^{\dagger})/\delta\}}{n}}\bigg]\Bigg),
\$
where $\Delta_{{\mathcal{U}}^{\star}_{2}}:=\{\rho_k: \|\frac{\rho_k}{\mu}\|_{L_2(\mu)} < {\mathcal{U}}^{\star}_{2}\}$, $\mathfrak{C}^{1}_{\bar{V},\lambda, L}, \mathfrak{C}^{2}_{\bar{V},\lambda, L}$ are some constant terms, and the function class complexity $\operatorname{Vol}({\Theta}^{\dagger}) = (e^{D}\max\{D_{\Omega},D_{\mathcal{Q}},D_{\Pi}\}+1)^3(\{1 \vee L\}{\mathcal{U}}^{\tau}_{2})^{2D}$ for $D=D_{\Omega}+D_{\mathcal{Q}}+D_{\Pi}$.
\end{thm}

\textbf{Trajectory-adaptive exploratory data distribution.} Similar to Theorem \ref{main_thm_reg}, the penalized algorithm also exhibits a desirable extrapolation property for minimizing extrapolation error while simultaneously preserving small on-support estimation errors. This is achieved through adaptations of the implicit exploratory data distributions, $\rho_{k}$ for $k \in [\bar{K}]$. In contrast to the information-theoretic algorithm, the automatic splitting by $\rho_{k}$ now depends on the optimization trajectory. At each iteration $k$, the penalized algorithm allows each implicit exploratory data distribution $\rho_{k}$ to adapt to the comparator policy $\pi$. This results in a more flexible adaptation than the one in the information-theoretic algorithm, either for balancing the trade-off between on-support bias and uncertainty incurred by the distributional mismatch between $d_{\pi}$ and $\rho_{k}$, or for selecting the best implicit exploratory to minimize model extrapolation error.

\textbf{Opimization error.} Blessed by the reparametrization in the proximal-mapping policy update, which projects the mixture policies into the parametric space $\Pi_{\omega}$, the complexity of the restricted policy class is independent of the class of $\mathcal{Q}$ and the horizon optimization trajectory $\bar{K}$. As a result, the optimization error $\mathcal{O}(\sqrt{\bar{V}\log|\mathcal{A}|/\bar{K}})$ can be reduced arbitrarily by increasing the maximum number of iterations, $\bar{K}$, without sacrificing overall regret to balance statistical error and optimization error. This allows for the construction of tight regret bounds. This distinguishes our algorithm from API-style algorithms, which do not possess a policy class that is independent of $\mathcal{Q}$ \citep{antos2008learning,scherrer2012use,xie2021bellman}.

\vspace{-2mm}
\subsection{An Application to Linear MDPs with Refined Concentrability Coefficient}
\vspace{-1.5mm}

In this section, we conduct a case study in  linear MDPs with insufficient data coverage. The concept of the linear MDP is initially developed in the fully exploratory setting \citep{yang2020reinforcement}. Let $\phi: \mathcal{S} \times \mathcal{A} \rightarrow \mathbb{R}^d$ be a $d$-dimensional feature mapping. We assume throughout that these feature mappings are normalized, such that
$\|\phi(s,a)\|_{L_2} \leq 1$ uniformly for all  $(s,a) \in \mathcal{S} \times \mathcal{A}$. We focus on action-value functions that are linear in $\phi$ and consider families of the following form:
$
\mathcal{Q}_{\theta}:= \left\{(s, a) \mapsto\langle\phi(s, a), \theta\rangle \mid\|\theta\|_{L_2} \leq c_{\theta}\right\}
$,
where $c_{\theta} \in [0, \bar{V}]$. For stochastic policies, we consider the soft-max policy class
$
\Pi_{\omega}:= \{\pi_{\omega}(a|s)  \propto  e^{\langle\phi(s, a), \omega\rangle} \mid\|\omega\|_{L_2} \leq c_\omega\},
$
where $c_\omega \in (0, \infty)$. Note that the softmax policy class is consistent with the implicit policy class produced by the mirror descent updates with negative entropy in Algorithm \ref{prox_map}, where the exponentiated gradient update rule is applied in each iteration. For the importance-weight class, we also consider the following form:
$
\Omega_{\psi} := \left\{(s, a) \mapsto\langle\phi(s, a), \psi\rangle \mid\|\psi\|_{L_2} \leq c_{\psi}\right\}
$ where
$c_{\psi} \in (0, \infty)$. 
To simplify the analysis, we assume the  realizability condition for $\mathcal{Q}_{\theta}$ is exactly met. In this linear MDP setting, we further refine the density ratio to a relative condition number to characterize partial coverage. This concept is recently introduced in the policy gradient literature \citep{agarwal2020optimality} and is consistently upper-bounded by the $L_{\infty}$-based density ratio concentrability coefficient.

\begin{definition}[Relative condition number]
\label{def_relative_cond}
For any policy $\pi \in \Pi_{\omega}$ and behavior policy $\pi_{b}$ such that $d_{\pi_{b}}=\mu$,  the relative condition number is defined as
$
\iota(d_{\pi},\mu) = \sup _{x \in \mathbb{R}^d} \frac{x^T \mathbb{E}_{ d_{\pi}}\left[\phi(s, a) \phi(s, a)^{\top}\right] x}{x^{\top} \mathbb{E}_{\mu}\left[\phi(s, a) \phi(s, a)^{\top}\right] x} .
$
\end{definition}



\begin{assump}[Bounded relative condition number]
\label{rcn}
For any $\pi \in \Pi_{\omega}$,
$
\iota(d_{\pi},\mu) < \infty
$.
\end{assump}

Intuitively, this implies that as long as a high-quality comparator policy exists, which only visits the subspace defined by the feature mapping $\phi$ and is covered by the offline data, our algorithm can effectively compete against it \citep{uehara2022pessimistic}. This partial coverage assumption, in terms of the relative condition number, is considerably weaker than density ratio-based assumptions. In the following, we present our main near-optimal guarantee in linear MDPs. In addition, we design and conduct numerical experiments to empirically validate Theorem \ref{lr_opt} in terms of the regret rate of convergence. 

\begin{thm}
\label{lr_opt}
Under Assumption \ref{rcn}, if we set propertly choose $\lambda = \lambda(c_{\psi})$ and $c^{*} = \widetilde{\mathcal{O}}\big(\sqrt[4]{n/d\ln\{(1+e\sqrt{n}(1\vee L)\bar{V}c_{\psi}c_\omega)/\delta\}}\big) $, and suppose $\widehat{\pi}^{\text{lr}}$ is returned by Algorithm \ref{prox_map} with linear function approxiamiton after running $\bar{K} \gg \log|\mathcal{A}|$ rounds, then for any policy in $\pi \in \Pi_{\omega}(\mathcal{U}^{\text{lr}}_{2})$ for $\mathcal{U}^{\text{lr}}_{2} \geq 1$, w.p. $\geq 1-\delta$, $J(\pi) - J(\widehat{\pi}^{\text{lr}})$ is bounded by 
 \$
\widetilde{\mathcal{O}}\bigg(\frac{\sqrt{\min\{\kappa^2
c^{2}_{\psi}\{\mathcal{U}^{\text{lr}}_{2}\}d^2,\iota(d_{\pi},\mu)d  \}}}{1-\gamma}
\sqrt[4]{\frac{\mathfrak{C}_{\bar{V},\lambda, L}d\ln \{(1+e\sqrt{n}(1\vee L)\bar{V}c_{\psi}c_\omega)/\delta\}}{n}}\bigg),
 \$
where $\kappa = \text{trace}(\mathbb{E}_{\mu}[\phi(s,a)\phi(s,a)^{\top}])$ and 
$
c_{\psi}\{\mathcal{U}^{\text{lr}}_{2}\} = \sup_{\{\psi: \|\phi(s,a)^{\top}\psi \|_{L_2(\mu)} = \mathcal{U}^{\text{lr}}_{2}\}}\|\psi\|_{L_\infty}$.
\end{thm}

To the best of our knowledge, this is the first result PAC guarantees for an offline model-free RL algorithm in linear MDPs, requiring only realizability and single-policy concentrability. The regret bound we obtain is at least linear and, at best, sub-linear with respect to the feature dimension $d$. Our approach demonstrates a sample complexity improvement in terms of feature dimension compared to prior work by \citep{jin2021pessimism}, with a complexity of $\mathcal{O}(d^{1/2})$ versus $\mathcal{O}(d)$. It is worth noting that \citep{jin2021pessimism} only establishes results that compete with the optimal policy, and when specialized to linear MDPs, assumes the offline data has global coverage. Another previous study by \cite{xie2021bellman} achieves a similar sub-linear rate in $d$ as our approach; however, their algorithm is computationally intractable, relying on a much stronger Bellman-completeness assumption and requiring a small action space.

\vspace{-2mm}
\section{Experiments}
\vspace{-1mm}

In this section, we evaluate the performance of our 
practical algorithm by comparing to the model-free offline RL baselines including CQL \citep{kumar2020conservative}, BEAR \citep{kumar2019stabilizing}, BCQ \citep{fujimoto2019off}, OptiDICE \citep{lee2021optidice}, ATAC \citep{cheng2022adversarially}, IQL \citep{kostrikov2021offline}, and TD3+BC \cite{fujimoto2021minimalist}. We also compete with a popular model-based approach COMBO \citep{yu2021combo}. 

\textbf{Synthetic data.} We consider two synthetic environments: a synthetic CartPole environment from the OpenAI Gym \citep{brockman2016openai} and a simulated environment. Detailed discussions on the experimental designs are deferred to the Appendix. In both settings, following \citep{uehara2020minimax}, we first learn a sub-optimal policy using DQN \citep{mnih2015human} and then apply softmax to its $q$-function, divided by a temperature parameter $\alpha$ to set the action probabilities to define a behavior policy $\pi_{b}$.
   \begin{wrapfigure}{r}{0.58\textwidth} 
        \vspace{-3.5mm}
    \centering   
\includegraphics[width=0.58\textwidth]{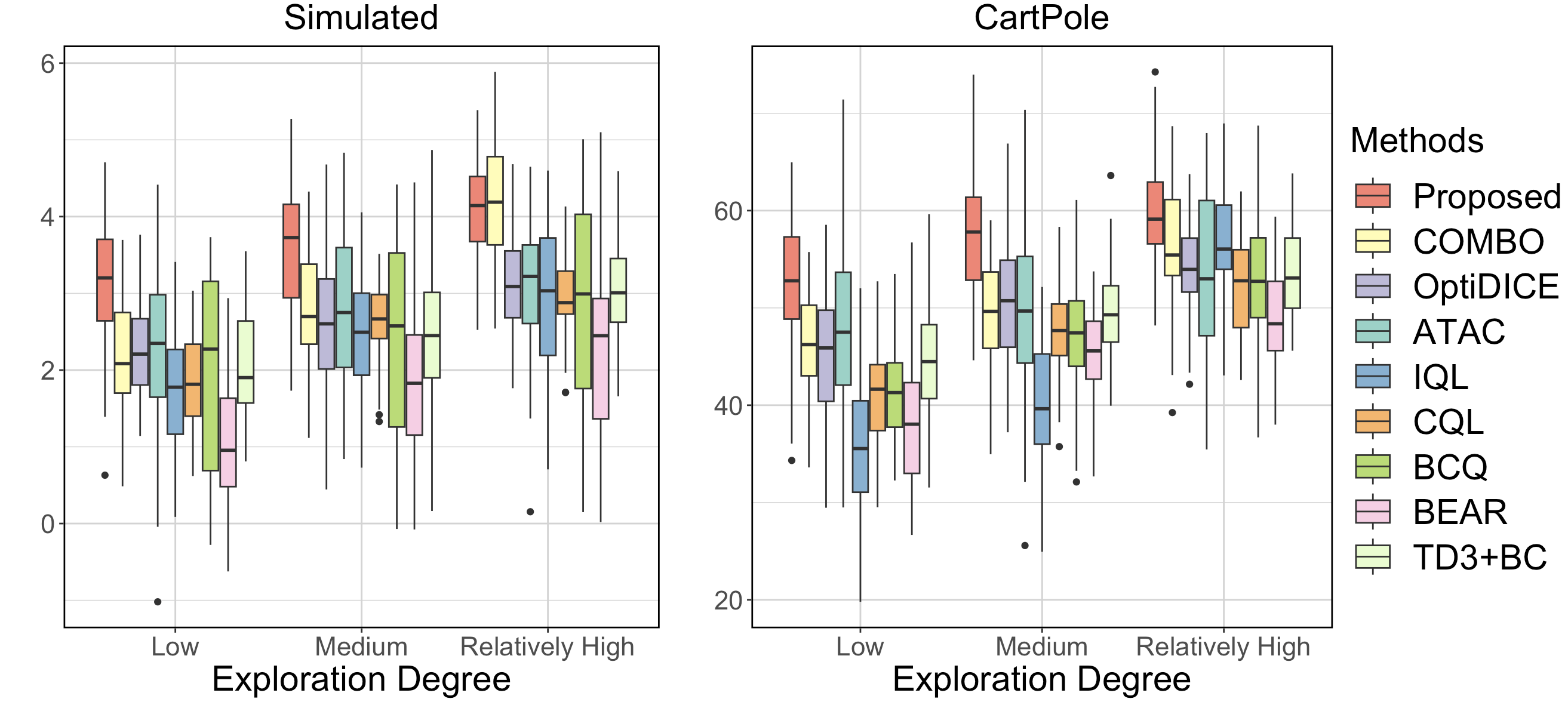}
        \vspace{-5mm}
    \caption{The boxplot of the discounted return over $50$ repeated experiments.}
    \vspace{-4mm}
     \label{fig:simu}
\end{wrapfigure}
A smaller $\alpha$ implies $\pi_{b}$ is less explored, and thus the support of $\mu=d_{\pi_{b}}$ is relatively small. We vary different values of $\alpha$ for evaluating the algorithm performance in ``low'', ``medium'' and ``relatively high'' offline data exploration scenarios.  We use $\gamma=0.95$ with the sample-size $n=1500$ in all experiments. Tuning parameter selection is an open problem in offline policy optimization. Fortunately, Theorem \ref{lr_opt} suggests an offline selection rule for hyper-parameters $\lambda$ and $c^{*}$. In the following experiments, we set the hyper-parameters satisfying the condition $\mathcal{O}(\frac{n^{1/4}}{d\log(\bar{V}\sqrt{n})})$. 
Figure \ref{fig:simu} shows that our algorithm almost consistently outperforms competing methods in different settings. This performance mainly benefits from the advantages exposed in our theoretical analysis, such as model extrapolation enhancement, relaxation of completeness-type assumptions on function approximation, etc. The only exception is the slightly poorer performance compared to COMBO in a high exploration setting, where COMBO may learn a good dynamic model with relatively sufficient exploration. We provide the experiment details in Appendix due to page limit.

\textbf{Benchmark data.} We evaluate our proposed approach on the D4RL benchmark of OpenAI Gym locomotion (walker2d, hopper, halfcheetah) and Maze2D tasks \citep{fu2020d4rl}, which encompasses a variety of dataset settings
and domains and positions our algorithm within the existing baselines. We take the results of COMBO, OptiDICE and ATAC from their original papers for Gym locomotion, and run COMBO and ATAC using author-provided implementations for Maze2D. The results of BCQ, BEAR methods from the D4RL original paper. In addition, CQL, IQL and TD3+BC are re-run to ensure a fair evaluation process for all tasks. As shown in Table \ref{benc}, the proposed algorithm achieves the best performance in 7 tasks and is comparable to the baselines in the remaining tasks. In addition to the evaluation of the policy performance, we also conduct sensitivity analyses on the hyperparameter-tuning and study the regret rate of convergence. 

\vspace{-2mm}

\begin{table}  [ht]
\scriptsize
\setlength{\tabcolsep}{0.08em}
\renewcommand{\arraystretch}{1.06}
		\centering
    \caption{The normalized score of the policy at
the last iteration of training, averaged over $5$ random seeds. The highest performing scores
are highlighted. The \textit{med}, \textit{med-rep}, and \textit{med-exp} is shorthand for \textit{medium}, \textit{medium-replay}, and \textit{medium-expert}, respectively. }
  \vspace{-1mm}
\begin{tabular}{c|ccccccccc }
\hline Tasks & Proposed & COMBO & BCQ & BEAR & OptiDICE & ATAC & CQL & IQL & TD3+BC \\
\hline  walker2d-med  & $80.8 \pm 5.1$ & $81.9 \pm 2.8$ & $53.1$ & $59.1$ & $21.8 \pm 7.1$  & $\mathbf{89.6}$  & $77.2 \pm 4.2$  & $78.3 \pm 4.3$ & $81.7 \pm 2.3$ \\
        hopper-med  & $94.9 \pm 4.3$ & $97.2 \pm 2.2$ & $54.5$ & $52.1$ & $94.1 \pm 3.7$  & $85.6$  & $74.3 \pm 5.8$  & $66.3 \pm 6.4$ & $\mathbf{98.4} \pm  1.6$ \\
         halfcheetah-med  & $\mathbf{58.1} \pm 1.4$ & $54.2 \pm 1.5$ & $40.7$ & $41.7$ & $38.2 \pm 0.1$  & $53.3$  & $37.2 \pm 0.3$  & $47.4 \pm 1.1$ &  $27.8 \pm 0.7$\\
         \hline
        walker2d-med-rep  & $\mathbf{99.6} \pm 2.9$ & $56.0 \pm 8.6$ & $15.0$ & $19.2$ & $21.6 \pm 2.1$  & $92.5$  & $20.8 \pm 1.6$  & $73.9 \pm 2.8$ & $34.4 \pm 4.2$ 
        \\
        hopper-med-rep  & $\mathbf{113.0} \pm 2.1$ & $89.5 \pm 1.8$ & $33.1$ & $33.7$ & $36.4 \pm 1.1$ & $102.5$  & $32.6 \pm 1.9$  & $94.7 \pm 1.5$ & $44.4 \pm 3.7$ \\
       halfcheetah-med-rep & $49.3 \pm 2.1$ & $\mathbf{55.1} \pm 1.0$ & $38.2$ & $38.6$ & $39.8 \pm 0.3$ & $48.0$  & $41.9 \pm 1.1$  & $44.2 \pm 2.5$ & $48.3 \pm 0.7$\\
       \hline
        walker2d-med-exp  & $108.2 \pm 7.4$ & $103.3 \pm 5.6$ & $57.5$ & $40.1$ & $74.8 \pm 9.2$  & $\mathbf{114.2}$  & $103.8 \pm 6.9$  & $109.6 \pm 7.0$ & $100.5 \pm 8.9$ \\
        hopper-med-exp & $117.8 \pm 1.9$ & $111.1 \pm 2.9$ & $110.9$ & $96.3$ & $111.5 \pm 0.6$ & $\mathbf{119.2}$  & $111.4 \pm 1.2$  & $91.5 \pm 2.2$ & $112.4 \pm 0.3$\\  
       halfcheetah-med-exp  & $\mathbf{98.5} \pm 3.8$ & $90.0 \pm 5.6$ & $64.7$ & $53.4$ & $91.1 \pm 3.7$ & $94.8$  & $66.7 \pm 8.9$  & $86.7 \pm 3.6$ & $95.9 \pm 3.9$\\    
       \hline
        walker2d-random  & $\mathbf{11.2} \pm 3.8$ & $7.0 \pm 3.6$ & $4.9$ & $7.3$ & $9.9 \pm 4.3$  & $6.8$  & $4.7 \pm 1.5$  & $5.8 \pm 2.8 $ & $3.4 \pm 1.7$\\   
        hopper-random  & $\mathbf{18.7} \pm 1.5$ & $17.9 \pm 1.4$ & $10.6$ & $11.4$ & $11.2 \pm 1.1$ & $17.5$  & $10.7 \pm 0.1$  & $10.8 \pm 0.6$ & $11.1 \pm 0.2$ \\
        halfcheetah-random  & $37.6 \pm 2.4$ & $\mathbf{38.8} \pm 3.7$ & $2.2$ & $25.1$ & $11.6 \pm 1.2$ & $3.9$  & $26.7 \pm 1.4$  & $22.4 \pm 1.8$ & $26.1 \pm 1.8$\\
        \hline  
       maze2d-umaze   & $96.5 \pm 27.8$ & $34.2 \pm 8.6$ & $12.8$ & $3.4$ & $\mathbf{111.0} \pm 8.3$  & $84.4 \pm 24.8$  & $50.5 \pm 7.9$  & $41.5 \pm 4.7$ &  $13.8 \pm 22.8$\\
        maze2d-med   & $137.5 \pm 18.9$ & $49.9 \pm 13.9$ & $8.3$ & $29.0$ & $145.2 \pm 17.5$ & $\mathbf{152.3} \pm 34.6$  & $28.6 \pm 9.2$  & $38.5 \pm 4.2$ & $59.1 \pm 44.7$\\
        maze2d-large  & $\mathbf{187.8} \pm 15.2$ & $128.2 \pm 17.3$ & $6.2$ & $4.6$ & $155.7 \pm 33.4$ & $142.1 \pm 33.8$  & $46.2 \pm 16.2$  & $54.2 \pm 18.1$ & $87.6 \pm 15.4$\\    
        \hline 
\end{tabular} 
\label{benc}
\end{table}


\vspace{-1mm}

\textbf{Real-world application.} The Ohio Type 1 Diabetes (OhioT1DM) dataset \citep{marling2020ohiot1dm} comprises a cohort of patients with Type-1 diabetes, where each patient exhibits different dynamics and 8 weeks of life-event data, including health status measurements and insulin injection dosages. Clinicians aim to adjust insulin injection dose levels \citep{marling2020ohiot1dm, bao2011improving} based on a patient's health status in order to maintain glucose levels within a specific range for safe dose recommendations. The state variables consist of health status measurements, and the action space is a bounded insulin dose range. The glycemic index serves as a reward function to assess the quality of dose suggestions. Since the data-generating process is unknown, we follow \citep{luckett2020estimating,li2023quasi} to utilize the Monte Carlo approximation of the estimated value function on the initial state of each trajectory to evaluate the performance of each method. The mean and standard deviation of the improvements on the Monto Carlo discounted returns are presented in Table \ref{real_data}. As a result, our algorithm achieves the best performance for almost all patients, except for Patient $552$. The main reason for the desired performance in real data is from the enhanced model extrapolation and relaxed function approximation requirements and outperforms the competing methods. This finding is consistent with the results
in the synthetic and benchmark datasets, demonstrating the potential applicability of the proposed algorithm in real-world environments.


\begin{table}  [ht]
\setlength{\tabcolsep}{0.25em}
\scriptsize
\vspace{-1mm}
  \caption{The baseline policy improvements over 50 repeated experiments in the OhioT1DM dataset.}
  \vspace{-1mm}
  \renewcommand{\arraystretch}{1.06}
		\centering
\begin{tabular}{c|ccccccccc}
\hline Patient ID & Proposed & COMBO & BCQ & BEAR & OptiDICE & ATAC & CQL & IQL & TD3+BC \\
\hline  
        $596$  & $\mathbf{6.5} \pm 1.1$ & $4.1 \pm 0.8$ & $3.8 \pm 0.9$ & $2.7 \pm 1.1$ & $4.7 \pm 1.1$ & $5.1 \pm 2.0$  & $4.6 \pm \mathbf{0.6}$  & $3.4 \pm 0.7$ & $4.8 \pm 1.3$\\
            $584$  & $\mathbf{33.1} \pm 1.8$ & $27.0 \pm 1.3$ & $20.3 \pm 1.2$ & $22.9 \pm 1.6$ & $27.7 \pm 1.9$  & $26.9 \pm 2.6$  & $21.6 \pm \mathbf{1.2}$  & $22.7 \pm 1.3$  & $22.4 \pm 1.7$ \\
        $567$  & $\mathbf{36.9} \pm \mathbf{1.3}$ & $30.6 \pm 2.0$ & $24.3 \pm 1.4$ & $25.6 \pm 1.4$ & $28.8 \pm 2.2$  & $29.7 \pm 2.8$  & $26.5 \pm 1.4$  & $25.8 \pm 1.4$ & $ 
 27.8 \pm 1.5 $ \\
                    $552$  & $7.9 \pm 0.9$ & $6.8 \pm 0.7$ & $5.7 \pm 0.5$ & $5.0 \pm  0.8$ & $\mathbf{8.1} \pm 0.9$  & $7.2 \pm 1.5$  & $6.7 \pm \mathbf{0.4}$  & $6.1 \pm 0.5$ & $ 7.4 \pm 0.8 $ \\
                            $544$  & $\mathbf{13.2} \pm 1.9$ & $9.8 \pm 1.5$ & $7.5 \pm 2.5$ & $5.9 \pm \mathbf{0.8}$ & $10.3 \pm 1.8$  & $10.1 \pm 2.1$  & $8.7 \pm 1.0$  & $7.8 \pm 0.9$ & $9.7 \pm 0.8$\\
$540$  & $\mathbf{20.4} \pm \mathbf{0.5}$ & $17.5 \pm 0.9$ & $14.3 \pm 0.6$ & $12.7 \pm 0.5$ & $17.9 \pm 0.9$  & $18.2 \pm 1.4$  & $16.5 \pm 0.5$  & $14.0 \pm 0.6$ & $17.1 \pm 0.8$ \\
\hline
\end{tabular} \label{real_data}
\end{table}

\vspace{-3mm}
\section{Conclusion}
\vspace{-1mm}

We study offline RL with limited exploration in function approximation settings. We propose a bi-level policy optimization framework, which can be further solved by a computationally practical penalized adversarial estimation algorithm, offering strong theoretical and empirical guarantees. Regarding limitations and future work, while the penalized adversarial estimation is more computationally efficient than the previously constrained problem, it may still be more challenging to solve than single-stage optimization problems. Another future direction is to explore environments with unobservable confounders. It will be interesting to address these limitations in future works.


\section{Acknowledments}

The author is grateful to the five anonymous reviewers and the area chair for their valuable comments and suggestions. 


\bibliographystyle{plain}
\bibliography{mycite}  

\newpage

\appendix

{\center \Large
\textbf{Supplementary Material to 
\center 
``Bi-Level Offline Policy Optimization with Limited Exploration''}
}


\counterwithin{figure}{section}
\counterwithin{table}{section}
\counterwithin{equation}{section}
\counterwithin{assump}{section}
\counterwithin{algorithm}{section}

{
  \hypersetup{linkcolor=black}
  \tableofcontents
}
  \addtocontents{toc}{\protect\setcounter{tocdepth}{3}}

\section{Discussion on Algorithm 1}
\label{detailed_algo_sec}

In this section, we provide a pronounced discussion on Algorithm 1 in maintext by offering more details on solving adversarial estimation over $q$ and $\tau$, as well as establishing the theoretical convergence guarantee in Theorem \ref{conv_thm}. The convergence of the step provides a basis for us to use the mirror descent for policy updating. The detailed version of Algorithm 1 in maintext is summarized in Algorithm \ref{prox_map_detail}. 

In function approximation settings, the $\Omega, \mathcal{Q}, \Pi$ are often represented by compact parametric functions in practice, either in linear or non-linear function classes \cite{sutton2018reinforcement}. In the following, we denote these parameters as $\psi$ and $\theta$ and $\omega$ corresponding to $\Omega_{\psi}$ and $\mathcal{Q}_{\theta}$, and $\Pi_{\omega}$ respectively. 

Under this parametric setting, we focus on solving the adversarial loss ${\mathcal{L}^{\circ}}(q,\tau,\pi,c^{*},\lambda)$, which can be expressed as:
\#
 {\mathcal{L}}(q_{\theta},\tau_{\psi},\pi_{\omega},c^{*},\lambda)   = & q_{\theta} (s^{0}, \pi) + \frac{1}{(1-\gamma)n}\bigg\{c^{*}\Big|\sum^{n}_{i=1}   \tau_{\psi}(s_i,a_i)\left(q_{\theta}(s_i,a_i)-r_i -\gamma q_{\theta}(s^{\prime}_i,\pi_{\omega})\right)\Big| \notag \\
& \qquad \qquad  \qquad \qquad -  \lambda\sum^{n}_{i=1} \mathbb{D}(\tau_{\psi}(s_i,a_i))\bigg\}.
\label{algo_eq1}
\#
As stated in Algorithm 1, at each iteration, we aim to solve 
$\max_{\psi} \min_{\theta}  {\mathcal{L}^{\circ}}(q_{\theta},\tau_{\psi}, \pi_{\omega},c^{*},\lambda)$, 
which forms a saddle-point formulation, and we denote the saddle point as $(\psi^{*},\theta^{*})$ (should depend on $\pi_{\omega}$, but we omit here for simplifying the notation). At the same time, we denote the population loss as 
\#
 {\mathcal{L}^{\circ}}(q_{\theta},\tau_{\psi},\pi_{\omega},c^{*},\lambda)   = & q_{\theta}(s^{0}, \pi) + \frac{1}{(1-\gamma)}\bigg\{c^{*}\Big|\mathbb{E}_{\mu} \left[  \tau_{\psi}(s,a)\left(q_{\theta}(s,a)-r(s,a) -\gamma q_{\theta}(s^{\prime},\pi_{\omega})\right)\right]\Big| \notag \\
& \qquad \qquad  \qquad \qquad  - \lambda \mathbb{E}_{\mu}[\mathbb{D}(\tau_{\psi}(s,a))]\bigg\}.
\#
In the following, we omit the arguments $\pi_{\omega},c^{*},\lambda$ in the expression for simplicity, and thus using 
\#
{\mathcal{L}^{\circ}}(q_{\theta},\tau_{\psi})
\label{algo_eq2}
\#
to denote the population loss for  ${\mathcal{L}^{\circ}}(q_{\theta},\tau_{\psi},\pi_{\omega},c^{*},\lambda)$. 

We can observe that the inner minimization problem is relatively easy to solve. 
In addition to the closed-form solution as discussed in maintext, the feature mapping class is sufficient for modeling $\mathcal{Q}$, as demonstrated in \citep{dai2020coindice}. The feature mapping class simplifies the optimization, making it efficiently solvable by various algorithms as discussed in \cite{sriperumbudur2011universality}. In contrast, the more challenging aspect is optimizing $\tau_{\psi}$. Due to its complex structure, it demands a sufficiently flexible non-linear function approximation class, e.g., deep neural networks, for optimization \citep{jiang2020minimax}. Unfortunately, concavity typically does not hold for non-linear function approximation classes, and thus the outer maximization of $
\max_{\psi} \min_{\theta}  {\mathcal{L}^{\circ}}(q_{\theta},\tau_{\psi})$ is also affected. As a result, we need to develop a more efficient and convergent algorithm. Therefore, we regard solving a non-concave maximization problem, conditional on the solved global optimizer $\bar{q}_{\theta} := \argmin_{\theta} {\mathcal{L}^{\circ}}(q_{\theta},\tau_{\psi})$. Under this framework, we first study the gradients of the objective function with respect to ${\psi}$. Define $\bar{{\mathcal{L}^{\circ}}}(\tau_{\psi}) = {\mathcal{L}^{\circ}}( \bar{q}_{\theta},\tau_{\psi})$, then the gradient of $\bar{{\mathcal{L}^{\circ}}}(\tau_{\psi})$ with respect to $\psi$ satisfies 
\# \nabla_{\psi}\bar{{\mathcal{L}^{\circ}}}(\tau_{\psi}) = &  \frac{\mathbb{E}_{\mu}[(r(s,a)+q_{\theta}(s, a)-\gamma q_{\theta}(s^{\prime}, \pi))(\tau_{\psi}(s,a)/|\tau_{\psi}(s,a)|)\nabla_{\psi}\tau_{\psi}(s,a)]}{1-\gamma} \\
& - \frac{\lambda\mathbb{E}_{\mu}[\mathbb{D}^{\prime}(\tau(s,a))\nabla_{\psi}\tau_{\psi}(s,a)]}{1-\gamma}.
\label{gradient_theory}
\#
With the gradients provided in \eqref{gradient_theory}, we propose a stochastic approximation algorithm to update $\tau_{\psi}$. At each iteration, we update $\tau_{\psi}$ by solving the proximal mapping \cite{parikh2014proximal}: 
\#
\text{Proj}_{\psi}(\psi^{*}, \nabla; D_{Berg}):= \argmax_{\psi} \{ \langle \psi, \nabla \rangle - D_{Berg}(\psi^{*}, \psi)\},
\label{algo_eq3}
\#
where $\psi^{*}$ can be viewed as the current update of the parameter, $D_{Berg}(\cdot, \cdot)$ denotes the Bregman divergence as discussed in \cite{reem2019re}, and $\nabla$ represents the scaled stochastic gradient of the parameter of interest. In practice, we may consider using the Euclidean distance to reduce the computational burden. Once $q$ and $\tau$ are solved, we apply mirror descent in terms of the negative entropy $D_{\text{NegEntropy}}$ \citep{beck2017first}. That is, given a
stochastic gradient direction of $\pi$ we solve the prox-mapping in each iteration. Note that, it follows from \cite{parikh2014proximal}, step 4 in Algorithm 1 (step 13 in Algorithm \ref{prox_map_detail}) has a closed-form exponential updating rule, particularly with the negative entropy $D_{\text{NegEntropy}}$, as 
\$
\pi_{w^{k}}(\cdot|s) \propto \pi_{w^{k-1}}\exp(\zeta q^{k}(s,\cdot)),
\$
for any $s$. The detailed version of the proposed optimization algorithm is presented in Algorithm \ref{prox_map_detail}. 

\renewcommand\footnoterule{}
\noindent\begin{minipage}{\textwidth}
\vspace{-4mm}
\begin{algorithm}[H]
\setstretch{1.12}
	\caption{Adversarial proximal-mapping algorithm (detailed version)}
\label{prox_map_detail}
	\begin{algorithmic}[1]
	\STATE \textbf{Input} observed data $\mathcal{D}_{1:n}=\{(s_i,a_i,r_i,s_i^{\prime})\}^n_{i=1}$ and and the initial state $s^{0}$.  
 	\STATE \textbf{Initialize} 
the parameters $\theta^{(0)},\psi^{(0)},\omega^{(0)},c^{*}$, $\lambda$, $\zeta$, $\eta^{0}$, $\bar{K}$ and $\bar{T}$. 
				\STATE \textbf{For} $k=1$ to $\bar{K}$:
		\STATE \; \textbf{Update} $\psi^{(k)}$ and $\theta^{(k)}$:
  \STATE  \;  \;   \; 
 \textbf{Initialize} $\psi^{0} = \psi^{(k-1)}$ and $\theta^{0} = \theta^{(k-1)}$ and $\eta^{0} = \eta^{0}$.
    			\STATE  \;  \;  \;  \textbf{For} $t=1$ to $t = T$:
   	\STATE \; \;  \;  \;  \ Update $\theta^{t}$ by solving 
${\mathcal{L}}(q_{\theta},\tau_{\psi^{t-1}},\pi_{\omega^{k-1}},c^{*},\lambda)$ in \eqref{algo_eq1}. 
    \STATE \; \;  \;  \; \ Decay the stepsize $\eta^{t}$ of the rate $\mathcal{O}(t^{-1/4})$. 
      \STATE \; \;  \;  \; \ Compute the stochastic gradient with respect to $\psi$ as $\widetilde{\nabla}_{\psi}{\mathcal{L}^{\circ}}(\tau_{\psi},q_{\theta^{t}})$ in \eqref{algo_eq2}. 
       \STATE \; \;  \;  \; \ Update $\psi^{t}$ by solving: $\psi^{t} = \text{Proj}_{\psi}(\psi^{t-1}, \eta^{t}\widetilde{\nabla}_{\psi}{\mathcal{L}^{\circ}}(\tau_{\psi},q_{\theta^{t}});D_{Berg})$ in \eqref{algo_eq3}.
       \STATE \;  \;  \; \textbf{End for}
       \STATE \;  \;  \; \textbf{Output} $\psi^{(k)} =  \psi^{T}$ and  $\theta^{(k)} = \theta^{T}$.
	\STATE \; \textbf{Update}  $\omega^{k}$ by solving 
	$\underset{\omega}{\operatorname{argmax}} \; \zeta \left\langle q^{k}(s,\cdot), \pi_{\omega}(\cdot|s) \right\rangle-D_{\text{NegEntropy}}\left(\pi_{\omega}(\cdot|s), \pi_{\omega^{k-1}}(\cdot|s)\right).
	$	
\STATE \textbf{Return} the policy $\widehat{\pi}$, which randomly selects a policy from the set $\{\pi^{k}\}^{\bar{K}}_{k=1}$.
\end{algorithmic}
\end{algorithm}
\end{minipage}

In the following, we demonstrate that our algorithm is convergent with a sublinear rate even under non-linear (non-concave) settings regarding solving the steps $6$ to $12$ in Algorithm \ref{prox_map_detail}. Before we state our convergence guarantee, we make the following regular assumptions as stated in \citep{zhou2023distributional}. 

\begin{assump}[$L_0$-Lipschitz continuity on gradient]
\label{lip_tau}
For any $\tau_{\psi} \in \Omega_{\psi}$, $\tau_{\psi}$ is differentiable (not necessarily convex or concave), bounded from below,
$
\| \nabla_{\psi}\tau_{\psi_{1}}(s,a) - \nabla_{\psi}\tau_{\psi_{2}}(s,a) \| \leq L_{0}\| \psi_{1} - \psi_{2}\|, \; \text{for any} \; s,a 
$,
where $L_{0} < \infty$ is some universal Lipschitz constant and $\|\cdot\|$ denotes the Euclidean norm. 
\end{assump}

Assumption \ref{lip_tau} imposes the first-order smoothness condition on the specified function class. 


\begin{assump}[Smooth function class]
\label{lip_q}
$
    | q_{\theta_1}(s,a) - q_{\theta_2}(s,a) | \leq L_{0}\| \theta_1 - \theta_2\|,  \; \text{for any} \; s,a, \; \text{and} \; q_{\theta} \in \mathcal{Q}_{\theta}$.
\end{assump}
Assumption \ref{lip_q} holds for a wide range of function approximation classes, including feature mapping space with smooth basis functions, non-linear approximation classes, DNNs with Leaky ReLU activation function, or spectral normalization on ReLU activation \citep{izmailov2018averaging}.

\begin{assump}
\label{bound_gradient}
The gradient of function $\tau_{\psi}(\cdot)$ evaluated at saddle point $\psi^{*}$ is bounded above; i.e., 
$\nabla_{\psi}\tau_{\psi^{*}}(s,a) < c_3$ uniformly over $(s,a)$ for some finite and positive constant $c_3$. 
\end{assump}

Assumption \ref{bound_gradient} is a much weaker assumption compared to the bounded variance of stochastic gradients assumption which is commonly made 
in the existing literature \citep{reddi2016stochastic,mokhtari2020unified}. In the following, we derive the convergence rate, which holds for non-concave function approximation class $\Omega_{\psi}$.

\begin{thm}[Convergence to a stationary point \citep{zhou2023distributional}]
\label{conv_thm}
Under Assumption 3 in maintext, and Assumptions \ref{lip_tau}-\ref{bound_gradient} above, suppose the steps $6$-$12$ in Algorithm \ref{prox_map_detail} runs $T \geq1$ rounds with stepsize 
\$
\eta^{t} = \min\{\sqrt[\leftroot{-1}\uproot{2}\scriptstyle 4]{tT4\mathbb{G}^2/\sigma^4_{\max}C_1},1/C_1\},
\$ 
for $t=1,...,T$ and Euclidean distance is used for Bergman divergence. If we pick up the solution output $\psi^{T^{\star}}$ following the probability mass function
\$
P(T^{\star}=t) = \frac{2\eta^{t}-(\eta^{t})^2C_1}{\sum^{T}_{t=1}(2\eta^{t}-(\eta^{t})^2C_1)},
\$
then it follows that 
\#
\mathbb{E}[\|\nabla_{\psi}\bar{{\mathcal{L}^{\circ}}}(\tau_{\psi^{T^{\star}}}) \|^2] \leq  \sqrt{\frac{2\mathbb{G} C_1\sigma^2_{\max}}{T}}  + \frac{\sqrt{2\mathbb{G} C_1\sigma^2_{\max}}}{T^{3/4}} + \frac{2\mathbb{G} C_1}{T} ,
\label{conv_rate}
\#
where $\bar{{\mathcal{L}^{\circ}}}(\tau_{\psi})$ is defined in \eqref{gradient_theory} and $\mathbb{G}:= \bar{{\mathcal{L}^{\circ}}}(\tau_{\psi^{0}}) -  \min_{\psi}\bar{{\mathcal{L}^{\circ}}}(\tau_{\psi})$ measures the distance of the initial and optimal solution, $C_1$ is Lipschitz constant depending on $c^{*},c_2,c_3,M, L_0, \bar{V}$ and $\lambda$. Recall that $c_2$ and $M$ are from the definition of $\mathbb{D}$. Here the variance of the stochastic gradient is bounded above by $
\sigma_{\max} := \max_{t\in 1:T}\sqrt{ c_4\|\widetilde{\theta}(\psi^{t}) - \theta^{*}\|^2 + c_5\|\psi^{t} - \psi^{*}\|^2}$,
for some constants $c_4,c_5$ depending on $c^{*}, c_2,c_3, L_0,\bar{V},\lambda$ and $\gamma$. Here, $\widetilde{\theta}(\psi^{t})$ is the optimizer for ${\mathcal{L}^{\circ}}(q_\theta,\tau_{\psi^{t}})$.
\end{thm}

Theorem \ref{conv_thm} is adapted
from Theorem 6.5 in \citep{zhou2023distributional} on local convergence. Theorem \ref{conv_thm} implies that the steps 6-12 in Algorithm \ref{prox_map_detail} can converge sublinearly to a stationary point if the $\sigma_{\max}$ is sufficiently small. The rate of convergence is also affected by the smoothness of the class $\Omega_{\psi}$ and the distance of the initial and optimal solution. 

\section{Experiment Details}

We include our source code for experiments and algorithm, and the guideline for access to the OhioT1DM dataset in this \href{https://anonymous.4open.science/r/bilevel-offline-policy-optimization-2685}{GitHub repository}.

\subsection{Environment Settings}

\textbf{Simulated environment.} For the simulated environment setting, the system dynamics are given by
$$
\begin{aligned}
s^{t+1} & =\left(\begin{array}{cc}
0.75\left(2 a^{ t}-1\right) & 0 \\
0 & 0.75\left(1-2 a^{ t}\right)
\end{array}\right) s^{t}+\left(\begin{array}{cc}
0 & 1 \\
1 & 0
\end{array}\right)\odot
s^{t}{s^{t}}^{\top}\mathbb{I}_{2\times 1} + 
\varepsilon^t, \\
r^{t} & ={s^{t+1}}^{\top}\left(\begin{array}{l}
2 \\
1
\end{array}\right)-\frac{1}{4}\left(2 a^{ t}-1\right) + ({s^{t+1}}^{\top}s^{t+1})^{\frac{3}{2}}\odot\left(\begin{array}{c}
0.25 \\
0.5 
\end{array}\right),
\end{aligned}
$$
for $t \geq 0$, where $\odot$ denotes the Hadamard product, $\mathbb{I}$ is the identity matrix, the noise $\left\{\varepsilon^t\right\}_{t \geq 0} \stackrel{i i d}{\sim} N\left({0}_{2\times 1}, 0.25\mathbb{I}_{2\times 2}\right)$ and the initial state variable $s^{0} \sim N\left({0}_{2\times 1}, 0.25\mathbb{I}_{2\times 2}\right)$. The transition dynamic mainly follows the design in \cite{shi2020statistical}, but the reward function we consider here is more complex. In this setting, we consider a binary action space $a^{t} = \{0,1\}$. 

\textbf{CartPole environment.} 
We utilize the CartPole environment from OpenAI Gym \cite{brockman2016openai}, a standard benchmark in RL for evaluating policies. The 4-dimensional state space in this environment is represented as 
 $s^{t} = (s^{t}_{[1]}, s^{t}_{[2]}, s^{t}_{[3]}, s^{t}_{[4]})$, encompassing both the cart's position and velocity and the pole's angle and angular velocity. The action space is binary, with actions $\{0, 1\}$, representing pushes to the left or right, respectively. To enhance the differentiation between various policy values, we adopt a modified reward function, as in \cite{shi2022minimax,zhou2023distributional}. The reward function is defined as:
\$
r^{t}=-1+\left|2-\frac{s^{t}_{[1]}}{{s^{t}_{[1]}}({\text{clip}})}\right|\left|2-\frac{s^{t}_{[3]}}{{s^{t}_{[3]}}({\text{clip}})}\right|.
\$
Here, $s^{t}_{[1]}$ and $s^{t}_{[3]}$ represent the cart's position and the pole's angle, respectively. The terms ${s^{t}_{[1]}}({\text{clip}})$ and ${s^{t}_{[3]}}({\text{clip}})$ denote the thresholds at which the episode terminates (done = True) if either $|s^{t}_{[1]}| \geq {s^{t}_{[1]}}({\text{clip}})$ or $ |s^{t}_{[3]}| \geq {s^{t}_{[3]}}({\text{clip}})$ is satisfied. Under this definition, a higher reward is obtained when the cart is closer to the center and the pole's angle is closer to the perpendicular position.

\textbf{D4RL benchmark environments.} We use Maze2D and Gym-locomotion environments of D4RL benchmark \citep{fu2020d4rl,lee2021optidice} to evaluate the proposed algorithm in continuous control tasks. We summarize the descriptions of different task settings in \citep{fu2020d4rl} in the following: 

\textbf{Maze2D} is a navigation task set within a 2D state space where the agent aims to reach a predetermined goal location. By leveraging previously collected trajectories, the agent's objective is to determine the shortest path to the destination. The complexity of the mazes increases in the sequence: "maze2d-umaze," "maze2d-medium," and "maze2d-large."

\textbf{Gym-locomotion.} For each task within the Gym-locomotion continuous controls set, which includes \{hopper, walker2d, halfcheetah\}. We refer the readers to \citep{fu2020d4rl} for detailed background for the above-mentioned tasks. In our experiments, data is generated and collected in the following manners:

\begin{itemize}
\item random: This dataset is produced using a policy initialized at random for each task.
\item medium: This dataset is derived from a policy trained with the SAC algorithm in \citep{haarnoja2018softac}. The training is stopped prematurely through early stopping.
    \item medium-replay: This combines two subsets. The ``replay'' subset consists of samples collected during the training of the policy for the ``medium'' dataset. Therefore, the "medium-replay" dataset encompasses both the ``medium'' and "replay" data.
    \item 
medium-expert: This dataset supplements an equal number of expert trajectories with suboptimal trajectories. The suboptimal samples are sourced either from a uniformly random policy or from a medium-performance policy. 
\end{itemize}

\textbf{Real world enviroment: OhioT1DM offline dataset. } 

We applied the proposed algorithm on the Ohio Type 1 Diabetes Mobile Health (OhioT1DM) study \citep{marling2020ohiot1dm}. This dataset comprises six patients with type 1 diabetes, each contributing eight weeks of life-event data—spanning health status measurements to insulin injection dosages. Given the unique glucose dynamics of each patient, we treat each patient's data as an individual dataset, in line with \cite{zhu2020causal}. Thus, daily data is seen as an individual trajectory. Data points are aggregated over 60-minute intervals, ensuring a maximum horizon length of 24. After the exclusion of missing samples and outliers, the total number of transition pairs for each patient's dataset approximates $n=360$. The state variable $s^{t}$ is set to be a three-dimensional vector including the average blood glucose levels $s^{t}_{[1]}$, the average heart rate $s^{t}_{[2]}$ and the total carbohydrates $s^{t}_{[3]}$ intake during the period time $[t-1,t]$. Here, the reward is defined as the average of the index of glycemic control \citep{rodbard2009interpretation,li2023quasi} between time $t-1$ and $t$, measuring the health status of the patient's glucose level. That is
\$
{r^{t}}=-\frac{\mathbb{I}({s^{t}_{[1]}}>140)|{s^{t}_{[1]}}-140|^{1.10} + \mathbb{I}({s^{t}_{[1]}}<80)({s^{t}_{[1]}} - 80)^{2}}{30},
\$
which implies that reward $r^{t}$ is non-positive and a larger value is preferred. Then we estimate the optimal policy by treating each day as an independent sample.
We study the individualized dose-finding problem
by selecting the optimal continuous dose level for intervention options. For model performance evaluation, since the data-generating process is unknown, we follow \cite{luckett2020estimating} to utilize the Monte Carlo approximation of the estimated function of the initial state of each trajectory to evaluate the performance of each method. To better evaluate the stability and performance of each method, we randomly select $20$ trajectories from each individual based on available trajectories $50$ times and apply all methods to the selected data.  The mean and standard deviation of the improvements on the Monto Carlo discounted returns are presented in Table 2 in maintext.

\subsection{Implementation Details}

In the synthetic environments,  we first learn a sub-optimal policy using DQN \citep{mnih2015human} and then apply softmax to its $q$-function, divided by a temperature parameter $\alpha$ to set the action probabilities to define a behavior policy $\pi_{b}$. In particular, we set $\alpha=0.1,0.5,1$ for the three degree of exploration ``Low'', ``Medium'', and ``Relatively High'', respectively. For the implementation, we set the detection function as a quadratic form, i.e., $\mathbb{D}(x) = \frac{1}{2}(x-1)^2$, which satisfies the definition of $\mathbb{D}(x)$ in Definition 3.1 in maintext. To evaluate the policy obtained from the proposed method in synthetic experiments, we generate 100 independent trajectories, each with a length of 100 based on the learned policy. We sample each action by the learned policy $\pi(a|s)$ and calculate the discounted sum of reward for each trajectory. We compare the discounted return of each method and output the results in maintext.

For function approximation in $\mathcal{Q}_{\theta}$ class in our practical implementation, we set the function spaces $\mathcal{Q}_{\theta}$ to RKHSs to facilitate the computation. For function modeling in $\Omega_{\psi}$, we model $\Omega_{\psi}$ by feedforward neural networks to handle the complex behavior of $\tau$.  The radius of the function class is selected to be sufficiently large to ensure the flexibility of the $\Omega_{\psi}$. For the feedforward neural networks modeling, we are parameterized by a two-layer neural network with a layer width 
$256$ and using ReLU as activation functions.  For the RKHS modeling, we use the Gaussian RBF kernel. RBF kernel, for any sample $x$ and $x^{\prime}$
\$
K\left(x;x^{\prime}\right):=\exp \left(-\frac{\left\|x-x^{\prime}\right\|^2}{2 \text{bw}^2}\right),
\$
where $\text{bw}$ is the bandwidth. In our numerical experiments, we use Silverman's rule of thumb for bandwidth selection \citep{silverman1986density}. In particular, we apply the finite representer theorem in RKSH to model $\theta \in \mathcal{Q}$ as 
$
q_{\theta}(s,a) =\sum^{n}_{i=1}K(\{s,a\},\{s_i,a_i\})\theta_i,
$
for the parameters of interest $\{\theta_i\}^{n}_{i=1}$. In step 7 in Algorithm \ref{prox_map_detail}, we optimize $\theta^{t}$ with a fixed $\psi^{t-1}$ using stochastic gradient descent with learning rate $5\times10^{-3}$, and set the stepsize $\eta^{0} = 1\times10^{-3}$. We set the decay learning rate $\eta^{t} $ for the $t$th iteration be $\frac{\eta^0}{1+0.3\cdot t^{1/4}}$, where $\alpha_0$ is the learning rate of the initial iteration for optimizating $\psi^{t}$. For updating the policy, we model the policy class $\Pi_{\omega}$ by a softmax policy class or Gaussian distribution a two-layer neural network with a layer width $64$. The updating rate $\zeta$ is also set to $3\times10^{-3}$. The class $\Omega_{\omega}$ and $\Pi_{\omega}$ and $\mathcal{Q}_{\theta}$ are optimized with Adam \citep{kingma2014adam}. For hyperparameters-tuning, we set hyper-parameters satisfying the condition $\lambda = c^{*} = \frac{2\cdot n^{1/4}}{3\cdot d\log(\bar{V}\sqrt{n})}$ via a offline selection rule inspired from Theorem 5.2.


For the implementation of competing methods, we implement the methods BEAR, CQL, IQL, BCQ, and COMBO mainly based on the popular offline deep reinforcement learning library \citep{seno2022d3rlpy}. For the general optimization and function approximation settings,  we use a multi-layer perceptron (MLP) with 2 hidden layers, each with 256 units for function approximation. We set the batch size to be 64, and use ReLU function as the activation function. In addition to the explicitly mentioned in the following, we choose the learning rate from the set of $\{3\times10^{-4},1\times10^{-4},3\times10^{-5}\}$. We use Adam as the optimizer for learning the neural network parameters. Specifically, for BEAR, the MMD constraint parameter is tuned over the candidate set $\{0.1, 0.25, 0.5, 0.75, 1\}$
 as in \citep{kumar2019stabilizing}. The samples of MMD 
 is tuned over the set $5,10,15$. The KL-control baseline uses automatic temperature tuning as in \citep{kumar2019stabilizing}. For CQL, 
we follow the author-released default settings but we modify the actor learning rate and use a fixed $\alpha$ instead of the Lagrange variant. This modification is to match the 
hyperparameters defined in their paper as \cite{fujimoto2021minimalist} found
the original hyperparameters performed better. For IQL, we use cosine schedule for the actor learning rate. For COMBO, we selected the conservative coefficient from the set $\{0.5, 1, 2.5\}$ and found $1$ is the best. We choose $\rho(s,a)$ in \citep{yu2021combo} as the soft-maximum of the $q$-values and estimated with log-sum-exp. In addition, we set up the learning rate for policy and value function updates as $1 \times 10^{-4}$ and $3 \times 10^{-5}$, respectively. For the implementation of the methods, ATAC and OptiDICE, we use the source code provided by the authors  \cite{cheng2022adversarially} and \cite{lee2021optidice}. In particular, we follow the basic implementation for OptiDICE setup in \citep{lee2021optidice}, we model the value function class, the advantage function class and the policy class using  fully-connected MLPs with two hidden layers and ReLU
activations, where the number of hidden units on each layer is equal to 256.  For the optimization of each network, we use Adam optimizer and its learning rate $0.0003$. The batch size is set to be $32$. We select the regularization coefficient to be $0.1$. Before training neural networks, we
preprocess the dataset $\mathcal{D}_{1:n}$ by standardizing observations and rewards. In terms of the details for implementing ATAC, we follow \cite{cheng2022adversarially}, employing separate 3-layer fully connected neural networks for realizing the policy and the critics. Each hidden layer comprises 256 neurons and utilizes a ReLU activation function, while the output layer employs a linear function. We use a softmax policy class for the policy. Optimization is performed using Adam with a minibatch size of 64, and we set the two-timescale stepsizes in \cite{cheng2022adversarially} as $\eta_{\text {fast }}=0.0005$ and $\eta_{\text {slow }}=10^{-3} \eta_{\text {fast }}$,  with values $\eta_{\text {fast }} = 5 \times 10^{-4}$ and $\eta_{\text {slow }}=5 \times 10^{-5}$. The mixing weights in a combination of the temporal difference (TD) losses of the critic and its delayed targets are set to $w = 0.5$ to ensure stability. Finally, for TD3+BC, we follow the default implementation in the original paper but we make a flexible choice on the hyperparameter $\lambda$ not fix $\lambda = \alpha$ in the original paper. We set and implement $\lambda=\frac{\alpha}{\frac{1}{n}\sum_{(s, a)}|q(s, a)|}$, which decreases the value of $\lambda$ when the function estimate is divergent due to
extrapolation error \citep{fujimoto2021minimalist,fujimoto2019off}. We found this setup helps to improve the performance of the algorithm.

\subsection{Additition Experiments Results}

\textbf{Sensitivity Analyses}  Tuning parameter selection is an open problem in offline policy optimization. Fortunately, our algorithm has desired robustness to choices of hyperparameters, when we set the hyperparameters satisfying the conditions in Theorem 5.2, i.e., $\mathcal{O}(\frac{n^{1/4}}{ d\log(\bar{V}\sqrt{n})})$. To validate the robustness of the proposed algorithm with respect to the hyperparameter-tuning, we conduct sensitivity analyses on the walker2d, hopper, and halfcheetah datasets. Figure \ref{fig:sen} shows that the policy performance is robust over a wide value range of $c^{*}$ and $\lambda$ ($\lambda,c^{*}$ in $[1,0.01]$), and the performance of under our choice $(c^{*}=0.1,\lambda=0.1)$ shown in Table \ref{sen_va} is close the best.

In Tables \ref{hop}-\ref{maze}, we report the results of the experiments for sensitivity analyses on the values of the hyperparameters vs policy performance on the additional D4RL benchmarks (hopper, walker2d, maze2d), in addition to the results (halfcheetah) we previously presented. Each number in the following tables is the normalized score of the policy at the last iteration of training, averaged over $3$ random seeds. From the tables, we can see that, our algorithm demonstrates robustness over a wide value range of hyperparameters. Also, the policy performance under our hyperparameter choice is close to the best performance in the table, which indicates the effectiveness of our proposed hyperparameter selection rule.

\begin{table}[h!]
\footnotesize
\setlength{\tabcolsep}{0.1em}
\label{hop}
\centering
\caption{Hopper-medium-replay: Our selection rule chooses $\lambda=c^*=0.25$ with the policy performance $114.0 \pm 2.4$.}
\begin{tabular}{c|c|c|c|c|c|c}
\hline
$c^*(\mathrm{col}), \lambda(\mathrm{row})$ & $\mathbf{2.5}$ & $\mathbf{1}$ & $\mathbf{0.1}$ & $\mathbf{0.01}$ & $\mathbf{0.0025}$ & $\mathbf{0.001}$ \\
\hline
$\mathbf{2.5}$ & $108.1 \pm 2.7$ & $109.7 \pm 2.4$ & $111.6 \pm 3.1$ & $111.1 \pm 2.7$ & $109.5 \pm 2.1$ & $108.3 \pm 4.4$ \\
\hline
$\mathbf{1}$ & $109.3 \pm 1.7$ & $111.8 \pm 2.4$ & $113.2 \pm 2.6$ & $112.9 \pm 2.2$ & $112.0 \pm 3.0$ & $110.7 \pm 3.3$ \\
\hline
$\mathbf{0.1}$ & $112.6 \pm 2.1$ & $113.3 \pm 2.0$ & $114.4 \pm 2.9$ & $114.6 \pm 2.1$ & $113.2 \pm 2.9$ & $112.5 \pm 3.4$ \\
\hline
$\mathbf{0.01}$ & $111.8 \pm 2.8$ & $112.0 \pm 3.6$ & $114.6 \pm 2.9$ & $114.2 \pm 3.3$ & $113.1 \pm 2.7$ & $110.2 \pm 3.4$ \\
\hline
$\mathbf{0.0025}$ & $109.7 \pm 2.6$ & $111.5 \pm 3.2$ & $113.8 \pm 2.6$ & $113.3 \pm 4.4$ & $112.6 \pm 3.7$ & $110.1 \pm 3.7$ \\
\hline
$\mathbf{0.001}$ & $108.2 \pm 3.0$ & $109.5 \pm 3.8$ & $111.9 \pm 3.5$ & $111.2 \pm 2.6$ & $109.8 \pm 3.1$ & $108.4 \pm 4.6$ \\
\hline
\end{tabular}
\end{table}

\begin{table}[h!]
\centering
\footnotesize
\label{walk}
\setlength{\tabcolsep}{0.1em}
\caption{Walker2d-medium-replay: Our selection rule chooses $\lambda=c^*=0.1$ with the policy performance $101.2 \pm 3.2$.}
\begin{tabular}{c|c|c|c|c|c|c}
\hline
$c^*(\mathrm{col}), \lambda \text{ (row) }$ & $\mathbf{2.5}$ & $\mathbf{1}$ & $\mathbf{0.1}$ & $\mathbf{0.01}$ & $\mathbf{0.0025}$ & $\mathbf{0.001}$ \\
\hline
$\mathbf{2.5}$ & $95.8 \pm 2.5$ & $97.4 \pm 2.8$ & $98.4 \pm 2.5$ & $99.1 \pm 3.2$ & $97.8 \pm 3.0$ & $97.9 \pm 3.4$ \\
\hline
$\mathbf{1}$ & $97.3 \pm 2.7$ & $98.0 \pm 3.1$ & $98.8 \pm 2.8$ & $99.4 \pm 3.4$ & $98.7 \pm 2.7$ & $98.1 \pm 3.2$ \\
\hline
$\mathbf{0.1}$ & $97.4 \pm 2.8$ & $98.3 \pm 2.9$ & $101.2 \pm 3.2$ & $101.3 \pm 3.4$ & $98.9 \pm 3.6$ & $97.5 \pm 4.2$ \\
\hline
$\mathbf{0.01}$ & $98.2 \pm 2.8$ & $99.5 \pm 2.9$ & $101.7 \pm 3.9$ & $102.6 \pm 3.4$ & $100.2 \pm 3.1$ & $98.4 \pm 3.3$ \\
\hline
$\mathbf{0.0025}$ & $98.0 \pm 3.6$ & $97.5 \pm 4.2$ & $100.1 \pm 3.6$ & $100.8 \pm 3.5$ & $99.2 \pm 4.2$ & $97.4 \pm 4.0$ \\
\hline
$\mathbf{0.001}$ & $97.2 \pm 3.8$ & $98.5 \pm 3.3$ & $98.2 \pm 3.8$ & $99.3 \pm 3.6$ & $97.8 \pm 5.2$ & $98.2 \pm 4.1$ \\
\hline
\end{tabular}
\end{table}

\begin{table}[h!]
\label{maze}
\footnotesize
\centering
\setlength{\tabcolsep}{0.1em}
\caption{Maze2d-medium: Our selection rule chooses $\lambda=c^*=2.25$ with the policy performance $138.1 \pm 7.6$.}
\begin{tabular}{c|c|c|c|c|c|c}
\hline
$c^*(\mathrm{col}), \lambda \text{ (row) }$ & $\mathbf{15}$ & $\mathbf{10}$ & $\mathbf{5}$ & $\mathbf{2.5}$ & $\mathbf{1}$ & $\mathbf{0.5}$ \\
\hline
$\mathbf{15}$ & $134.5 \pm 4.6$ & $133.9 \pm 5.8$ & $134.8 \pm 4.5$ & $136.7 \pm 6.2$ & $134.8 \pm 6.0$ & $134.9 \pm 5.4$ \\
\hline
$\mathbf{10}$ & $133.7 \pm 4.2$ & $136.7 \pm 5.1$ & $135.8 \pm 6.8$ & $138.4 \pm 7.4$ & $135.7 \pm 12.2$ & $137.5 \pm 8.2$ \\
\hline
$\mathbf{5}$ & $133.9 \pm 5.8$ & $137.3 \pm 6.9$ & $136.6 \pm 5.5$ & $138.3 \pm 9.2$ & $134.9 \pm 7.0$ & $135.1 \pm 5.2$ \\
\hline
$\mathbf{2.5}$ & $137.5 \pm 6.3$ & $135.8 \pm 5.9$ & $140.7 \pm 10.9$ & $138.9 \pm 9.2$ & $132.2 \pm 8.1$ & $133.7 \pm 6.5$ \\
\hline
$\mathbf{1}$ & $134.0 \pm 4.2$ & $137.2 \pm 10.7$ & $133.8 \pm 6.9$ & $137.3 \pm 9.5$ & $138.2 \pm 5.2$ & $137.6 \pm 8.0$ \\
\hline
$\mathbf{0.5}$ & $135.2 \pm 11.8$ & $133.7 \pm 8.3$ & $136.1 \pm 7.8$ & $134.5 \pm 6.7$ & $137.2 \pm 9.2$ & $135.5 \pm 7.1$ \\
\hline
\end{tabular}
\end{table}

\begin{table}  [ht]
\label{sen_va}
\footnotesize
\setlength{\tabcolsep}{0.15em}
\renewcommand{\arraystretch}{1.05}
		\centering
    \caption{Hyperparameter values for D4RL benchmark.}
\begin{tabular}{c|ccccccccc }
\hline Gym locomotion Tasks & Hypereparameters \\
\hline  walker2d-medium   & $0.25$  \\
        walker2d-medium-replay  & $0.1$  \\
        walker2d-medium-expert  & $0.35$ \\
        walker2d-random  & $0.25$ \\
        \hline  
        hopper-medium  & $0.4$ \\
        hopper-medium-replay  & $0.25$  \\
        hopper-medium-expert  & $0.5$ \\     
        hopper-random  & $0.4$  \\
        \hline  
        halfcheetah-medium  & $0.25$ \\
        halfcheetah-medium-replay  & $0.1$ \\
        halfcheetah-medium-expert  & $0.35$\\
        halfcheetah-random  & $0.25$\\
        \hline  
        Maze2d Tasks & Hyperparamters \\
\hline

       maze2d-umaze   & $2$ \\
        maze2d-medium   & $2.25$ \\
        maze2d-large  & $2.5$ \\  
        \hline 
\end{tabular} 
\end{table}

\begin{figure}[H]
    \centering
\includegraphics[width=0.68\textwidth]{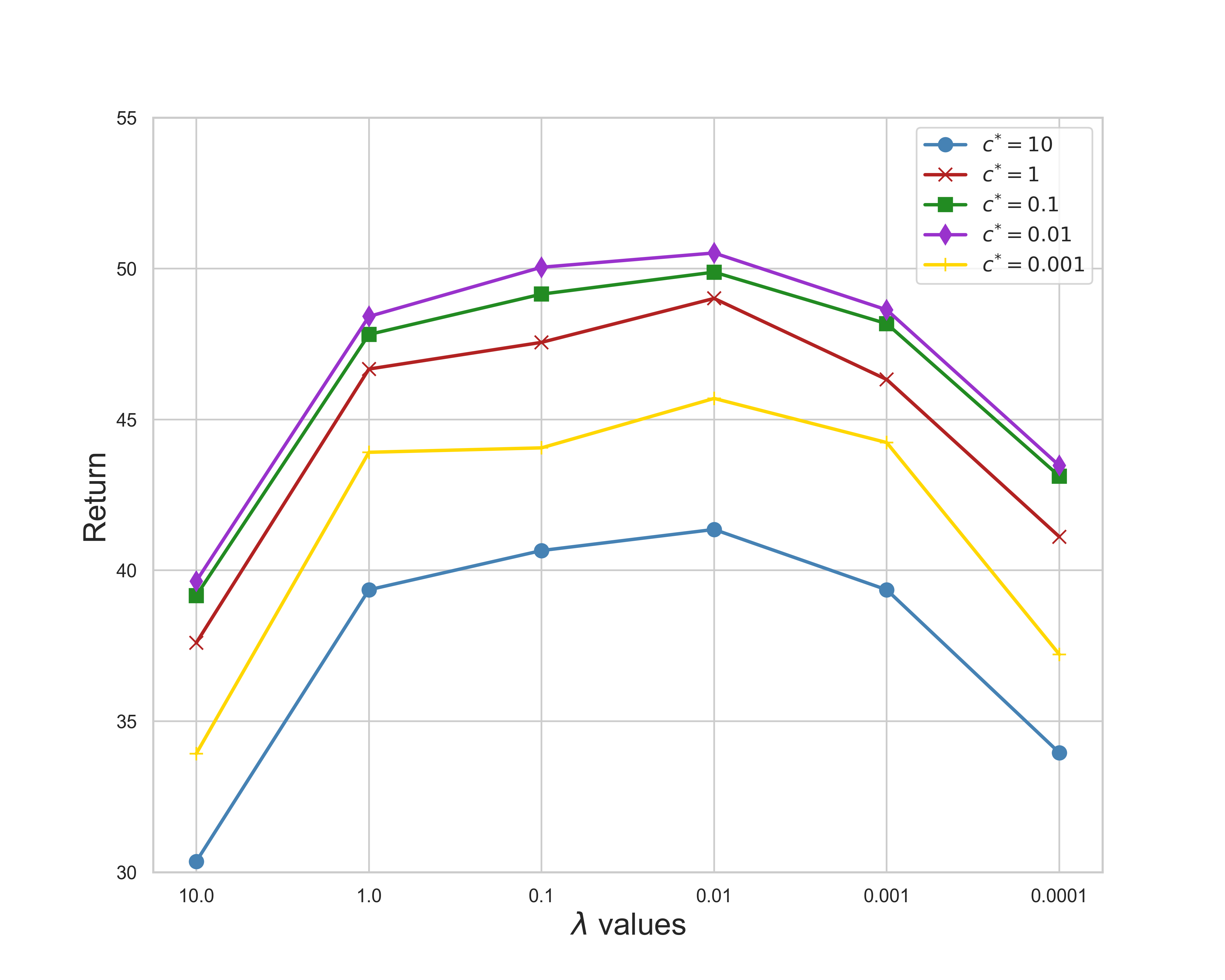}
    \caption{Sensitivity analysis on the effects of hyperparameters $\lambda$ and $c^{*}$ for model performance with halfcheetah-medium-replay dataset.}
    \label{fig:sen}
\end{figure}

\textbf{Empirical evaluation on theoretical results.} We also empirically validate the regret bound in Theorem 5.2. In general, we have no information on the optimal policy and whether it is covered by offline datasets, which makes it challenging to accurately compute the regret in order to verify our theoretic bound. Thus, we carefully design a synthetic environment. We describe the environment in the following: the reward 
$r(s, a)=(a-\beta s)^{\top} \Lambda(a-\beta s)$ with coefficient matrix $\beta$ and the negative definite matrix $\Lambda$. Therefore, the optimal policy has an analytical form $\pi^{\star}(s)=s^{\top} \beta$, which is important to calculate precise regret. The dataset is generated following $\pi_b$ such that $a = \beta s + \mathcal{N}\left(0, \sigma_0^2 I \right)$, indicating the behavior policy is more different from the optimal one and the data is more explored when $\sigma_0$ is large.

In Figure \ref{regret_conv} we study the convergence rate of regret, which validates the  $\mathcal{O}(n^{-1/4})$ rate in Theorem 5.1 and 5.2. The plot shows that the convergence rate is close to $\mathcal{O}(n^{-1/4})$ in all scenarios, which validates the theoretical regret bound of our practical algorithm in Theorem 5.1 and 5.2.

\begin{figure}[H]
\label{regret_conv}
    \centering
\includegraphics[width=0.68\textwidth]{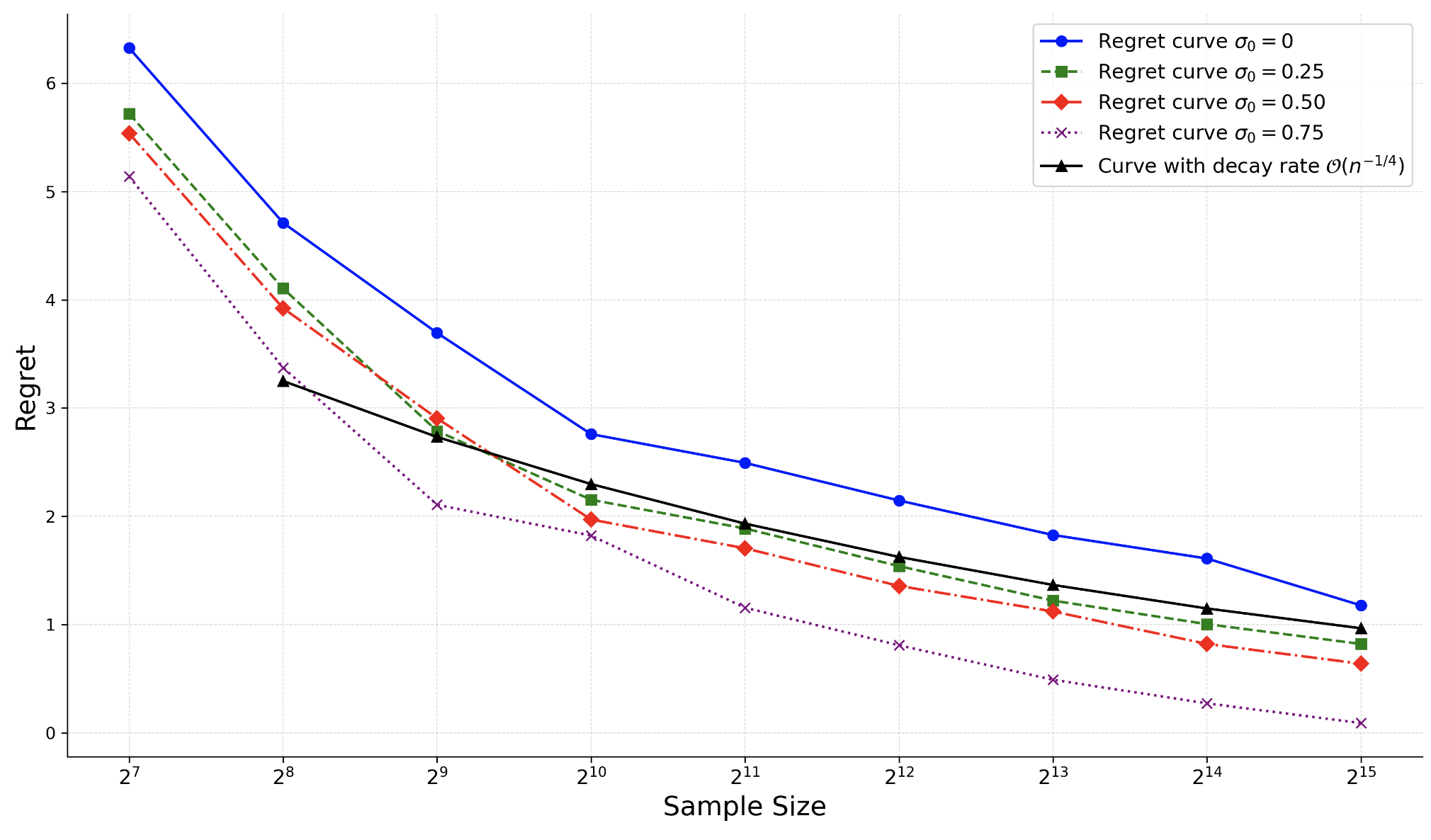}
    \caption{Convegence rate of the near-optimal regret (compete to the optimal policy) on the synthetic dataset with different degrees of exploration $\sigma_{0}$. A smaller $\sigma_{0}$ indicates the training data is less explored.}
\end{figure}

\section{Proof of Theorem 3.1}

\subsection{Proof of Lemma \ref{tele_lemma}}

\begin{lemma}[\citep{zhou2023distributional}]
\label{tele_lemma}
For any target policy $\pi \in \Pi$ and  $\tau \in \Omega$, 
\$
\frac{\mathbb{E}_{\mu}[\tau(s,a)r(s,a)]}{1-\gamma} - J^{\mathbb{D}}(\pi)  = \frac{\mathbb{E}_{\mu}\left[\tau(s,a)\left(q^{\pi}(s, a)-\gamma q^{\pi}\left(s^{\prime}, \pi\right)\right) - \lambda \mathbb{D}(\tau(s,a))\right]}{1-\gamma} - q^{\pi}(s^0,\pi),
\$
where $J^{\mathbb{D}}(\pi) := J(\pi) - \lambda \xi(\mathbb{D},\tau)$ for $\xi(\mathbb{D},\tau) := \mathbb{E}_{\mu}[\frac{ \mathbb{D}(\tau(s,a))}{1-\gamma}]$, and $q^{\pi}$ is the unique fixed point of Bellman equation $\mathcal{B}^{\pi}q=q$. 
\end{lemma}

\begin{proof}[Proof of Lemma \ref{tele_lemma}]
It follows the definition $J^{\mathbb{D}}(\pi) = J(\pi) + \lambda \xi(\mathbb{D},\tau)$. Then it is sufficient to show 
\$
\frac{\mathbb{E}_{\mu}[\tau(s,a)r(s,a)]}{1-\gamma} - J(\pi)  = \frac{\mathbb{E}_{\mu}\left[\tau(s,a)\left(q^{\pi}(s, a)-\gamma q^{\pi}\left(s^{\prime}, \pi\right)\right) \right]}{1-\gamma} - q^{\pi}(s^0,\pi).
\$
We rearrange the equation as 
\$
J(\pi) - q^{\pi}(s^0,\pi) = \frac{\mathbb{E}_{\mu}\left[\tau(s,a)\left(-q^{\pi}(s, a) + r(s,a)+ \gamma q^{\pi}\left(s^{\prime}, \pi\right) \right)\right]}{1-\gamma}.
\$
Following the definition of 
$J(\pi) = \mathbb{E}\left[\sum_{t=0}^{\infty} \gamma^{t} r^{t} | \pi\right] =q^{\pi}(s^0,\pi),  
$
Therefore, it leaves to show the 
$
\frac{\mathbb{E}_{\tau}\left[r+ \gamma q^{\pi}\left(s^{\prime}, \pi\right)-q^{\pi}(s, a)  \right]}{1-\gamma} = 0 .
$
As $q^{\pi}(s, a) = r(s,a)+\mathbb{E}_{s^{\prime}\sim \mathds{P}(\cdot|s,a)}[q^{\pi}(s^{\prime},\pi)]$ by Bellman evaluation equation, thus we concldue that 
\$
& \frac{\mathbb{E}_{\mu}\left[\tau(s,a)\left(r(s,a)+ \gamma q^{\pi}\left(s^{\prime}, \pi\right)-q^{\pi}(s, a)  \right)\right]}{1-\gamma} \\
= &   \frac{\mathbb{E}_{\mu}\left[\tau(s,a)\big(r+ \gamma q^{\pi}\left(s^{\prime}, \pi\right)-q^{\pi}(s, a)\big)  \right]}{1-\gamma} \\
= & \frac{\int_{s,a}\mu(s,a)\left[\tau(s,a)\mathbb{E}_{s^{\prime}\sim \mathds{P}(\cdot|s,a)}\big[\big(r(s,a)+ \gamma q^{\pi}\left(s^{\prime}, \pi\right)-q^{\pi}(s, a)\big)\big]\right]}{1-\gamma} \\
= & \frac{\int_{s,a}\mu(s,a)\left[\tau(s,a)\big(r(s,a)+ \gamma\mathbb{E}_{s^{\prime}\sim \mathds{P}(\cdot|s,a)}\big[ q^{\pi}\left(s^{\prime}, \pi\right)\big]-q^{\pi}(s, a)\big)\right]}{1-\gamma} 
=  0.
\$
This completes the proof.
\end{proof}

\subsection{Proof of Theorem 3.2}

\begin{proof}
To prove the theorem, we follow the proof of Theorem 3.4 in \citep{zhou2023distributional}. We need to establish appropriate confidence in upper and lower bounds at the same time. To simplify the notation, we denote $\mathbb{E}_{\tau}[\cdot] = \mathbb{E}_{\mu}[\tau(s,a)\cdot]$. At first, we prove for the confidence lower bound. It follows Lemma \ref{tele_lemma} and for any $\lambda>0$, we have 
\$
\frac{\mathbb{E}_{\tau}\left[r(s,a) + \gamma q^{\pi}\left(s^{\prime}, \pi\right) - q^{\pi}(s, a) - \lambda \mathbb{D}(\tau(s,a))/\tau(s,a)\right]}{1-\gamma} =  J(\pi) - q^{\pi}(s^0,\pi) - \lambda \xi(\mathbb{D},\tau).
\$
This immediately implies that
\$
 &J(\pi) - q^{\pi}(s^0,\pi) \geq \frac{\mathbb{E}_{\tau}\left[r(s,a) + \gamma q^{\pi}\left(s^{\prime}, \pi\right) - q^{\pi}(s, a) - \lambda \mathbb{D}(\tau(s,a))/\tau(s,a)\right]}{1-\gamma} \\
 \iff & J(\pi) \geq \frac{\mathbb{E}_{\tau}\left[r(s,a) - \gamma q^{\pi}\left(s^{\prime}, \pi\right) - q^{\pi}(s, a) - \lambda \mathbb{D}(\tau(s,a))/\tau(s,a)\right]}{1-\gamma} + q^{\pi}(s^0,\pi). 
\$
The above equation helps to obtain the lower bound for the bias evaluation but without concern about the uncertainty quantification due to sampling. To construct the sample estimator for the lower bound and incorporate the uncertainty deviation, we first observe that it suffices to approximation 
\$
\frac{\mathbb{E}_{\tau}\left[r(s,a) - \gamma q^{\pi}\left(s^{\prime}, \pi\right) - q^{\pi}(s, a)  -\lambda\mathbb{D}(\tau(s,a))/\tau(s,a)\right]}{1-\gamma},
\$
by its sample counterparts. That is, 
\$
\frac{\frac{1}{n}\sum^{n}_{i=1}\tau(s_i,a_i)\left(r_i + \gamma q^{\pi}(s^{\prime}_i,\pi)-q^{\pi}(s_i,a_i)\right)}{1-\gamma} -\lambda\xi_n(\mathbb{D},\tau)
\$
To bound below the uncertainty, this is equivalent to finding a good $\sigma_n$ such that for any $\tau \in \Omega$, 
\#
& \frac{\mathbb{E}_{\tau}\left[r(s,a) - \gamma q^{\pi}\left(s^{\prime}, \pi\right) - q^{\pi}(s, a) - \lambda \mathbb{D}(\tau(s,a))/\tau(s,a)\right]}{1-\gamma} \notag \\
\geq &\frac{\frac{1}{n}\sum^{n}_{i=1}\tau(s_i,a_i)\left(r_i + \gamma q^{\pi}(s^{\prime}_i,\pi)-q^{\pi}(s_i,a_i)\right)}{1-\gamma} -\lambda\xi_n(\mathbb{D},\tau) -\sigma^{\circ}_n
\label{uncertain_bound}
\#
with probability at least $1-\delta/2$. Note that the power of $\delta/2$ is due to that we need to further consider the upper confidence bound with also $\delta/2$ power so that the confidence interval holds w.p. $\geq 1-\delta$.    

According to Bellman equation, we know that 
$r(s,a) - \gamma \mathbb{E}_{s^{\prime} \sim \mathds{P}(\cdot|s,a)}[q^{\pi}\left(s^{\prime}, \pi\right)] - q^{\pi}(s, a)=0$ for any $s,a$. This implies that 
\#
& \frac{\mathbb{E}_{\tau}\left[r(s,a) - \gamma q^{\pi}\left(s^{\prime}, \pi\right) - q^{\pi}(s, a) - \lambda \mathbb{D}(\tau(s,a))/\tau(s,a)\right]}{1-\gamma}  \notag \\
= &\frac{\mathbb{E}_{\tau}\left[-\lambda \mathbb{D}(\tau(s,a))/\tau(s,a)\right]}{1-\gamma}  \leq  0
\label{plug_in_un}
\#
where the last inequality comes from the fact of Definition 3.1 on the detection function $\mathbb{D}(\cdot)$ which is always non-negative. 

Combine the inequalities \eqref{uncertain_bound} and  \eqref{plug_in_un}, it is sufficient to obtain $\sigma^{\circ}_{n}$ satisfying the following condition:
\$
 \sigma^{\circ}_{n} \geq \frac{\frac{1}{n}\sum^{n}_{i=1}\tau(s_i,a_i)\left(r_i + \gamma q^{\pi}(s^{\prime}_i,\pi)-q^{\pi}(s_i,a_i)\right)}{1-\gamma} -\lambda\xi_n(\mathbb{D},\tau), 
\$
for any $\tau \in \Omega$. We can rewrite it to use a uniform argument, that is 
\$
\sup_{\tau \in \Omega}\left\{\frac{1}{n}\sum^{n}_{i=1}\frac{\tau(s_i,a_i)\left(r_i + \gamma q^{\pi}(s^{\prime}_i,\pi)-q^{\pi}(s_i,a_i)\right)}{1-\gamma} -\lambda\xi_n(\mathbb{D},\tau)\right\} \leq \sigma^{\circ}_{n}
\$
Now, recall that we have a condition that 
\$
\sup_{\tau \in \Omega}\Big| \frac{1}{n(1-\gamma)}\sum^{n}_{i=1}\tau(s_i,a_i)\left(r_i +\gamma q^{\pi}(s^{\prime}_i,\pi)-q^{\pi}(s_i,a_i)\right)-\lambda\xi_n(\mathbb{D},\tau)\Big| \leq \sigma_n,
\$
which directly implies that 
\$
\sup_{\tau \in \Omega}\left\{\frac{1}{n}\sum^{n}_{i=1}\frac{\tau(s_i,a_i)\left(r_i + \gamma q^{\pi}(s^{\prime}_i,\pi)-q^{\pi}(s_i,a_i)\right)}{1-\gamma} -\lambda\xi_n(\mathbb{D},\tau)\right\} \leq \sigma_{n}. 
\$

Therefore, we set $\sigma^{\circ}_n = \sigma_{n}$, and  combine with \eqref{uncertain_bound}, it obtains that for any $\tau \in \Omega$,
\$
 J(\pi) \geq &\frac{1}{n}\sum^{n}_{i=1} \frac{\tau(s_i,a_i)\left(r_i + \gamma q^{\pi}(s^{\prime}_i,\pi)-q^{\pi}(s_i,a_i)\right)}{1-\gamma} + q^{\pi}(s^0,\pi) -\lambda\xi_n(\mathbb{D},\tau) -\sigma_{n} \\
J(\pi) \geq &\frac{1}{n}\sum^{n}_{i=1}\frac{r_i\tau(s_i,a_i)}{1-\gamma} + \inf_{q \in \mathcal{Q}} \frac{\frac{1}{n}\sum^{n}_{i=1} \tau(s_i,a_i)(\gamma q(s^{\prime}_i,\pi) - q(s_i,a_i)) + (1-\gamma)q(s^0,\pi)}{1-\gamma} \\
& -\lambda\xi_n(\mathbb{D},\tau)-\sigma^{L}_n \\
 J(\pi) \geq  &\frac{1}{n}\sum^{n}_{i=1}\frac{r_i\tau(s_i,a_i)}{1-\gamma} - \sup_{q \in \mathcal{Q}} \underbrace{ \frac{\frac{1}{n}\sum^{n}_{i=1} \tau(s_i,a_i)(q(s_i,a_i)-\gamma q(s^{\prime}_i,\pi)) - (1-\gamma)q(s^0,\pi)}{1-\gamma}}_{\widehat{M}_{n}(-q,\tau)} \\
 & -\lambda\xi_n(\mathbb{D},\tau) -\sigma^{L}_n.
 \$
This completes the proof for the confidence lower bound.

Now, it remains to prove the result for the confidence upper bound. According to the value interval in (3) of the maintext, we observe 
\$
 J(\pi) \leq \frac{\mathbb{E}_{\tau}\left[r(s,a) - \gamma q^{\pi}\left(s^{\prime}, \pi\right) - q^{\pi}(s, a) + \lambda \mathbb{D}(\tau(s,a))/\tau(s,a)\right]}{1-\gamma} + q^{\pi}(s^0,\pi). 
\$
 To construct the sample estimator for the lower bound and incorporate the uncertainty deviation, we first observe that it suffices to approximation 
\$
\frac{\mathbb{E}_{\tau}\left[r(s,a) - \gamma q^{\pi}\left(s^{\prime}, \pi\right) - q^{\pi}(s, a)  +\lambda\mathbb{D}(\tau(s,a))/\tau(s,a)\right]}{1-\gamma},
\$
by its sample counterparts. That is, 
\$
\frac{\frac{1}{n}\sum^{n}_{i=1}\tau(s_i,a_i)\left(r_i + \gamma q^{\pi}(s^{\prime}_i,\pi)-q^{\pi}(s_i,a_i)\right)}{1-\gamma} +\lambda\xi_n(\mathbb{D},\tau)
\$
To bound below the uncertainty, this is equivalent to finding a good $\sigma_n$ such that for any $\tau \in \Omega$, 
\#
& \frac{\mathbb{E}_{\tau}\left[r(s,a) - \gamma q^{\pi}\left(s^{\prime}, \pi\right) - q^{\pi}(s, a) + \lambda \mathbb{D}(\tau(s,a))/\tau(s,a)\right]}{1-\gamma} \notag \\
\leq &\frac{\frac{1}{n}\sum^{n}_{i=1}\tau(s_i,a_i)\left(r_i + \gamma q^{\pi}(s^{\prime}_i,\pi)-q^{\pi}(s_i,a_i)\right)}{1-\gamma} +\lambda\xi_n(\mathbb{D},\tau) + \sigma^{*}_n
\label{uncertain_bound_upper}
\#
with probability at least $1-\delta/2$. According to Bellman equation, we know that 
$r(s,a) - \gamma \mathbb{E}_{s^{\prime} \sim \mathds{P}(\cdot|s,a)}[q^{\pi}\left(s^{\prime}, \pi\right)] - q^{\pi}(s, a)=0$ for any $s,a$. This implies that 
\#
& \frac{\mathbb{E}_{\tau}\left[r(s,a) - \gamma q^{\pi}\left(s^{\prime}, \pi\right) - q^{\pi}(s, a) + \lambda \mathbb{D}(\tau(s,a))/\tau(s,a)\right]}{1-\gamma}  \notag \\
= &\frac{\mathbb{E}_{\tau}\left[\lambda \mathbb{D}(\tau(s,a))/\tau(s,a)\right]}{1-\gamma} \geq  0.
\label{plug_in_un_upper}
\#

Combine the inequalities \eqref{uncertain_bound_upper} and  \eqref{plug_in_un_upper}, it is sufficient to obtain $\sigma^{*}_{n}$ satisfying the following condition:
\$
\sigma^{*}_{n} \geq \frac{\frac{1}{n}\sum^{n}_{i=1}\tau(s_i,a_i)\left(r_i + \gamma q^{\pi}(s^{\prime}_i,\pi)-q^{\pi}(s_i,a_i)\right)}{1-\gamma} +\lambda\xi_n(\mathbb{D},\tau), 
\$
for any $\tau \in \Omega$. This could be satisfied by the uncertainty deviation condition in Theorem 3.1 that 
\$
\sup_{\tau \in \Omega}\left\{\frac{1}{n}\sum^{n}_{i=1}\frac{\tau(s_i,a_i)\left(-r_i - \gamma q^{\pi}(s^{\prime}_i,\pi)+q^{\pi}(s_i,a_i)\right)}{1-\gamma} + \lambda\xi_n(\mathbb{D},\tau)\right\} \geq -\sigma_{n}. 
\$
by taking $\sigma^{*}_{n} = \sigma_{n}$. It then obtains that for any $\tau \in \Omega$,
\$
\frac{1}{n}\sum^{n}_{i=1}\frac{\tau(s_i,a_i)\left(-r_i - \gamma q^{\pi}(s^{\prime}_i,\pi)+q^{\pi}(s_i,a_i)\right)}{1-\gamma} + \lambda\xi_n(\mathbb{D},\tau) \geq &  -\sigma_n \\ 
\frac{1}{n}\sum^{n}_{i=1}\frac{\tau(s_i,a_i)\left(-r_i - \gamma q^{\pi}(s^{\prime}_i,\pi)+q^{\pi}(s_i,a_i)\right)}{1-\gamma} + \lambda\xi_n(\mathbb{D},\tau) + (J(\pi)-q^{\pi}(s^{0},\pi))\geq&  -\sigma_n
\$
By some algebra, this implies 
\$
J(\pi) \leq &\frac{1}{n}\sum^{n}_{i=1}\frac{r_i\tau(s_i,a_i)}{1-\gamma} + \sup_{q \in \mathcal{Q}} \underbrace{\frac{\frac{1}{n}\sum^{n}_{i=1} \tau(s_i,a_i)(\gamma q(s^{\prime}_i,\pi) - q(s_i,a_i)) + (1-\gamma)q(s^0,\pi)}{1-\gamma}}_{\widehat{M}_{n}(q,\tau)} \\
& + \lambda\xi_n(\mathbb{D},\tau)+\sigma_n
\$
This completes the proof for the confidence upper bound. 
\end{proof}

\section{Proof of Theorem 3.2}

\begin{proof}

It follows the definition of $\widehat{J}^{-}_{n}(\pi;\tau)$, we have 
\$
\widehat{J}^{-}_{n}(\pi;\tau) =  \frac{1}{n}\sum^{n}_{i=1}\frac{r_i\tau(s_i,a_i)}{1-\gamma} - \sup_{q \in \mathcal{Q}}\widehat{M}_{n}(-q,\tau) -  \lambda\xi_n(\mathbb{D},\tau)  - \sigma_n
\$
and we obtain the maximizer $\max_{\pi \in \Pi} \left\{\sup_{\tau\in\Omega}\widehat{J}^{-}_{n}(\pi;\tau)\right\}$. Therefore, to provide the equivalence, it suffices to show, for any $\pi \in \Pi$, the optimization 
\#
\sup_{\tau\in\Omega}\left\{\frac{1}{n}\sum^{n}_{i=1}\frac{r_i\tau(s_i,a_i)}{1-\gamma} - \sup_{q \in \mathcal{Q}}\widehat{M}_{n}(-q,\tau) -  \lambda\xi_n(\mathbb{D},\tau)  - \sigma_n\right\}
\label{prof_dual}
\#
is equivalent to the optimization 
\$
& \min_{q\in \mathcal{Q}_{\varepsilon_{n}}}q(s^{0},\pi),
\\
& \mathcal{Q}_{\varepsilon_{n}} = \big\{q \in \mathcal{Q}:  \sup_{\tau \in \widetilde{\Omega}_{\widetilde{\sigma}_{n}}}\big|n^{-1}\sum^{n}_{i=1}\tau(s_i,a_i)(r_i +\gamma q^{\pi}(s^{\prime}_i,\pi)-q^{\pi}(s_i,a_i))\big| \leq \varepsilon_{n} \big\}, \\
& \widetilde{\Omega}_{\widetilde{\sigma}_{n}} =    \left\{\tau_{\circ}/\sup_{\tau_{\circ} \in \Omega}\|\tau_{\circ}\|_{\Omega} \; \text{for} \; \tau_{\circ} \in \Omega: \xi_n(\mathbb{D},\tau_{\circ})) \leq \widetilde{\sigma}_{n}  \right\}. 
\$
which can be re-expressed as a prime form: 
\#
&\min_{q\in \mathcal{Q}}q(s^0,\pi), \; \textbf{s.t.} \; q \ \text{satisfies} \notag \\
& \qquad \qquad \sup_{\tau \in \widetilde{\Omega}_{\widetilde{\sigma}_{n}} }\left\{\left| \frac{1}{n}\sum^{n}_{i=1} \frac{\tau(s_i,a_i)\left(r_i +\gamma q(s^{\prime}_i,\pi)-q(s_i,a_i)\right)}{1-\gamma}\right| \right\} \leq \frac{\varepsilon_{n}}{1-\gamma}  := \widetilde{\varepsilon}_{n},
\label{prime_perturbed}
\#
where 
\$
\widetilde{\Omega}_{\widetilde{\sigma}_{n}} := \left\{\frac{\tau_{\circ}}{\sup_{\tau_{\circ} \in \Omega}\|\tau_{\circ}\|_{\Omega_{\circ}}}, \tau_{\circ} \in \Omega:  \left|\frac{1}{n}\sum^{n}_{i=1}\frac{ \mathbb{D}(\tau_{\circ}(s_i,a_i))}{1-\gamma}\right| \leq \widetilde{\sigma}_{n}  \right\}.
\$
Note that, it follows the definition of $\mathbb{D}(\cdot)$, the above form of $\widetilde{\Omega}_{\widetilde{\sigma}_{n}}$ can be further relaxed to 
\$
\widetilde{\Omega}_{\widetilde{\sigma}_{n}}:= \left\{\frac{\tau_{\circ}}{\sup_{\tau_{\circ} \in \Omega_{\circ}}\|\tau_{\circ}\|_{\Omega}}, \tau_{\circ} \in \Omega: \frac{1}{n}\sum^{n}_{i=1}\frac{ \mathbb{D}(\tau_{\circ}(s_i,a_i))}{1-\gamma} \leq \widetilde{\sigma}_{n}  \right\}.
\$

Therefore, it is sufficient to show the optimization \ref{prime_perturbed} is equivalent to the optimization \ref{prof_dual}. 

First, by the rule of $\sup \& \inf$: $\sup\{A_{n}\} = -\inf\{-A_n\}$ for any sequence $A_n$, we observe that
\$
\sup_{\tau \in \widetilde{\Omega}_{\widetilde{\sigma}_{n}}}\left| \frac{1}{n}\sum^{n}_{i=1} \frac{\tau(s_i,a_i)\left(r_i +\gamma q(s^{\prime}_i,\pi)-q(s_i,a_i)\right)}{1-\gamma}\right| \leq \widetilde{\varepsilon}_{n}\\
\implies
\inf_{\tau \in \widetilde{\Omega}_{\widetilde{\sigma}_{n}}}\left\{ \frac{1}{n}\sum^{n}_{i=1} \frac{\tau(s_i,a_i)\left(q(s_i,a_i)-r_i -\gamma q(s^{\prime}_i,\pi)\right)}{1-\gamma} \right\} \geq -\widetilde{\varepsilon}_{n}
\$
For a fixed $\widetilde{\Omega}_{\widetilde{\sigma}_{n}}$ and $\Upsilon \geq 0$, the optimization \eqref{prime_perturbed} is equivalent to, 
\$
&\min_{q\in \mathcal{Q}}q(s^0,\pi), \; \textbf{s.t.} \; q \in  \left\{\inf_{\tau \in \widetilde{\Omega}_{\widetilde{\sigma}_{n}}}\left\{ -\left( \frac{1}{n}\sum^{n}_{i=1} \frac{\tau(s_i,a_i)\left(r_i +\gamma q(s^{\prime}_i,\pi)-q(s_i,a_i)\right)}{1-\gamma}\right)\right\} \geq -\widetilde{\varepsilon}_n  \right\}\\
\iff & \min_{q\in \mathcal{Q}}q(s^0,\pi), \; \textbf{s.t.} \; q \in  \left\{ -\inf_{\tau \in \widetilde{\Omega}_{\widetilde{\sigma}_{n}}}\left\{ - \left(\frac{1}{n}\sum^{n}_{i=1} \frac{\tau(s_i,a_i)\left(r_i +\gamma q(s^{\prime}_i,\pi)-q(s_i,a_i)\right)}{1-\gamma}\right)\right\} \leq \widetilde{\varepsilon}_n  \right\}\\
\iff & \min_{q\in \mathcal{Q}}q(s^0,\pi), \; \textbf{s.t.} \; q \ \in \  \left\{\sup_{\tau \in \widetilde{\Omega}_{\widetilde{\sigma}_{n}}}\left\{  \frac{1}{n}\sum^{n}_{i=1} \frac{\tau(s_i,a_i)\left(r_i +\gamma q(s^{\prime}_i,\pi)-q(s_i,a_i)\right)}{1-\gamma}\right\} \leq \widetilde{\varepsilon}_n \right\}\\ 
\iff & \min_{q\in \mathcal{Q}}q(s^0,\pi) + \sup_{\Upsilon \geq 0}\Upsilon \left( \sup_{\tau \in \widetilde{\Omega}_{\widetilde{\sigma}_{n}}}\left\{  \frac{1}{n}\sum^{n}_{i=1} \frac{\tau(s_i,a_i)\left(r_i +\gamma q(s^{\prime}_i,\pi)-q(s_i,a_i)\right)}{1-\gamma}\right\} -\widetilde{\varepsilon}_n\right) 
\$
Furthermore, we can express the above optimization as a prime form w.r.t. to $\tau$. Also, we have an observation that the space 
\$
\widetilde{\Omega}:=\left\{\frac{\tau_{\circ}}{\sup_{\tau_{\circ} \in \Omega}\|\tau_{\circ}\|_{\Omega}}: \tau_{\circ} \in \Omega\right\}
\$
where, for any $\tau \in \widetilde{\Omega}$, $\|\tau\|_{\widetilde{\Omega}} \leq 1$. Then we can further write 
\$
& \min_{q\in \mathcal{Q}}q(s^0,\pi) + \sup_{\Upsilon \geq 0}\Upsilon \left( \sup_{\tau \in \widetilde{\Omega}}\left\{  \frac{1}{n}\sum^{n}_{i=1} \frac{\tau(s_i,a_i)\left(r_i +\gamma q(s^{\prime}_i,\pi)-q(s_i,a_i)\right)}{1-\gamma}\right\} -\widetilde{\varepsilon}_n\right) \\
& \qquad \qquad \textbf{s.t.} \; \left\{ \frac{1}{n}\sum^{n}_{i=1}\frac{ \mathbb{D}(\tau(s_i,a_i)\sup_{\tau_{\circ} \in \Omega}\|\tau_{\circ}\|_{\Omega})}{1-\gamma} \leq \widetilde{\sigma}_{n}  \right\}
\$
It follows that the exchange of variables, $\tau^{\circ}(s,a) = \tau(s,a)\sup_{\tau_{\circ} \in \Omega}\|\tau_{\circ}\|_{\Omega} \in \Omega$ and $\tau(s,a) = \tau^{\circ}(s,a)/\sup_{\tau_{\circ} \in \Omega}\|\tau_{\circ}\|_{\Omega} \in \widetilde{\Omega}$, so that $\tau$ and $\tau^{\circ}$ is bijective. 
\$
& \min_{q\in \mathcal{Q}}q(s^0,\pi) + \sup_{\Upsilon \geq 0}\Upsilon \left( \sup_{\tau \in \widetilde{\Omega}}\left\{  \frac{1}{n}\sum^{n}_{i=1} \frac{\tau(s_i,a_i)\left(r_i +\gamma q(s^{\prime}_i,\pi)-q(s_i,a_i)\right)}{1-\gamma}\right\} -\widetilde{\varepsilon}_n\right) \\
& \qquad \qquad \textbf{s.t.} \; \left\{ \frac{1}{n}\sum^{n}_{i=1}\frac{ \mathbb{D}(\tau_{\circ}(s_i,a_i))}{1-\gamma} \leq \widetilde{\sigma}_{n}\right\}.  
\$
This is further equivalent to the form 
\$
&  \min_{q\in \mathcal{Q}}\sup_{\Upsilon \leq 0} \sup_{\tau \in \widetilde{\Omega}} \left \{q(s^0,\pi) + \Upsilon \left(\left\{  \frac{1}{n}\sum^{n}_{i=1} \frac{\tau(s_i,a_i)\left(r_i +\gamma q(s^{\prime}_i,\pi)-q(s_i,a_i)\right)}{1-\gamma}\right\} -\widetilde{\varepsilon}_n\right)\right\} \\
& \qquad \qquad \textbf{s.t.} \; \tau \in \left\{ \left|\frac{1}{n}\sum^{n}_{i=1}\frac{ \mathbb{D}(\tau_{\circ}(s_i,a_i))}{1-\gamma}\right| \leq \widetilde{\sigma}_{n}  \right\}
\$
Next, we transform the above prime form to its duality, for any dual variable $\Psi \geq 0$,  
\$
& \min_{q\in \mathcal{Q}}\sup_{\Upsilon \geq 0} \sup_{\tau \in \Omega} \sup_{\Psi \geq 0} \Bigg\{ q(s^0,\pi) + \Upsilon \left(\left\{  \frac{1}{n}\sum^{n}_{i=1} \frac{\tau(s_i,a_i)\left(r_i +\gamma q(s^{\prime}_i,\pi)-q(s_i,a_i)\right)}{1-\gamma}\right\} -\widetilde{\varepsilon}_n\right) \\ 
& \quad - \Psi\left(\frac{1}{n}\sum^{n}_{i=1}\frac{ \mathbb{D}(\tau_{\circ}(s_i,a_i))}{1-\gamma} - \widetilde{\sigma}_{n} \right) \Bigg\} \\
\iff & \min_{q\in \mathcal{Q}} q(s^0,\pi)  + \sup_{\Upsilon \geq 0} \sup_{\tau \in \Omega} \sup_{\Psi \geq 0} \Bigg\{ \Upsilon \left(\left\{  \frac{1}{n}\sum^{n}_{i=1} \frac{\tau(s_i,a_i)\left(r_i +\gamma q(s^{\prime}_i,\pi)-q(s_i,a_i)\right)}{1-\gamma}\right\} -\widetilde{\varepsilon}_n\right) \\ 
& \quad - \Psi \left(\frac{1}{n}\sum^{n}_{i=1}\frac{ \mathbb{D}(\tau_{\circ}(s_i,a_i))}{1-\gamma} - \widetilde{\sigma}_{n} \right) \Bigg\}\\
\iff &  \min_{q\in \mathcal{Q}} q(s^0,\pi)  + \sup_{\Upsilon \geq 0} \sup_{\tau \in \Omega} \sup_{\Psi \geq 0} \Bigg\{ \Upsilon \left(\left\{  \frac{1}{n}\sum^{n}_{i=1} \frac{\tau(s_i,a_i)\left(r_i +\gamma q(s^{\prime}_i,\pi)-q(s_i,a_i)\right)}{1-\gamma}\right\} \right)\\ 
& \quad - \Psi\left(\frac{1}{n}\sum^{n}_{i=1}\frac{ \mathbb{D}(\tau_{\circ}(s_i,a_i))}{1-\gamma} \right)  \Bigg\} - \Upsilon\widetilde{\varepsilon}_n + \Psi\widetilde{\sigma}_{n} \\
\iff &  \min_{q\in \mathcal{Q}} q(s^0,\pi)  + \sup_{\Upsilon \geq 0} \sup_{\tau \in \Omega} \sup_{\Psi \geq 0} \Bigg\{ \Upsilon \left(\left\{  \frac{1}{n}\sum^{n}_{i=1} \frac{\tau(s_i,a_i)\left(q(s_i,a_i)-r_i -\gamma q(s^{\prime}_i,\pi)\right)}{1-\gamma}\right\} \right)\\ 
& \quad - \Psi \left(\frac{1}{n}\sum^{n}_{i=1}\frac{ \mathbb{D}(\tau_{\circ}(s_i,a_i))}{1-\gamma} \right)  \Bigg\} - \Upsilon\widetilde{\varepsilon}_n + \Psi\widetilde{\sigma}_{n}. 
\$
Let $\tau_{\circ}(s,a) = \Upsilon\tau(s,a)$ over the space $\Omega_{\circ}$, such that $\Upsilon$ is replaced by $\sup_{\tau_{\circ} \in \Omega}\|\tau_{\circ}\|_{\Omega}$. Moreover, it is feasible to select  
 $\lambda$ equals to the maximizer of $\Psi$, i.e., $\Psi^{*} = \lambda$, this directly implies that 
\#
& \min_{q\in \mathcal{Q}}\sup_{\tau \in \Omega}\bigg\{q(s^0,\pi) + \frac{1}{n}\sum^{n}_{i=1} \frac{\tau(s_i,a_i)\left(r_i +\gamma q(s^{\prime}_i,\pi)-q(s_i,a_i)\right)}{1-\gamma} \notag \\
& \qquad \qquad \qquad -\frac{\lambda}{n}\sum^{n}_{i=1}\frac{ \mathbb{D}(\tau(s_i,a_i))}{1-\gamma} - \|\tau\|_{\Omega}
\widetilde{\varepsilon}_n + \lambda\widetilde{\sigma}_{n} \bigg\} \notag \\
\iff & \min_{q\in \mathcal{Q}}\sup_{\tau \in \Omega}\bigg\{q(s^0,\pi) + \frac{1}{n}\sum^{n}_{i=1} \frac{\tau(s_i,a_i)\left(r_i +\gamma q(s^{\prime}_i,\pi)-q(s_i,a_i)\right)}{1-\gamma} \notag \\
& \qquad \qquad \qquad -\frac{\lambda}{n}\sum^{n}_{i=1}\frac{ \mathbb{D}(\tau(s_i,a_i))}{1-\gamma} - (\|\tau\|_{\Omega}
\widetilde{\varepsilon}_n -\lambda\widetilde{\sigma}_{n}) \bigg\}.
\label{circ_tau}
\#
Denote the inner maximizer of \eqref{circ_tau}
as $\tau^{*}$, then we set up 
\$\sigma_{n} = &\|\tau^{*}\|_{\Omega}\widetilde{\varepsilon}_n -\lambda\widetilde{\sigma}_{n}\\
= &\|\tau^{*}\|_{\Omega}(1-\gamma){\varepsilon}_n -\lambda\widetilde{\sigma}_{n}.
\$
Then the above expression is equivalent to 
\$
\min_{q\in \mathcal{Q}}\sup_{\tau \in \Omega}\left\{q(s^0,\pi) + \frac{1}{n}\sum^{n}_{i=1} \frac{\tau(s_i,a_i)\left(r_i +\gamma q(s^{\prime}_i,\pi)-q(s_i,a_i)\right)}{1-\gamma}  -\frac{\lambda}{n}\sum^{n}_{i=1}\frac{ \mathbb{D}(\tau(s_i,a_i))}{1-\gamma}\right\}- \sigma_n
\$
We check the Slater’s condition \citep{nesterov2018lectures}, as 
\#
q(s^0,\pi) + \frac{1}{n}\sum^{n}_{i=1} \frac{\tau(s_i,a_i)\left(r_i +\gamma q(s^{\prime}_i,\pi)-q(s_i,a_i)\right)}{1-\gamma}  -\frac{\lambda}{n}\sum^{n}_{i=1}\frac{ \mathbb{D}(\tau(s_i,a_i))}{1-\gamma}
\#
is linear on $q$, and also 
\$
\inf_{\tau \in \Omega}\left\{ -\left( \frac{1}{n}\sum^{n}_{i=1} \frac{\tau(s_i,a_i)\left(r_i +\gamma q(s^{\prime}_i,\pi)-q(s_i,a_i)\right)}{1-\gamma}-\frac{\lambda}{n}\sum^{n}_{i=1}\frac{ \mathbb{D}(\tau(s_i,a_i))}{1-\gamma}\right)\right\} 
\$
is convex on $q$, as it is the supremum of a linear function of $q$, then Slater’s condition is satisfied and strong duality holds,
\#
& \min_{q\in \mathcal{Q}}\sup_{\tau \in \Omega}\left\{q(s^0,\pi) + \frac{1}{n}\sum^{n}_{i=1} \frac{\tau(s_i,a_i)\left(r_i +\gamma q(s^{\prime}_i,\pi)-q(s_i,a_i)\right)}{1-\gamma}  -\frac{\lambda}{n}\sum^{n}_{i=1}\frac{ \mathbb{D}(\tau(s_i,a_i))}{1-\gamma}\right\}- \sigma_n \notag \\
\iff  &\sup_{\tau \in \Omega}\min_{q\in \mathcal{Q}}\left\{q(s^0,\pi) + \frac{1}{n}\sum^{n}_{i=1} \frac{\tau(s_i,a_i)\left(r_i +\gamma q(s^{\prime}_i,\pi)-q(s_i,a_i)\right)}{1-\gamma}  -\frac{\lambda}{n}\sum^{n}_{i=1}\frac{ \mathbb{D}(\tau(s_i,a_i))}{1-\gamma}\right\}- \sigma_n
\label{slack}
\#
where the order of $\min_{q\in \mathcal{Q}}\sup_{\tau \in \Omega}$ is exchanged to 
$\sup_{\tau \in \Omega}\min_{q\in \mathcal{Q}}$. According to the max-min form in \eqref{slack}, we have  
\$
&\sup_{\tau \in \Omega}\min_{q\in \mathcal{Q}}\left\{q(s^0,\pi) + \frac{1}{n}\sum^{n}_{i=1} \frac{\tau(s_i,a_i)\left(r_i +\gamma q(s^{\prime}_i,\pi)-q(s_i,a_i)\right)}{1-\gamma}-\frac{\lambda}{n}\sum^{n}_{i=1}\frac{ \mathbb{D}(\tau(s_i,a_i))}{1-\gamma}\right\} - \sigma_n  \\ 
\iff & \sup_{\tau \in \Omega}\min_{q\in \mathcal{Q}}\bigg\{ \frac{1}{n}\sum^{n}_{i=1}\frac{r_i\tau(s_i,a_i)}{1-\gamma} + 
q(s^0,\pi)  \\
& \qquad \qquad \qquad + \frac{1}{n}\sum^{n}_{i=1} \frac{\tau(s_i,a_i)\left(\gamma q(s^{\prime}_i,\pi)-q(s_i,a_i)\right)}{1-\gamma}  -\frac{\lambda}{n}\sum^{n}_{i=1}\frac{ \mathbb{D}(\tau(s_i,a_i))}{1-\gamma}\bigg\} - \sigma_n \\
\iff & \sup_{\tau \in \Omega}\bigg\{ \frac{1}{n}\sum^{n}_{i=1}\frac{r_i\tau(s_i,a_i)}{1-\gamma} + \min_{q\in \mathcal{Q}}\left\{
q(s^0,\pi) + \frac{1}{n}\sum^{n}_{i=1} \frac{\tau(s_i,a_i)\left(\gamma q(s^{\prime}_i,\pi)-q(s_i,a_i)\right)}{1-\gamma}\right\} \\
&\qquad \qquad \qquad \qquad \qquad \qquad  -\frac{\lambda}{n}\sum^{n}_{i=1}\frac{ \mathbb{D}(\tau(s_i,a_i))}{1-\gamma}\bigg\} - \sigma_n  \\
\iff & \sup_{\tau \in \Omega}\bigg\{ \frac{1}{n}\sum^{n}_{i=1}\frac{r_i\tau(s_i,a_i)}{1-\gamma} - \max_{q\in \mathcal{Q}}\left\{
\frac{1}{n}\sum^{n}_{i=1} \frac{\tau(s_i,a_i)\left(q(s_i,a_i)-\gamma q(s^{\prime}_i,\pi)\right)}{1-\gamma} -q(s^0,\pi) \right\} \\
& \qquad \qquad \qquad \qquad \qquad \qquad -\frac{\lambda}{n}\sum^{n}_{i=1}\frac{ \mathbb{D}(\tau(s_i,a_i))}{1-\gamma}\bigg\}- \sigma_n.
\$
It follows the definition of $\widehat{M}_{n}(-q,\tau)$ and $\xi_n(\mathbb{D},\tau)$, we conclude the above form is equivalent to 
\$
\frac{1}{n}\sum^{n}_{i=1}\frac{r_i\tau(s_i,a_i)}{1-\gamma} - \sup_{q \in \mathcal{Q}}\widehat{M}_{n}(-q,\tau) - \lambda\xi_n(\mathbb{D},\tau) - \sigma_n = \widehat{J}^{-}_{n}(\pi;\tau).
\$
This completes the proof.
\end{proof}

\section{Proof of Theorem 4.1}

\subsection{Proof of Lemma \ref{eval_error}}

\begin{lemma}[Evaluation error lemma]\label{eval_error}
For any target policy $\pi$ and  $q \in \mathcal{Q}$, 
\$
J(\pi)-q\left(s^0, \pi\right)=\frac{\mathbb{E}_{d^\pi}\left[r+\gamma q\left(s^{\prime}, \pi\right)-q(s, a)\right]}{1-\gamma}.
\$
\end{lemma}
\begin{proof}[Proof of Lemma \ref{eval_error}]
We follow the proof of Lemma 1 in \cite{xie2020q}. First, we observe that $
J(\pi)=\frac{\mathbb{E}_{d^\pi}[r]}{1-\gamma},
$ 
then it suffices to show 
$
\mathbb{E}_{d^\pi}\left[q(s, a)-\gamma q\left(s^{\prime}, \pi\right)\right]=(1-\gamma)q\left(s^0, \pi\right).
$
It follows the definition of $d_{\pi}$ and $q(s^{\prime},\pi)$, we have
\$
& \frac{\mathbb{E}_{d_\pi}\left[q(s, a)-\gamma q\left(s^{\prime}, \pi\right)\right]}{1-\gamma} \\
= & \int_{s, a} \sum_{t=0}^{\infty} \gamma^t P\left(s^t=s, a^t=a | s^0, \pi\right) q(s, a)-\int_{s, a} \sum_{t=1}^{\infty} \gamma^t P\left(s^t=s | s^0, \pi\right) q(s, \pi) \\
= & \int_{s, a} \sum_{t=0}^{\infty} \gamma^t P\left(s^t=s, a^t=a | s^0, \pi\right) q(s, a)-\int_{s, a} \sum_{t=1}^{\infty} \gamma^t P\left(s^t=s, a^t=a | s^0, \pi\right) q(s, a) \\
= & \int_{s, a} P\left(s^0=s, a^0=a | s^0, \pi\right) q(s, a)=q(s^0, \pi),
\$
where the conditional probability 
 $P(\cdot|s^0,\pi)$ is taken by assuming that the system follows the policy $\pi$ with initial state $s^0$. This completes the proof.
\end{proof}

\subsection{Proof of Lemma \ref{risk_bound_lm}}

\begin{lemma}\label{risk_bound_lm}
Suppose for $\tau \in \Omega$, $\sup_{\tau}{\|\tau(s,a)\|_{L_2(\mu)}} \leq \mathcal{U}^{\tau}_{2}$ and $\sup_{\tau}{\|\tau(s,a)\|_{L_\infty}} \leq \mathcal{U}^{\tau}_{\infty}$ and $\sup_{q}{\|q(s,a)\|_{L_{\infty}}} \leq \bar{V}$, given an offline data $\mathcal{D}_{1:n}=\{s_{i},a_{i},r_{i},s^{\prime}_i)\}^{n}_{i=1}$, w.p. $\geq 1-\delta$,
 \$
& \left|\mathbb{E}_{\mu}\left[\tau(s,a)(r(s,a)+\gamma q \left(s^{\prime}, \pi\right)-q(s, a)) \right] - \mathbb{P}_{n}\left[\tau(s_i,a_i)\left(r_i+\gamma q \left(s^{\prime}_i, \pi\right)-q(s_i, a_i)  \right)\right] \right| \\
\leq & \mathcal{U}^{\tau}_{2}\sqrt{\frac{2 \bar{V}^2\ln\frac{(e^{D}\max\{D_{\Omega},D_{\mathcal{Q}},D_{\Pi}\}+1)^3( \mathcal{U}^{\tau}_{2})^{2D}}{\delta}}{n}} + \frac{2 \mathcal{U}^{\tau}_{\infty} \bar{V} \ln \frac{(e^{D}\max\{D_{\Omega},D_{\mathcal{Q}},D_{\Pi}\}+1)^3( \mathcal{U}^{\tau}_{2})^{2D}}{\delta}}{3 n}.
\$
holds for any $\pi \in \Pi, \tau \in \Omega$ and  $q \in \mathcal{Q}$, and empirical measure $\mathbb{P}_{n}$. The terms $D_{\Omega}$, $D_{\mathcal{Q}}$ and $D_{\Pi}$ are the pseudo-dimension of $\Omega$, $\mathcal{Q}$ and $\Pi$, respectively, and $D = D_{\Omega} +  D_{\mathcal{Q}} + D_{\Pi}$ is so-called effective pseudo-dimension.
\end{lemma}

\begin{proof}
First, we observe that it suffices to bound
provide a uniform deviation bound that applies to all $\pi \in \Pi, \tau \in \Omega$, and  $q \in \mathcal{Q}$. According to the definition of $\varepsilon$-covering number in Definition \ref{cover_def}, and we define the $\varepsilon$-covering number with resecp to a weighted $L^2$ norm $\|\cdot\|_{L_{2}(\mu)}$ in the space of $\Omega$, $\mathcal{Q}$ and  $\Pi$ as follows:
\#
\|\tau_1-\tau_2\|_{L_{2}(\mu)} := &\sqrt{\int_{\mathcal{S}\times \mathcal{A}}|\tau_1(s,a) - \tau_2(s,a)|^2d\mu(\mathcal{S},\mathcal{A})} \notag \\
\|q_1-q_2\|_{L_{2}(\mu)} := &\sqrt{\int_{\mathcal{S}\times \mathcal{A}}|q_1(s,a) - q_2(s,a)|^2d\mu(\mathcal{S} \times \mathcal{A})} \notag \\
\|\pi_1-\pi_2\|_{L_{2}(\mu)} := &\sqrt{\int_{\mathcal{S}}|\pi_1(\cdot|s) - \pi_2(\cdot|s)|^2 d\mu(\mathcal{S})}.
\label{metric_def_1}
\#
where $\mu(\mathcal{S})$ is the marginal measure for $\mu(\mathcal{S} \times \mathcal{A})$. For the product space $\mathcal{G} := \Omega \times \mathcal{Q} \times \Pi$, where the function $g \in \mathcal{G}$ that 
\$
g(s,a,r,s^{\prime}) = {\tau}(s,a)(r+\gamma {q} (s^{\prime},{\pi})-{q}(s, a)).
\$
We have the $L_2(\mu)$ metric for $g_1, g_2 \in \mathcal{G}$, that $\|g_1-g_2\|_{L_{2}(\mu)}$ is upper bounded by 
\#
\sqrt{\int_{\mathcal{S} \times \mathcal{A} }|\mathbb{E}_{r = r(s,a), s^{\prime} \sim \mathds{P}(\cdot|s,a)}[g_1(s,a,r,s^{\prime})] - \mathbb{E}_{r = r(s,a), s^{\prime} \sim \mathds{P}(\cdot|s,a)}[g_2(s,a,r,s^{\prime})]|^2d\mu (\mathcal{S} \times \mathcal{A})} .
\label{metric_def_2}
\#
Based on this $L_{2}(\mu)$ metric, and to complete the proof, it is sufficient to establish the supremum bound, w.p. $\geq 1-\delta$,
\$
\sup_{g \in \mathcal{G}}\left|\mathbb{E}_{\mu}\left[g(s,a,r,s^{\prime})\right] - \mathbb{P}_{n}\left[g(s_i,a_i,r_i,s^{\prime}_i)\right] \right|.
\$
To apply Bernstein -type concentration inequality, we need to first examine the boundedness. According to the boundedness conditions on the function classes, we have
\$
&\sup _{{\tau} \in \Omega}\|{\tau}(s,a)\|_{L_2(\mu)} \leq \mathcal{U}^{\tau}_{2};\quad \sup _{{\tau} \in \Omega}\|{\tau}(s,a)\|_{L_{\infty}} \leq \mathcal{U}^{\tau}_{\infty}\\
& \left(r(s,a)+\gamma {q} \left(s^{\prime}, {\pi}\right)-{q}(s, a)  \right) \in [-\bar{V},\bar{V}], \forall s,a,s^{\prime}. 
\$
It is easy to conclude that 
\$
\sup_{g\in \mathcal{G}}\|g(s,a,r,s^{\prime})\|_{L_2(\mu)} \leq & \, \mathcal{U}^{\tau}_{2}\bar{V} \\ 
\sup_{g\in \mathcal{G}}\|g(s,a,r,s^{\prime})\|_{L_{\infty}} \leq &  \, \mathcal{U}^{\tau}_{\infty}\bar{V},
\$
To quantify the complexity of the product space $\mathcal{G}$, we first need to calculate the $L_{2}(\mu)$-distance in $\mathcal{G}$. With some calculations, for $g_1, g_2 \in \mathcal{G}$ corresponding to $\tau_1 \times q_1 \times \pi_1$ and $\tau_2 \times q_2 \times \pi_2$, respectively, 
\#
& \|g_1(s,a,r,s^{\prime}) - g_2(s,a,r,s^{\prime})\|_{L_2(\mu)} \notag \\
= & \|\tau_1(s,a)(r+\gamma q_1(s^{\prime},\pi_1) - q_1(s,a)) - \tau_2(s,a)(r+\gamma q_2(s^{\prime},\pi_2) - q_2(s,a)) \|_{L_2(\mu)} \notag \\
\leq & 
\bar{V}\|\tau_1(s,a) - {\tau_2}(s,a)\|_{L_2(\mu)} + \mathcal{U}^{\tau}_{2}\|(r+\gamma q_1(s^{\prime},\pi_1) -q_1(s,a)) -(r+\gamma {q}_2(s^{\prime},\pi_2)-q_2(s,a))\|_{L_2(\mu)} \notag \\
&+\mathcal{U}^{\tau}_{2}\|(r+\gamma {q}_1(s^{\prime},\pi_1) -q_1(s,a))-  (r+\gamma {q}_2(s^{\prime},{\pi_2})-{q}_2(s,a))\|_{L_2(\mu)} \notag \\
\leq & \bar{V}\|\tau_1(s,a) - {\tau_2}(s,a)\|_{L_2(\mu)} + \mathcal{U}^{\tau}_{2}(1+\gamma)\|q_1(s,a) - q_2(s,a)\|_{L_2(\mu)} \notag \\
& + \gamma \bar{V}\mathcal{U}^{\tau}_{2}\|\pi_1(\cdot|s^{\prime}) - 
{\pi_2}(\cdot|s^{\prime})\|_{L_2(\mu)} 
\label{decom_cover}
\#
which leads to
\#
\mathcal{N}(3\widetilde{C}{\varepsilon}, \mathcal{G}, \|\cdot\|_{L_2(\mu)}) \leq \mathcal{N}({\varepsilon}, \Omega, \|\cdot\|_{L_2(\mu)}) \mathcal{N}({\varepsilon}, \mathcal{Q}, \|\cdot\|_{L_2(\mu)}) \mathcal{N}({\varepsilon}, \Pi, \|\cdot\|_{L_2(\mu)}),
\label{cover_nonreg_upper_num}
\#
where $\widetilde{C} := \mathcal{U}^{\tau}_{2}(1+\gamma+\gamma \bar{V}) + \bar{V}$. 

To upper bound these factors, we apply the generalized version of Corollary 2 in \citep{haussler1995sphere}. For the pseudo-dimension of $\Omega$, $\mathcal{Q}$ and $\Pi$, i.e., $D_{\Omega}$, $D_{\mathcal{Q}}$ and $D_{\Pi}$, and for some $\varepsilon^{\prime} > 0 $,
\$
\mathcal{N}\left(3\widetilde{C}\epsilon^{\prime}, \mathcal{G}, \|\cdot\|_{L_2(\mu)}\right) \leq e^3\left(D_{\Omega}+1\right)\left(D_{\mathcal{Q}}+1\right)\left(D_{\Pi}+1\right)\left(\frac{4 e \widetilde{C}}{\epsilon^{\prime}}\right)^{D_{\Omega}+D_{\mathcal{Q}}+D_{\Pi}}.
\$
 This also implies 
\#
\mathcal{N}\left(\epsilon, \mathcal{G}, \|\cdot\|_{L_2(\mu)}\right) \leq e^3\left(D_{\Omega}+1\right)\left(D_{\mathcal{Q}}+1\right)\left(D_{\Pi}+1\right)\left(\frac{12 e \widetilde{C}^2}{\epsilon^{\prime}}\right)^{D_{\Omega}+D_{\mathcal{Q}}+D_{\Pi}} = C_{1}\left(\frac{1}{\varepsilon}\right)^{D}.
\label{cover_num_func}
\#
where $C_1 = e^3\left(D_{\Omega}+1\right)\left(D_{\mathcal{Q}}+1\right)\left(D_{\Pi}+1\right)(12 e \widetilde{C}^2)^{D}$ and $D = D_{\Omega}+D_{\mathcal{Q}}+D_{\Pi}$, i.e., the ``effective'' pseudo-dimension. With the calculated function class complexity, we apply empirical Bernstein inequality and union bound, w.p. $\geq 1-\delta$ with $Z = g(s,a,r,s^{\prime})$,
\$
 & \left|\mathbb{E}_{\mu}\left[g(s,a,r,s^{\prime})\right] - \mathbb{P}_{n}\left[g(s_i,a_i,r_i,s^{\prime}_i)\right] \right|\\
  \leq & \frac{1}{n}\sqrt{2 \sum_{i=1}^n Var_{\mu}\left[Z\right]\ln \frac{8\mathcal{N}\left(\epsilon, \mathcal{G}, \|\cdot\|_{L_2(\mu)}\right)}{\delta}} + \frac{2 \|Z\|_{L_{\infty}} \ln \frac{8\mathcal{N}\left(\epsilon, \mathcal{G}, \|\cdot\|_{L_2(\mu)}\right)}{\delta}}{3 n}\\
    \leq & \frac{1}{n}\sqrt{2 \sum_{i=1}^n \mathbb{E}_{\mu}\left[Z^2\right]\ln \frac{8\mathcal{N}\left(\epsilon, \mathcal{G}, \|\cdot\|_{L_2(\mu)}\right)}{\delta}} + \frac{2 \mathcal{U}^{\tau}_{\infty} \bar{V} \ln \frac{8\mathcal{N}\left(\epsilon, \mathcal{G}, \|\cdot\|_{L_2(\mu)}\right)}{\delta}}{3 n}\\
        \leq & \sqrt{\frac{2n  \mathcal{U}^2_{\tau,2}\bar{V}^2\ln \frac{8\mathcal{N}\left(\epsilon, \mathcal{G}, \|\cdot\|_{L_2(\mu)}\right)}{\delta}}{n^2}}+ \frac{2 \mathcal{U}^{\tau}_{\infty} \bar{V} \ln \frac{8\mathcal{N}\left(\epsilon, \mathcal{G}, \|\cdot\|_{L_2(\mu)}\right)}{\delta}}{3 n}\\
 = & \mathcal{U}^{\tau}_{2}\sqrt{\frac{2 \bar{V}^2\ln \frac{8\mathcal{N}\left(\epsilon, \mathcal{G}, \|\cdot\|_{L_2(\mu)}\right)}{\delta}}{n}} + \frac{2 \mathcal{U}^{\tau}_{\infty} \bar{V} \ln \frac{8\mathcal{N}\left(\epsilon, \mathcal{G}, \|\cdot\|_{L_2(\mu)}\right)}{\delta}}{3 n}. 
\$

We set $\varepsilon = \mathcal{O}(\frac{1}{\sqrt{n}})$, and combine with the upper bound for covering number \eqref{cover_num_func} by some algebra, thus we have 
\$
 & \left|\mathbb{E}_{\mu}\left[\tau(s,a)(r(s,a)+\gamma q \left(s^{\prime}, \pi\right)-q(s, a)) \right] - \mathbb{P}_{n}\left[\tau(s_i,a_i)\left(r_i+\gamma q \left(s^{\prime}_i, \pi\right)-q(s_i, a_i)  \right)\right] \right| \\
\leq &  \mathcal{U}^{\tau}_{2}\sqrt{\frac{2 \bar{V}^2\ln\frac{(e^{D}\max\{D_{\Omega},D_{\mathcal{Q}},D_{\Pi}\}+1)^3(\mathcal{U}^{\tau}_{2})^{2D}}{\delta}}{n}} + \frac{2 \mathcal{U}^{\tau}_{\infty} \bar{V} \ln \frac{(e^{D}\max\{D_{\Omega},D_{\mathcal{Q}},D_{\Pi}\}+1)^3(\mathcal{U}^{\tau}_{2})^{2D}}{\delta}}{3 n}.
\$
This completes the proof. 
\end{proof}

\subsection{Proof of Lemma \ref{L_2MMD}}

\begin{lemma}\label{L_2MMD}
For some admissible probability measure or empirical probability measure, denoted as $\nu$, suppose $\sqrt{\mathbb{E}_{\nu}[(\tau(s,a))^2]} \leq C$ for some positive constant $C$ and $\tau \in \Omega$, and 
\$
\sup_{\tau \in \Omega}|\mathbb{E}_{\nu}[\tau(s,a)(r(s,a)+\gamma q(s^{\prime},\pi)-q(s,a))]| \leq \varepsilon,
\$
then it holds that 
$
\sqrt{\mathbb{E}_{\nu}[(r+\gamma q(s^{\prime},\pi)-q(s,a))^2] } \leq \varepsilon/C 
$. 
\end{lemma}

\begin{proof}
To facilitate the proof, we first define
$
\widetilde{\tau} := \sup_{\tau \in \Omega }\left|\mathbb{E}_{\nu}\left[\tau(s,a) \Delta(q,\pi) \right]\right| 
$,
for $q \in \mathcal{Q}$, and denote $\Delta(q,\pi) = r(s,a)+\gamma q(s^{\prime},\pi)-q(s,a)$. Then for any $\pi \in \Pi$ and $q \in \mathcal{Q}$, it satisfies that  
\$
 C \sqrt{\mathbb{E}_{\nu}[(r(s,a)+\gamma q(s^{\prime},\pi)-q(s,a))^2] } = &C \| \Delta(q,\pi) \|_{L_2(\nu)}\\
= & C \left\langle \Delta(q,\pi),\Delta(q,\pi) \right\rangle^{\frac{1}{2}}_{\nu} \\ 
 = & \left\langle \Delta(q,\pi),\Delta(q,\pi) \right\rangle^{\frac{1}{2}}_{\nu}  \cdot \left\langle \frac{C\widetilde{\tau}(s,a)}{\|\widetilde{\tau}(s,a)\|_{L_2(\nu)}} ,\frac{C\widetilde{\tau}(s,a)}{\|\widetilde{\tau}(s,a)\|_{L_2(\nu)}}  \right\rangle^{\frac{1}{2}}_{\nu}\\
 = & \left\langle \Delta(q,\pi), \frac{C\widetilde{\tau}}{\|\widetilde{\tau}\|_{L_2(\nu)}} \right\rangle_{\nu}\\
 = & \sup_{\tau \in \Omega}\left|\mathbb{E}_{\nu}\left[\tau(s,a) \Delta(q,\pi) \right]\right|,
 \$
 where the third equality comes from that the direction $\widetilde{\tau}/\|\widetilde{\tau}(s,a)\|_{L_2(\nu)}$ is aligned with the direction of the maximizer of inner product, and the fourth equality comes from the exact equality condition for Cauchy-Schwarz inequality. This completes the proof. 
\end{proof}

\subsection{Proof of Lemma \ref{risk_bound_alpha_2q}}

\begin{lemma}\label{risk_bound_alpha_2q}
For any $\pi \in \Pi$ and any $\tau \in \widetilde{\Omega}_{\widetilde{\sigma}_n}$ where $\sup_{\tau \in \widetilde{\Omega}_{\widetilde{\sigma}_n}}{\|\tau(s,a)\|_{L_2(\mu)}} \leq \mathcal{U}^{\star}_{2}$ and $\sup_{\tau \in \widetilde{\Omega}_{\widetilde{\sigma}_n}}{\|\tau(s,a)\|_{L_\infty}} \leq \mathcal{U}^{\star}_{\infty}$, and  any $q_1,q_2 \in \mathcal{Q}_{\varepsilon_n}$, given an offline data $\mathcal{D}_{1:n}=\{s_{i},a_{i},r_{i},s^{\prime}_i)\}^{n}_{i=1}$, w.p. $\geq 1-\delta$, it holds that 
\$
& \bigg| \mathbb{E}_{\mu}\left[\tau(s,a)\left[\left(q_1(s, a)-r(s,a)-\gamma q_1 \left(s^{\prime}, \pi\right) \right) - \left(q_2 (s, a)-r(s,a)-\gamma q_2  \left(s^{\prime}, \pi\right)\right)\right]\right] \\
&- \mathbb{P}_{n}\tau(s_i,a_i)\left[\left(q_1(s_i, a_i)-r_i-\gamma q_1 \left(s^{\prime}_i, \pi\right) \right) - \left(q_2 (s_i, a_i)-r_i-\gamma q_2  \left(s^{\prime}_i, \pi\right)\right)\right] \bigg| \\
\lesssim &  \mathcal{U}^{\star}_{2}\sqrt{\frac{32 \bar{V}^2\ln \frac{8\mathcal{N}\left(\epsilon, \mathcal{G}_{\varepsilon_n, \widetilde{\sigma}_n}, \|\cdot\|_{L_2(\mu)}\right)}{\delta}}{n}} + \frac{8 \mathcal{U}^{\star}_{\infty} \bar{V} \ln \frac{8\mathcal{N}\left(\epsilon, \mathcal{G}_{\varepsilon_n, \widetilde{\sigma}_n}, \|\cdot\|_{L_2(\mu)}\right)}{\delta}}{3 n}. 
\$
\end{lemma}

\begin{proof}
At first, we define a product space 
\#
\mathcal{G}_{\varepsilon_n, \widetilde{\sigma}_n} := \widetilde{\Omega}_{\widetilde{\sigma}_n} \times \mathcal{Q}_{\varepsilon_n} \times \Pi.
\label{G_space}
\#
equipped with the $L_2(\mu)$ weighted metric, so that any $g \in \mathcal{G}_{\varepsilon_n, \widetilde{\sigma}_n}$ can be expressed as 
\$
g(s,a,r,s^{\prime}) = & 
\tau(s,a)\bigg[\left(q_1 (s, a) -r-\gamma q_1  \left(s^{\prime}, \pi\right)\right) - \left(q_2  (s, a) -r-\gamma q_2  \left(s^{\prime}, \pi\right)\right)\bigg].
\$
With some calculation, the $L_{2}(\mu)$-distance  in $\mathcal{G}_{\varepsilon_n, \widetilde{\sigma}_n}$ can be bounded. That is, for $g_1, g_2 \in \mathcal{G}_{\varepsilon_n, \widetilde{\sigma}_n}$ corresponding to $\tau_1, q_1, q_2, \pi_1$ and $\tau_2, q^{\prime}_1, q^{\prime}_2, \pi_2$, the ${L_2(\mu)}$ distance between $g_1$ and $g_2$ is upper bounded as,
\#
& 
 \|g_1(s,a,r,s^{\prime}) - g_2(s,a,r,s^{\prime})\|_{L_2(\mu)} \notag \\
\leq & \|\tau_1(s,a)(r+\gamma q_1^{\prime}(s^{\prime},\pi) - q_1^{\prime}(s,a)) - {\tau}_2(s,a)(r+\gamma q_2^{\prime}(s^{\prime},\pi) - q_2^{\prime}(s,a))\|_{L_2(\mu)} \notag  \\
&+ \|\tau_1(s,a)(r+\gamma q_1(s^{\prime},\pi) -q_1(s,a)) - {\tau}_2(s,a)(r+\gamma q_2(s^{\prime},\pi) - q_2(s,a))\|_{L_2(\mu)} \notag \\
\leq & \bar{V}\|\tau_1(s,a) - {\tau_2}(s,a)\|_{L_2(\mu)} + \mathcal{U}^{\tau}_{2}(1+\gamma)\|q_1(s,a) - q_2(s,a)\|_{L_2(\mu)} \notag \\
& + \gamma \bar{V}\mathcal{U}^{\tau}_{2}\|\pi_1(\cdot|s^{\prime}) - 
{\pi_2}(\cdot|s^{\prime})\|_{L_2(\mu)}. 
\label{metric_double_2}
\#
Based on \eqref{metric_double_2}, and following the definition of the covering number in Definition \ref{cover_def}, then we  have 
\$
\mathcal{N}(6\widetilde{C}{\varepsilon}, \mathcal{G}_{\varepsilon_n, \widetilde{\sigma}_n}, \|\cdot\|_{L_2(\mu)}) \leq \mathcal{N}({\varepsilon}, \widetilde{\Omega}_{\widetilde{\sigma}_n}, \|\cdot\|_{L_2(\mu)}) \mathcal{N}({\varepsilon}, \mathcal{Q}_{\varepsilon_n}, \|\cdot\|_{L_2(\mu)}) \mathcal{N}({\varepsilon}, \Pi, \|\cdot\|_{L_2(\mu)}) 
\$
where $\widetilde{C} := \mathcal{U}^{\star}_{2}(1+\gamma+\gamma \bar{V}) + \bar{V}$. Accordingly, follows Corollary 2 in \citep{haussler1995sphere}, by some algebra, 
\$
 \mathcal{N}\left(\epsilon, \mathcal{G}_{\varepsilon_n, \widetilde{\sigma}_n}, \|\cdot\|_{L_2(\mu)}\right) \leq & e^3\left(D_{\widetilde{\Omega}_{\widetilde{\sigma}_n}}+1\right)\left(D_{\mathcal{Q}_{\varepsilon_n}}+1\right)\left(D_{\Pi}+1\right)\left(\frac{24 e \widetilde{C}^2}{\epsilon^{\prime}}\right)^{D_{\widetilde{\Omega}_{\widetilde{\sigma}_n}}+D_{\mathcal{Q}_{\varepsilon_n}}+D_{\Pi}}  \\
=  &C^{\prime}_{1}\left(\frac{1}{\varepsilon}\right)^{D}
\$
where $C^{\prime}_1 = e^3\left(D_{\widetilde{\Omega}_{\widetilde{\sigma}_n}}+1\right)\left(D_{\mathcal{Q}_{\varepsilon_n}}+1\right)\left(D_{\Pi}+1\right)(24 e \widetilde{C}^2)^{D}$ and $D = D_{\widetilde{\Omega}_{\widetilde{\sigma}_n}}+D_{\mathcal{Q}_{\varepsilon_n}}+D_{\Pi}$. Next, we apply empirical Bernstein concentration inequality and union bound as in the proof of Lemma \ref{risk_bound_lm}, we conclude that 
\#
& \bigg| \mathbb{E}_{\mu}\left[\tau(s,a)\left[\left(q_1(s, a)-r(s,a)-\gamma q_1 \left(s^{\prime}, \pi\right) \right) - \left(q_2 (s, a)-r(s,a)-\gamma q_2  \left(s^{\prime}, \pi\right)\right)\right]\right] \notag \\
&- \mathbb{P}_{n}\tau(s_i,a_i)\left[\left(q_1(s_i, a_i)-r_i-\gamma q_1 \left(s^{\prime}_i, \pi\right) \right) - \left(q_2 (s_i, a_i)-r_i-\gamma q_2  \left(s^{\prime}_i, \pi\right)\right)\right] \bigg| \notag  \\
\lesssim &  \mathcal{U}^{\star}_{2}\sqrt{\frac{32 \bar{V}^2\ln \frac{8\mathcal{N}\left(\epsilon, \mathcal{G}_{\varepsilon_n, \widetilde{\sigma}_n}, \|\cdot\|_{L_2(\mu)}\right)}{\delta}}{n}} + \frac{8 \mathcal{U}^{\star}_{\infty} \bar{V} \ln \frac{8\mathcal{N}\left(\epsilon, \mathcal{G}_{\varepsilon_n, \widetilde{\sigma}_n}, \|\cdot\|_{L_2(\mu)}\right)}{\delta}}{3 n}. 
\label{diff_ineq_2}
\#
This completes the proof. 
\end{proof}

\subsection{Proof of Lemma \ref{iden_mdp_version}}

\begin{lemma}
\label{iden_mdp_version}
Define the maximizer $\overline{q^{\pi}}:= \argmax_{q\in \mathcal{Q}_{\varepsilon_n}}q(s^{0},\pi)$ and the minimizer $\underline{q^{\pi}}:= \argmin_{q\in \mathcal{Q}_{\varepsilon_n}}q(s^{0},\pi)$ over the confidence set $\mathcal{Q}_{\varepsilon_n}$ with both $q=\overline{q^{\pi}}$ and $q=\underline{q^{\pi}}$ satisfy that 
\$
\left|\frac{1}{n}\sum^{n}_{i=1}\tau(s_i,a_i)\left(r_i+\gamma q \left(s^{\prime}_i, \pi\right)-q(s_i, a_i)  \right)\right| \leq \varepsilon_n, \forall \tau \in \widetilde{\Omega}_{\widetilde{\sigma}_n}.
\$
There must exist an MDP $\left\{\mathcal{S}, \mathcal{A}, \mathds{P}_{\max}(\mathds{P}_{\min}), \gamma, r_{\max}(r_{\min}), s^{0}\right\}$ which is identical to the true environment MDP with $\mathds{P}_{\max}(\mathds{P}_{\min}) = \mathds{P}$: $\text{MDP}^{\star}$ only except that the reward function $r_{\max}(r_{\min})$ is redefined as 
\$
r_{\max}(s,a) = \overline{q^{\pi}}(s,a) - \gamma \mathbb{E}_{s^{\prime} \sim \mathds{P}_{\max}(\cdot|s,a)}\left[ \sum_{a^{\prime}\in \mathcal{A}} \pi(a^{\prime}|s^{\prime})\overline{q^{\pi}}(s^{\prime},a^{\prime}) \right], \\
r_{\min}(s,a) =\underline{q^{\pi}}(s,a) - \gamma \mathbb{E}_{s^{\prime} \sim \mathds{P}_{\min}(\cdot|s,a)}\left[ \sum_{a^{\prime}\in \mathcal{A}} \pi(a^{\prime}|s^{\prime})\underline{q^{\pi}}(s^{\prime},a^{\prime}) \right].
\$
In addition, the reward function $r_{\max}(r_{\min})$ is approximating the true reward, i.e., $ \| r_{\max}(s,a) - r(s,a) \|_{L_{2}(\mu)}$ or  $\| r_{\min}(s,a) - r(s,a) \|_{L_{2}(\mu)}$ is upper bounded by 
\$
& \sqrt{\frac{2 \bar{V}^2\ln \frac{8\mathcal{N}\left(\epsilon, \mathcal{G}_{\varepsilon_n, \widetilde{\sigma}_n}, \|\cdot\|_{L_2(\mu)}\right)}{\delta}}{n}} +  \frac{2\mathcal{U}_{\infty} \bar{V} \ln \frac{8\mathcal{N}\left(\epsilon, \mathcal{G}_{\varepsilon_n, \widetilde{\sigma}_n}, \|\cdot\|_{L_2(\mu)}\right)}{\delta}}{3 n} + \frac{\varepsilon_n}{\mathcal{U}^{\star}_{2}}.
\$
for some constant $\mathcal{U}_{\infty} := \frac{\mathcal{U}^{\star}_{\infty}}{\mathcal{U}^{\star}_{2}}$ for $\sup_{\tau \in \widetilde{\Omega}_{\widetilde{\sigma}_n}}{\|\tau(s,a)\|_{L_2(\mu)}} \leq \mathcal{U}^{\star}_{2}$ and $\sup_{\tau \in \widetilde{\Omega}_{\widetilde{\sigma}_n}}{\|\tau(s,a)\|_{L_\infty}} \leq \mathcal{U}^{\star}_{\infty}$.
\end{lemma}

\begin{proof}
Without loss of generality, we prove the lemma for $r_{\max}$, and the results for $r_{\min}$ can be obtained in a similar way. It follows from the definition of $r_{\max}$, i.e., 
\#
r_{\max}(s,a) =  \overline{q^{\pi}}(s,a) - \gamma \mathbb{E}_{s^{\prime} \sim \mathds{P}_{\max}(\cdot|s,a)}\left[ \sum_{a^{\prime}\in \mathcal{A}} \pi^{k}(a^{\prime}|s^{\prime}) \overline{q^{\pi}}(s^{\prime},a^{\prime}) \right],
\label{max_def}
\#
we re-arrange the terms as follows: 
\#
 \overline{q^{\pi}}(s,a) =& r_{\max}(s,a)  + \gamma \mathbb{E}_{s^{\prime} \sim \mathds{P}_{\max}(\cdot|s,a)}\left[ \sum_{a^{\prime}\in \mathcal{A}} \pi(a^{\prime}|s^{\prime}) \overline{q^{\pi}}(s^{\prime},a^{\prime}) \right].
\label{supp_2_version}
\#
This is exactly the Bellman equation over the MDP $\left\{\mathcal{S}, \mathcal{A}, \mathds{P}_{\max}, \gamma, r_{\max}, s^{0}\right\}$. It follows the equivalence between $\mathds{P}_{\max}$ and $\mathds{P}$, we further observe that 
$\overline{q^{\pi}}(\cdot,\cdot)$ is the true $q$-function for policy $\pi$ in the MDP $\left\{\mathcal{S}, \mathcal{A}, \mathds{P}, \gamma, r_{\max}, s^{0}\right\}$. 
To show the reward function $r_{\max}$ approximates the true reward function $r$. It follows from Lemma \ref{risk_bound_lm} with $\|\tau(s,a) \|^2_{L_{2}(\mu)} \leq  \mathcal{U}^2_{\text{prime},2}$ as $\tau \in \widetilde{\Omega}_{\widetilde{\sigma}_n}$, we have $\sup_{\tau \in \widetilde{\Omega}_{\widetilde{\sigma}_n}}\left|\mathbb{E}_{\mu}\left[\tau(s,a) \Delta( \overline{q^{\pi}},\pi) \right]\right|$ for $\Delta(q,\pi) = r(s,a)+\gamma q(s^{\prime},\pi)-q(s,a)$, is upper bounded by 
\$
 \mathcal{U}^{\star}_{2}\sqrt{\frac{2 \bar{V}^2\ln \frac{8\mathcal{N}\left(\epsilon, \mathcal{G}_{\varepsilon_n, \widetilde{\sigma}_n}, \|\cdot\|_{L_2(\mu)}\right)}{\delta}}{n}} +  \frac{2 \mathcal{U}_{\infty}\mathcal{U}^{\star}_{2} \bar{V} \ln \frac{8\mathcal{N}\left(\epsilon, \mathcal{G}_{\varepsilon_n, \widetilde{\sigma}_n}, \|\cdot\|_{L_2(\mu)}\right)}{\delta}}{3 n} + \varepsilon_n.
\$
With $\sup_{\tau \in \widetilde{\Omega}_{\widetilde{\sigma}_n}}\left|\mathbb{E}_{\mu}\left[\tau(s,a) \Delta( \overline{q^{\pi}},\pi) \right]\right|$ and Lemma \ref{L_2MMD} on $
\|\tau(s,a) \|^2_{L_{2}(\mu)} \leq  \mathcal{U}^2_{\text{prime},2}
$, we obtain 
\$
 \| \Delta( \overline{q^{\pi}},\pi) \|_{L_{2}(\mu)} \leq &   \sqrt{\frac{2 \bar{V}^2\ln \frac{8\mathcal{N}\left(\epsilon, \mathcal{G}_{\varepsilon_n, \widetilde{\sigma}_n}, \|\cdot\|_{L_2(\mu)}\right)}{\delta}}{n}} +  \frac{2 \mathcal{U}_{\infty}\bar{V} \ln \frac{8\mathcal{N}\left(\epsilon, \mathcal{G}_{\varepsilon_n, \widetilde{\sigma}_n}, \|\cdot\|_{L_2(\mu)}\right)}{\delta}}{3 n} + \varepsilon_n/\mathcal{U}^2_{\text{prime},2}.
\$
Since $\mathds{P}_{\max} = \mathds{P}$, and follow the definition of $r_{\max}$ in \eqref{max_def}, we have
\$
\|r(s,a) - r_{\max}(s,a) \|_{L_{2}(\mu)}  =& \left\|r(s,a) -  \overline{q^{\pi}}(s,a) + \gamma \mathbb{E}_{s^{\prime} \sim \mathds{P}_{\max}(\cdot|s,a)}\left[ \sum_{a^{\prime}\in \mathcal{A}} \pi(a^{\prime}|s^{\prime}) \overline{q^{\pi}}(s^{\prime},\pi) \right]  \right\|_{L_{2}(\mu)} \\
= &\| \Delta( \overline{q^{\pi}},\pi) \|_{L_{2}(\mu)}.
\$
Combining with the upper bound on $\| \Delta( \overline{q^{\pi}},\pi) \|_{L_{2}(\mu)}$, this completes the proof. 
\end{proof}

\subsection{Proof of Lemma \ref{upper_version_bdd}}

\begin{lemma}[Upper Bound for Version Space Function over $\mu$]
\label{upper_version_bdd}
On the notations and definitions in Lemma \ref{iden_mdp_version}, 
where $ \overline{q^{\pi}}(\cdot,\cdot)$ and $\underline{q^{\pi}}(\cdot,\cdot)$ are the true action-value function under policy $\pi$ for the MDPs $\left\{\mathcal{S}, \mathcal{A}, \mathds{P}_{\max}, \gamma, r_{\max}, s^{0}\right\}$ and $\left\{\mathcal{S}, \mathcal{A}, \mathds{P}_{\min}, \gamma, r_{\min}, s^{0}\right\}$, respectively. Then 
\$
\| \underline{q^{\pi}}(s,a) -  \overline{q^{\pi}}(s,a)\|_{L_{2}(\mu)} \leq \frac{2 \varepsilon_{r}}{1-\gamma}, 
\$
where 
\$
\varepsilon_{r} =& \sqrt{\frac{2 \bar{V}^2\ln \frac{8\mathcal{N}\left(\epsilon, \mathcal{G}_{\varepsilon_n, \widetilde{\sigma}_n}, \|\cdot\|_{L_2(\mu)}\right)}{\delta}}{n}} +  \frac{2 \mathcal{U}_{\infty}\bar{V} \ln \frac{8\mathcal{N}\left(\epsilon, \mathcal{G}_{\varepsilon_n, \widetilde{\sigma}_n}, \|\cdot\|_{L_2(\mu)}\right)}{\delta}}{3 n} + \varepsilon_n/\mathcal{U}^2_{\text{prime},2},
\$
for $\mathcal{U}_{\infty} \geq 0 $.
\end{lemma}

\begin{proof}
According to Lemma \ref{iden_mdp_version} and the definitions of $r_{\max}$ and $r_{\min}$, we have the reward functions $r_{\max}$ and $r_{\min}$ are bounded above over $L_2(\mu)$, i.e.,
\#
\|r_{\max}(s,a)-r_{\min}(s,a)\|_{L_{2}(\mu)} =& \|r_{\max}(s,a) - r(s,a)+r(s,a)-r_{\min}(s,a)\|_{L_{2}(\mu)} \notag \\
\leq & \|r_{\max}(s,a) - r(s,a)\|_{L_{2}(\mu)}+\|r(s,a)-r_{\min}(s,a)\|_{L_{2}(\mu)} \notag \\
\leq & 2 \varepsilon_{r}.
\label{reward_bound}
\#
By the fact that $ \overline{q^{\pi}}(\cdot,\cdot)$ and $\underline{q^{\pi}}(\cdot,\cdot)$ are the true action-value functions for policy $\pi$ in the MDPs $\left\{\mathcal{S}, \mathcal{A}, \mathds{P}_{\max}, \gamma, r_{\max}, s^{0}\right\}$ and $\left\{\mathcal{S}, \mathcal{A}, \mathds{P}_{\min}, \gamma, r_{\min}, s^{0}\right\}$, respectively. 
Then by the definition of the reward functions $r_{\max}$ and $r_{\min}$, we have 
\#
& \|r_{\max}(s,a)-r_{\min}(s,a)\|_{L_{2}(\mu)} \notag  \\
= \, & \Bigg\|\Bigg( \overline{q^{\pi}}(s,a) - \gamma \mathbb{E}_{s^{\prime} \sim \mathds{P}_{\max}(\cdot|s,a)}\left[ \sum_{a^{\prime}\in \mathcal{A}} \pi(a^{\prime}|s^{\prime}) \overline{q^{\pi}}(s^{\prime},a^{\prime}) \right]\Bigg) \notag  \\
& - \Bigg(\underline{q^{\pi}}(s,a) - \gamma \mathbb{E}_{s^{\prime} \sim \mathds{P}_{\min}(\cdot|s,a)}\left[ \sum_{a^{\prime}\in \mathcal{A}} \pi(a^{\prime}|s^{\prime})\underline{q^{\pi}}(s^{\prime},a^{\prime}) \right]\Bigg) \Bigg\|_{L_{2}(\mu)} \notag \\
=: & \|( \overline{q^{\pi}}(s,a) - \mathds{P}^{\pi} \overline{q^{\pi}}(s,a)) - (\underline{q^{\pi}}(s,a) - \mathds{P}^{\pi}\underline{q^{\pi}}(s,a)) \|_{L_{2}(\mu)} 
\label{reward_diff}
\#
where $\mathds{P}^{\pi}$ is the transition kernel under the policy $\pi$. We re-organize \eqref{reward_diff}, and obtain  
\$
& \|( \overline{q^{\pi}}(s,a) -\gamma \mathds{P}^{\pi} \overline{q^{\pi}}(s,a)) - (\underline{q^{\pi}}(s,a) - \gamma \mathds{P}^{\pi}\underline{q^{\pi}}(s,a))  \|_{L_{2}(\mu)} \notag \\
= & \| \overline{q^{\pi}}(s,a) -\underline{q^{\pi}}(s,a) +\gamma \mathds{P}^{\pi}(\underline{q^{\pi}}(s,a) -  \overline{q^{\pi}}(s,a))
\|_{L_{2}(\mu)} \\
= & \| (\mathds{I}-\gamma\mathds{P}^{\pi})(\underline{q^{\pi}}(s,a) -  \overline{q^{\pi}}(s,a))
\|_{L_{2}(\mu)}\\
\geq &  \| \underline{q^{\pi}}(s,a) -  \overline{q^{\pi}}(s,a)\|_{L_{2}(\mu)} - \gamma \|\mathds{P}^{\pi} (\underline{q^{\pi}}(s,a) -  \overline{q^{\pi}}(s,a))\|_{L_{2}(\mu)} \\
\geq &  \| \underline{q^{\pi}}(s,a) -  \overline{q^{\pi}}(s,a)\|_{L_{2}(\mu)} - \gamma \| (\underline{q^{\pi}}(s,a) -  \overline{q^{\pi}}(s,a))\|_{L_{2}(\mu)}\\
\geq &  (1-\gamma)\| \underline{q^{\pi}}(s,a) -  \overline{q^{\pi}}(s,a)\|_{L_{2}(\mu)},
\$
where the second inequality comes from each element of $\mathds{P}^{\pi}$
is a convex average of $\underline{q^{\pi}}(s,a) -  \overline{q^{\pi}}(s,a)$. Combine with the inequality \eqref{reward_bound}, we conclude that 
\$
 (1-\gamma)\| \underline{q^{\pi}}(s,a) -  \overline{q^{\pi}}(s,a)\|_{L_{2}(\mu)} \leq  2 \varepsilon_{r} 
 \implies \| \underline{q^{\pi}}(s,a) -  \overline{q^{\pi}}(s,a)\|_{L_{2}(\mu)} \leq \frac{2 \varepsilon_{r}}{1-\gamma}.
\$
We explicitly express the weighted $L_2(\mu)$ norm on $\underline{q^{\pi}}(s,a) -  \overline{q^{\pi}}(s,a)$, i.e., 
\$
\| \underline{q^{\pi}}(s,a) -  \overline{q^{\pi}}(s,a)\|^2_{L_{2}(\mu)} =& \sum_{a \in \mathcal{A}}\sum_{s \in \mathcal{S}} \left[\underline{q^{\pi}}(s,a) -  \overline{q^{\pi}}(s,a)\right]^2 \cdot \mu(s,a) 
\leq   \left(\frac{2 \varepsilon_{r}}{1-\gamma}\right)^2.
\$
By the non-negative of the term $\mu(s,a)$, we conclude that  
\$
\sup_{a\in \mathcal{A},s\in\mathcal{S}} \left|\underline{q^{\pi}}(s,a) -  \overline{q^{\pi}}(s,a)\right| \leq \frac{2 \varepsilon_{r}}{1-\gamma}, \quad \text{almost surely for } \; (s,a) \; \text{with} \; \mu(s,a)>0.
\$ 
This completes the proof. 
\end{proof}

\subsection{Proof of Lemma \ref{risk_bound_alpha}}

\begin{lemma}\label{risk_bound_alpha}
Suppose for $\tau \in \Omega$, $\sup_{\tau}{\|\tau(s,a)\|_{L_2(\mu)}} \leq \mathcal{U}^{\tau}_{2}$ and $\sup_{\tau}{\|\tau(s,a)\|_{L_\infty}} \leq \mathcal{U}^{\tau}_{\infty}$ and $\sup_{q}{\|q(s,a)\|_{L_{\infty}}} \leq \bar{V}$, given an offline data $\mathcal{D}_{1:n}=\{s_{i},a_{i},r_{i},s^{\prime}_i)\}^{n}_{i=1}$, w.p. $\geq 1-\delta$,
 \$
&  \bigg|\mathbb{E}_{\mu}\Big[\tau(s,a) \left(r+\gamma q \left(s^{\prime}, \pi\right)-q(s, a)\right) -\lambda \mathbb{D}(\tau(s,a))\Big] \\
& - \mathbb{P}_{n}\Big(\tau(s_i,a_i)\left(r_i+\gamma q \left(s^{\prime}_i, \pi\right)-q(s_i, a_i)  \right)-\lambda\mathbb{D}(\tau(s_i,a_i))\Big)  \bigg| \\ 
\lesssim & \big(\mathcal{U}^{\tau}_{2}\bar{V}+\lambda \|\mathbb{D}(\tau(s,a))\|^{\text{UB}}_{L_2(\mu)}\big)\sqrt{\frac{2\ln\frac{(e^{D}\max\{D_{\Omega},D_{\mathcal{Q}},D_{\Pi}\}+1)^3(L\mathcal{U}^{\tau}_{2})^{2D}}{\delta}}{n}} \notag  \\
 & \qquad \qquad + \frac{2 \big(\mathcal{U}^{\tau}_{\infty}\bar{V}+\lambda\|\mathbb{D}(\tau(s,a))\|^{\text{UB}}_{L_{\infty}}\big)\ln\frac{(e^{D}\max\{D_{\Omega},D_{\mathcal{Q}},D_{\Pi}\}+1)^3(L\mathcal{U}^{\tau}_{2})^{2D}}{\delta}}{3 n},
\$
where $L < \infty$ issome bounded Lipschitz constant of $\mathbb{D}(\cdot)$, and 
$\|\mathbb{D}(\tau(s,a))\|^{\text{UB}}_{L_2(\mu)} = \sup_{\tau \in \Omega}\|\mathbb{D}({\tau}(s,a))\|_{L_2(\mu)}$ and $\|\mathbb{D}(\tau(s,a))\|^{\text{UB}}_{L_{\infty}} =  \sup_{\tau \in \Omega}\|\mathbb{D}({\tau}(s,a))\|_{L_{\infty}} $. 
\end{lemma}

\begin{proof}
Define the product space $\widetilde{\mathcal{G}} := \Omega \times \mathcal{Q} \times \Pi$, and the $L_2$ weighted metric
\$
& \|g_1-g_2\|_{L_{2}(\mu)} \\
:= & \sqrt{\int_{\mathcal{S} \times \mathcal{A} }|\mathbb{E}_{r \sim r(s,a), s^{\prime} \sim \mathds{P}(\cdot|s,a)}[g_1(s,a,r,s^{\prime})] - \mathbb{E}_{r \sim r(s,a), s^{\prime} \sim \mathds{P}(\cdot|s,a)}[g_2(s,a,r,s^{\prime})]|^2d\mu (\mathcal{S} \times \mathcal{A})}
\$
 where $g_1, g_2 \in \widetilde{\mathcal{G}}$ for given $\lambda > 0 $ such that
$
g(s,a,r,s^{\prime}) = {\tau}(s,a)(r+\gamma {q} (s^{\prime},{\pi})-{q}(s, a))-\lambda\mathbb{D}(\tau(s,a))
$ for any $g \in \widetilde{\mathcal{G}}$. To apply empirical Bernstein inequality, we study the boundedness conditions:
$
\left(r+\gamma {q} \left(s^{\prime}, {\pi}\right)-{q}(s, a)  \right) \in [-\bar{V},\bar{V}],  \forall q \in \mathcal{Q}, 
 \sup_{\tau \in \Omega}\|\tau(s,a)\|_{L_2(\mu)} \leq \mathcal{U}^{\tau}_{2},
\sup_{\tau \in \Omega}\|\tau(s,a)\|_{L_{\infty}} \leq \mathcal{U}^{\tau}_{\infty}, 
\lambda \sup_{\tau \in \Omega}\|\mathbb{D}({\tau}(s,a))\|_{L_2(\mu)} \in [0,\lambda \|\mathbb{D}(\tau(s,a))\|^{\text{UB}}_{L_2(\mu)}], 
\lambda \sup_{\tau \in \Omega}\|\mathbb{D}({\tau}(s,a))\|_{L_{\infty}} \in [0,\lambda \|\mathbb{D}(\tau(s,a))\|^{\text{UB}}_{L_{\infty}}]
$.
By some calculation, the $L_{2}(\mu)$-distance  in $\widetilde{\mathcal{G}}$ can be bounded. For $g_1, g_2 \in \widetilde{\mathcal{G}}$ corresponding to $\tau_1 \times q_1 \times \pi_1$ and $\tau_2 \times q_2 \times \pi_2$, respectively, we have 
\#
& \|g_1(s,a,r,s^{\prime}) - g_2(s,a,r,s^{\prime})\|_{L_2(\mu)} \notag \\
= &\bigg|\tau_1(s,a)(r+\gamma q_1(s^{\prime},\pi_1) - q_1(s,a)) -\lambda \mathbb{D}(\tau_1(s,a)) \notag \\
& - 
\left(\tau_2(s,a)(r+\gamma {q}_2(s^{\prime},{\pi}_2) - {q}_2(s,a))-\lambda \mathbb{D}({\tau}_2(s,a))\right)\bigg|\notag \\
\leq & 
2\bar{V}\|\tau_1(s,a) - \tau_2(s,a)\|_{L_2(\mu)}  + \mathcal{U}^{\tau}_{2}\|(r+\gamma q_1(s^{\prime},\pi_1) -q_1(s,a)) -(r+\gamma q_2(s^{\prime},\pi_2)-{q}_2(s,a))\|_{L_2(\mu)}\notag \\
&+\mathcal{U}^{\tau}_{2}\|(r+\gamma{q}_1(s^{\prime},\pi) -{q}_1(s,a))-  (r+\gamma {q}_2(s^{\prime},{\pi}_2)-{q}_2(s,a))\|_{L_2(\mu)} + \lambda\|\mathbb{D}(\tau_1(s,a)) - \mathbb{D}(\tau_2(s,a))\|_{L_2(\mu)} \notag  \\ 
\leq & 
2\bar{V}\|\tau_1(s,a) - \tau_2(s,a)\|_{L_2(\mu)}  + \mathcal{U}^{\tau}_{2}\|(r+\gamma q_1(s^{\prime},\pi_1) -q_1(s,a)) -(r+\gamma q_2(s^{\prime},\pi_2)-{q}_2(s,a))\|_{L_2(\mu)}\notag \\
&+\mathcal{U}^{\tau}_{2}\|(r+\gamma{q}_1(s^{\prime},\pi) -{q}_1(s,a))-  (r+\gamma {q}_2(s^{\prime},{\pi}_2)-{q}_2(s,a))\|_{L_2(\mu)} + \lambda L\|\tau_1(s,a) - \tau_2(s,a)\|_{L_2(\mu)} \notag \\
\leq & 
2\bar{V}\|\tau_1(s,a) - \tau_2(s,a)\|_{L_2(\mu)}+ \gamma \bar{V}\mathcal{U}^{\tau}_{2}\|\pi_1 - 
{\pi_2}\|_{L_2(\mu)} \notag \\
&+\mathcal{U}^{\tau}_{2}(1+\gamma)\|q_1(s,a) -q_2(s,a)\|_{L_2(\mu)} + \lambda L\|\tau_1(s,a) - \tau_2(s,a)\|_{L_2(\mu)},
\label{metric_dist_reg}
\#
where the last inequality comes from the $\mathbb{D}$ is $M$-strongly convex function and thus locally Lipschitz with a  Lipschitz constant $L \leq \infty$. Also, we note that 
\$
\|\sup_{\tau \in \Omega}\mathbb{D}(\tau(s,a)) - 0 \|_{L_2(\mu)} \leq  L\|\tau^{\star}(s,a) - 1\|_{L_2(\mu)} 
\leq  L\max\{1, \mathcal{U}^{\tau}_{2}\} 
\leq  L \mathcal{U}^{\tau}_{2},
\$
where the last inequality holds for $\mathcal{U}^{\tau}_{2} \geq 1$. The metric distance \eqref{metric_dist_reg} directly leads to the upper bound for the covering number over $\widetilde{\mathcal{G}}$: 
\$
\mathcal{N}(4{C}^{\star}{\varepsilon}, \widetilde{\mathcal{G}}, \|\cdot\|_{L_2(\mu)}) \leq \mathcal{N}({\varepsilon}, \Omega, \|\cdot\|_{L_2(\mu)}) \mathcal{N}({\varepsilon}, \mathcal{Q}, \|\cdot\|_{L_2(\mu)}) \mathcal{N}({\varepsilon}, \Pi, \|\cdot\|_{L_2(\mu)}),
\$
where ${C}^{\star} := \bar{V}(2+\gamma\mathcal{U}^{\tau}_{2}) + \mathcal{U}^{\tau}_{2}(1+\gamma) + \lambda L$.
Apply the generalize version of Corollary 2 in \citep{haussler1995sphere}, which implies 
\#
\mathcal{N}\left(\epsilon, \widetilde{\mathcal{G}}, \|\cdot\|_{L_2(\mu)}\right) \leq e^3\left(D_{\Omega}+1\right)\left(D_{\mathcal{Q}}+1\right)\left(D_{\Pi}+1\right)\left(\frac{16 e (C^{\star})^2}{\epsilon}\right)^{D},
\label{cover_num_func_2}
\#
where $D = D_{\Omega}+D_{\mathcal{Q}}+D_{\Pi}$.
By empirical Bernstein inequality and a union bound, we have that with probability at least $1-\delta$ and $Z = g(s,a,r,s^{\prime})$, 
\#
& \left|\mathbb{E}_{\mu}\left[g(s,a,r,s^{\prime})\right] - \mathbb{P}_{n}\left[g(s_i,a_i,r_i,s^{\prime}_i)\right] \right| \notag \\
  \leq & \frac{1}{n}\sqrt{2 \sum_{i=1}^n Var_{\mu}\left[Z\right]\ln \frac{8\mathcal{N}\left(\epsilon, \widetilde{\mathcal{G}}, \|\cdot\|_{L_2(\mu)}\right)}{\delta}} + \frac{2 \big(\mathcal{U}^{\tau}_{\infty}\bar{V}+\lambda \|\mathbb{D}(\tau(s,a))\|^{\text{UB}}_{L_{\infty}}\big) \ln \frac{8\mathcal{N}\left(\epsilon, \widetilde{\mathcal{G}}, \|\cdot\|_{L_2(\mu)}\right)}{\delta}}{3 n} \notag \\
    \leq & \frac{1}{n}\sqrt{2 \sum_{i=1}^n \mathbb{E}_{\mu}\left[Z^2\right]\ln \frac{8\mathcal{N}\left(\epsilon, \widetilde{\mathcal{G}}, \|\cdot\|_{L_2(\mu)}\right)}{\delta}} + \frac{2 \big(\mathcal{U}^{\tau}_{\infty}\bar{V}+\lambda \|\mathbb{D}(\tau(s,a))\|^{\text{UB}}_{L_{\infty}}\big) \ln \frac{8\mathcal{N}\left(\epsilon, \widetilde{\mathcal{G}}, \|\cdot\|_{L_2(\mu)}\right)}{\delta}}{3 n} \notag \\
 \leq & \big(\mathcal{U}^{\tau}_{2}\bar{V}+\lambda \|\mathbb{D}(\tau(s,a))\|^{\text{UB}}_{L_2(\mu)}\big)\sqrt{\frac{2 \ln \frac{8\mathcal{N}\left(\epsilon, \widetilde{\mathcal{G}}, \|\cdot\|_{L_2(\mu)}\right)}{\delta}}{n}}+ \frac{2 \big(\mathcal{U}^{\tau}_{\infty}\bar{V}+\lambda\|\mathbb{D}(\tau(s,a))\|^{\text{UB}}_{L_{\infty}}\big) \ln \frac{8\mathcal{N}\left(\epsilon, \widetilde{\mathcal{G}}, \|\cdot\|_{L_2(\mu)}\right)}{\delta}}{3 n}.
    \label{decom_eq2}
\#
We set $\varepsilon = \mathcal{O}(\frac{1}{\sqrt{n}})$, and combine with the upper bound for covering number in  \eqref{cover_num_func_2} and $\|\mathbb{D}(\tau(s,a))\|^{\text{UB}}_{L_2(\mu)} < L \mathcal{U}^{\tau}_{2}$, it follows similar arguments in the proof of Lemma \ref{risk_bound_lm}, by some algera, we conclude that 
\$
& \bigg|\mathbb{E}_{\mu}\Big[\tau(s,a) \left(r+\gamma q \left(s^{\prime}, \pi\right)-q(s, a)\right) -\lambda \mathbb{D}(\tau(s,a))\Big] \notag \\
& - \mathbb{P}_{n}\Big(\tau(s_i,a_i)\left(r_i+\gamma q \left(s^{\prime}_i, \pi\right)-q(s_i, a_i)  \right)-\lambda\mathbb{D}(\tau(s_i,a_i))\Big)  \bigg|  \\
\lesssim & \big(\mathcal{U}^{\tau}_{2}\bar{V}+\lambda \|\mathbb{D}(\tau(s,a))\|^{\text{UB}}_{L_2(\mu)}\big)\sqrt{\frac{2\ln\frac{(e^{D}\max\{D_{\Omega},D_{\mathcal{Q}},D_{\Pi}\}+1)^3(L\mathcal{U}^{\tau}_{2})^{2D}}{\delta}}{n}} \notag  \\
 & \qquad \qquad + \frac{2 \big(\mathcal{U}^{\tau}_{\infty}\bar{V}+\lambda\|\mathbb{D}(\tau(s,a))\|^{\text{UB}}_{L_{\infty}}\big)\ln\frac{(e^{D}\max\{D_{\Omega},D_{\mathcal{Q}},D_{\Pi}\}+1)^3(L\mathcal{U}^{\tau}_{2})^{2D}}{\delta}}{3 n}.
\$
This completes the proof.
\end{proof}

\subsection{Proof of Lemma \ref{onlyreg_risk_bound}}

\begin{lemma}
\label{onlyreg_risk_bound}
Given an offline data $\mathcal{D}_{1:n}=\{s_{i},a_{i},r_{i},s^{\prime}_i)\}^{n}_{i=1}$, for any $\tau \in \Omega$,
\$
& \left|\mathbb{P}_{n}\mathbb{D}(\tau(s_i,a_i)) -  \mathbb{E}_{\mu}[  \mathbb{D}(\tau(s,a))] \right| \\
\leq &
\|\mathbb{D}(\tau(s,a))\|^{\text{UB}}_{L_2(\mu)}\sqrt{\frac{2\ln\frac{(e^{D_{\Omega}}(D_{\Omega}+1))(\{1 \vee  L \}\mathcal{U}^{\tau}_{2})^{2D_{\Omega}}}{\delta}}{n}} + \frac{2\|\mathbb{D}(\tau(s,a))\|^{\text{UB}}_{L_{\infty}}\ln\frac{(e^{D_{\Omega}}(D_{\Omega}+1))(\{1 \vee L \}\mathcal{U}^{\tau}_{2})^{2D_{\Omega}}}{\delta}}{3 n}.
\$
holds w.p. $\geq 1-\delta$.
\end{lemma}

\begin{proof}

We equip the space $\mathcal{G}^{\mathbb{D}}$  such that 
$g(s,a) = \mathbb{D}({\tau}(s,a))$ for any $g \in \mathcal{G}^{\mathbb{D}}$ with the $L_2$ weighted metric. To apply empirical Bernstein inequality, we study the boundedness conditions:
$
 \sup_{\tau \in \Omega}\|\tau(s,a)\|_{L_2(\mu)} \leq \mathcal{U}^{\tau}_{2},
\sup_{\tau \in \Omega}\|\tau(s,a)\|_{L_{\infty}} \leq \mathcal{U}^{\tau}_{\infty}, 
 \sup_{\tau \in \Omega}\|\mathbb{D}({\tau}(s,a))\|_{L_2(\mu)} \in [0,\|\mathbb{D}(\tau(s,a))\|^{\text{UB}}_{L_2(\mu)}], 
\sup_{\tau \in \Omega}\|\mathbb{D}({\tau}(s,a))\|_{L_{\infty}} \in [0, \|\mathbb{D}(\tau(s,a))\|^{\text{UB}}_{L_{\infty}}]
$.
By some calculation, the $L_{2}(\mu)$-distance  in $\widetilde{\mathcal{G}}$ can be bounded. For $g_1, g_2 \in \mathcal{G}^{\mathbb{D}}$ corresponding to $\tau_1$ and $\tau_2 $, respectively, we have 
\#
\|g_1(s,a) - g_2(s,a)\|_{L_2(\mu)}
= \bigg|- \mathbb{D}(\tau_1(s,a)) + \mathbb{D}({\tau}_2(s,a))\bigg|
\leq \lambda L\|\tau_1(s,a) - \tau_2(s,a)\|_{L_2(\mu)}
\label{metric_GG}
\#
where the last inequality comes from the $\mathbb{D}$ is $M$-strongly convex function and thus locally Lipschitz with a  Lipschitz constant $L \leq \infty$. \$
\mathcal{N}(4{C}^{\star}{\varepsilon}, \mathcal{G}^{\mathbb{D}}, \|\cdot\|_{L_2(\mu)}) \leq \mathcal{N}({\varepsilon}, \Omega, \|\cdot\|_{L_2(\mu)})
\$
where ${C}^{\star} := \mathcal{U}^{\tau}_{2} L$.
Apply the generalized version of Corollary 2 in \citep{haussler1995sphere}, which implies 
\#
\mathcal{N}\left(\epsilon, \mathcal{G}^{\mathbb{D}}, \|\cdot\|_{L_2(\mu)}\right) \leq e\left(D_{\Omega}+1\right)\left(\frac{4 e (C^{\star})^2}{\epsilon}\right)^{D_{\Omega}},
\label{cover_num_func_GG}
\#
By empirical Bernstein inequality and a union bound, we have that with probability at least $1-\delta$, following proof of Lemma \ref{risk_bound_lm}, by some algebra, we have
\$
& \left|\mathbb{P}_{n}\mathbb{D}(\tau(s_i,a_i)) -  \mathbb{E}_{\mu}[  \mathbb{D}(\tau(s,a))] \right| \\
\leq &
\|\mathbb{D}(\tau(s,a))\|^{\text{UB}}_{L_2(\mu)}\sqrt{\frac{2\ln\frac{(e^{D_{\Omega}}(D_{\Omega}+1))(\{1 \vee  L \}\mathcal{U}^{\tau}_{2})^{2D_{\Omega}}}{\delta}}{n}} + \frac{2\|\mathbb{D}(\tau(s,a))\|^{\text{UB}}_{L_{\infty}}\ln\frac{(e^{D_{\Omega}}(D_{\Omega}+1))(\{1 \vee L \}\mathcal{U}^{\tau}_{2})^{2D_{\Omega}}}{\delta}}{3 n}.
\$
\end{proof}

\subsection{Proof of Lemma \ref{reg_finder}}

\begin{lemma}\label{reg_finder}
Given an offline data $\mathcal{D}_{1:n}=\{s_{i},a_{i},r_{i},s^{\prime}_i)\}^{n}_{i=1}$, 
for any $\tau$ in a subset of $\Omega$, i.e., $\widetilde{\Omega}$ such that $\sup_{\tau \in \widetilde{\Omega}}\|\tau(s,a)\|_{L_2(\mu)} \leq C$ for some constant $C \geq 1$, then it suffices to ensure 
\$
 \sum^{n}_{i=1}\mathbb{D}(\tau(s_i,a_i)) 
\leq  & \frac{M(C^2-1)}{2}  +   \|\mathbb{D}(\tau(s,a))\|^{\text{UB}}_{L_2(\mu)}\sqrt{\frac{2\ln\frac{\operatorname{Vol}({\mathcal{G}^{\mathbb{D}}})}{\delta}}{n}}  + \frac{ \|\mathbb{D}(\tau(s,a))\|^{\text{UB}}_{\infty}\ln\frac{\operatorname{Vol}({\mathcal{G}^{\mathbb{D}}})}{\delta}}{3 n},
\$
where $\operatorname{Vol}(\mathcal{G}^{\mathbb{D}}) = (e^{D_{\Omega}}(D_{\Omega}+1))(\{1 \vee L \}\mathcal{U}^{\tau}_{2})^{2D_{\Omega}}$.
\end{lemma}

\begin{proof}
In this proof, we first convert the upper bound from $\sum^{n}_{i=1}\mathbb{D}(\tau(s_i,a_i))$ to $\mathbb{E}_{\mu}[\mathbb{D}(\tau(s,a))]$. In the second part, we leverage the strongly-convexity of $\mathbb{D}$ for upper bound $\|\tau(s,a)\|_{L_{2}(\mu)}$. It follows from Lemma \ref{onlyreg_risk_bound} and apply the norm triangle inequality, we have 
\#
 \mathbb{E}_{\mu}[\mathbb{D}(\tau(s,a))] 
\leq & \sum^{n}_{i=1}\mathbb{D}(\tau(s_i,a_i)) -  \|\mathbb{D}(\tau(s,a))\|^{\text{UB}}_{L_2(\mu)}\sqrt{\frac{2\ln\frac{\operatorname{Vol}(\mathcal{G}^{\mathbb{D}})}{\delta}}{n}}  - \frac{ \|\mathbb{D}(\tau(s,a))\|^{\text{UB}}_{\infty}\ln\frac{\operatorname{Vol}(\mathcal{G}^{\mathbb{D}})}{\delta}}{3 n} \notag \\
:= & \sum^{n}_{i=1}\mathbb{D}(\tau(s_i,a_i)) - \varepsilon^{\circ}_{n}.
\label{g_only_emp}
\#
According to the zero value of detection function  $\mathbb{D}(\cdot)$ and its non-negative property  functions, i.e., 
$
\mathbb{D}(1) = 0 
$
Then we immediately have 
\$
\mathbb{E}_{\mu}\left[ \mathbb{D}(\tau(s_i,a_i))\right] \leq \sum^{n}_{i=1}\mathbb{D}(\tau(s_i,a_i)) - \varepsilon_{n}
\iff 
\left|\mathbb{E}_{\mu}\left[\mathbb{D}(\tau(s_i,a_i))\right]\right| \leq \sum^{n}_{i=1}\mathbb{D}(\tau(s_i,a_i)) - \varepsilon^{\circ}_{n}.
\$
Furthermore, we have 
$
\left|\mathbb{E}_{\mu}\left[ \mathbb{D}(\tau(s_i,a_i))\right] - \mathbb{E}_{\mu}\left[ \mathbb{D}(\tau_{0}(s_i,a_i))\right] \right| \leq \sum^{n}_{i=1}\mathbb{D}(\tau(s_i,a_i)) - \varepsilon_{n},
$
where $\tau_{0}(\cdot,\cdot) \equiv 1$, such that 
$
\mathbb{E}_{\mu}\left[ \mathbb{D}(\tau_{0}(s,a))\right] = 0 
$
Motivated by the Lipschitz continuity of $\mathbb{D}(\cdot)$, we construct two target functions in order to facilitate the proof,
\$
\mathbb{D}^{\diamond}(\tau): =- \mathbb{E}_{\mu}\left[ \mathbb{D}(\tau(s,a))\right]; \widetilde{\mathbb{D}}(\tau): = \mathbb{E}_{\mu}\left[ \frac{M}{2}(\tau(s,a))^2 -  \mathbb{D}(\tau(s,a))\right].
\$
Since $\mathbb{D}(\tau)$ is $M$-strongly-convex over $\tau$, which implies that  $\widetilde{\mathbb{D}}(\tau)$ is concave, which implies that $\widetilde{\mathbb{D}}(\tau)$ is $M$-strongly-concave with respect to $\tau$ and $\sqrt{\mathbb{E}_{\mu}[(\cdot)^2]}$. Then 
\#
& \mathbb{E}_{\mu}\left[(\tau_0(s,a) - \tau(s,a))^2\right] \leq \frac{2(\mathbb{D}^{\diamond}(\tau)-\mathbb{D}^{\diamond}(\tau_0))}{M} \notag \\ 
\implies & \mathbb{E}_{\mu}\left[(\tau_0(s,a) - \tau(s,a))^2\right] \leq \frac{2(\mathbb{E}_{\mu}\left[\mathbb{D}(\tau(s,a))\right]-\mathbb{E}_{\mu}\left[ \mathbb{D}(\tau_0(s,a))\right])}{M}.
\label{diver_concave}
\#
According to the definition of $\tau_0$, by some algebra, we have \eqref{diver_concave} is equivalent to 
\#
\mathbb{E}_{\mu}\left[(1 - \tau(s,a))^2\right] \leq  \frac{2\mathbb{E}_{\mu}\left[\mathbb{D}(\tau(s,a))\right]}{M} 
\implies  \|\tau(s,a)\|^2_{L_2(\mu)}   \leq  \frac{2\mathbb{E}_{\mu}\left[ \mathbb{D}(\tau(s,a))\right]+M}{M}.
\label{g_lip}
\#
According to \eqref{g_only_emp}, then we have 
$
 \|\tau(s,a)\|^2_{L_2(\mu)}  \leq 
(2(\sum^{n}_{i=1}\mathbb{D}(\tau(s_i,a_i)) -  \varepsilon^{\circ}_{n})+M)/M 
$.
By some algebra, where we solve for $C= \|\tau(s,a)\|^2_{L_2(\mu)}$, then we conclude that 
\$
\sum^{n}_{i=1}\mathbb{D}(\tau(s_i,a_i)) \leq  &\frac{M(C^2-1)}{2}  +   \|\mathbb{D}(\tau(s,a))\|^{\text{UB}}_{L_2(\mu)}\sqrt{\frac{2\ln\frac{\operatorname{Vol}({\mathcal{G}^{\mathbb{D}}})}{\delta}}{n}}  + \frac{ \|\mathbb{D}(\tau(s,a))\|^{\text{UB}}_{\infty}\ln\frac{\operatorname{Vol}({\mathcal{G}^{\mathbb{D}}})}{\delta}}{3 n}.
\$
This completes the proof. 
\end{proof}

\subsection{Proof of Theorem 4.1}

\begin{proof}
In this proof, we aim to bound the regret 
$
J(\pi) - J(\widehat{\pi})
$ for $\widehat{\pi}$ is return from (6) in maintext. First, recall that we have a consistent confident set of value estimates as 
\#
& \mathcal{Q}_{\varepsilon_{n}} = \big\{q \in \mathcal{Q}:  \sup_{\tau \in \widetilde{\Omega}_{\widetilde{\sigma}_{n}}}\big|n^{-1}\sum^{n}_{i=1}\tau(s_i,a_i)(r_i +\gamma q^{\pi}(s^{\prime}_i,\pi)-q^{\pi}(s_i,a_i))\big| \leq \varepsilon_{n} \big\},
\label{prime_perturbed}
\#
with the uncertainty control on important-weight class 
\#
\widetilde{\Omega}_{\widetilde{\sigma}_{n}} =    \left\{\tau_{\circ}/\sup_{\tau_{\circ} \in \Omega}\|\tau_{\circ}\|_{\Omega} \; \text{for} \; \tau_{\circ} \in \Omega: \xi_n(\mathbb{D},\tau_{\circ})) \leq \widetilde{\sigma}_{n}  \right\}. 
\label{omega_cond}
\#
We can rewrite this confidence set $\mathcal{Q}_{\varepsilon_{n}}$ as 
\$
\mathcal{Q}_{\varepsilon_{n}} = \left\{  \sup_{\tau \in \widetilde{\Omega}_{\widetilde{\sigma}_{n}}}\big|n^{-1}\sum^{n}_{i=1}\tau(s_i,a_i)(r_i +\gamma q^{\pi}(s^{\prime}_i,\pi)-q^{\pi}(s_i,a_i))\big| \leq \varepsilon_{n}, \; \forall \ \tau \in \widetilde{\Omega}_{\widetilde{\sigma}_{n}} \right\}. 
\$
Now, for any fixed policy $\pi \in \Pi$ and $\varepsilon_{n},\widetilde{\sigma}_{n}$, we define the maximizer and minimizer in $\mathcal{Q}_{\varepsilon_{n}}$ as 
$ \overline{q^{\pi}}$ and $\underline{q^{\pi}}$, i.e., the maximizer $\overline{q^{\pi}}:= \argmax_{q\in \mathcal{Q}_{\varepsilon_n}}q(s^{0},\pi)$ and the minimizer $\underline{q^{\pi}}:= \argmin_{q\in \mathcal{Q}_{\varepsilon_n}}q(s^{0},\pi)$ over the confidence set $\mathcal{Q}_{\varepsilon_n}$, so that the follow inequalites hold, for any $\tau \in \widetilde{\Omega}_{\widetilde{\sigma}_{n}}$,
\#
&\left| \frac{1}{n}\sum^{n}_{i=1} \frac{\tau(s_i,a_i)\left(r_i +\gamma  \overline{q^{\pi}}(s^{\prime}_i,\pi)- \overline{q^{\pi}}(s_i,a_i)\right)}{1-\gamma}\right| \leq \varepsilon_{n} \label{q_max_bound}\\
&\left| \frac{1}{n}\sum^{n}_{i=1} \frac{\tau(s_i,a_i)\left(r_i +\gamma \underline{q^{\pi}}(s^{\prime}_i,\pi)-\underline{q^{\pi}}(s_i,a_i)\right)}{1-\gamma}\right| \leq \varepsilon_{n}. \label{q_min_bound} 
\#
In addition, it is obvious that, for any $\lambda>0$ and $\tau \in \widetilde{\Omega}_{\widetilde{\sigma}_{n}}$ and $q \in \mathcal{Q}_{\varepsilon_{n}}$, 
\#
& \left|\frac{1}{n}\sum^{n}_{i=1} \frac{\tau(s_i,a_i)\left(r_i +\gamma q(s^{\prime}_i,\pi)-q(s_i,a_i)\right)}{1-\gamma} - \lambda\frac{1}{n}\sum^{n}_{i=1}\frac{ \mathbb{D}(\tau(s_i,a_i))}{1-\gamma} \right| \notag \\
\leq& \left|\frac{1}{n}\sum^{n}_{i=1} \frac{\tau(s_i,a_i)\left(r_i +\gamma q(s^{\prime}_i,\pi)-q(s_i,a_i)\right)}{1-\gamma}\right| +  \lambda\left|\frac{1}{n}\sum^{n}_{i=1}\frac{ \mathbb{D}(\tau(s_i,a_i))}{1-\gamma} \right| \leq  \varepsilon_n + \lambda \widetilde{\sigma}_{n}.
\label{reg_bound_spec}
\#
where the last inequality comes from 
the conditions \eqref{prime_perturbed} and \eqref{omega_cond}. According to the definition of the discounted return, $J(\pi) = q^{\pi}(s^0,\pi)$ for any $\pi \in \Pi$, then we have 
\$
J(\pi) - J(\widehat{\pi}) &= J(\pi) - q^{\pi}(s^0, \widehat{\pi}) \leq J(\pi) - \underline{q^{\pi}}(s^0, \widehat{\pi}) \leq J(\pi) - \underline{q^{\pi}}(s^0, \pi).
\$
where the second equality from $\underline{q^{\pi}}(s^0, \widehat{\pi})$ is the lower bound of $q^{\pi}(s^0, \widehat{\pi})$ for $\widehat{\pi} \in \Pi$, and the last inequality comes from $\widehat{\pi}$ is the maximizer with respect to pessimistic value estimate. According the evaluation error Lemma \ref{eval_error}, and note that $ \overline{q^{\pi}}(s^0, \pi)$ is the upper bound of $q^{\pi}(s^0, \pi)$, thus we have 
\$
& J(\pi) - J(\widehat{\pi}) \\
=&  J(\pi) - \underline{q^{\pi}}(s^0, \pi) \\
\leq & \overline{q^{\pi}}(s^0, \pi)  - \underline{q^{\pi}}(s^0, \pi) \\
= & { \overline{q^{\pi}}}(s^0, \pi)  - J(\pi) + J(\pi) -  \underline{q^{\pi}}(s^0, \pi) \\
 = & \underline{q^{\pi}}(s^0, \pi)  - \left(\underline{q^{\pi}}(s^0, \pi) +\frac{\mathbb{E}_{d^\pi}\left[r+\gamma  \overline{q^{\pi}}\left(s^{\prime}, \pi\right)- \overline{q^{\pi}}(s, a)\right]}{1-\gamma}\right)  \\
& + \left(\underline{q^{\pi}}(s^0, \pi) +  \frac{\mathbb{E}_{d^\pi}\left[r+\gamma \underline{q^{\pi}} \left(s^{\prime}, \pi\right)-\underline{q^{\pi}}(s, a)\right]}{1-\gamma}\right) -  \underline{q^{\pi}}(s^0, \pi) \\
= & \frac{\mathbb{E}_{d^\pi}[r(s,a)+\gamma \underline{q^{\pi}} \left(s^{\prime}, \pi\right)-\underline{q^{\pi}}(s, a)]}{1-\gamma} -  \frac{\mathbb{E}_{d^\pi}[r(s,a)+\gamma  \overline{q^{\pi}}\left(s^{\prime}, \pi\right)- \overline{q^{\pi}}(s, a)]}{1-\gamma} \\
 = & \frac{\mathbb{E}_{d^\pi}\left[\left(r(s,a)+\gamma \underline{q^{\pi}} \left(s^{\prime}, \pi\right)-\underline{q^{\pi}}(s, a)\right)   - \left(r(s,a)+\gamma  \overline{q^{\pi}}\left(s^{\prime}, \pi\right)- \overline{q^{\pi}}(s, a)\right)\right]}{1-\gamma} \\
 := & \frac{\mathbb{E}_{d^\pi}\left[ \Delta( \overline{q^{\pi}},\pi)  - \Delta(\underline{q^{\pi}},\pi)\right]}{1-\gamma},
\$
where we use the notation $\Delta(q,\pi) = q(s, a)-r(s,a)-\gamma q (s^{\prime}, \pi)$. Based on this, it is sufficient to bound the $
\frac{\mathbb{E}_{d^\pi}\left[ \Delta( \overline{q^{\pi}},\pi)  - \Delta(\underline{q^{\pi}},\pi)\right]}{1-\gamma}
$ 
in order to bound the regret $J(\pi) - J(\widehat{\pi})$. To proceed the proof, we define admissible implicit exploratory distribution as $\rho$ that satisfies the condition on uncertainty control (7) in maintext, i.e.,  
$
\left\|\frac{\rho(s,a)}{\mu(s,a)}\right\|_{L_2(\mu)} := \mathcal{U}^{\star}_{2} \leq \sup_{\tau \in \widetilde{\Omega}_{\widetilde{\sigma}_{n}}}{\|\tau(s,a)\|_{L_2(\mu)}} := \mathcal{U}^{\star}_{2}
$. 
With this implicit exploratory distribution, we can decompose the regret error over $\rho$ as,
\#
& \frac{\mathbb{E}_{d^\pi}\left[ \Delta(\overline{q^{\pi}},\pi)  - \Delta(\underline{q^{\pi}},\pi)\right]}{1-\gamma} 
=  \underbrace{\frac{\mathbb{E}_{\mu}\left[\left(\frac{\rho(s,a)}{\mu(s,a)} - \tau^{\widetilde{\Omega}_{\widetilde{\sigma}_n} }_{\rho/\mu}(s,a)\right) \left(\Delta(\overline{q^{\pi}},\pi) -\Delta(\underline{q^{\pi}},\pi)\right)\right]}{1-\gamma}}_{\text{err}_1} \notag \\
& 
+ \underbrace{\frac{\mathbb{E}_{\mu}\left[\tau^{\widetilde{\Omega}_{\widetilde{\sigma}_n} }_{\rho/\mu}(s,a) \left(\Delta(\overline{q^{\pi}},\pi) -\Delta(\underline{q^{\pi}},\pi) \right)\right] }{1-\gamma}}_{\text{err}_2} +  \underbrace{\frac{\mathbb{E}_{d^\pi}\left[\Delta(\overline{q^{\pi}},\pi) -\Delta(\underline{q^{\pi}},\pi) \right] - \mathbb{E}_{\rho}\left[\Delta(\overline{q^{\pi}},\pi) -\Delta(\underline{q^{\pi}},\pi) \right]}{1-\gamma}  }_{\text{err}_3}.
\label{total_bound_alpha1}
\#
where $\tau^{\widetilde{\Omega}_{\widetilde{\sigma}_n} }_{\rho/\mu}$ is the importance-weight estimator that $\tau^{\widetilde{\Omega}_{\widetilde{\sigma}_n} }_{\rho/\mu} \in \widetilde{\Omega}_{\widetilde{\sigma}_n}$, with the definition as 
\$
\tau^{\widetilde{\Omega}_{\widetilde{\sigma}_n} }_{\rho/\mu}(s,a) := \argmin_{\tau \in \text{lr-hull}(\widetilde{\Omega}_{\widetilde{\sigma}_n} )} \mathbb{E}_{\mu}\left[\left(\frac{\rho(s,a)}{\mu(s,a)} - \tau(s,a)\right)\left(\Delta(\overline{q^{\pi}},\pi) -\Delta(\underline{q^{\pi}},\pi)\right)\right],
\$
where $\text{lr-hull}(\widetilde{\Omega}_{\widetilde{\sigma}_n})$ is the linear hull of the function class $\widetilde{\Omega}_{\widetilde{\sigma}_n}$. Note that, we use the linear hull to enhance the expressivity of the function class over $\tau$ and more robustness to the function approximation error. With the above error decomposition, we bound the major three terms $\text{err}_1$, $\text{err}_2$ and $\text{err}_3$ subsequently. 

\noindent \textbf{Bounding $\text{err}_1$.} Intuitively, the term $\text{err}_1$ is introduced by the function approximation error. Due to the construction of uncertainty control class $\widetilde{\Omega}_{\widetilde{\sigma}_{n}}$, the function approximation is well-controlled, almost cannot be detected under small importance-weight class, i.e., $\widetilde{\sigma}_{n}$ is small. We explain this in the following. According to Cauchy–Schwarz inequality, 
\$
\text{err}_1 = &
\mathbb{E}_{\mu}\bigg[\left(\frac{\rho(s,a)}{\mu(s,a)} - \tau^{\widetilde{\Omega}_{\widetilde{\sigma}_n} }_{\rho/\mu}(s,a)\right) \bigg[\left(\Delta(\overline{q^{\pi}},\pi) -\Delta(\underline{q^{\pi}},\pi)\right)
\bigg]\bigg] \\
\leq &
\mathbb{E}_{\mu}\bigg[\left|\frac{\rho(s,a)}{\mu(s,a)} - \tau^{\widetilde{\Omega}_{\widetilde{\sigma}_n} }_{\rho/\mu}(s,a)\right|\bigg|\left(\Delta(\overline{q^{\pi}},\pi) -\Delta(\underline{q^{\pi}},\pi)\right)
\bigg|\bigg] \\
\leq & \sqrt{\mathbb{E}_{\mu}\left[\left(\frac{\rho(s,a)}{\mu(s,a)} - \tau^{\widetilde{\Omega}_{\widetilde{\sigma}_n}}_{\rho/\mu}(s,a)\right)^2\right]}  \cdot \sqrt{\mathbb{E}_{\mu}\left[\left\{\left(\Delta(\overline{q^{\pi}},\pi) -\Delta(\underline{q^{\pi}},\pi)\right)
\right\}^2\right]} \\
= &  \underbrace{\|\frac{\rho(s,a)}{\mu(s,a)} - \tau^{\widetilde{\Omega}_{\widetilde{\sigma}_n}}_{\rho/\mu}(s,a)\|_{L_2(\mu)}}_{\text{err}_{11}}  \cdot \underbrace{\|\Delta(\overline{q^{\pi}},\pi) -\Delta(\underline{q^{\pi}},\pi)
\|_{L_2(\mu)}}_{\text{err}_{12}}.
\$
On this point, it suffices to bound the terms $\text{err}_{11}$ and $\text{err}_{12}$. 

\noindent \textbf{Bounding $\text{err}_{11}$.} It follows the definition of $\rho$, it observes that $
\left\|\frac{\rho(s,a)}{\mu(s,a)}\right\|_{L_2(\mu)}  = \mathcal{U}^{\star}_{2} \leq \mathcal{U}^{\star}_{2}$. Also, 
since the importance-weight estimator $\tau^{\widetilde{\Omega}_{\widetilde{\sigma}_n}}_{\rho/\mu}(s,a) \in \widetilde{\Omega}_{\widetilde{\sigma}_n}$, so that 
$
\left\|\tau^{\widetilde{\Omega}_{\widetilde{\sigma}_n}}_{\rho/\mu}(s,a)\right\|_{L_2(\mu)} \leq \mathcal{U}^{\star}_{2}
$. 
Due to the non-negativity of $\rho(s,a), \mu(s,a)$ and $\tau^{\widetilde{\Omega}_{\widetilde{\sigma}_n}}_{\rho/\mu}(s,a)$ for any $s,a$ over the support of $\mu$, we can obtain the upper bound
\#
\mathbb{E}_{\mu}\left[\left(\frac{\rho(s,a)}{\mu(s,a)} - \tau^{\widetilde{\Omega}_{\widetilde{\sigma}_n}}_{\rho/\mu}(s,a)\right)^2\right] \leq & 
\min\left\{\mathbb{E}_{\mu}\left[\left(\frac{\rho(s,a)}{\mu(s,a)}\right)^2\right], \mathbb{E}_{\mu}\left[\left(\tau^{\widetilde{\Omega}_{\widetilde{\sigma}_n}}_{\rho/\mu}(s,a)\right)^2\right]\right\} \notag \\
\leq & \min \{ (\mathcal{U}^{\circ}_{2})^2,  \mathcal{U}^2_{\text{prime}, 2} \}. 
\label{mis_term2}
\#
\noindent \textbf{Bounding $\text{err}_{12}$.}
It follows from the norm triangle inequality, we have
\$
\text{err}_{12} = & \|\Delta(\overline{q^{\pi}},\pi) -\Delta(\underline{q^{\pi}},\pi)
\|_{L_2(\mu)} \leq \|\Delta(\overline{q^{\pi}},\pi)\|_{L_2(\mu)} + \|\Delta(\underline{q^{\pi}},\pi)
\|_{L_2(\mu)}.
\$
According to \eqref{q_max_bound} and \eqref{q_min_bound}, and it follows from 
 Lemma \ref{risk_bound_lm} over 
the product space 
$\mathcal{G}_{\varepsilon_n, \widetilde{\sigma}_n} := \widetilde{\Omega}_{\widetilde{\sigma}_n} \times \mathcal{Q}_{\varepsilon_n} \times \Pi$. Accordingly, for $q \in \mathcal{Q}_{\varepsilon_n}$ and $\pi \in \Pi$, we have 
$
\sup_{\tau \in \widetilde{\Omega}_{\widetilde{\sigma}_n}}\mathbb{E}_{\mu}[\tau(s,a)\Delta(q,\pi)] \lesssim  \varepsilon^{1}_n
$ for 
\$
\varepsilon^{1}_n = \mathcal{U}^{\star}_{2}\sqrt{\frac{\bar{V}^2\ln \frac{8\mathcal{N}\left(\epsilon, \mathcal{G}_{\varepsilon_n, \widetilde{\sigma}_n}, \|\cdot\|_{L_2(\mu)}\right)}{\delta}}{n}} + \frac{\mathcal{U}^{\star}_{\infty} \bar{V} \ln \frac{8\mathcal{N}\left(\epsilon, \mathcal{G}_{\varepsilon_n, \widetilde{\sigma}_n}, \|\cdot\|_{L_2(\mu)}\right)}{\delta}}{3 n} + \varepsilon_n, 
\$
The above inequality also holds for $\overline{q^{\pi}}$ and $\underline{q^{\pi}}$. Then, as $\|\tau(s,a)\|_{L_{2}(\mu)} \leq \mathcal{U}_{\text{prime} ,2}$ for $\tau \in \widetilde{\Omega}_{\widetilde{\sigma}_{n}}$,
it follows from Lemma \ref{L_2MMD}, we have 
$
 \|\Delta(\overline{q^{\pi}},\pi)\|_{L_2(\mu)} + \|\Delta(\underline{q^{\pi}},\pi)
\|_{L_2(\mu)} \leq 2 \varepsilon^{1}_n/ \mathcal{U}^{\star}_{2}\ 
$,
which implies 
\#
\mathcal{U}^{\star}_{2}\text{err}_{12} \lesssim  \varepsilon^{1}_n.
\label{err12_main_reg}
\#
Combine with the bounds in \eqref{err12_main_reg}
and \eqref{mis_term2}, we conclude the upper bound for $\text{err}_1$: 
\#
\text{err}_1 \leq \frac{1}{(1-\gamma)}\left\{
\mathcal{U}_{\star, 2}\sqrt{\frac{\bar{V}^2\ln \frac{8\mathcal{N}\left(\epsilon, \mathcal{G}_{\varepsilon_n, \widetilde{\sigma}_n}, \|\cdot\|_{L_2(\mu)}\right)}{\delta}}{n}} + \frac{\mathcal{U}^{\star}_{\infty} \bar{V} \ln \frac{8\mathcal{N}\left(\epsilon, \mathcal{G}_{\varepsilon_n, \widetilde{\sigma}_n}, \|\cdot\|_{L_2(\mu)}\right)}{\delta}}{3 n} + \varepsilon_n\right\}.
\label{misspecified_end}
\#

\noindent\textbf{Bounding $\text{err}_2$.} It first observes that 
 $\tau^{\widetilde{\Omega}_{\widetilde{\sigma}_n}}_{\rho/\mu} \in \text{lr-hull}(\widetilde{\Omega}_{\widetilde{\sigma}_n})$, and 
$
(1-\gamma)\text{err}_2 \leq 
\sup_{\tau \in \text{lr-hull}(\widetilde{\Omega}_{\widetilde{\sigma}_n})}\mathbb{E}_{\mu}\left[\tau(s,a) \left(\Delta(\overline{q^{\pi}},\pi) -\Delta(\underline{q^{\pi}},\pi) \right)\right] 
$. Before we proceed to bound, we first shows the equivalence between the $\tau \in \Omega$ and $\tau \in \text{lr-hull}(\Omega)$ when measuring the statistical complexity for any linear functional $h_{\text{linear}}(\cdot)$ with respect to $\tau$. That is   
$
\sup_{\tau \in \text{lr-hull}(\Omega)}|h_{\text{linear}}(\cdot)|=\sup_{\tau \in \Omega}|h_{\text{linear}}(\cdot)|
$.
Let's consider any $\tau^{\dag} \in \text{lr-hull}(\Omega)$, i.e., $\tau^{\dag} = \sum_{i}\beta_i\tau_i$, where $\tau_{i}\in \Omega$ for any $i$ and $\sum_{i}|\beta_i|=1$. For any $h_{\text{linear}}(\cdot)$ and any $\tau^{\dag} \in \text{lr-hull}(\Omega)$, we have that
\#
|h_{\text{linear}}(\tau^{\dag})|= \left|h\left(\sum_i \beta_i \tau_i\right)\right|=& \left|\sum_i \beta_i h\left(\tau_i\right)\right| \leq  \sum_i\left|\beta_i\right|\left|h\left(\tau_i\right)\right| \leq   \sup _{\tau \in \Omega} \left|h(\tau)\right|.
\label{eq:lr-hull}
\#
As \eqref{eq:lr-hull} holds for any $\tau^{\dag} \in \text{lr-hull}(\Omega)$. 
Take maximum over $\tau^{\dag} \in \text{lr-hull}(\Omega)$ on the LHS, we have 
\#
\sup_{\tau^{\dag} \in \text{lr-hull}(\Omega)} \left|h_{\text{linear}}(\tau^{\dag})\right| \leq \sup_{\tau \in \Omega}\left|h(\tau)\right|. 
\label{eq:front}
\#
On the other side, as $\Omega \subset \text{lr-hull}(\Omega)$, it is easy to observe that 
\#
\sup _{\tau^{\dag} \in \text{lr-hull}(\Omega)} \left|h_{\text{linear}}(\tau^{\dag})\right| \geq \sup_{\tau \in \Omega}\left|h_{\text{linear}}(\tau)\right|. 
\label{eq:back}
\#
Combine \eqref{eq:front} and \eqref{eq:back}, we conclude for $h_{\text{linear}}(\cdot)$,
$
\sup _{\tau^{\dag} \in \text{lr-hull}(\Omega)} \left|h_{\text{linear}}(\tau^{\dag})\right| = \sup_{\tau \in \Omega}\left|h_{\text{linear}}(\tau)\right|.
$
Note that 
$
\frac{1}{1-\gamma}\mathbb{E}_{\mu}\left[\tau(s,a) \left[\Delta(\underline{q^{\pi}},\pi) -\Delta(\overline{q^{\pi}},\pi) \right]\right] 
$
is linear in $\tau$ due to $\tau$ is the weights over average Bellman error $\Delta(\cdot)$ which enjoys the linearity, and belongs to $h_{\text{linear}}(\cdot)$, therefore the above derivation can be applied. According to the equivalence between $\Omega$ and $\text{lr-hull}(\Omega)$, to quantify the statistical error is sufficient to bound
\$
 &\sup_{\tau \in \text{lr-hull}(\widetilde{\Omega}_{\widetilde{\sigma}_n})}\mathbb{E}_{\mu}\left[\tau(s,a) \Delta(\overline{q^{\pi}},\pi)\right]
-\inf_{\tau \in \text{lr-hull}(\widetilde{\Omega}_{\widetilde{\sigma}_n})}\mathbb{E}_{\mu}\left[\tau(s,a) \Delta(\underline{q^{\pi}},\pi)\right] \\
= &\underbrace{\sup_{\tau \in \widetilde{\Omega}_{\widetilde{\sigma}_n}}\mathbb{E}_{\mu}\left[\tau(s,a) \Delta(\overline{q^{\pi}},\pi)\right]}_{\text{err}_{21}}
-\underbrace{\inf_{\tau \in \widetilde{\Omega}_{\widetilde{\sigma}_n}}\mathbb{E}_{\mu}\left[\tau(s,a) \Delta(\underline{q^{\pi}},\pi)\right]}_{\text{err}_{22}}.
\$
Let $\tau_{\max}$ and $\tau_{\min}$ be the optimizer of $\text{err}_{21}$ and $\text{err}_{22}$, respectively. And we define an auxiliary objective function, which leverages the convexity of $\mathbb{D}(\cdot)$ over $\tau$. That is, 
\$
\mathbb{P}_n\mathcal{L}_{\mathbb{D}}(\tau,q) :=
\frac{1}{n}\sum^{n}_{i=1}\left[\tau(s_i,a_i)\left(r_i+\gamma q\left(s^{\prime}_i, \pi\right)-q(s_i, a_i)\right) -  \lambda \mathbb{D}(\tau(s_i,a_i))\right], 
\$
for $\tau \in \widetilde{\Omega}_{\widetilde{\sigma}_n}$, $q \in \mathcal{Q}_{\varepsilon_n}$. We make the decomposition:
\$
\text{err}_{21} - \text{err}_{22}  
= & \underbrace{\mathbb{E}_{\mu}\left[\tau_{\max}(s,a) \Delta(\overline{q^{\pi}},\pi)\right] - \mathbb{E}_{\mu}\left[\tau_{\min}(s,a) \Delta(\overline{q^{\pi}},\pi)\right]}_{{\text{err}}_{21}}\\
&+  \underbrace{\mathbb{E}_{\mu}\left[\tau_{\min}(s,a) \Delta(\overline{q^{\pi}},\pi)\right] - \mathbb{E}_{\mu}\left[\tau_{\min}(s,a) \Delta(\underline{q^{\pi}},\pi)\right]}_{{\text{err}}_{22}}.
\$
And therefore, to bound $\text{err}_{21} - \text{err}_{22}$, we are sufficient to bound ${\text{err}}_{21}$ and ${\text{err}}_{22}$. 

\textbf{Bounding ${\text{err}}_{21}$.} It follows from Cauchy-Schwarz inequality, then we have 
\#
{\text{err}}_{21} \leq 
\sqrt{ \underbrace{\mathbb{E}_{\mu}\left[(\tau_{\max}(s,a) - \tau_{\min}(s,a))^2\right]}_{{\text{err}}_{211}}
\underbrace{\mathbb{E}_{\mu}\left[(\Delta(\overline{q^{\pi}},\pi))^2\right]}_{{\text{err}}_{212}}}.
\label{cauch_4}
\#

\textbf{Bounding ${\text{err}}_{211}$.} Recall 
$
\mathcal{L}_{\mathbb{D}}(\tau,q) := 
\mathbb{E}_{\mu}\left[\tau(s,a)\left(r+\gamma q \left(s^{\prime}, \pi\right)-q(s, a)\right) -  \lambda\mathbb{D}(\tau(s,a))\right] 
$.
For fixed $q \in \mathcal{Q}_{\varepsilon_n}$ and $\tau \in \widetilde{\Omega}_{\widetilde{\sigma}_n}$, we show  $\mathcal{L}_{\mathbb{D}}(\tau,q)$, is $\lambda M$-strongly concave with respect to $\tau$ and $\mathbb{E}_{\mu}[(\cdot)^2]$. Let us consider an counterpart for $\mathcal{L}_{\mathbb{D}}(\tau,q)$, i.e., 
$
\mathcal{L}^{\circ}_{\mathbb{D}}(\tau,q)  := \mathcal{L}_{\mathbb{D}}(\tau,q) + \frac{\lambda M}{2}\mathbb{E}_{\mu}\left[(\tau(s,a))^2\right],
$
so that we have 
$
\mathcal{L}^{\circ}_{\mathbb{D}}(\tau,q) = \mathbb{E}_{\mu}\left[\tau(s,a)\left(r+\gamma q \left(s^{\prime}, \pi\right)-q(s, a)\right) -  \lambda[ \mathbb{D}(\tau(s,a) ) - \frac{M}{2}(\tau(s,a))^2]\right] 
$.
Since $\mathbb{D}(\cdot)$ is $M$-strongly convex with respect to $\tau$, so $\mathcal{L}^{\circ}_{\mathbb{D}}(\tau,q)$ is concave, which implies that $\mathcal{L}_{\mathbb{D}}(\tau,q)$ is $\lambda M$-strongly-concave with respect to $\tau$ and $\mathbb{E}_{\mu}[(\cdot)^2]$. It follows from the strongly-concavity, and plug-in $\tau_{\max}$, $\tau_{\min}$ and $\overline{q^{\pi}}$, 
\#
\mathbb{E}_{\mu}\left[(\tau_{\max}(s,a) - \tau_{\min}(s,a))^2\right] \leq \frac{2(\mathcal{L}_{\mathbb{D}}(\tau_{\max},\overline{q^{\pi}})-\mathcal{L}_{\mathbb{D}}(\tau^{\circ}_{\min},\overline{q^{\pi}}))}{\lambda M}.
\label{concave_tau}
\#
This implies it is sufficient to bound
$
\mathcal{L}_{\mathbb{D}}(\tau_{\max},\overline{q^{\pi}})-\mathcal{L}_{\mathbb{D}}(\tau_{\min},\overline{q^{\pi}})
$ for bound ${\text{err}}_{211}$. 
\$
 \mathcal{L}_{\mathbb{D}}(\tau_{\max},\overline{q^{\pi}})-\mathcal{L}_{\mathbb{D}}(\tau_{\min},\overline{q^{\pi}}) =  & \underbrace{\mathcal{L}_{\mathbb{D}}(\tau_{\max},\overline{q^{\pi}}) - \mathbb{P}_n\mathcal{L}_{\mathbb{D}}(\tau_{\max},\overline{q^{\pi}})}_{{\text{err}}_{2111}} 
+ \underbrace{\mathbb{P}_n\mathcal{L}_{\mathbb{D}}(\tau_{\max},\overline{q^{\pi}}) - \mathbb{P}_n\mathcal{L}_{\mathbb{D}}(\tau_{\min},\overline{q^{\pi}})}_{{\text{err}}_{2112}} \\
&+ \underbrace{\mathbb{P}_n\mathcal{L}_{\mathbb{D}}(\tau_{\min},\overline{q^{\pi}}) - \mathcal{L}_{\mathbb{D}}(\tau_{\min},\overline{q^{\pi}})}_{{\text{err}}_{2113}}.
\$
It follows from Lemma \ref{risk_bound_alpha}, with the defintion on the boundedness of $\mathbb{D}$ class terms, i.e., $\|\mathbb{D}(\tau(s,a))\|^{\text{prime}}_{L_2(\mu)} = \sup_{\tau \in \widetilde{\Omega}_{\widetilde{\sigma}_n}}\|\mathbb{D}({\tau}(s,a))\|_{L_2(\mu)}$ and $\|\mathbb{D}(\tau(s,a))\|^{\text{prime}}_{L_{\infty}} =  \sup_{\tau \in \widetilde{\Omega}_{\widetilde{\sigma}_n}}\|\mathbb{D}({\tau}(s,a))\|_{L_{\infty}}$, the terms ${\text{err}}_{2111}$ and ${\text{err}}_{2112}$ is upper bounded by 
\$
\varepsilon_2 :=& \big(\mathcal{U}^{\star}_{2}\bar{V}+\lambda \|\mathbb{D}(\tau(s,a))\|^{\text{prime}}_{L_2(\mu)}\big)\sqrt{\frac{2\ln\frac{(e^{D}\max\{D_{\Omega},D_{\mathcal{Q}},D_{\Pi}\}+1)^3(L\mathcal{U}^{\tau}_{2})^{2D}}{\delta}}{n}} \notag  \\
 &  + \frac{2 \big(\mathcal{U}_{\text{prime},
    \infty}\bar{V}+\lambda\|\mathbb{D}(\tau(s,a))\|^{\text{prime}}_{\infty}\big)\ln\frac{(e^{D}\max\{D_{\Omega},D_{\mathcal{Q}},D_{\Pi}\}+1)^3(L\mathcal{U}^{\tau}_{2})^{2D}}{\delta}}{3 n},
\$
According to \eqref{reg_bound_spec}, as $\overline{q^{\pi}} \in \mathcal{Q}_{\varepsilon_n}$ and $\tau \in \widetilde{\Omega}_{\widetilde{\sigma}_n}$, thus $
{\text{err}}_{2112} \leq 2(\varepsilon_n + \lambda \widetilde{\sigma}_n)$. Therefore, combine with $ {\text{err}}_{2112}$, we conclude that 
\#
 \mathbb{E}_{\mu}\left[(\tau_{\max}(s,a) - \tau_{\min}(s,a))^2\right] 
\leq \frac{2}{\lambda M}\bigg(\varepsilon_{2} + \varepsilon_n + \lambda \widetilde{\sigma}_n\bigg) := \frac{2\varepsilon_{3}}{\lambda M}.
\label{tau_dist_l2}
\#
\textbf{Bounding ${\text{err}}_{212}$.} First, it observes that the density ratio class $\widetilde{\Omega}_{\widetilde{\sigma}_n}$ is upper bounded with respect to weighted $L_{2}(\mu)$ norm, i.e., 
$\sup_{\tau \in \widetilde{\Omega}_{\widetilde{\sigma}_n}}\|\tau(s,a)\|_{L_2(\mu)} \leq \mathcal{U}^{\star}_{2}$. It follows from Lemma \ref{L_2MMD}, we have 
$
\mathcal{U}^{\star}_{2}\| \Delta(\overline{q^{\pi}},\pi)\|_{L_2(\mu)} = \sup_{\tau \in \widetilde{\Omega}_{\widetilde{\sigma}_n}}\left|\mathbb{E}_{\mu}\left[\tau(s,a) \Delta(\overline{q^{\pi}},\pi) \right]\right|  \leq \varepsilon_n
$,
where the last inequality comes from $\overline{q^{\pi}} \in \mathcal{Q}_{\varepsilon_n}$ and  \eqref{prime_perturbed}. Therefore, we conclude that 
\#
\sqrt{{\text{err}}_{212}} \leq \frac{1}{\mathcal{U}^{\star}_{2}}\varepsilon_n. 
\label{tmp}
\#
It combines with \eqref{tau_dist_l2}, \eqref{cauch_4} and \eqref{tmp}, we have
\#
{\text{err}}_{21} \leq \frac{\varepsilon_1}{\mathcal{U}^{\star}_{2}}\sqrt{\frac{2\varepsilon_{3}}{\lambda M}}.
\label{err21bound}
\#

\textbf{Bounding ${\text{err}}_{22}$.} We first observe that 
 $
 \mathbb{E}_{\mu}\left[\tau_{\min}(s,a) \Delta(\overline{q^{\pi}},\pi)\right] - \mathbb{E}_{\mu}\left[\tau_{\min}(s,a) \Delta(\underline{q^{\pi}},\pi)\right] =  \mathbb{E}_{\mu}\left[\tau_{\min}(s,a)\left( \Delta(\overline{q^{\pi}},\pi) - \Delta(\underline{q^{\pi}},\pi)\right)\right]. 
 $
 Then it follows from Lemma \ref{risk_bound_alpha_2q} and I\eqref{prime_perturbed}, with the norm triangle inequality, we can conclude that 
\#
 {\text{err}}_{22}
\leq \varepsilon^{1}_{n}.
\label{err22bound}
\#
Now we summarize the bound for $\text{err}_2$ throught combining the upper bounds on  \eqref{err21bound} and \eqref{err22bound}, we have 
\#
\text{err}_2 \leq \frac{1}{1-\gamma}\bigg(& \mathcal{U}^{\star}_{2}\sqrt{\frac{32 \bar{V}^2\ln \frac{8\mathcal{N}\left(\epsilon, \mathcal{G}_{\varepsilon_n, \widetilde{\sigma}_n}, \|\cdot\|_{L_2(\mu)}\right)}{\delta}}{n}} + \frac{8 \mathcal{U}^{\star}_{\infty} \bar{V} \ln \frac{8\mathcal{N}\left(\epsilon, \mathcal{G}_{\varepsilon_n, \widetilde{\sigma}_n}, \|\cdot\|_{L_2(\mu)}\right)}{\delta}}{3 n} \notag \\ 
& + \frac{\varepsilon_n}{\mathcal{U}^{\star}_{2}}\sqrt{\frac{2\varepsilon_{3}}{\lambda M}} + 2\varepsilon_n\bigg).
\label{err2bound}
\#

\noindent \textbf{Bounding $\text{err}_3$.} In the following, we proceed to bound the term $\text{err}_3$. First, we make the decomposition as follows: 
\$
(1-\gamma)\text{err}_3
= & \sum_{a\in\mathcal{A}, s\in\mathcal{S}}[d_{\pi}(s,a)-\rho(s,a)][\Delta(\overline{q^{\pi}},\pi)-\Delta(\underline{q^{\pi}},\pi)]\\
= & \sum_{a\in\mathcal{A}, s\in\mathcal{S}}\mathds{1}_{\{d_{\pi}(s,a)-\rho(s,a) \geq 0 \} }[d_{\pi}(s,a)-\rho(s,a)][\Delta(\overline{q^{\pi}},\pi)-\Delta(\underline{q^{\pi}},\pi)] \\
& + \sum_{a\in\mathcal{A}, s\in\mathcal{S}}\mathds{1}_{\{d_{\pi}(s,a)-\rho(s,a) < 0\} }[d_{\pi}(s,a)-\rho(s,a)][\Delta(\overline{q^{\pi}},\pi)-\Delta(\underline{q^{\pi}},\pi)]\\
\leq & 
\underbrace{\sum_{a\in\mathcal{A}, s\in\mathcal{S}}\mathds{1}_{\{d_{\pi}(s,a)-\rho(s,a) < 0\} }[\rho(s,a)-d_{\pi}(s,a)]|\Delta(\overline{q^{\pi}},\pi)-\Delta(\underline{q^{\pi}},\pi)|}_{\text{err}_{31}}\\
& + 
\underbrace{\sum_{a\in\mathcal{A}, s\in\mathcal{S}}\mathds{1}_{\{d_{\pi}(s,a)-\rho(s,a) \geq 0 \} }[d_{\pi}(s,a)-\rho(s,a)][\Delta(\overline{q^{\pi}},\pi)-\Delta(\underline{q^{\pi}},\pi)]}_{\text{err}_{32}} .
\$

Next, we bound the terms $\text{err}_{31}$ and $\text{err}_{32}$ separately. 

\textbf{Bounding $\text{err}_{31}$.} We first observe that $\mathds{1}_{\{d_{\pi}(s,a)-\rho(s,a) < 0\} }[\rho(s,a)-d_{\pi}(s,a)] = \left(\rho(s,a)-d_{\pi}(s,a)\right)^{+}$, and we have 
\$
\text{err}_{31} = & \sum_{a\in\mathcal{A}, s\in\mathcal{S}}\mathds{1}_{\{d_{\pi}(s,a)-\rho(s,a) < 0\} }[\rho(s,a)-d_{\pi}(s,a)][\Delta(\overline{q^{\pi}},\pi)-\Delta(\underline{q^{\pi}},\pi)] \\
=& \sum_{a\in\mathcal{A}, s\in\mathcal{S}}\left(\rho(s,a)-d_{\pi}(s,a)\right)^{+}[\Delta(\overline{q^{\pi}},\pi)-\Delta(\underline{q^{\pi}},\pi)]\\
= & \ \mathbb{E}_{\left(\rho(s,a)-d_{\pi}(s,a)\right)^{+}}[\Delta(\overline{q^{\pi}},\pi)-\Delta(\underline{q^{\pi}},\pi)] .
\$
By the condition that $\|\tau(s,a)\|_{L_2({\mu})} \leq \mathcal{U}^{\star}_{2}$, and it follows from Lemma \ref{L_2MMD},  
\#
& \sqrt{\mathbb{E}_{\mu}[(\Delta(\overline{q^{\pi}},\pi))^2]}  \; \text{or} \; \sqrt{\mathbb{E}_{\mu}[(\Delta(\underline{q^{\pi}},\pi))^2]}  \notag \\
\lesssim &
\frac{\mathcal{U}^{\star}_{2}\sqrt{\frac{\bar{V}^2\ln \frac{8\mathcal{N}\left(\epsilon, \mathcal{G}_{\varepsilon_n, \widetilde{\sigma}_n}, \|\cdot\|_{L_2(\mu)}\right)}{\delta}}{n}} + \frac{\mathcal{U}^{\star}_{\infty} \bar{V} \ln \frac{8\mathcal{N}\left(\epsilon, \mathcal{G}_{\varepsilon_n, \widetilde{\sigma}_n}, \|\cdot\|_{L_2(\mu)}\right)}{\delta}}{3 n} + \mathcal{U}^{\star}_{2}\sqrt{\varepsilon_{\mathcal{Q}}}}{\mathcal{U}^{\star}_{2}}.
\label{l2_split_2}
\#
Since $\|\frac{\rho(s,a)}{\mu(s,a)}\|_{L_{2}(\mu)} \leq \mathcal{U}^{\star}_{2}$ and  $\left(\rho(s,a)-d_{\pi}(s,a)\right)^{+} \in [0, \rho(s,a)]$ for any $(s,a)$, we have 
\#
& \mathbb{E}_{\left(\rho(s,a)-d_{\pi}(s,a)\right)^{+}}|\Delta(\overline{q^{\pi}},\pi)-\Delta(\underline{q^{\pi}},\pi)| \notag \\
 \leq & \mathbb{E}_{\rho}[|\Delta(\overline{q^{\pi}},\pi)|] + \mathbb{E}_{\rho}[|\Delta(\underline{q^{\pi}},\pi)|] \notag \\
 = & \mathbb{E}_{\mu}\left[\frac{\rho(s,a)}{\mu(s,a)}|\Delta(\overline{q^{\pi}},\pi)|\right] +\mathbb{E}_{\mu}\left[\frac{\rho(s,a)}{\mu(s,a)}|\Delta(\underline{q^{\pi}},\pi)|\right]  \notag\\
 \leq & \|\frac{\rho(s,a)}{\mu(s,a)}\|_{L_{2}(\mu)}\|\Delta(\overline{q^{\pi}},\pi) \|_{L_{2}(\mu)} + \|\frac{\rho(s,a)}{\mu(s,a)}\|_{L_{2}(\mu)}\|\Delta(\underline{q^{\pi}},\pi) \|_{L_{2}(\mu)}  \notag\\
\lesssim &
\mathcal{U}^{\star}_{2}\sqrt{\frac{\bar{V}^2\ln \frac{8\mathcal{N}\left(\epsilon, \mathcal{G}_{\varepsilon_n, \widetilde{\sigma}_n}, \|\cdot\|_{L_2(\mu)}\right)}{\delta}}{n}} + \frac{2\mathcal{U}^{\star}_{\infty} \bar{V} \ln \frac{8\mathcal{N}\left(\epsilon, \mathcal{G}_{\varepsilon_n, \widetilde{\sigma}_n}, \|\cdot\|_{L_2(\mu)}\right)}{\delta}}{3 n} + \mathcal{U}^{\star}_{2}\sqrt{\varepsilon_{\mathcal{Q}}}.
\label{err31_bound}
\#

\textbf{Bounding $\text{err}_{32}$.} We now bound the term $\text{err}_{32}$. It observes that 
\$
\text{err}_{32} =&\sum_{a\in\mathcal{A}, s\in\mathcal{S}}\left(d_{\pi}(s,a)-\rho(s,a)\right)^{+}[\Delta(\underline{q^{\pi}},\pi)-\Delta(\overline{q^{\pi}},\pi)] \\
= &\sum_{a\in\mathcal{A}, s\in\mathcal{S}}\left(d_{\pi}(s,a)-\rho(s,a)\right)^{+}(\mathbb{I}-\gamma \mathds{P}^{\pi})\Delta_{\overline{q^{\pi}}-\underline{q^{\pi}}},
\$
where $\Delta_{\overline{q^{\pi}}-\underline{q^{\pi}}} = \overline{q^{\pi}}(s,a) - \underline{q^{\pi}}(s,a)
$. We make the decomposition with respect to $\mu(s,a) >0 $ and $\mu(s,a) = 0 $, that is
\$
\text{err}_{32} =& \sum_{a\in\mathcal{A}, s\in\mathcal{S}}\mathds{1}_{\mu(s,a)>0}\left(d_{\pi}(s,a)-\rho(s,a)\right)^{+}(\mathbb{I}-\gamma \mathds{P}^{\pi})\Delta_{\overline{q^{\pi}}-\underline{q^{\pi}}} \\
&+ \sum_{a\in\mathcal{A}, s\in\mathcal{S}}\mathds{1}_{\mu(s,a)=0}\left(d_{\pi}(s,a)-\rho(s,a)\right)^{+}(\mathbb{I}-\gamma \mathds{P}^{\pi})\Delta_{\overline{q^{\pi}}-\underline{q^{\pi}}} \\
\leq & \sum_{a\in\mathcal{A}, s\in\mathcal{S}}\mathds{1}_{\mu(s,a)>0}\left(d_{\pi}(s,a)-\rho(s,a)\right)^{+}\left|(\mathbb{I}-\gamma \mathds{P}^{\pi})\Delta_{\overline{q^{\pi}}-\underline{q^{\pi}}}\right| \\
&+ \sum_{a\in\mathcal{A}, s\in\mathcal{S}}\mathds{1}_{\mu(s,a)=0}\left(d_{\pi}(s,a)-\rho(s,a)\right)^{+}(\mathbb{I}-\gamma \mathds{P}^{\pi})\Delta_{\overline{q^{\pi}}-\underline{q^{\pi}}}.
\$
It observes that, for state-action pairs $(s,a) \in \mathcal{S} \times \mathcal{A}$ with $\mu(s,a) > 0$, 
\$
|(\mathbb{I}-\gamma \mathds{P}^{\pi})\Delta_{\overline{q^{\pi}}-\underline{q^{\pi}}}| = &|\gamma  \overline{q^{\pi}}(s^{\prime},\pi) -  \overline{q^{\pi}}(s,a) -  (\gamma \underline{q^{\pi}}(s^{\prime},\pi) -  \underline{q^{\pi}}(s,a)| \\
 \leq & |\gamma  (\overline{q^{\pi}}(s^{\prime},\pi) - \underline{q^{\pi}}(s^{\prime},\pi))| + |  \overline{q^{\pi}}(s,a) -  \underline{q^{\pi}}(s,a)|  \\
 \leq  & \sup_{a\in \mathcal{A},s\in\mathcal{S}, \mu(s,a)>0}(1+\gamma) \left|\overline{q^{\pi}}(s,a) - \underline{q^{\pi}}(s,a)\right|
 \lesssim  \frac{2 \varepsilon_{r}(1+\gamma)}{1-\gamma},
\$
where the last inequality comes from Lemma \ref{upper_version_bdd}, and $\varepsilon_{r}$ is defined in Lemma \ref{upper_version_bdd} with some modifications adapting to $\varepsilon_1$ and $\varepsilon_2$ in \eqref{prime_perturbed} and \eqref{omega_cond}, and the function class $\mathcal{G}_{\varepsilon_n, \widetilde{\sigma}_n}$. Therefore, we have 
\$
\varepsilon_{r} = \sqrt{\frac{2\bar{V}^2\ln \frac{8\mathcal{N}\left(\epsilon, \mathcal{G}_{\varepsilon_n, \widetilde{\sigma}_n}, \|\cdot\|_{L_2(\mu)}\right)}{\delta}}{n}} + \frac{2\mathcal{U}_{\infty}\bar{V} \ln \frac{8\mathcal{N}\left(\epsilon, \mathcal{G}_{\varepsilon_n, \widetilde{\sigma}_n}, \|\cdot\|_{L_2(\mu)}\right)}{\delta}}{3 n}  + \frac{\varepsilon_n}{\mathcal{U}^{\star}_{2}}.
\$
where $\mathcal{U}_{\infty} \geq 0$. Next, we have that 
\$
\text{err}_{32} \leq & \sum_{a\in\mathcal{A}, s\in\mathcal{S}}\mathds{1}_{\mu(s,a)>0}\left(d_{\pi}(s,a)-\rho(s,a)\right)^{+}\frac{2 \varepsilon_{r}(1+\gamma)}{1-\gamma}\\
&+ \sum_{a\in\mathcal{A}, s\in\mathcal{S}}\mathds{1}_{\mu(s,a)=0}\left(d_{\pi}(s,a)-\rho(s,a)\right)^{+}(\mathbb{I}-\gamma \mathds{P}^{\pi})\Delta_{\overline{q^{\pi}}-\underline{q^{\pi}}}\\
\leq &  \sum_{a\in\mathcal{A}, s\in\mathcal{S}}\left(d_{\pi}(s,a)-\rho(s,a)\right)^{+}\frac{2 \varepsilon_{r}(1+\gamma)}{1-\gamma} \\ 
&+ \sum_{a\in\mathcal{A}, s\in\mathcal{S}}\mathds{1}_{\mu(s,a)=0}\left(d_{\pi}(s,a)-\rho(s,a)\right)^{+}(\mathbb{I}-\gamma \mathds{P}^{\pi})\Delta_{\overline{q^{\pi}}-\underline{q^{\pi}}} \\
= & \frac{2 \varepsilon_{r}(1+\gamma)}{1-\gamma}\sum_{a\in\mathcal{A}, s\in\mathcal{S}}\left(d_{\pi}(s,a)-\rho(s,a)\right)^{+} \\
&+ \sum_{a\in\mathcal{A}, s\in\mathcal{S}}\mathds{1}_{\mu(s,a)=0}\Big[\left(d_{\pi}(s,a)-\rho(s,a)\right)^{+}(\mathbb{I}-\gamma \mathds{P}^{\pi})\Delta_{\overline{q^{\pi}}-\underline{q^{\pi}}}\Big]  \\
\lesssim & \frac{ \varepsilon_{r}}{1-\gamma}\sum_{a\in\mathcal{A}, s\in\mathcal{S}}\left(d_{\pi}(s,a)-\rho(s,a)\right)^{+} \\
&+ \sum_{a\in\mathcal{A}, s\in\mathcal{S}}\mathds{1}_{\mu(s,a)=0}\Big[\left(d_{\pi}(s,a)-\rho(s,a)\right)^{+}(\mathbb{I}-\gamma \mathds{P}^{\pi})\Delta_{\overline{q^{\pi}}-\underline{q^{\pi}}}\Big] .
\$
Optimizing  $\text{err}_{32}$ for the set contains $\{\rho: \|\frac{\rho(s,a)}{\mu(s,a)}\|_{L_{2}(\mu)} \leq \mathcal{U}^{\star}_{2}\}$, we obtain the tight bound that 
\#
\text{err}_{32} \lesssim  
\min_{\left\{\rho: \left\|\frac{\rho(s,a)}{\mu(s,a)}\right\|_{L_{2}(\mu)} \leq \mathcal{U}^{\star}_{2}\right\}} \bigg\{ &  \sum_{a\in\mathcal{A}, s\in\mathcal{S}}\mathds{1}_{\mu(s,a)=0}\Big(\left(d_{\pi}(s,a)-\rho(s,a)\right)^{+}(\mathbb{I}-\gamma \mathds{P}^{\pi})\Delta_{\overline{q^{\pi}}-\underline{q^{\pi}}}\Big)  \notag \\
& \quad +  \frac{\varepsilon_{r}}{1-\gamma}\sum_{a\in\mathcal{A}, s\in\mathcal{S}}\left(d_{\pi}(s,a)-\rho(s,a)\right)^{+}\bigg\}.
\label{err32_bound}
\#
Combine the upper bounds \eqref{err31_bound} and \eqref{err32_bound}, and we summarize the upper bound for $\text{err}_3$ as follows:
\#
\text{err}_3 \lesssim & \frac{1}{1-\gamma}\Bigg( \mathcal{U}^{\star}_{2}\sqrt{\frac{\bar{V}^2\ln \frac{8\mathcal{N}\left(\epsilon, \mathcal{G}_{\varepsilon_n, \widetilde{\sigma}_n}, \|\cdot\|_{L_2(\mu)}\right)}{\delta}}{n}} + \frac{2\mathcal{U}^{\star}_{\infty} \bar{V} \ln \frac{8\mathcal{N}\left(\epsilon, \mathcal{G}_{\varepsilon_n, \widetilde{\sigma}_n}, \|\cdot\|_{L_2(\mu)}\right)}{\delta}}{3 n} \notag \\
& + \min_{\left\{\rho: \left\|\frac{\rho(s,a)}{\mu(s,a)}\right\|_{L_{2}(\mu)} \leq \mathcal{U}^{\star}_{2}\right\}} \bigg\{ \sum_{a\in\mathcal{A}, s\in\mathcal{S}}\mathds{1}_{\mu(s,a)=0}\Big(\left(d_{\pi}(s,a)-\rho(s,a)\right)^{+}(\mathbb{I}-\gamma \mathds{P}^{\pi})\Delta_{\overline{q^{\pi}}-\underline{q^{\pi}}}\Big)  \notag \\
& \qquad \qquad \qquad \qquad \qquad \qquad   +  \frac{\varepsilon_{r}}{1-\gamma}\sum_{a\in\mathcal{A}, s\in\mathcal{S}}\left(d_{\pi}(s,a)-\rho(s,a)\right)^{+}\bigg\}  + \mathcal{U}^{\star}_{2}\sqrt{\varepsilon_{\mathcal{Q}}} \Bigg).
\label{err3_bound}
\#
According to the calculation of $\mathcal{N}\left(\epsilon, \mathcal{G}_{\varepsilon_n, \widetilde{\sigma}_n}, \|\cdot\|_{L_2(\mu)}\right)$, and use the notation $\operatorname{Vol}({\Theta})$ for the function class complexity, i.e., $\operatorname{Vol}({\Theta}) = (e^{D}\max\{D_{\Omega},D_{\mathcal{Q}},D_{\Pi}\}+1)^3(\{1 \vee L\}\mathcal{U}^{\tau}_{ 2})^{2D}$ where $D=D_{\Omega}+D_{\mathcal{Q}}+D_{\Pi}$, 
we set $\varepsilon_{n}$ as:  
$\varepsilon_{n} = \widetilde{\mathcal{O}}(n^{-1/2}\mathcal{U}^{\tau}_{2}(\sqrt{\ln \{\operatorname{Vol}({\Theta})/\delta\}} + \mathcal{U}^{\tau}_{\infty}\sqrt{\varepsilon_{\mathcal{Q}}})$
for ensuring the best approximator for $q^{\pi}$ is in $\mathcal{Q}_{\varepsilon_n}$. According to Lemma \ref{reg_finder}, we set $\widetilde{\sigma}_n =  \widetilde{\mathcal{O}}(n^{-1/2}\mathcal{U}^{\star}_{2} L\sqrt{\ln \{\operatorname{Vol}({\Theta})/\delta\}}+M(\mathcal{U}^{\tau}_{2}-1)^2)$  
The set up for $\widetilde{\sigma}_n$ is to ensure $\sup_{\tau \in \widetilde{\Omega}_{\widetilde{\sigma}_n}}\|\tau(s,a)\|_{L_2(\mu)} \leq \mathcal{U}^{\star}_{2}$, where $\mathcal{U}^{\star}_{2} \in [1,\mathcal{U}^{\tau}_{2})$. And
Then it follows from the regret decomposition \eqref{total_bound_alpha1}, and the upper error  bounds for $\text{err}_1, \text{err}_2$ and $\text{err}_3$ in \eqref{misspecified_end}, \eqref{err2bound} and \eqref{err3_bound}, by some algebra and if we ignore the high-order fast terms, we conclude that, w.p. $\geq 1-\delta$, 
Then it follows from the regret decomposition \eqref{total_bound_alpha1}, and the upper error  bounds for $\text{err}_1, \text{err}_2$ and $\text{err}_3$ in \eqref{misspecified_end}, \eqref{err2bound} and \eqref{err3_bound}, by some algebra, w.p. $\geq 1-\delta$, we have
\#
J(\pi) - J(\widehat{\pi}) \leq & \; \frac{1}{1-\gamma}\mathcal{O}\Bigg(\mathcal{E}^{n}_{1} + \sqrt{\left(1 + \mathcal{U}^{\tau}_{\infty} + \frac{\mathcal{U}^{\tau}_{\infty}}{M}\right)}\max\{(\varepsilon_{\mathcal{Q}})^{1/2},(\varepsilon_{\mathcal{Q}})^{3/4}\} \notag \\
& \quad + \min_{\left\{\rho: \left\|\frac{\rho}{\mu}\right\|_{L_{2}(\mu)} \leq \mathcal{U}^{\star}_{2}\right\}} \bigg\{ \mathbb{E}_{\left(d_{\pi}-\rho\right)^{+}}\big[\mathds{1}_{\mu=0}(\mathbb{I}-\gamma \mathds{P}^{\pi})\Delta_{\overline{q^{\pi}}-\underline{q^{\pi}}}(s,a) + \mathds{1}_{\mu>0}\mathcal{E}^{n}_{2}\big]\bigg\}\Bigg),
\label{full_main_regret}
\#
where $\mathcal{E}^{n}_{1}=
\mathcal{U}^{\star}_{2}(\bar{V}+L)\sqrt{\ln \{\operatorname{Vol}({\Theta})/\delta\}/(nM)} + \sqrt{\mathcal{U}^{\tau}_{2}(\bar{V}^3+\bar{V}^2L)/M}(\ln \{\operatorname{Vol}({\Theta})/\delta\}/n)^{\frac{3}{4}} + \mathcal{U}^{\tau}_{\infty}(\bar{V}+L)\ln \{\operatorname{Vol}({\Theta})/\delta\}/n$ and $\mathcal{E}^{n}_{2}=  
{(1-\gamma)^{-1}}((\bar{V}+L)\sqrt{\ln \{\operatorname{Vol}({\Theta})/\delta\}/n}  \linebreak + 
(\mathcal{U}^{\tau}_{\infty}\bar{V}/\mathcal{U}^{\star}_{2})\ln \{\operatorname{Vol}({\Theta})/\delta\}/n)$.
Furthermore, if we ignore the high-order fast terms using a big-Oh notation $\widetilde{\mathcal{O}}$, by some algebra, we conclude that 
\#
& J(\pi) - J(\widehat{\pi}) \leq \; \frac{1}{1-\gamma}\widetilde{\mathcal{O}}\Bigg(\mathcal{U}^{\star}_{2}\mathfrak{C}_{\bar{V}, L}\sqrt{\frac{\ln\{\operatorname{Vol}({\Theta})/\delta\}}{nM}}   + \sqrt{\frac{\mathfrak{C}_{\mathcal{U}^{\tau}_{\infty}}}{M}}\max\{(\varepsilon_{\mathcal{Q}})^{1/2},(\varepsilon_{\mathcal{Q}})^{3/4}\} \\
&  + \min_{\left\{\rho: \left\|\frac{\rho}{\mu}\right\|_{L_{2}(\mu)} \leq \mathcal{U}^{\star}_{2}\right\}} \bigg\{ \mathbb{E}_{\left(d_{\pi}-\rho\right)^{+}}\big[\mathds{1}_{\mu=0}(\mathbb{I}-\gamma \mathds{P}^{\pi})\Delta_{\overline{q^{\pi}}-\underline{q^{\pi}}}(s,a) + \mathds{1}_{\mu>0}\mathfrak{C}_{\bar{V},\gamma}\sqrt{\frac{\ln\{\operatorname{Vol}({\Theta})/\delta\}}{n}} \big]\bigg\}\Bigg),
\label{simplified_main_regret}
\#
where we use  $\mathfrak{C}_{x}$ denote constant terms depending on $x$. This completes the proof. 
\end{proof}

\section{Proof of Corollary 4.1}

\begin{proof}
To complete the proof, it is sufficient to choose a particular $\rho^{\diamond}$ such that $\left\|\frac{\rho}{\mu}\right\|_{L_{2}(\mu)} \leq \mathcal{U}^{\star}_{2}$ to obtain a regret as the upper bound for the regret in \eqref{full_main_regret}. For any comparator policy $\pi^{\diamond}$, since
\$
& \min_{\left\{\rho: \left\|\frac{\rho}{\mu}\right\|_{L_{2}(\mu)} \leq \mathcal{U}^{\star}_{2}\right\}} \bigg\{ \mathbb{E}_{\left(d_{\pi^{\diamond}}-\rho\right)^{+}}\big[\mathds{1}_{\mu=0}(\mathbb{I}-\gamma \mathds{P}^{\pi})\Delta_{\overline{q^{\pi}}-\underline{q^{\pi}}}(s,a) + \mathds{1}_{\mu>0}\mathcal{E}^{n}_{2}\big]\bigg\}\Bigg) \\
\leq &  
\mathbb{E}_{\left(d_{\pi^{\diamond}}-\rho^{\diamond}\right)^{+}}\big[\mathds{1}_{\mu=0}(\mathbb{I}-\gamma \mathds{P}^{\pi})\Delta_{\overline{q^{\pi}}-\underline{q^{\pi}}}(s,a) + \mathds{1}_{\mu>0}\mathcal{E}^{n}_{2}\big].
\$ 
therefore when we set $\rho^{\diamond} = d_{\pi^{\diamond}}$ which satisfies the condition that $\left\|\frac{\rho}{\mu}\right\|_{L_{2}(\mu)} \leq \mathcal{U}^{\star}_{2}$, because $\pi^{\diamond} \in \Pi({\mathcal{U}}^{\tau}_{2})$ based on the definition of $\Pi({\mathcal{U}}^{\tau}_{2})$. In this case, it observes that 
\$
\mathbb{E}_{\left(d_{\pi^{\diamond}}-\rho^{\diamond}\right)^{+}}\big[\mathds{1}_{\mu=0}(\mathbb{I}-\gamma \mathds{P}^{\pi})\Delta_{\overline{q^{\pi}}-\underline{q^{\pi}}}(s,a) + \mathds{1}_{\mu>0}\mathcal{E}^{n}_{2}\big] = 0 
\$
Then we have 
\#
J(\pi^{\diamond}) - J(\widehat{\pi}) \leq  \; \frac{1}{1-\gamma}\mathcal{O}\Bigg(\mathcal{E}^{n}_{1} + \sqrt{\left(1 + \mathcal{U}^{\tau}_{\infty} + \frac{\mathcal{U}^{\tau}_{\infty}}{M}\right)}\max\{(\varepsilon_{\mathcal{Q}})^{1/2},(\varepsilon_{\mathcal{Q}})^{3/4}\} \Bigg)
\label{full_first}
\#
with the assumption that $\varepsilon_{\mathcal{Q}} \in (0,1]$, we have $\max\{(\varepsilon_{\mathcal{Q}})^{1/2},(\varepsilon_{\mathcal{Q}})^{3/4}\} =(\varepsilon_{\mathcal{Q}})^{1/2}$, and therefore if we ignore the high-order fast terms, we conclude that 
\$
J(\pi^{\diamond}) - J(\widehat{\pi}) \leq \frac{1}{1-\gamma}\widetilde{\mathcal{O}}\Bigg( \mathcal{U}^{\star}_{2}(\bar{V}+L)\sqrt{\frac{\ln\{\operatorname{Vol}({\Theta})/\delta\}}{nM}} +  \sqrt{\left(1 + \mathcal{U}^{\tau}_{\infty} + \mathcal{U}^{\tau}_{\infty}/M\right) \varepsilon_{\mathcal{Q}}} \Bigg).
\$
\end{proof}

\section{Proof of Corollary 4.2}

\begin{proof}
On the condition of Corollary 4.1 and set $\varepsilon_{\mathcal{Q}} = 0$, we can follow the proof of Corollary 4.1, and obtain the regret bound in \eqref{full_first} but with the modification as, w.p. $1-\delta$,
\$
J(\pi^{\diamond}) - J(\widehat{\pi}) \leq & \; \frac{1}{1-\gamma}\mathcal{O}(\mathcal{E}^{n}_{1}),
\$
for $\mathcal{E}^{n}_{1}=
\mathcal{U}^{\star}_{2}(\bar{V}+L)\sqrt{\ln \{\operatorname{Vol}({\Theta})/\delta\}/(nM)} + \sqrt{\mathcal{U}^{\tau}_{2}(\bar{V}^3+\bar{V}^2L)/M}(\ln \{\operatorname{Vol}({\Theta})/\delta\}/n)^{\frac{3}{4}} + \mathcal{U}^{\tau}_{\infty}(\bar{V}+L)\ln \{\operatorname{Vol}({\Theta})/\delta\}/n$. We let $\varepsilon = \mathcal{E}^{n}_{1}/(1-\gamma)$, and we solve this equation for $n$, by some algebra, we obtain the sample complexity:
\$
n = \mathcal{O}\Bigg(\Big(\frac{(\mathcal{U}^{\star}_{2}(\bar{V}+L)/\sqrt{M})^2}{\varepsilon^{2}(1-\gamma)^2} + \frac{(\mathcal{U}^{\tau}_{2}\bar{V}^2(\bar{V}+L)/M)^{0.67}}{\varepsilon^{1.33}(1-\gamma)^{1.33}} + \frac{\mathcal{U}^{\tau}_{\infty}(\bar{V}+L)}{\varepsilon(1-\gamma)}\Big) \ln\frac{\operatorname{Vol}({\Theta})}{\delta}\Bigg).
\$
This completes the proof. 
\end{proof}

\section{Proof of Theorem 4.2}

\begin{proof}
To complete the proof, it is sufficient to set $\rho = d^{\pi_{b}}$ and we can obtain the regret following the proof of Corollary 4.1 with $\varepsilon_{\mathcal{Q}} = 0$, i.e.,
\$
J(\pi^{\diamond}) - J(\widehat{\pi}) \leq  \; \frac{1}{1-\gamma}\mathcal{O}(\mathcal{E}^{n}_{1}).
\$
where $\mathcal{E}^{n}_{1}=
\mathcal{U}^{\star}_{2}(\bar{V}+L)\sqrt{\ln \{\operatorname{Vol}({\Theta})/\delta\}/(nM)} + \sqrt{\mathcal{U}^{\tau}_{2}(\bar{V}^3+\bar{V}^2L)/M}(\ln \{\operatorname{Vol}({\Theta})/\delta\}/n)^{\frac{3}{4}} + \mathcal{U}^{\tau}_{\infty}(\bar{V}+L)\ln \{\operatorname{Vol}({\Theta})/\delta\}/n$. As $d^{\pi_{b}} = \mu$, so that $\|d^{\pi_{b}} /\mu(s,a)\|_{L_2(\mu)} = \tau_{d^{\pi_{b}}/\mu} = 1$. Therefore, it is feasible to set $\mathcal{U}^{\star}_{2} = \mathcal{U}^{\tau}_{\infty} =1$, and this completes the proof. 
\end{proof}

\section{Proof of Theorem 5.1}

\subsection{Proof of Lemma \ref{iden_mdp}}

\begin{lemma}
\label{iden_mdp}
For $k \in [\bar{K}]$, suppose $q^{k} \in \mathcal{Q}$ and $\tau^{k} \in \Omega$ such that $\|\tau^{k}(s,a)\|_{L_2(\mu)} \leq C_1$ and  $\|\tau^{k}(s,a)\|_{L_{\infty}} \leq C_2$ where the constants $C_2 \geq C_1 > 0 $, it satisfies that 
\$
\left|\frac{\frac{1}{n}\sum^{n}_{i=1}\tau^{k}(s_i,a_i)\left(r_i+\gamma q^{k} \left(s^{\prime}_i, \pi^{k}\right)-q^{k}(s_i, a_i)  \right)}{1-\gamma}\right| \leq \varepsilon.
\$
for some $\varepsilon \geq 0 $. There must exist an MDP $\left\{\mathcal{S}, \mathcal{A}, \mathds{P}_{k}, \gamma, r_{k}, s^{0}\right\}$ which is identical to the true environment MDP $\left\{\mathcal{S}, \mathcal{A}, \mathds{P}, \gamma, r, s^{0}\right\}$: $\text{MDP}^{\star}$ only except the reward function $r_{k}(s,a)$ is iterative based when $\mathds{P}_{k} = \mathds{P}$, which is defined as
\$
r_{k}(s,a) = q^{k}(s,a) - \gamma \mathbb{E}_{s^{\prime} \sim \mathds{P}(\cdot|s,a)}\left[ \sum_{a^{\prime}\in \mathcal{A}} \pi^{k}(a^{\prime}|s^{\prime})q^{k}(s^{\prime},a^{\prime}) \right].
\$
In addition, such reward functions, for any $k \in [\bar{K}]$, are approximating to the true reward, 
\$
 \| r_{k}(s,a) - r(s,a)  \|_{L^{2}(\mu)}
\leq & C_1\sqrt{\frac{2 \bar{V}^2\ln \frac{8\mathcal{N}\left(\epsilon, \mathcal{G}^{k}, \|\cdot\|_{L_2(\mu)}\right)}{\delta}}{n}} +  \frac{2C_2 \bar{V} \ln \frac{8\mathcal{N}\left(\epsilon, \mathcal{G}^{k}, \|\cdot\|_{L_2(\mu)}\right)}{\delta}}{3 n} + \varepsilon,
\$
where $\mathcal{G}^{k} := \widetilde{\Omega}^{k} \times \mathcal{Q} \times \Pi$ for $\tau^{k} \in \widetilde{\Omega}^{k}$. And $q^{k}$ is the true action-value function under the policy $\pi^{k}$ in the MDP $\left\{\mathcal{S}, \mathcal{A}, \mathds{P}, \gamma, r_{k}, s^{0}\right\}$. 
\end{lemma}

\begin{proof}
Following the definition of $r_{k}$, we observe that 
\#
q^{k}(s,a) =& r_{k}(s,a)  + \gamma \mathbb{E}_{s^{\prime} \sim \mathds{P}(\cdot|s,a)}\left[ \sum_{a^{\prime}\in \mathcal{A}} \pi^{k}(a^{\prime}|s^{\prime})q^{k}(s^{\prime},a^{\prime}) \right] \notag \\ 
=& r_{k}(s,a)  + \gamma \mathbb{E}_{s^{\prime} \sim \mathds{P}_{k}(\cdot|s,a)}\left[ \sum_{a^{\prime}\in \mathcal{A}} \pi^{k}(a^{\prime}|s^{\prime})q^{k}(s^{\prime},a^{\prime}) \right].
\label{supp_2}
\#
The second equality comes from $\mathds{P}_{k} = \mathds{P}$ for any $k$. Thus the equation realizes a Bellman equation over the MDP $\left\{\mathcal{S}, \mathcal{A}, \mathds{P}_{k}, \gamma, r_{k}, s^{0}\right\}$ for the policy $\pi^{k}$. Further, this implies that $q^{k}(s,a)$ is the corresponding true action-value function. Following the proof of Lemma \ref{risk_bound_lm}, for the MDP $\left\{\mathcal{S}, \mathcal{A}, \mathds{P}_{k}, \gamma, r_{k}, s^{0}\right\}$ and we define the subset of $\tau^{k}$ with $\|\tau^{k}(s,a)\|_{L_2(\mu)} \leq C_1$, $\|\tau(s,a)\|_{L_{\infty}} \leq C_2$ as $\widetilde{\Omega}^{k}$. Then we have 
\$
& \sup_{\tau^{k} \in \widetilde{\Omega}^{k}
}\left|\mathbb{E}_{\mu}\left[\tau^{k}(s,a) \left(r(s,a)+\gamma q^{k} \left(s^{\prime}, \pi^{k}\right)-q^{k}(s, a)  \right)\right]\right| \\
 \leq & C_1\sqrt{\frac{2 \bar{V}^2\ln \frac{8\mathcal{N}\left(\epsilon, \mathcal{G}^{k}, \|\cdot\|_{L_2(\mu)}\right)}{\delta}}{n}} +  \frac{2 C_2 \bar{V} \ln \frac{8\mathcal{N}\left(\epsilon, \mathcal{G}^{k}, \|\cdot\|_{L_2(\mu)}\right)}{\delta}}{3 n} + \varepsilon.
\$
Since $\|\tau^{k}(s,a)\|_{L_2(\mu)} \leq C_1$, it follows from Lemma \ref{L_2MMD}, we have 
\#
& C_1\sqrt{\mathbb{E}_{\mu}[\left(r(s,a)+\gamma q^{k} \left(s^{\prime}, \pi^{k}\right)-q^{k}(s, a)  \right)^2]} \\
\leq &
C_1\sqrt{\frac{\bar{V}^2\ln \frac{8\mathcal{N}\left(\epsilon, \mathcal{G}^{k}, \|\cdot\|_{L_2(\mu)}\right)}{\delta}}{n}} + \frac{2C_2 \bar{V} \ln \frac{8\mathcal{N}\left(\epsilon, \mathcal{G}^{k}, \|\cdot\|_{L_2(\mu)}\right)}{\delta}}{3 n} + \varepsilon.
\#
Due to the equivalence $\mathds{P}_{k} = \mathds{P}$, we have 
\$
\|r(s,a) - r_{k}(s,a) \|_{L^{2}(\mu)}  =& \left\|r(s,a) - q^{k}(s,a) + \gamma \mathbb{E}_{s^{\prime} \sim \mathds{P}_{k}(\cdot|s,a)}\left[ \sum_{a^{\prime}\in \mathcal{A}} \pi^{k}(a^{\prime}|s^{\prime})q^{k}(s^{\prime},\pi^{k}) \right]  \right\|_{L^{2}(\mu)} \\
= &\sqrt{\mathbb{E}_{\mu}[\left(r(s,a)+\gamma q^{k} \left(s^{\prime}, \pi^{k}\right)-q^{k}(s, a)  \right)^2]}\\
\leq & C_1\sqrt{\frac{\bar{V}^2\ln \frac{8\mathcal{N}\left(\epsilon, \mathcal{G}^{k}, \|\cdot\|_{L_2(\mu)}\right)}{\delta}}{n}} + \frac{2C_2 \bar{V} \ln \frac{8\mathcal{N}\left(\epsilon, \mathcal{G}^{k}, \|\cdot\|_{L_2(\mu)}\right)}{\delta}}{3 n} + \varepsilon.
\$
This completes the proof. 
\end{proof}

\subsection{Proof of Lemma \ref{misspecified_q_alpha}}

\begin{lemma}
\label{misspecified_q_alpha}
Define 
\$
\widetilde{q}^{\pi} := \inf _{q \in \mathcal{Q}} \sup _{ \rho}\mathbb{E}_{\rho}[\left(q(s,a)-\mathcal{B}^{\pi} q(s,a)\right)^{2}]. 
\$
Under Assumption 1-3 in maintext, for any $\pi \in \Pi, \tau \in \Omega$ and  $q \in \mathcal{Q}$, given an offline data $\mathcal{D}_{1:n}=\{s_{i},a_{i},r_{i},s^{\prime}_i)\}^{n}_{i=1}$, then w.p. $\geq 1-\delta$, 
 \$
|\mathbb{P}_{n}\tau(s_i,a_i)\left(r_i+\gamma \widetilde{q}^{\pi}\left(s^{\prime}_i, \pi\right)-\widetilde{q}^{\pi} (s_i, a_i)  \right)| \leq  \varepsilon^{\diamond}_{n}
 \$
for 
 \$
\varepsilon^{\diamond}_{n} =  &\big(3\sqrt{2}\mathcal{U}^{\tau}_{2}\bar{V}+2\sqrt{2}\lambda \|\mathbb{D}(\tau(s,a))\|^{\text{UB}}_{L_2(\mu)}\big)\sqrt{\frac{\ln\frac{\operatorname{Vol}({\Theta}^{\dagger})}{\delta}}{n}} \notag  \\
 & + \frac{\big(6\mathcal{U}^{\tau}_{\infty}\bar{V}+4\lambda\|\mathbb{D}(\tau(s,a))\|^{\text{UB}}_{L_{\infty}}\big)\ln\frac{\operatorname{Vol}({\Theta}^{\dagger})}{\delta}}{3 n} 
+ \mathcal{U}^{\tau}_{2}\sqrt{\varepsilon_{\mathcal{Q}}},
 \$
 where $\operatorname{Vol}({\Theta}^{\dagger}) = (e^{D}\max\{D_{\Omega},D_{\mathcal{Q}},D_{\Pi}\}+1)^3(\{1 \vee L\}{\mathcal{U}}^{\tau}_{2})^{2D}$ for $D=D_{\Omega}+D_{\mathcal{Q}}+D_{\Pi}$.
\end{lemma}

\begin{proof}
First, we plug-in $\widetilde{q}^{\pi}$ into $\mathbb{E}_{\mu}[\tau(s,a) \Delta(\widetilde{q}^{\pi},\pi)]$ for any $\tau \in \Omega$, where $\Delta(q,\pi):= r(s,a)+\gamma {q}^{\pi} \left(s^{\prime}, \pi\right)-{q}^{\pi}(s, a)$. According to Cauchy–Schwarz inequality,  
\#
 |\mathbb{E}_{\mu}[\tau(s,a) \Delta(\widetilde{q}^{\pi},\pi)]| \leq \mathbb{E}_{\mu}[|\tau(s,a)|
 |\Delta(\widetilde{q}^{\pi},\pi)|]
 \leq \sqrt{\mathbb{E}_{\mu}[\tau^2(s,a)]\mathbb{E}_{\mu}\left[ \Delta^2(\widetilde{q}^{\pi},\pi) \right]} 
 \leq &\mathcal{U}^{\tau}_{2}\sqrt{\varepsilon_{\mathcal{Q}}}.
 \label{mis_q_obj}
\#
where the last inequality comes from the weightd $L_2(\mu)$ boundedness over $\tau$ and Assumption 1 in maintext on realizibility error over $\mathcal{Q}$. It then follows from Lemma \ref{risk_bound_alpha}, 
 \$
&  \bigg|\mathbb{E}_{\mu}\Big[\tau(s,a) \Delta(q,\pi) -\lambda \mathbb{D}(\tau(s,a))\Big] - \mathbb{P}_{n}\Big(\tau(s_i,a_i)\Delta_i(q,\pi)-\lambda\mathbb{D}(\tau(s_i,a_i))\Big)  \bigg| \leq \varepsilon_{1,n}
\$
where $\Delta_i(q,\pi):= r_i+\gamma {q} \left(s^{\prime}_i, \pi\right)-{q}(s_i, a_i)$ and $ \varepsilon_{1,n}$ denotes the upper bound of the inequality in Lemma \ref{risk_bound_alpha}. With the norm triangle inequality, 
\$
 & \bigg|\mathbb{E}_{\mu}\Big[\tau(s,a) \Delta(\widetilde{q}^{\pi},\pi) -\lambda \mathbb{D}(\tau(s,a))\Big] - \mathbb{P}_{n}\Big(\tau(s_i,a_i)\Delta_i(\widetilde{q}^{\pi},\pi)-\lambda\mathbb{D}(\tau(s_i,a_i))\Big)  \bigg| \\
\geq &\bigg|\mathbb{P}_{n}\Big(\tau(s_i,a_i)\Delta_i(\widetilde{q}^{\pi},\pi)-\lambda\mathbb{D}(\tau(s_i,a_i))\Big)  +\mathbb{E}_{\mu}\big[\lambda \mathbb{D}(\tau(s,a))\big] \bigg| - \bigg|\mathbb{E}_{\mu}\big[\tau(s,a) \Delta(\widetilde{q}^{\pi},\pi)\big]\bigg| 
\$
This indicates
\$
\bigg|\mathbb{P}_{n}\Big(\tau(s_i,a_i)\Delta_i(\widetilde{q}^{\pi},\pi)-\lambda\mathbb{D}(\tau(s_i,a_i))\Big)  +\mathbb{E}_{\mu}\big[\lambda \mathbb{D}(\tau(s,a))\big] \bigg|  \leq \varepsilon_{1,n} + \mathcal{U}^{\tau}_{2}\sqrt{\varepsilon_{\mathcal{Q}}}
\$
where the inequality comes from \eqref{mis_q_obj}. Apply Triangle inequality again, the above inequality implies 
\$
& \bigg|\mathbb{P}_{n}\Big(\tau(s_i,a_i)\Delta_i(\widetilde{q}^{\pi},\pi)\bigg| \leq  \bigg|\mathbb{P}_n \lambda \mathbb{D}(\tau(s_i,a_i))-\mathbb{E}_{\mu}\big[\lambda\mathbb{D}(\tau(s,a))\big]\bigg| + \mathcal{U}^{\tau}_{2}\sqrt{\varepsilon_{\mathcal{Q}}} + \varepsilon_{1,n}
\$
If follow Lemma \ref{onlyreg_risk_bound} and plug-in $\varepsilon_{1,n}$ from Lemma \ref{risk_bound_alpha}, by some algebra, we conclude that 
\$ \bigg|\mathbb{P}_{n}\Big(\tau(s_i,a_i)\Delta_i(\widetilde{q}^{\pi},\pi)\bigg|
\leq &
\big(3\sqrt{2}\mathcal{U}^{\tau}_{2}\bar{V}+2\sqrt{2}\lambda \|\mathbb{D}(\tau(s,a))\|^{\text{UB}}_{L_2(\mu)}\big)\sqrt{\frac{\ln\frac{\operatorname{Vol}({\Theta}^{\dagger})}{\delta}}{n}} \notag  \\
 & + \frac{\big(6\mathcal{U}^{\tau}_{\infty}\bar{V}+4\lambda\|\mathbb{D}(\tau(s,a))\|^{\text{UB}}_{L_{\infty}}\big)\ln\frac{\operatorname{Vol}({\Theta}^{\dagger})}{\delta}}{3 n} 
+ \mathcal{U}^{\tau}_{2}\sqrt{\varepsilon_{\mathcal{Q}}}.
\$
This completes the proof. 
\end{proof}

\subsection{Proof of Lemma \ref{alpha_true}}

\begin{lemma}\label{alpha_true}
Define 
\$
\widetilde{q}^{\pi} := \inf _{q \in \mathcal{Q}} \sup _{ \rho}\mathbb{E}_{\rho}[\left(q(s,a)-\mathcal{B}^{\pi} q(s,a)\right)^{2}], 
\$
for some admissible distribution $\rho$. Under Assumption 1-3 in maintext, for any $\pi \in \Pi, \tau \in \Omega$ and  $q \in \mathcal{Q}$, given an offline data $\mathcal{D}_{1:n}=\{s_{i},a_{i},r_{i},s^{\prime}_i)\}^{n}_{i=1}$, then w.p. $\geq 1-\delta$, 
\$
&\left|\mathbb{E}_{\mu}\left[\tau(s,a)\left(r(s,a)+\gamma \widetilde{q}^{\pi}\left(s^{\prime}, \pi\right)-\widetilde{q}^{\pi} (s, a)  \right)\right]\right| \leq  2\big(2 \mathcal{U}^{\tau}_{2}\bar{V}+2\sqrt{2}\lambda \|\mathbb{D}(\tau(s,a))\|^{\text{UB}}_{L_2(\mu)}\big)\sqrt{\frac{2 \ln \frac{\operatorname{Vol}({\Theta}^{\dagger})}{\delta}}{n}}\\
& \qquad \qquad \qquad \qquad  \qquad \qquad + \frac{4\big(2  \mathcal{U}^{\tau}_{\infty}\bar{V}+4\lambda\|\mathbb{D}(\tau(s,a))\|^{\text{UB}}_{L_{\infty}}\big) \ln \frac{\operatorname{Vol}({\Theta}^{\dagger})}{\delta}}{3 n} + \mathcal{U}^{\tau}_{\infty}\sqrt{\varepsilon_{\mathcal{Q}}}.
\$
 where $\operatorname{Vol}({\Theta}^{\dagger}) = (e^{D}\max\{D_{\Omega},D_{\mathcal{Q}},D_{\Pi}\}+1)^3(\{1 \vee L\}{\mathcal{U}}^{\tau}_{2})^{2D}$ for $D=D_{\Omega}+D_{\mathcal{Q}}+D_{\Pi}$.
\end{lemma}

\begin{proof}
According to Lemma \ref{risk_bound_lm}, for any $q \in \mathcal{Q},\pi \in \Pi,\tau \in \Omega$, we have   
\$
& \left|\mathbb{E}_{\mu}\left[\tau(s,a)\left(r(s,a)+\gamma q\left(s^{\prime}, \pi\right)-q(s, a)  \right)\right] - \mathbb{P}_{n}\left[\tau(s_i,a_i)\left(r_i+\gamma q\left(s^{\prime}_i, \pi\right)-q (s_i, a_i)  \right)\right]\right| \\
\lesssim & \mathcal{U}^{\tau}_{2}\sqrt{\frac{2 \bar{V}^2\ln \frac{\operatorname{Vol}({\Theta}^{\dagger})}{\delta}}{n}}  + \frac{2 \mathcal{U}^{\tau}_{\infty} \bar{V} \ln \frac{\operatorname{Vol}({\Theta}^{\dagger})}{\delta}}{3 n}.
\$
As this holds for any $q \in \mathcal{Q}$, thus it must hold for $\widetilde{q}$ which is in $\mathcal{Q}$ with approximation erorr $\varepsilon_{\mathcal{Q}}$.  Then by trainagle inequality, we have 
\$
\left|\mathbb{E}_{\mu}\left[\tau(s_i,a_i)\left(r(s,a)+\gamma \widetilde{q}^{\pi}\left(s^{\prime}, \pi\right)-\widetilde{q}^{\pi} (s, a)  \right)\right]\right| \leq & \left| \mathbb{P}_{n}\left[\tau(s_i,a_i)\left(r_i+\gamma \widetilde{q}^{\pi}\left(s^{\prime}_i, \pi\right)-\widetilde{q}^{\pi} (s_i, a_i)  \right)\right]\right| \\
& + \mathcal{U}^{\tau}_{2}\sqrt{\frac{2 \bar{V}^2\ln \frac{\operatorname{Vol}({\Theta}^{\dagger})}{\delta}}{n}}  + \frac{2 \mathcal{U}^{\tau}_{\infty} \bar{V} \ln \frac{\operatorname{Vol}({\Theta}^{\dagger})}{\delta}}{3 n}.
\$
According to Lemma \ref{misspecified_q_alpha}, we conclude that 
\$
&\left|\mathbb{E}_{\mu}\left[\tau(s,a)\left(r(s,a)+\gamma \widetilde{q}^{\pi}\left(s^{\prime}, \pi\right)-\widetilde{q}^{\pi} (s, a)  \right)\right]\right| \leq  2\big(2 \mathcal{U}^{\tau}_{2}\bar{V}+2\sqrt{2}\lambda \|\mathbb{D}(\tau(s,a))\|^{\text{UB}}_{L_2(\mu)}\big)\sqrt{\frac{2 \ln \frac{\operatorname{Vol}({\Theta}^{\dagger})}{\delta}}{n}}\\
& \qquad \qquad \qquad \qquad  \qquad \qquad + \frac{4\big(2  \mathcal{U}^{\tau}_{\infty}\bar{V}+4\lambda\|\mathbb{D}(\tau(s,a))\|^{\text{UB}}_{L_{\infty}}\big) \ln \frac{\operatorname{Vol}({\Theta}^{\dagger})}{\delta}}{3 n} + \mathcal{U}^{\tau}_{\infty}\sqrt{\varepsilon_{\mathcal{Q}}}.
\$
\end{proof}

\subsection{Proof of Lemma \ref{prime_bound_tau}}

\begin{lemma}
\label{prime_bound_tau}
Suppose 
$
\mathbb{P}_{n}\mathbb{D}(\tau(s_i,a_i)) \leq \varepsilon^{\mathbb{D}}_{n}
$ for some $\tau \in \Omega$ and $\varepsilon^{\mathbb{D}}_{n}$ depends on $n$ but it is not necessary to be $0$. Then, w.p., $\geq 1-\delta$, 
\$
\|\tau(s,a)\|_{L_2(\mu)} \leq  \frac{1}{\sqrt{M}}\bigg\{L\mathcal{U}^{\tau}_{2}\sqrt{\frac{2 \ln \frac{\operatorname{Vol}(\mathcal{G}^{\mathbb{D}})}{\delta}}{n}} +  2\sqrt{\frac{ L\mathcal{U}_{\infty}\mathcal{U}^{\tau}_{2}\ln \frac{\operatorname{Vol}(\mathcal{G}^{\mathbb{D}})}{\delta}}{3 n}}  + \sqrt{2\varepsilon^{\mathbb{D}}_{n}} + \sqrt{M}\bigg\}.
\$
where $L$ is the local Lipschitz constant. 
\end{lemma}

\begin{proof}
To proceed the proof, we first convert the upper bound for $\mathbb{P}_{n}\mathbb{D}(\tau(s_i,a_i))$ to the upper bound for $\mathbb{E}_{\mu}[\mathbb{D}(\tau(s_i,a_i))]$.
According to Lemma \ref{onlyreg_risk_bound}, we have 
\#
\mathbb{E}_{\mu}[\mathbb{D}(\tau(s,a))]  \lesssim \; \mathbb{P}_{n}\mathbb{D}(\tau(s_i,a_i)) + \varepsilon^{\diamond}_{n},
\label{concen_alpha}
\#
where 
\$
 \varepsilon^{\diamond}_{n}  = \|\mathbb{D}(\tau(s,a))\|^{\text{UB}}_{L_2(\mu)}\sqrt{\frac{2\ln\frac{\operatorname{Vol}(\mathcal{G}^{\mathbb{D}})}{\delta}}{n}} + \frac{2\|\mathbb{D}(\tau(s,a))\|^{\text{UB}}_{L_{\infty}}\ln\frac{\operatorname{Vol}(\mathcal{G}^{\mathbb{D}})}{\delta}}{3 n}.
\$
for $\operatorname{Vol}(\mathcal{G}^{\mathbb{D}}) = (e^{D_{\Omega}}(D_{\Omega}+1))(\{1 \vee  L \}\mathcal{U}^{\tau}_{2})^{2D_{\Omega}}$.
It follows the inequality in \eqref{g_lip} and 
combine with \eqref{concen_alpha}, we have 
$
 \|\tau(s,a)\|^2_{L_2(\mu)}  \leq    \frac{2(\varepsilon^{\mathbb{D}}_{n}+\varepsilon^{\diamond}_{n})+M}{M} 
$
To simplify the notation, we define 
\$
\varepsilon^{\diamond,1}_{n} =  \sqrt{\frac{2 \ln \frac{\operatorname{Vol}(\mathcal{G}^{\mathbb{D}})}{\delta}}{n}}; 
\varepsilon^{\diamond,2}_{n} = \frac{2 \|\mathbb{D}(\tau(s,a))\|^{\text{UB}}_{L_{\infty}} \ln \frac{\operatorname{Vol}(\mathcal{G}^{\mathbb{D}})}{\delta}}{3 n}.
\$
According to the Lipschitz continuity of $\mathbb{D}(\cdot)$ with local Lipschitz constant $L$, we have 
\$
&  \|\tau(s,a)\|^2_{L_2(\mu)}  \leq \frac{ \|\mathbb{D}(\tau(s,a))\|^{\text{UB}}_{L_2(\mu)}\varepsilon^{\diamond,1}_{n} + 2(\varepsilon^{\diamond,2}_{n}+\varepsilon^{\mathbb{D}}_{n})+M}{M} \leq \frac{L\mathcal{U}^{\tau}_{2}\varepsilon^{\diamond,1}_{n} + 2(\varepsilon^{\diamond,2}_{n}+\varepsilon^{\mathbb{D}}_{n})+M}{M}\\
 \implies &  \|\tau(s,a)\|^2_{L_2(\mu)} + \left(\frac{L\mathcal{U}^{\tau}_{2}\varepsilon^{\diamond,1}_{n}}{M}\right)\|\tau(s,a)\|_{L_2(\mu)} - \frac{2}{M}(\varepsilon^{\diamond,2}_{n}+\varepsilon^{\mathbb{D}}_{n}) -\frac{M}{M}  \leq 0. 
\$
Therefore, it suffices to solve the root of 
\$
\|\tau(s,a)\|^2_{L_2(\mu)} + \left(\frac{L\mathcal{U}^{\tau}_{2}\varepsilon^{\diamond,1}_{n}}{M}\right)\|\tau(s,a)\|_{L_2(\mu)} + \frac{2}{M}(\varepsilon^{\diamond,2}_{n}+\varepsilon^{\mathbb{D}}_{n}) + \frac{M}{M}   = 0, 
\$
 and we conclude that 
\$
\sqrt{M}\|\tau(s,a)\|_{L_2(\mu)} \leq & L\mathcal{U}^{\tau}_{2}\varepsilon^{\diamond,1}_{n}+\sqrt{2\varepsilon^{\diamond,2}_{n}}+\sqrt{2\varepsilon^{\mathbb{D}}_{n}} + \sqrt{M}\\
    \leq & L\mathcal{U}^{\tau}_{2}\sqrt{\frac{2 \ln \frac{\operatorname{Vol}(\mathcal{G}^{\mathbb{D}})}{\delta}}{n}} +  2\sqrt{\frac{ L\mathcal{U}_{\infty}\mathcal{U}^{\tau}_{2}\ln \frac{\operatorname{Vol}(\mathcal{G}^{\mathbb{D}})}{\delta}}{3 n}}  + \sqrt{2\varepsilon^{\mathbb{D}}_{n}} + \sqrt{M}.
\$ 
where the last inequality comes from the Lipschitz continuity of $\mathbb{D}(\cdot)$. This completes the proof. 
\end{proof}

\subsection{Proof of Lemma \ref{reg_tau_l2_bound}}

\begin{lemma}
\label{reg_tau_l2_bound}
Define 
\$
\tau_{*}^{k} :=\argmax_{\tau}\bigg\{ {q}^{k}(s^0,\pi^{k}) +\frac{c^{*}}{(1-\gamma)n} \Big| \sum^{n}_{i=1} \tau(s_i,a_i)\left( {q}^{k}(s_i,a_i)-r_i -\gamma  {q}^{k}(s^{\prime}_i,\pi^{k})\right)\Big| - \lambda \xi_{n}\big(\tau (s_i,a_i)\big)\bigg\},
\$
for $\widetilde{q}^{\pi^{k}} := \inf_{q \in \mathcal{Q}} \mathbb{E}_{\mu}\left[\left(q(s,a)-\mathcal{B}^{\pi^{k}} q(s,a)\right)^{2}\right]$. Then for any each iteration $k \in [\bar{K}]$, the maximizer $\tau_{*}^{k}(s_i,a_i)$ at $k$-th iteration satisfies that 
\$
 \mathbb{P}_{n}\mathbb{D}(\tau_{*}^{k}(s_i,a_i)) \leq &  \frac{(1-\gamma)}{\lambda}\bigg(2\bar{V} +  \frac{c^{*}}{1-\gamma}\Bigg\{\big(3 \mathcal{U}^{\tau}_{2}\bar{V}+2\sqrt{2}\lambda \|\mathbb{D}(\tau(s,a))\|^{\text{UB}}_{L_2(\mu)}\big)\sqrt{\frac{2 \ln \frac{\operatorname{Vol}({\Theta}^{\dagger})}{\delta}}{n}}\\
& \qquad \qquad \qquad + \frac{2\big(3  \mathcal{U}^{\tau}_{\infty}\bar{V}+2\lambda\|\mathbb{D}(\tau(s,a))\|^{\text{UB}}_{L_{\infty}}\big) \ln \frac{\operatorname{Vol}({\Theta}^{\dagger})}{\delta}}{3 n} + \mathcal{U}^{\tau}_{\infty}\sqrt{\varepsilon_{\mathcal{Q}}}\bigg),
\$
where $\operatorname{Vol}({\Theta}^{\dagger}) = (e^{D}\max\{D_{\Omega},D_{\mathcal{Q}},D_{\Pi}\}+1)^3(\{1 \vee L\}{\mathcal{U}}^{\tau}_{2})^{2D}$ for $D=D_{\Omega}+D_{\mathcal{Q}}+D_{\Pi}$.
\end{lemma}

\begin{proof}
To proceed the proof, we first define a constant 
$
\tau_{0}(s,a) := 1 \,\text{for any} \, s,a
$, and we observe that 
\$
& {q}^{k}(s^0,\pi^{k}) + \sup_{\tau}\bigg\{\frac{c^{*}}{(1-\gamma)n} \Big| \sum^{n}_{i=1} \tau(s_i,a_i)\left({q}^{k}(s_i,a_i)-r_i -\gamma {q}^{k}(s^{\prime}_i,\pi^{k})\right)\Big|\bigg\} -\lambda\xi_{n}\big(\tau_{*}^{k}(s_i,a_i)\big) \\
\geq & {q}^{k}(s^0,\pi^{k}) + \sup_{\tau}\bigg\{\frac{c^{*}}{(1-\gamma)n} \Big| \sum^{n}_{i=1} \tau(s_i,a_i)\left( {q}^{k}(s_i,a_i)-r_i -\gamma  {q}^{k}(s^{\prime}_i,\pi^{k})\right)\Big|-\lambda\xi_{n}\big(\tau(s_i,a_i)\big) \bigg\} \\ 
\geq & \widetilde{q}^{\pi^k}(s^0,\pi^{k}) + \sup_{\tau}\bigg\{\frac{c^{*}}{(1-\gamma)n} \Big| \sum^{n}_{i=1} \tau(s_i,a_i)\left( \widetilde{q}^{\pi^k}(s_i,a_i)-r_i -\gamma  \widetilde{q}^{\pi^k}(s^{\prime}_i,\pi^{k})\right)\Big|-\lambda\xi_{n}\big(\tau(s_i,a_i)\big) \bigg\} \\ 
= &  \widetilde{q}^{\pi^k}(s^0,\pi^{k}) + \frac{c^{*}}{(1-\gamma)n} \Big| \sum^{n}_{i=1} \tau_{*}^{k}(s_i,a_i)\left( \widetilde{q}^{\pi^k}(s_i,a_i)-r_i -\gamma  \widetilde{q}^{\pi^k}(s^{\prime}_i,\pi^{k})\right)\Big| -\lambda\xi_{n}\big(\tau_{*}^{k}(s_i,a_i)\big) \\
\geq & \widetilde{q}^{\pi^k}(s^0,\pi^{k}) + \frac{c^{*}}{(1-\gamma)n} \Big| \sum^{n}_{i=1} \tau_{0}(s_i,a_i)\left( \widetilde{q}^{\pi^k}(s_i,a_i)-r_i -\gamma  \widetilde{q}^{\pi^k}(s^{\prime}_i,\pi^{k})\right)\Big| -\lambda\xi_{n}\big(\tau_{0}(s_i,a_i)\big) \\ 
= & \widetilde{q}^{\pi^k}(s^0,\pi^{k}) + \frac{c^{*}}{(1-\gamma)n} \Big| \sum^{n}_{i=1}\left( \widetilde{q}^{\pi^k}(s_i,a_i)-r_i -\gamma  \widetilde{q}^{\pi^k}(s^{\prime}_i,\pi^{k})\right)\Big| \geq -\bar{V} ,
\$
where the last inequality comes from the boundedness condition on $ \widetilde{q}^{\pi^k}$, and the non-negativity of the second term. Based on this, we further have 
\$
& \lambda\xi_{n}\big(\tau_{*}^{k} (s_i,a_i)\big) - \widetilde{q}^{\pi^k}(s^0,\pi^{k}) - \sup_{\tau}\bigg\{\frac{c^{*}}{(1-\gamma)n} \Big| \sum^{n}_{i=1} \tau(s_i,a_i)\left(\widetilde{q}^{\pi^k}(s_i,a_i)-r_i -\gamma \widetilde{q}^{\pi^k}(s^{\prime}_i,\pi^{k})\right)\Big|\bigg\}\leq \bar{V} ,
\$
which directly implies with Lemma \ref{misspecified_q_alpha}, 
\$
 \xi_{n}\big(\tau_{*}^{k} (s_i,a_i)\big) 
\leq & \frac{1}{\lambda}\bigg(2\bar{V} +  \frac{c^{*}}{1-\gamma}\Bigg\{\big(3 \mathcal{U}^{\tau}_{2}\bar{V}+2\sqrt{2}\lambda \|\mathbb{D}(\tau(s,a))\|^{\text{UB}}_{L_2(\mu)}\big)\sqrt{\frac{2 \ln \frac{\operatorname{Vol}({\Theta}^{\dagger})}{\delta}}{n}}\\
& \qquad \qquad \qquad + \frac{2\big(3  \mathcal{U}^{\tau}_{\infty}\bar{V}+2\lambda\|\mathbb{D}(\tau(s,a))\|^{\text{UB}}_{L_{\infty}}\big) \ln \frac{\operatorname{Vol}({\Theta}^{\dagger})}{\delta}}{3 n} + \mathcal{U}^{\tau}_{\infty}\sqrt{\varepsilon_{\mathcal{Q}}}\bigg).
\$
\end{proof}

\subsection{Proof of Lemma \ref{pik_return_lemma}}

\begin{lemma}
\label{pik_return_lemma}
For $k \in [\bar{K}]$, the following inequality holds, w.p. $\geq 1-\delta$, 
\$
& \frac{1}{\bar{K}}\sum^{\bar{K}}_{k=1}\left\{J(\pi^{k};\left\{\mathcal{S}, \mathcal{A}, \mathds{P}_{k}, \gamma, r_{k}, s^{0}\right\}) - J(\pi^{k})\right\} \\ 
\leq &\frac{\sqrt{\varepsilon_{\mathcal{Q}}}}{1-\gamma} +  \frac{c^{*}}{1-\gamma}\Bigg\{\big(3\mathcal{U}^{\star}_{2}\bar{V}+2\lambda \|\mathbb{D}(\tau(s,a))\|^{\text{prime}}_{L_2(\mu)}\big)\sqrt{\frac{2 \ln \frac{\operatorname{Vol}({\Theta}^{\dagger})}{\delta}}{n}} \notag\\
 &+ \frac{2\big(3 \mathcal{U}_{\infty}\mathcal{U}^{\star}_{2}\bar{V}+2\lambda \|\mathbb{D}(\tau(s,a))\|^{\text{prime}}_{L_\infty}\big) \ln \frac{\operatorname{Vol}({\Theta}^{\dagger})}{\delta}}{3 n} +  \mathcal{U}_{\infty}\mathcal{U}^{\star}_{2}\sqrt{\varepsilon_{\mathcal{Q}}}\Bigg\}+  \mathcal{O}\left(\lambda\mathcal{U}_{\infty}\mathcal{U}^{\star}_{2}\frac{\sqrt{\varepsilon_{\mathcal{Q}}}}{1-\gamma} \right) + \mathcal{O}(\frac{1}{\sqrt{n}}).
\$
where $\mathcal{U}^{\star}_{2} \in [1,\mathcal{U}^{\tau}_{2})$, can be choose via controlling $\lambda$ and $c^{*}$, and  $\|\mathbb{D}(\tau(s,a))\|^{\text{prime}}_{L_2(\mu)} = \sup_{\tau: \|\tau(s,a)\|_{L_{2}(\mu)} \leq \mathcal{U}^{\star}_{2}} \|\mathbb{D}(\tau(s,a))\|_{L_2(\mu)}$, and $\|\mathbb{D}(\tau(s,a))\|^{\text{prime}}_{L_\infty} = \sup_{\tau: \|\tau(s,a)\|_{L_{2}(\mu)} \leq \mathcal{U}^{\star}_{2}} \|\mathbb{D}(\tau(s,a))\|_{L_{\infty}}$, and $\operatorname{Vol}({\Theta}^{\dagger}) = (e^{D}\max\{D_{\Omega},D_{\mathcal{Q}},D_{\Pi}\}+1)^3(\{1 \vee L\}{\mathcal{U}}^{\tau}_{2})^{2D}$ for $D=D_{\Omega}+D_{\mathcal{Q}}+D_{\Pi}$.
\end{lemma}

\begin{proof}
To facilitate the proof, we first define some useful optimizers as follows:
\$
\widetilde{q}^{\pi^{k}} := & \inf_{q \in \mathcal{Q}} \mathbb{E}_{\mu}\left[\left(q(s,a)-\mathcal{B}^{\pi^{k}} q(s,a)\right)^{2}\right] \\
\tau_{*}^{k} :=&\argmax_{\tau}\bigg\{ q^{k}(s^0,\pi^{k}) +\frac{c^{*}}{(1-\gamma)n} \Big| \sum^{n}_{i=1} \tau(s_i,a_i)\left(q^{k}(s_i,a_i)-r_i -\gamma q^{k}(s^{\prime}_i,\pi^{k})\right)\Big| - \lambda\xi_{n}\big(\tau(s_i,a_i)\big)\bigg\} \\
\widetilde{\tau}^{k} :=&\argmax_{\tau}\bigg\{ \widetilde{q}^{\pi^{k}}(s^0,\pi^{k}) +\frac{c^{*}}{(1-\gamma)} \mathbb{E}_{\mu}\big[ \tau(s,a)\big( \widetilde{q}^{\pi^{k}}(s,a)-r_i -\gamma\widetilde{q}^{\pi^{k}}(s^{\prime},\pi^{k})\big)\big] - \lambda\mathbb{E}_{\mu}[\mathbb{D}(\tau (s,a))]\bigg\} \\ 
\tau_{\star}^{k} : =& \argmax_{\tau} \Bigg\{\min_{q}\bigg\{q(s^0,\pi^{k}) + \bigg\{\frac{c^{*}}{(1-\gamma)n} \Big| \sum^{n}_{i=1} \tau_{\psi}(s_i,a_i)\left(q(s_i,a_i)-r_i -\gamma q(s^{\prime}_i,\pi^{k})\right)\Big| - \lambda\xi_{n}\big(\tau(s_i,a_i)\big)\bigg\}\bigg\} \Bigg\}
\$
According to Lemma \ref{iden_mdp}, $q^{k}$ is the true action-value function with respect to the MDP $\left\{\mathcal{S}, \mathcal{A}, \mathds{P}_{k}, \gamma, r_{k}, s^{0}\right\}$, therefore we have, 
\$
& J(\pi^{k};\left\{\mathcal{S}, \mathcal{A}, \mathds{P}_{k}, \gamma, r_{k}, s^{0}\right\}) - J(\pi^{k})= q^{k}(s^0,\pi^{k}) - J(\pi^{k}). 
\$
Based on this observation, it suffices to upper-bound the following term, 
\#
 & q^{k}(s^0,\pi^{k}) - J(\pi^{k})\notag  \\
\leq & q^{k}(s^0,\pi^{k}) + \sup_{\tau}\frac{c^{*}}{(1-\gamma)n} \Big| \sum^{n}_{i=1}  \tau(s_i,a_i)\left(q^{k}(s_i,a_i)-r_i -\gamma q^{k}(s^{\prime}_i,\pi^{k})\right)\Big| - \lambda\xi_{n}\big(\widetilde{\tau}^{k}(s_i,a_i)\big) \notag\\
& + \lambda\xi_{n}\big(\widetilde{\tau}^{k}(s_i,a_i)\big)  - J(\pi^{k})\notag \\
\leq & q^{k}(s^0,\pi^{k}) + \sup_{\tau}\bigg\{\frac{c^{*}}{(1-\gamma)n} \Big| \sum^{n}_{i=1}  \tau(s_i,a_i)\left(q^{k}(s_i,a_i)-r_i -\gamma q^{k}(s^{\prime}_i,\pi^{k})\right)\Big| - \lambda\xi_{n}\big(\tau(s_i,a_i)\big)\bigg\} \notag \\
& + \lambda\xi_{n}\big(\widetilde{\tau}^{k}(s_i,a_i)\big)  - J(\pi^{k}),
\label{reg_thetak}
\#
Follow the definition of $q^{k}$ which is the minimizer of \$
q(s^0,\pi^{k}) + \sup_{\psi}\bigg\{\frac{c^{*}}{(1-\gamma)n} \Big| \sum^{n}_{i=1} \tau(s_i,a_i)\left(q(s_i,a_i)-r_i -\gamma q(s^{\prime}_i,\pi^{k})\right)\Big| - \lambda\xi_{n}\big(\tau_{\psi}(s_i,a_i)\big)\bigg\}.
\$
Therefore, for $\widetilde{q}^{\pi^{k}}$ in function class, we have 
\$
& q^{k}(s^0,\pi^{k}) + \sup_{\tau}\bigg\{\frac{c^{*}}{(1-\gamma)n} \Big| \sum^{n}_{i=1} \tau(s_i,a_i)\left(q^{k}(s_i,a_i)-r_i -\gamma q^{k}(s^{\prime}_i,\pi^{k})\right)\Big| - \lambda\xi_{n}\big(\tau(s_i,a_i)\big)\bigg\} \\
 =& \min_{q}\bigg\{q(s^0,\pi^{k}) + \sup_{\tau}\bigg\{\frac{c^{*}}{(1-\gamma)n} \Big| \sum^{n}_{i=1} \tau_{\psi}(s_i,a_i)\left(q(s_i,a_i)-r_i -\gamma q(s^{\prime}_i,\pi^{k})\right)\Big| - \lambda\xi_{n}\big(\tau(s_i,a_i)\big)\bigg\}\bigg\}\\
  \leq & \widetilde{q}^{\pi^{k}}(s^0,\pi^{k}) + \frac{c^{*}}{(1-\gamma)n} \Big| \sum^{n}_{i=1}  \tau_{\star}^{k} (s_i,a_i)\left(\widetilde{q}^{\pi^{k}}(s_i,a_i)-r_i -\gamma \widetilde{q}^{\pi^{k}}(s^{\prime}_i,\pi^{k})\right)\Big| - \lambda\xi_{n}\big(\tau_{\star}^{k} (s_i,a_i)\big)\\
 \leq & \widetilde{q}^{\pi^{k}}(s^0,\pi^{k}) + \sup_{\tau}\bigg\{\frac{c^{*}}{(1-\gamma)n} \Big| \sum^{n}_{i=1}  \tau(s_i,a_i)\left(\widetilde{q}^{\pi^{k}}(s_i,a_i)-r_i -\gamma \widetilde{q}^{\pi^{k}}(s^{\prime}_i,\pi^{k})\right)\Big| - \lambda\xi_{n}\big(\tau(s_i,a_i)\big)\bigg\}
\$
By this, we have \eqref{reg_thetak} is upper bounded by 
\$
& \widetilde{q}^{\pi^{k}}(s^0,\pi^{k}) + \sup_{\tau}\bigg\{\frac{c^{*}}{(1-\gamma)n} \Big| \sum^{n}_{i=1}  \tau(s_i,a_i)\left(\widetilde{q}^{\pi^{k}}(s_i,a_i)-r_i -\gamma \widetilde{q}^{\pi^{k}}(s^{\prime}_i,\pi^{k})\right)\Big| - \lambda\xi_{n}\big(\tau(s_i,a_i)\big)\bigg\} \\
& + \lambda\xi_{n}\big(\widetilde{\tau}^{k}(s_i,a_i)\big)  - J(\pi^{k})\\
\leq & \widetilde{q}^{\pi^{k}}(s^0,\pi^{k}) + \sup_{\tau}\bigg\{\frac{c^{*}}{(1-\gamma)n} \Big| \sum^{n}_{i=1} \tau_{\psi}(s_i,a_i)\left(\widetilde{q}^{\pi^{k}}(s_i,a_i)-r_i -\gamma \widetilde{q}^{\pi^{k}}(s^{\prime}_i,\pi^{k})\right)\Big| \bigg\} \\
& + \lambda\xi_{n}\big(\widetilde{\tau}^{k}(s_i,a_i)\big)  - J(\pi^{k})\\
= &  \underbrace{\widetilde{q}^{\pi^{k}}(s^0,\pi^{k}) - J(\pi^{k})}_{\Delta_1} +\underbrace{ \sup_{\tau}\bigg\{\frac{c^{*}}{(1-\gamma)n} \Big| \sum^{n}_{i=1} \tau(s_i,a_i)\left(\widetilde{q}^{\pi^{k}}(s_i,a_i)-r_i -\gamma \widetilde{q}^{\pi^{k}}(s^{\prime}_i,\pi^{k})\right)\Big| \bigg\}}_{\Delta_2} \\
& + \underbrace{ \lambda\xi_{n}\big(\widetilde{\tau}^{k}(s_i,a_i)\big)}_{\Delta_3}.
\$
\textbf{Bounding  $\Delta_1$.} According to Lemma \ref{eval_error}, we have 
\$
\Delta_1 =&  \frac{\mathbb{E}_{d^{\pi^{k}}}[\widetilde{q}^{\pi^k}(s,a)-r(s,a) -\gamma \widetilde{q}^{\pi^k}(s^{\prime},\pi^{k})]}{1-\gamma}
\leq & \frac{\|\widetilde{q}^{\pi^k}(s,a)-r(s,a) -\gamma \widetilde{q}^{\pi^k}(s^{\prime},\pi^{k})\|_{L(d^{\pi^{k}}})}{1-\gamma}.
\$
As $\pi^{k} \in \Pi$ and thus $d^{\pi^{k}}$ is an admissible data distribution, it follows Assumption 1 in maintext, it follows 
\$
\frac{\|\widetilde{q}^{\pi^k}(s,a)-r(s,a) -\gamma \widetilde{q}^{\pi^k}(s^{\prime},\pi^{k})\|_{L(d^{\pi^{k}}})}{1-\gamma} \leq \frac{\sqrt{\varepsilon_{\mathcal{Q}}}}{1-\gamma}.
\$
Thus, we have $\Delta_1 \leq \frac{\sqrt{\varepsilon_{\mathcal{Q}}}}{1-\gamma} $. 

\textbf{Bounding  $\Delta_2$.} According to Lemma \ref{misspecified_q_alpha} and \ref{prime_bound_tau}, with a proper choice of $\lambda$ and $c^{*}$ in the proof of Theorem 5.1, we can have a well-controlled uncertainty concentrability coefficient$\mathcal{U}^{\star}_{2}$. Replacing $\pi$ with $\pi^{k}$, we then have 
\$
 & \sup_{\tau}\bigg\{\frac{c^{*}}{(1-\gamma)n} \Big| \sum^{n}_{i=1} \tau(s_i,a_i)\left(\widetilde{q}^{\pi^{k}}(s_i,a_i)-r_i -\gamma \widetilde{q}^{\pi^{k}}(s^{\prime}_i,\pi^{k})\right)\Big| \bigg\}\\
 \leq & 
\frac{c^{*}}{1-\gamma}\Bigg\{\big(3\sqrt{2}\mathcal{U}^{\star}_{2}\bar{V}+2\sqrt{2}\lambda \|\mathbb{D}(\tau(s,a))\|^{\text{prime}}_{L_2(\mu)}\big)\sqrt{\frac{\ln\frac{\operatorname{Vol}({\Theta}^{\dagger})}{\delta}}{n}} \notag  \\
 & + \frac{\big(6\mathcal{U}^{\star}_{\infty}\bar{V}+4\lambda\|\mathbb{D}(\tau(s,a))\|^{\text{prime}}_{L_{\infty}}\big)\ln\frac{\operatorname{Vol}({\Theta}^{\dagger})}{\delta}}{3 n} 
+ \mathcal{U}^{\star}_{\infty}\sqrt{\varepsilon_{\mathcal{Q}}}\Bigg\}.
\$
where $\mathcal{U}^{\star}_{2} \in [1,\mathcal{U}^{\tau}_{2})$ and $\|\mathbb{D}(\tau(s,a))\|^{\text{prime}}_{L_2(\mu)} = \sup_{\tau: \|\tau(s,a)\|_{L_{2}(\mu)} \leq \mathcal{U}^{\star}_{2}} \|\mathbb{D}(\tau(s,a))\|_{L_2(\mu)}$, and $\|\mathbb{D}(\tau(s,a))\|^{\text{prime}}_{L_\infty} = \sup_{\tau: \|\tau(s,a)\|_{L_{2}(\mu)} \leq \mathcal{U}^{\star}_{2}} \|\mathbb{D}(\tau(s,a))\|_{L_{\infty}}$.

\textbf{Bounding $\Delta_3$.} Follow a similar argument in Proposition 1 of \citep{lee2021optidice} and the definition of $\widetilde{\tau}^{k}$, by some algebra, we have, for any $(s,a)$, 
\$
\widetilde{\tau}^{k}(s,a) = \left[\left(\mathbb{D}^{\prime}\right)^{-1}\left(\frac{c^{*}\big(r(s,a) +\gamma  \widetilde{q}^{\pi^k}(s^{\prime},\pi^k) - \widetilde{q}^{\pi^k}(s,a)\big)}{\lambda(1-\gamma)}\right)\right]^{+}.
\$
It follows from Lemma \ref{alpha_true}, for any $(s,a)$ such that $\mu(s,a) >0$, we have 
\$
& |r(s,a) +\gamma  \widetilde{q}^{\pi^k}(s^{\prime},\pi^k) - \widetilde{q}^{\pi^k}(s,a)| \\
\lesssim &  \big(\bar{V}+\lambda L\big)\sqrt{\frac{2 \ln \frac{\operatorname{Vol}({\Theta}^{\dagger})}{\delta}}{n}} + \frac{\mathcal{U}_{\infty}\big( \bar{V}+L\lambda\big) \ln \frac{\operatorname{Vol}({\Theta}^{\dagger})}{\delta}}{3 n} + \mathcal{U}_{\infty}\sqrt{\varepsilon_{\mathcal{Q}}} :=  \varepsilon^{\widetilde{q}^{\pi^k}}_n.
\$
And thus for any $(s_i,a_i)$ where $i=1,...,n$, we have 
\$
& \frac{1}{1-\gamma}c^{*}\big(r(s_i,a_i) +\gamma  \widetilde{q}^{\pi^k}(s^{\prime}_i,\pi^k) - \widetilde{q}^{\pi^k}(s_i,a_i)\big)
\lesssim \frac{1}{1-\gamma}c^{*} \varepsilon^{\widetilde{q}^{\pi^k}}_n.
\$
This directly leads to 
\$
(1-\gamma)\lambda\xi_{n}\big(\widetilde{\tau}^{k}(s_i,a_i)\big) \leq  &\frac{\lambda}{n}\sum^{n}_{i=1}\mathbb{D}\Bigg(\left[\left(\mathbb{D}^{\prime}\right)^{-1}\left(\frac{c^{*}\varepsilon^{\widetilde{q}^{\pi^k}}_n}{\lambda(1-\gamma)}\right)\right]^{+}\Bigg) \leq \lambda\mathbb{D}\Bigg(\left[\left(\mathbb{D}^{\prime}\right)^{-1}\left(\frac{c^{*}\varepsilon^{\widetilde{q}^{\pi^k}}_n}{\lambda(1-\gamma)}\right)\right]^{+}\Bigg).
\$
We note that for sufficient large $n$, we have $c^{*}\varepsilon^{\widetilde{q}^{\pi^k}}_n \rightarrow c^{*}\mathcal{U}_{\infty}\mathcal{U}^{\star}_{2}\sqrt{\varepsilon_{\mathcal{Q}}}$ and thus
\$
\left[\left(\mathbb{D}^{\prime}\right)^{-1}\left(\frac{c^{*}\varepsilon^{\widetilde{q}^{\pi^k}}_n}{\lambda(1-\gamma)}\right)\right]^{+} \stackrel{n \uparrow \infty}{\longrightarrow} \left[\left(\mathbb{D}^{\prime}\right)^{-1}(\frac{c^{*}\mathcal{U}_{\infty}\mathcal{U}^{\star}_{2}\sqrt{\varepsilon_{\mathcal{Q}}}}{\lambda(1-\gamma)}) \right]^{+}.
\$
As $\mathbb{D}^{\prime}$ is monotonic increasing, and since 
$
\frac{c^{*}\mathcal{U}_{\infty}\mathcal{U}^{\star}_{2}\sqrt{\varepsilon_{\mathcal{Q}}}}{\lambda(1-\gamma)} \geq 0 ,
$,
so 
$
\left(\mathbb{D}^{\prime}\right)^{-1}(\frac{c^{*}\mathcal{U}_{\infty}\mathcal{U}^{\star}_{2}\sqrt{\varepsilon_{\mathcal{Q}}}}{\lambda(1-\gamma)}) \geq 1,
$
and according the property of divergence function $\mathbb{D}^{\prime}(1) = 0$ thus 
\$
\left[\left(\mathbb{D}^{\prime}\right)^{-1}(\frac{c^{*}\mathcal{U}_{\infty}\mathcal{U}^{\star}_{2}\sqrt{\varepsilon_{\mathcal{Q}}}}{\lambda(1-\gamma)})\right]^{+} = \left(\mathbb{D}^{\prime}\right)^{-1}(\frac{c^{*}\mathcal{U}_{\infty}\mathcal{U}^{\star}_{2}\sqrt{\varepsilon_{\mathcal{Q}}}}{\lambda(1-\gamma)}),
\$
for sure. This immediately implies that 
\$
\mathbb{D}\Bigg(\left(\mathbb{D}^{\prime}\right)^{-1}\left(\frac{c^{*}\varepsilon^{\widetilde{q}^{\pi^k}}_n}{\lambda(1-\gamma)}\right)\Bigg) \stackrel{n \uparrow \infty}{\longrightarrow} \mathbb{D}\left(\left(\mathbb{D}^{\prime}\right)^{-1}(\frac{c^{*}\mathcal{U}_{\infty}\mathcal{U}^{\star}_{2}\sqrt{\varepsilon_{\mathcal{Q}}}}{\lambda(1-\gamma)}) \right) .
\$
 Therefore we conclude that 
\$
\Delta_3 \lesssim \lambda \mathbb{D}\left(\left(\mathbb{D}^{\prime}\right)^{-1}(\frac{c^{*}\mathcal{U}_{\infty}\mathcal{U}^{\star}_{2}\sqrt{\varepsilon_{\mathcal{Q}}}}{\lambda(1-\gamma)}) \right)
+ \mathcal{O}(\frac{1}{\sqrt{n}}).
\$
For sufficient small realizibility error $\sqrt{\varepsilon_{\mathcal{Q}}}$, and $c^{*} \asymp \lambda$, for bounded $\mathcal{U}_{\infty}$ and $\mathcal{U}^{\star}_{2}$, and also by the monotonicity of $\left(\mathbb{D}^{\prime}\right)^{-1}(\cdot)$, we have 
\$
\left(\mathbb{D}^{\prime}\right)^{-1}(\frac{c^{*}\mathcal{U}_{\infty}\mathcal{U}^{\star}_{2}\sqrt{\varepsilon_{\mathcal{Q}}}}{\lambda(1-\gamma)}) \asymp \left(\mathbb{D}^{\prime}\right)^{-1}(\frac{\mathcal{U}_{\infty}\mathcal{U}^{\star}_{2}\sqrt{\varepsilon_{\mathcal{Q}}}}{(1-\gamma)}) .
\$
According to the local Lipschitz continuity of $\mathbb{D}$ and $\mathbb{D}^{\prime}$ due to strongly convexity, we have 
\$
\lambda \mathbb{D}\left(\left(\mathbb{D}^{\prime}\right)^{-1}(\frac{\mathcal{U}_{\infty}\mathcal{U}^{\star}_{2}\sqrt{\varepsilon_{\mathcal{Q}}}}{(1-\gamma)}) \right) \leq  
 &\lambda \left|\mathbb{D}\left(\left(\mathbb{D}^{\prime}\right)^{-1}(\frac{\mathcal{U}_{\infty}\mathcal{U}^{\star}_{2}\sqrt{\varepsilon_{\mathcal{Q}}}}{(1-\gamma)}) \right)\right|\\
= & \lambda \left|\mathbb{D}\left(\left(\mathbb{D}^{\prime}\right)^{-1}(\frac{\mathcal{U}_{\infty}\mathcal{U}^{\star}_{2}\sqrt{\varepsilon_{\mathcal{Q}}}}{(1-\gamma)}) \right) - 0 \right| \\
\leq & \lambda L\left| \left(\mathbb{D}^{\prime}\right)^{-1}(\frac{\mathcal{U}_{\infty}\mathcal{U}^{\star}_{2}\sqrt{\varepsilon_{\mathcal{Q}}}}{(1-\gamma)})  - 1 \right| \\
\lesssim 
 & \lambda L\left|\frac{\mathcal{U}_{\infty}\mathcal{U}^{\star}_{2}\sqrt{\varepsilon_{\mathcal{Q}}}}{(1-\gamma)}  \right|.
\$
where $L $ is some Lipschtiz constants, and then we conclude that
\#
\Delta_3 \lesssim \mathcal{O}\left(\lambda\mathcal{U}_{\infty}\mathcal{U}^{\star}_{2}\frac{\sqrt{\varepsilon_{\mathcal{Q}}}}{1-\gamma} \right) + \mathcal{O}(\frac{1}{\sqrt{n}}).
\label{reg_asymp}
\#
Combine the results on bounding $\Delta_1$, $\Delta_2$ and $\Delta_3$, we have 
\$
 & J(\pi^{k};\left\{\mathcal{S}, \mathcal{A}, \mathds{P}_{k}, \gamma, r_{k}, s^{0}\right\}) - J(\pi^{k}) \leq \frac{\sqrt{\varepsilon_{\mathcal{Q}}}}{1-\gamma} +  \frac{c^{*}}{1-\gamma}\Bigg\{\big(3\mathcal{U}^{\star}_{2}\bar{V}+2\lambda \|\mathbb{D}(\tau(s,a))\|^{\text{prime}}_{L_2(\mu)}\big)\sqrt{\frac{2 \ln \frac{\operatorname{Vol}({\Theta}^{\dagger})}{\delta}}{n}} \notag\\
 &+ \frac{2\big(3 \mathcal{U}_{\infty}\mathcal{U}^{\star}_{2}\bar{V}+2\lambda \|\mathbb{D}(\tau(s,a))\|^{\text{prime}}_{L_\infty}\big) \ln \frac{\operatorname{Vol}({\Theta}^{\dagger})}{\delta}}{3 n} +  \mathcal{U}_{\infty}\mathcal{U}^{\star}_{2}\sqrt{\varepsilon_{\mathcal{Q}}}\Bigg\}+  \mathcal{O}\left(\lambda\mathcal{U}_{\infty}\mathcal{U}^{\star}_{2}\frac{\sqrt{\varepsilon_{\mathcal{Q}}}}{1-\gamma} \right) + \mathcal{O}(\frac{1}{\sqrt{n}}).
\$
As the above upper bound holds for any $k \in [\bar{K}]$, this completes the proof for $\frac{1}{\bar{K}}\sum^{\bar{K}}_{k=1}\left\{J(\pi^{k};\left\{\mathcal{S}, \mathcal{A}, \mathds{P}_{k}, \gamma, r_{k}, s^{0}\right\}) - J(\pi^{k})\right\}$. 
\end{proof}

\subsection{Proof of Lemma \ref{supp_lm1}}

\begin{lemma}
\label{supp_lm1}
For any $\pi \in \Pi$, and define the normalized negative entropy as 
\$
\mathcal{H}_{\text{NegEnt}}(\pi(\cdot|s)) = \sum_{a \in \mathcal{A}} \pi(a|s)\log(\pi(a|s)).
\$ 
Then we have 
\$
\sum^{\bar{K}}_{k=1}\left \langle \pi(\cdot|s) - \pi^{k+1}(\cdot|s), q^{k}(s,\cdot)\right \rangle \leq \mathcal{H}_{\text{NegEnt}}(\pi(\cdot|s))- \mathcal{H}_{\text{NegEnt}}(\pi^{0}(\cdot|s))
\$
where $\pi^{0}(\cdot|s)$ is the initial random policy for algorithm run.
\end{lemma}

\begin{proof}
To prove this Lemma, we use mathematical induction. Suppose the inequality holds for the round $(\bar{K}-1)$, i.e., 
\$
\sum^{\bar{K}-1}_{k=1}\left \langle \pi(\cdot|s) - \pi^{k+1}(\cdot|s), q^{k}(s,\cdot)\right \rangle \leq \mathcal{H}_{\text{NegEnt}}(\pi(\cdot|s))- \mathcal{H}_{\text{NegEnt}}(\pi^{0}(\cdot|s))
\$
Then we verify the inequality for $\bar{K}$ in the following, 
\$
& \sum^{\bar{K}}_{k=1}\left \langle \pi(\cdot|s) - \pi^{k+1}(\cdot|s), q^{k}(s,\cdot)\right \rangle \\ 
= & \sum^{\bar{K}}_{k=1}\left \langle \pi(\cdot|s), q^{k}(s,\cdot)\right \rangle - \sum^{\bar{K}}_{k=1}\left \langle \pi^{k+1}(\cdot|s), q^{k}(s,\cdot)\right \rangle - \mathcal{H}_{\text{NegEnt}}(\pi(\cdot|s)) + \mathcal{H}_{\text{NegEnt}}(\pi(\cdot|s)) \\ 
\leq & \sum^{\bar{K}}_{k=1}\left \langle \pi_{\bar{K}+1}(\cdot|s), q^{k}(s,\cdot)\right \rangle - \sum^{\bar{K}}_{k=1}\left \langle \pi^{k+1}(\cdot|s), q^{k}(s,\cdot)\right \rangle \\
& - \mathcal{H}_{\text{NegEnt}}(\pi_{\bar{K}+1}(\cdot|s)) + \mathcal{H}_{\text{NegEnt}}(\pi(\cdot|s)) \\
= & \sum^{\bar{K}-1}_{k=1}\left \langle \pi_{\bar{K}+1}(\cdot|s), q^{k}(s,\cdot)\right \rangle - \sum^{\bar{K}}_{k=1}\left \langle \pi^{k+1}(\cdot|s), q^{k}(s,\cdot)\right \rangle + \left \langle \pi_{ \bar{K}+1}(\cdot|s), q^{k}(s,\cdot)\right \rangle\\
& - \mathcal{H}_{\text{NegEnt}}(\pi_{\bar{K}+1}(\cdot|s)) +  \mathcal{H}_{\text{NegEnt}}(\pi(\cdot|s)) \\ 
= & \sum^{\bar{K}-1}_{k=1}\left \langle \pi_{\bar{K}+1}(\cdot|s), q^{k}(s,\cdot)\right \rangle - \sum^{\bar{K}-1}_{k=1}\left \langle \pi^{k+1}(\cdot|s), q^{k}(s,\cdot)\right \rangle + \left \langle \pi_{ \bar{K}+1}(\cdot|s), q^{k}(s,\cdot)\right \rangle \\
& - \left \langle \pi_{ \bar{K}+1}(\cdot|s), q^{k}(s,\cdot)\right \rangle  - \mathcal{H}_{\text{NegEnt}}(\pi_{\bar{K}+1}(\cdot|s)) +  \mathcal{H}_{\text{NegEnt}}(\pi(\cdot|s)) \\
\leq & \mathcal{H}_{\text{NegEnt}}(\pi_{ \bar{K}+1}(\cdot|s))- \mathcal{H}_{\text{NegEnt}}(\pi^{0}(\cdot|s)) -  \mathcal{H}_{\text{NegEnt}}(\pi_{\bar{K}+1}(\cdot|s)) +  \mathcal{H}_{\text{NegEnt}}(\pi(\cdot|s)) \\
= & \mathcal{H}_{\text{NegEnt}}(\pi(\cdot|s))  - \mathcal{H}_{\text{NegEnt}}(\pi^{0}(\cdot|s)). 
\$
This completes the proof. 
\end{proof}

\subsection{Proof of Lemma \ref{seq_bound}}

\begin{lemma}
\label{seq_bound}
For any policy $\pi$, it satisfies that 
\$
\sum^{\bar{K}}_{k=1}\left \langle \pi - \pi^{k}(\cdot|s), q^{k}(s,\cdot)\right \rangle \leq 
2\sqrt{2\bar{V}\bar{K}\log|\mathcal{A}|}. 
\$
\end{lemma}

\begin{proof}
Following the definition of the Bergman divergence in terms of negative entropy, we have that 
\#
D_{\text{NegEntropy}}(\pi^{k}, \pi^{k+1}) = & \mathcal{H}_{\text{NegEnt}}(\pi^{k}(\cdot|s))  - \mathcal{H}_{\text{NegEnt}}(\pi^{k+1}(\cdot|s)) \\
& - \left \langle \nabla \mathcal{H}_{\text{NegEnt}}(\pi^{k+1}(\cdot|s)), \pi^{k}(\cdot|s) - \pi^{k+1}(\cdot|s)  \right\rangle.
\label{neg_eq1}
\#
By the second-order Taylor expansion on $\mathcal{H}_{\text{NegEnt}}(\pi^{k}(\cdot|s))$ and evaluated at $\pi^{k+1}(\cdot|s)$, we have 
\#
\mathcal{H}_{\text{NegEnt}}(\pi^{k}(\cdot|s)) = & \mathcal{H}_{\text{NegEnt}}(\pi^{k+1}(\cdot|s)) - \left \langle \nabla \mathcal{H}_{\text{NegEnt}}(\pi^{k+1}(\cdot|s)), \pi^{k}(\cdot|s) - \pi^{k+1}(\cdot|s)  \right\rangle \notag \\
& + \frac{1}{2} (\pi^{k}(\cdot|s) - \pi^{k+1}(\cdot|s) )^{T}\nabla^2 \mathcal{H}_{\text{NegEnt}}(\pi_{k||k-1}(\cdot|s))(\pi^{k}(\cdot|s) - \pi^{k+1}(\cdot|s) ),
\label{neg_eq2}
\#
where $\pi_{k||k-1}(\cdot|s)$ lies on the line connecting $\pi^{k}(\cdot|s) $ and $\pi^{k+1}(\cdot|s) $. With \eqref{neg_eq1} and \eqref{neg_eq2}, we do the subtraction, then it obtains
\$
D_{\text{NegEntropy}}(\pi^{k}, \pi^{k+1}) =  \frac{1}{2} (\pi^{k}(\cdot|s) - \pi^{k+1}(\cdot|s) )^{T}\nabla^2 \mathcal{H}_{\text{NegEnt}}(\pi_{k||k-1}(\cdot|s))(\pi^{k}(\cdot|s) - \pi^{k+1}(\cdot|s) ).
\$
Then we proceed to bound
\#
& \left\langle \pi^{k}(\cdot|s) - \pi^{k+1}(\cdot|s), q^{k}(s,\cdot) \right\rangle \\
\leq  & 
\sqrt{q^{k}(s,\cdot)^{T}\nabla^{-2} \mathcal{H}_{\text{NegEnt}}(\pi_{k||k-1}(\cdot|s))q^{k}(s,\cdot)} \notag \\
& \cdot 
\sqrt{(\pi^{k}(\cdot|s) - \pi^{k+1}(\cdot|s) )^{T}\nabla^2 \mathcal{H}_{\text{NegEnt}}(\pi_{k||k-1}(\cdot|s))(\pi^{k}(\cdot|s) - \pi^{k+1}(\cdot|s) )} \notag  \\ 
&  = \sqrt{2q^{k}(s,\cdot)^{T}\nabla^{-2} \mathcal{H}_{\text{NegEnt}}(\pi_{k||k-1}(\cdot|s))q^{k}(s,\cdot)D_{\text{NegEntropy}}(\pi^{k}, \pi^{k+1})} \notag \\
&  \leq  \sqrt{2\zeta}\|q^{k} \|_{L_{\infty}}\sqrt{ D_{\text{NegEntropy}}(\pi^{k}, \pi^{k+1})},
\label{optim_main_eq1}
\#
where $\zeta$ is defined in Algorithm 1 maintext, which indicates the projection rate. 
Next, we aim to upper bound $\sqrt{ D_{\text{NegEntropy}}(\pi^{k}, \pi^{k+1})}$. 
Follow the soft policy improvement Lemma 2 in \citep{haarnoja2018softac} that 
$\pi^{k+1}$ is the global maximizer of $\sum^{k}_{k^{\prime}=1} \langle \pi(\cdot|s), q^{k^{\prime}}(s,\cdot)\rangle - \mathcal{H}_{\text{NegEnt}}(\pi(\cdot|s))$. 
By 
\$
0 = &  \sum^{k}_{k^{\prime}=1} \langle \pi^{k}(\cdot|s), q^{k^{\prime}}(s,\cdot)\rangle - \mathcal{H}_{\text{NegEnt}}(\pi^{k}(\cdot|s)) - \left(\sum^{k}_{k^{\prime}=1} \langle \pi^{k+1}(\cdot|s), q^{k^{\prime}}(s,\cdot)\rangle - \mathcal{H}_{\text{NegEnt}}(\pi^{k+1}(\cdot|s)) \right)  \\
& -  \left\langle \pi^{k}(\cdot|s) - \pi^{k+1}(\cdot|s), \nabla_{\pi} \sum^{k}_{k^{\prime}=1} \langle \pi^{k+1}(\cdot|s), q^{k^{\prime}}(s,\cdot)\rangle - \mathcal{H}_{\text{NegEnt}}(\pi^{k+1}(\cdot|s)) \right\rangle  \\
& - \Bigg( \sum^{k}_{k^{\prime}=1} \langle \pi^{k}(\cdot|s), q^{k^{\prime}}(s,\cdot)\rangle - \mathcal{H}_{\text{NegEnt}}(\pi^{k}(\cdot|s)) - \left(\sum^{k}_{k^{\prime}=1} \langle \pi^{k+1}(\cdot|s), q^{k^{\prime}}(s,\cdot)\rangle - \mathcal{H}_{\text{NegEnt}}(\pi^{k+1}(\cdot|s)) \right)  \\
& \qquad - \left\langle \pi^{k}(\cdot|s) - \pi^{k+1}(\cdot|s), \nabla_{\pi} \sum^{k}_{k^{\prime}=1} \langle \pi^{k+1}(\cdot|s), q^{k^{\prime}}(s,\cdot)\rangle - \mathcal{H}_{\text{NegEnt}}(\pi^{k+1}(\cdot|s)) \right\rangle \Bigg) .
\$
By  $\pi^{k+1}$ is the maximizer of $\sum^{k}_{k^{\prime}=1} \langle \pi(\cdot|s), q^{k^{\prime}}(s,\cdot)\rangle - \mathcal{H}_{\text{NegEnt}}(\pi(\cdot|s))$, then 
\$
\nabla_{\pi} \sum^{k}_{k^{\prime}=1} \langle \pi^{k+1}(\cdot|s), q^{k^{\prime}}(s,\cdot)\rangle - \mathcal{H}_{\text{NegEnt}}(\pi^{k+1}(\cdot|s))  =0.
\$
This implies that 
\$
&\sum^{k}_{k^{\prime}=1} \langle \pi^{k}(\cdot|s), q^{k^{\prime}}(s,\cdot)\rangle - \mathcal{H}_{\text{NegEnt}}(\pi^{k}(\cdot|s)) - 
\left(\sum^{k}_{k^{\prime}=1} \langle \pi^{k+1}(\cdot|s), q^{k^{\prime}}(s,\cdot)\rangle - \mathcal{H}_{\text{NegEnt}}(\pi^{k+1}(\cdot|s)) \right) \\
= & \Bigg( \sum^{k}_{k^{\prime}=1} \langle \pi^{k}(\cdot|s), q^{k^{\prime}}(s,\cdot)\rangle - \mathcal{H}_{\text{NegEnt}}(\pi^{k}(\cdot|s)) - \left(\sum^{k}_{k^{\prime}=1} \langle \pi^{k+1}(\cdot|s), q^{k^{\prime}}(s,\cdot)\rangle - \mathcal{H}_{\text{NegEnt}}(\pi^{k+1}(\cdot|s)) \right)  \\
& - \left\langle \pi^{k}(\cdot|s) - \pi^{k+1}(\cdot|s), \nabla_{\pi} \sum^{k}_{k^{\prime}=1} \langle \pi^{k+1}(\cdot|s), q^{k^{\prime}}(s,\cdot)\rangle - \mathcal{H}_{\text{NegEnt}}(\pi^{k+1}(\cdot|s)) \right\rangle \Bigg) \\
= &  \sum^{k}_{k^{\prime}=1} \langle \pi^{k}(\cdot|s), q^{k^{\prime}}(s,\cdot)\rangle - \sum^{k}_{k^{\prime}=1} \langle \pi^{k+1}(\cdot|s), q^{k^{\prime}}(s,\cdot)\rangle + \mathcal{H}_{\text{NegEnt}}(\pi^{k+1}(\cdot|s)) - \mathcal{H}_{\text{NegEnt}}(\pi^{k}(\cdot|s)) \\
& - \left\langle \pi^{k}(\cdot|s) - \pi^{k+1}(\cdot|s), \nabla_{\pi} \sum^{k}_{k^{\prime}=1} \langle \pi^{k+1}(\cdot|s), q^{k^{\prime}}(s,\cdot)\rangle  \right\rangle  -  \left\langle \pi^{k}(\cdot|s) - \pi^{k+1}(\cdot|s), \nabla_{\pi}  \mathcal{H}_{\text{NegEnt}}(\pi^{k+1}(\cdot|s)) \right\rangle.
\$
Since
\$
\sum^{k}_{k^{\prime}=1} \langle \pi^{k}(\cdot|s), q^{k^{\prime}}(s,\cdot)\rangle - \sum^{k}_{k^{\prime}=1} \langle \pi^{k+1}(\cdot|s), q^{k^{\prime}}(s,\cdot)\rangle \leq  \left\langle \pi^{k}(\cdot|s) - \pi^{k+1}(\cdot|s), \nabla_{\pi} \sum^{k}_{k^{\prime}=1} \langle \pi^{k+1}(\cdot|s), q^{k^{\prime}}(s,\cdot)\rangle  \right\rangle .
\$
Thus we have 
\$
&\sum^{k}_{k^{\prime}=1} \langle \pi^{k}(\cdot|s), q^{k^{\prime}}(s,\cdot)\rangle - \mathcal{H}_{\text{NegEnt}}(\pi^{k}(\cdot|s)) - 
\left(\sum^{k}_{k^{\prime}=1} \langle \pi^{k+1}(\cdot|s), q^{k^{\prime}}(s,\cdot)\rangle - \mathcal{H}_{\text{NegEnt}}(\pi^{k+1}(\cdot|s)) \right) \\
\leq & \mathcal{H}_{\text{NegEnt}}(\pi^{k+1}(\cdot|s)) - \mathcal{H}_{\text{NegEnt}}(\pi^{k}(\cdot|s)) - \left\langle \pi^{k}(\cdot|s) - \pi^{k+1}(\cdot|s), \nabla_{\pi}  \mathcal{H}_{\text{NegEnt}}(\pi^{k+1}(\cdot|s)) \right\rangle \\ 
=  & - D_{\text{NegEntropy}}(\pi^{k}, \pi^{k+1}).
\$
This implies that
\$
D_{\text{NegEntropy}}(\pi^{k}(\cdot|s), \pi^{k+1}(\cdot|s))  \leq &\sum^{k}_{k^{\prime}=1} \langle \pi^{k+1}(\cdot|s), q^{k^{\prime}}(s,\cdot)\rangle - \mathcal{H}_{\text{NegEnt}}(\pi^{k+1}(\cdot|s)) \\
& - \left(\sum^{k}_{k^{\prime}=1} \langle \pi^{k}(\cdot|s), q^{k^{\prime}}(s,\cdot)\rangle - \mathcal{H}_{\text{NegEnt}}(\pi^{k}(\cdot|s)) \right) \\
=& \sum^{k-1}_{k^{\prime}=1} \langle \pi^{k+1}(\cdot|s), q^{k^{\prime}}(s,\cdot)\rangle - \mathcal{H}_{\text{NegEnt}}(\pi^{k+1}(\cdot|s)) \\ & -\left(\sum^{k-1}_{k^{\prime}=1} \langle \pi^{k}(\cdot|s), q^{k^{\prime}}(s,\cdot)\rangle - \mathcal{H}_{\text{NegEnt}}(\pi^{k}(\cdot|s)) \right) \\
& -\left \langle  \pi^{k}(\cdot|s) - \pi^{k+1}(\cdot|s), q^{k}(s,\cdot) \right \rangle.
\$
Since
\$
&\sum^{k-1}_{k^{\prime}=1} \langle \pi^{k+1}(\cdot|s), q^{k^{\prime}}(s,\cdot)\rangle - \mathcal{H}_{\text{NegEnt}}(\pi^{k+1}(\cdot|s)) \\ & -\left(\sum^{k-1}_{k^{\prime}=1} \langle \pi^{k}(\cdot|s), q^{k^{\prime}}(s,\cdot)\rangle - \mathcal{H}_{\text{NegEnt}}(\pi^{k}(\cdot|s)) \right) \leq 0 
\$
Then we have 
\$
D_{\text{NegEntropy}}(\pi^{k}(\cdot|s), \pi^{k+1}(\cdot|s)) \leq -\left \langle  \pi^{k}(\cdot|s) - \pi^{k+1}(\cdot|s), q^{k}(s,\cdot) \right \rangle
\$
Combine with \eqref{optim_main_eq1} and boundedness condition on $q$-function, we have 
\$
-\left\langle \pi^{k}(\cdot|s) - \pi^{k+1}(\cdot|s), q^{k}(s,\cdot) \right\rangle \leq  & \sqrt{2\zeta}\bar{V}\sqrt{\left \langle  \pi^{k+1}(\cdot|s) - \pi^{k}(\cdot|s), q^{k}(s,\cdot) \right \rangle}\\
\leq & \sqrt{2\zeta \bar{V}\left \langle  \pi^{k+1}(\cdot|s) - \pi^{k}(\cdot|s), q^{k}(s,\cdot) \right \rangle}.
\$
Solving the equation 
\$
\left\langle \pi^{k}(\cdot|s) - \pi^{k+1}(\cdot|s), q^{k}(s,\cdot) \right\rangle^2 - 2\zeta\bar{V}^2 \left\langle \pi^{k}(\cdot|s) - \pi^{k+1}(\cdot|s), q^{k}(s,\cdot) \right\rangle = 0 .
\$
We obtain that 
\#
-\left\langle \pi^{k}(\cdot|s) - \pi^{k+1}(\cdot|s), q^{k}(s,\cdot) \right\rangle \leq \zeta2\bar{V}.
\label{part_res1}
\#
Then we proceed to bound 
\$
\sum^{\bar{K}}_{k=1}\left \langle \pi - \pi^{k}(\cdot|s), q^{k}(s,\cdot)\right \rangle 
 \leq  & \sum^{\bar{K}}_{k=1}\left \langle \pi(\cdot|s), q^{k}(s,\cdot)\right \rangle -  \mathcal{H}_{\text{NegEnt}}(\pi(\cdot|s)) \\
& - \left(  \sum^{\bar{K}}_{k=1}\left \langle \pi^{k}(\cdot|s), q^{k}(s,\cdot)\right \rangle- \mathcal{H}_{\text{NegEnt}}(\pi(\cdot|s))\right) \\
 = & \sum^{\bar{K}}_{k=1}\left \langle \pi(\cdot|s) - \pi^{k+1}(\cdot|s), q^{k}(s,\cdot)\right \rangle \\
 & + \sum^{\bar{K}}_{k=1}\left \langle \pi^{k+1}(\cdot|s), q^{k}(s,\cdot) \right\rangle -  \mathcal{H}_{\text{NegEnt}}(\pi(\cdot|s)) \\
  & - \left(  \sum^{\bar{K}}_{k=1}\left \langle \pi^{k}(\cdot|s), q^{k}(s,\cdot)\right \rangle- \mathcal{H}_{\text{NegEnt}}(\pi(\cdot|s))\right).
  \$
Leverage Lemma \ref{supp_lm1},
  \$
 \sum^{\bar{K}}_{k=1}\left \langle \pi - \pi^{k}(\cdot|s), q^{k}(s,\cdot)\right \rangle
 \leq & \ \mathcal{H}_{\text{NegEnt}}(\pi(\cdot|s))- \mathcal{H}_{\text{NegEnt}}(\pi^{0}(\cdot|s))\\
 & + \sum^{\bar{K}}_{k=1}\left \langle \pi^{k+1}(\cdot|s), q^{k}(s,\cdot) \right\rangle -  \mathcal{H}_{\text{NegEnt}}(\pi(\cdot|s)) \\
  & - \left(  \sum^{\bar{K}}_{k=1}\left \langle \pi^{k}(\cdot|s), q^{k}(s,\cdot)\right \rangle- \mathcal{H}_{\text{NegEnt}}(\pi(\cdot|s))\right)\\
 =&  \sum^{\bar{K}}_{k=1}\left \langle \pi^{k+1}(\cdot|s), q^{k}(s,\cdot) \right\rangle -  \sum^{\bar{K}}_{k=1}\left \langle \pi^{k}(\cdot|s), q^{k}(s,\cdot) \right\rangle \\
 & - \mathcal{H}_{\text{NegEnt}}(\pi^{0}(\cdot|s)) + \mathcal{H}_{\text{NegEnt}}(\pi(\cdot|s))\\
  \leq & \sum^{\bar{K}}_{k=1}\left \langle \pi^{k+1}(\cdot|s), q^{k}(s,\cdot) \right\rangle -  \sum^{\bar{K}}_{k=1}\left \langle \pi^{k}(\cdot|s), q^{k}(s,\cdot) \right\rangle \\
 & - \mathcal{H}_{\text{NegEnt}}(\pi^{0}(\cdot|s)).
\$
Combine with the inequality \eqref{part_res1}, we have 
\$
 \sum^{\bar{K}}_{k=1}\left \langle \pi - \pi^{k}(\cdot|s), q^{k}(s,\cdot)\right \rangle & \leq \sum^{\bar{K}}_{k=1}\zeta2\bar{V} - \frac{1}{\bar{K}} \mathcal{H}_{\text{NegEnt}}(\pi^{0}(\cdot|s))\\
 \leq & \zeta^{-1}(\zeta^22\bar{V}\bar{K}-\log\frac{1}{|\mathcal{A}|}). 
\$
Minimizing the $\zeta^{-1}(2\zeta^2\bar{V}\bar{K}+\log|\mathcal{A}|)$ over $\zeta$, we set
$
\zeta = \sqrt{\frac{\log|\mathcal{A}|}{\bar{K}2\bar{V}}}$ and thus
\#
\sum^{\bar{K}}_{k=1}\left \langle \pi - \pi^{k}(\cdot|s), q^{k}(s,\cdot)\right \rangle \leq &  \sqrt{2\bar{V}\bar{K}\log|\mathcal{A}|} + \frac{\sqrt{2\bar{V}\bar{K}}\log|\mathcal{A}|}{\sqrt{\log|\mathcal{A}|}} \notag \\
= & 2\sqrt{2\bar{V}\bar{K}\log|\mathcal{A}|}. 
\label{seq_bound_eq}
\#
\end{proof}

\subsection{Proof of Lemma \ref{iden_mdpret}}

\begin{lemma}
\label{iden_mdpret}
For any policy $\pi$, the average regret 
\$
\frac{1}{\bar{K}}\sum^{\bar{K}}_{k=1}\left \{ J(\pi;\left\{\mathcal{S}, \mathcal{A}, \mathds{P}_{k}, \gamma, r_{k}, s^{0}\right\}) - J(\pi^{k};\left\{\mathcal{S}, \mathcal{A}, \mathds{P}_{k}, \gamma, r_{k}, s^{0}\right\})\right \} 
\leq 
\frac{2\sqrt{2\bar{V}\log|\mathcal{A}|}}{\sqrt{\bar{K}}(1-\gamma)}.
\$
\end{lemma}

\begin{proof}
To faciliate the proof, we denote $\mathbb{E}^{k}[\cdot]$ is the expectation taken to the system of iterated MDP $\left\{\mathcal{S}, \mathcal{A}, \mathds{P}_{k}, \gamma, r_{k}, s^{0}\right\}$. It follows the definitions of the discounted return, we have 
\$
& \frac{1}{\bar{K}}\sum^{\bar{K}}_{k=1} J(\pi;\left\{\mathcal{S}, \mathcal{A}, \mathds{P}_{k}, \gamma, r_{k}, s^{0}\right\}) - J(\pi^{k};\left\{\mathcal{S}, \mathcal{A}, \mathds{P}_{k}, \gamma, r_{k}, s^{0}\right\}) \\
= & \frac{\frac{1}{\bar{K}}\sum^{\bar{K}}_{k=1} \mathbb{E}^{k}_{d^{\pi}}[q^{k}(s,\pi) - q^{k}(s,\pi^{k})]}{1-\gamma} =  \frac{\frac{1}{\bar{K}}\sum^{\bar{K}}_{k=1} \mathbb{E}^{k}_{d^{\pi}}\left[\left \langle q^{k}(s,\cdot), \pi(\cdot|s) - \pi^{k}(\cdot|s)\right\rangle \right]}{1-\gamma}
\$
As the dynamics of $\left\{\mathcal{S}, \mathcal{A}, \mathds{P}_{k}, \gamma, r_{k}, s^{0}\right\}$ is identical to $\left\{\mathcal{S}, \mathcal{A}, \mathds{P}, \gamma, r, s^{0}\right\}$ except for the reward functions. Therefore,  
\$
& \frac{\frac{1}{\bar{K}}\sum^{\bar{K}}_{k=1} \mathbb{E}^{k}_{d^{\pi}}\left[\left \langle q^{k}(s,\cdot), \pi(\cdot|s) - \pi^{k}(\cdot|s)\right\rangle \right]}{1-\gamma} \\
=  & \mathbb{E}_{d^{\pi}}\left[\frac{\frac{1}{\bar{K}}\sum^{\bar{K}}_{k=1}\left \langle q^{k}(s,\cdot), \pi(\cdot|s) - \pi^{k}(\cdot|s)\right\rangle }{1-\gamma}\right] 
\leq  \frac{2\sqrt{2\bar{V}\log|\mathcal{A}|}}{\sqrt{\bar{K}}(1-\gamma)}.
\$
where the last inequality comes from Lemma \ref{seq_bound}. 
\end{proof}

\subsection{Proof of Lemma \ref{iden_concen}}

\begin{lemma}
\label{iden_concen}
For any $k \in [\bar{K}]$, we have
\$
& \sup_{\tau}\bigg\{\frac{1}{(1-\gamma)}  \mathbb{E}_{\mu} \left[\tau(s,a)\left(r(s,a) +\gamma q^{k}(s^{\prime},\pi^{k})-q^{k}(s,a)\right)\right]\bigg\}\\
\leq & 
\frac{1}{1-\gamma}\Bigg\{2\big(2\mathcal{U}^{\star}_{2}\bar{V}+2\lambda \|\mathbb{D}(\tau(s,a))\|^{\text{prime}}_{L_2(\mu)}\big)\sqrt{\frac{2 \ln \frac{\operatorname{Vol}({\Theta}^{\dagger})}{\delta}}{n}} \notag\\
 &+ \frac{4\big(2\mathcal{U}_{\infty}\mathcal{U}^{\star}_{2}\bar{V}+2\lambda \|\mathbb{D}(\tau(s,a))\|^{\text{prime}}_{L_\infty}\big) \ln \frac{\operatorname{Vol}({\Theta}^{\dagger})}{\delta}}{3 n} +  \mathcal{U}_{\infty}\mathcal{U}^{\star}_{2}\sqrt{\varepsilon_{\mathcal{Q}}}\Bigg\}\notag \\
    & + \mathcal{O}\left(\frac{\lambda\mathcal{U}_{\infty}\mathcal{U}^{\star}_{2}\frac{\sqrt{\varepsilon_{\mathcal{Q}}}}{1-\gamma}}{c^{*}} \right)
    + \frac{\bar{V}}{c^{*}} + \mathcal{O}(\frac{\bar{V}\mathcal{U}_{\infty}\mathcal{U}^{\star}_{2}}{n}) + \mathcal{O}(\frac{\bar{V}\mathcal{U}_{\infty}\mathcal{U}^{\star}_{2}(1+\gamma)}{n}) + \mathcal{O}(\frac{\gamma \bar{V}}{n}) +\mathcal{O}\left(\frac{1}{\sqrt{n}}\right),
\$
where $\operatorname{Vol}({\Theta}^{\dagger}) = (e^{D}\max\{D_{\Omega},D_{\mathcal{Q}},D_{\Pi}\}+1)^3(\{1 \vee L\}{\mathcal{U}}^{\tau}_{2})^{2D}$ for $D=D_{\Omega}+D_{\mathcal{Q}}+D_{\Pi}$. 
\end{lemma}

\begin{proof}
To complete the proof, it suffices to show 
\$
\sup_{\tau}\bigg\{\frac{1}{(1-\gamma)n} \Big| \sum^{n}_{i=1} \tau(s_i,a_i)\left(q^{k}(s_i,a_i)-r_i -\gamma q^{k}(s^{\prime}_i,\pi^{k})\right)\Big| \bigg\}, 
\$
is upper bounded for any $k \in [\bar{K}]$. To facilitate the proof, we define 
\$
\widetilde{\tau}^{k} :=&\argmax_{\tau}\bigg\{ \widetilde{q}^{\pi^{k}}(s^0,\pi^{k}) +\frac{c^{*}}{(1-\gamma)} \mathbb{E}_{\mu}\big[ \tau(s,a)\big( \widetilde{q}^{\pi^{k}}(s,a)-r_i -\gamma\widetilde{q}^{\pi^{k}}(s^{\prime},\pi^{k})\big)\big] - \frac{\lambda}{1-\gamma}\mathbb{E}_{\mu}[\mathbb{D}(\tau (s,a))]\bigg\},
\$
and define
\$
\Delta_1 := q^{k}(s^0,\pi^{k}) + \sup_{\tau}\bigg\{\frac{c^{*}}{(1-\gamma)n} \Big| \sum^{n}_{i=1} \tau(s_i,a_i)\left(q^{k}(s_i,a_i)-r_i -\gamma q^{k}(s^{\prime}_i,\pi^{k})\right)\Big| \bigg\} .
\$
Then obviously we have
\$
\Delta_1 = & \Delta_1 - \lambda\xi_{n}\big(\widetilde{\tau}^{k}(s_i,a_i)\big) + \lambda\xi_{n}\big(\widetilde{\tau}^{k}(s_i,a_i)\big) \\ 
\leq &  q^{k}(s^0,\pi^{k}) + \sup_{\tau}\bigg\{\frac{c^{*}}{(1-\gamma)n} \Big| \sum^{n}_{i=1} \tau(s_i,a_i)\left(q^{k}(s_i,a_i)-r_i -\gamma q^{k}(s^{\prime}_i,\pi^{k})\right)\Big| - \lambda\xi_{n}\big( \tau(s_i,a_i)\big) \bigg\} \\
& + \lambda\xi_{n}\big(\widetilde{\tau}^{k}(s_i,a_i)\big). 
\$

where the first inequality comes from $\widetilde{\tau}^{k}$ is not the maximizer of 
\$
\frac{c^{*}}{(1-\gamma)n} \Big| \sum^{n}_{i=1} \tau(s_i,a_i)\left(q^{k}(s_i,a_i)-r_i -\gamma q^{k}(s^{\prime}_i,\pi^{k})\right)\Big| - \lambda\xi_{n}\big( \tau(s_i,a_i)\big).
\$
By the definition of $q^{k}(\cdot,\cdot)$, then we have
\$
 & q^{k}(s^0,\pi^{k}) + \sup_{\tau}\bigg\{\frac{c^{*}}{(1-\gamma)n} \Big| \sum^{n}_{i=1} \tau(s_i,a_i)\left(q^{k}(s_i,a_i)-r_i -\gamma q^{k}(s^{\prime}_i,\pi^{k})\right)\Big|-\lambda\xi_{n}\big(\tau(s_i,a_i)\big) \bigg\} \\
 & + \lambda\xi_{n}\big(\widetilde{\tau}^{k} (s_i,a_i)\big) \\
= & \min_{q}\bigg\{q(s^0,\pi^{k}) + \sup_{\tau}\bigg\{\frac{c^{*}}{(1-\gamma)n} \Big| \sum^{n}_{i=1} \tau(s_i,a_i)\left(q_{\theta}(s_i,a_i)-r_i -\gamma q(s^{\prime}_i,\pi^{k})\right)\Big| \\
& - \lambda\xi_{n}\big(\tau(s_i,a_i)\big)\bigg\}\bigg\} + \lambda\xi_{n}\big(\widetilde{\tau}^{k}(s_i,a_i)\big)
\$
As $\widetilde{q}^{\pi^k}(\cdot,\cdot)$ belongs to the function space associated with $q$, so 
\$
& \min_{q}\bigg\{q(s^0,\pi^{k}) + \sup_{\tau}\bigg\{\frac{c^{*}}{(1-\gamma)n} \Big| \sum^{n}_{i=1} \tau(s_i,a_i)\left(q(s_i,a_i)-r_i -\gamma q(s^{\prime}_i,\pi^{k})\right)\Big| - \lambda\xi_{n}\big(\tau(s_i,a_i)\big)\bigg\}\bigg\} \\
\leq & \widetilde{q}^{\pi^k}(s^0,\pi^{k}) + \sup_{\tau}\bigg\{\frac{c^{*}}{(1-\gamma)n} \Big| \sum^{n}_{i=1} \tau(s_i,a_i)\left(\widetilde{q}^{\pi^k}(s_i,a_i)-r_i -\gamma \widetilde{q}^{\pi^k}(s^{\prime}_i,\pi^{k})\right)\Big|-\lambda\xi_{n}\big(\tau(s_i,a_i)\big) \bigg\}\\
\leq & \widetilde{q}^{\pi^k}(s^0,\pi^{k}) + \sup_{\tau}\bigg\{\frac{c^{*}}{(1-\gamma)n} \Big| \sum^{n}_{i=1} \tau(s_i,a_i)\left(\widetilde{q}^{\pi^k}(s_i,a_i)-r_i -\gamma \widetilde{q}^{\pi^k}(s^{\prime}_i,\pi^{k})\right)\Big|\bigg\}
\$

where the last inequality comes from the second inequality comes from the non-negativity of $\xi_{n}(\cdot)$. This immediately implies that 
\$
\Delta_1 \leq & \widetilde{q}^{\pi^k}(s^0,\pi^{k}) + \sup_{\tau}\bigg\{\frac{c^{*}}{(1-\gamma)n} \Big| \sum^{n}_{i=1} \tau_{\tau}(s_i,a_i)\left(\widetilde{q}^{\pi^k}(s_i,a_i)-r_i -\gamma \widetilde{q}^{\pi^k}(s^{\prime}_i,\pi^{k})\right)\Big|\bigg\} + \lambda\xi_{n}\big(\widetilde{\tau}^{k}(s_i,a_i)\big) \\
\leq & \bar{V} + \sup_{\tau}\bigg\{\frac{c^{*}}{(1-\gamma)n} \Big| \sum^{n}_{i=1} \tau_{\tau}(s_i,a_i)\left(\widetilde{q}^{\pi^k}(s_i,a_i)-r_i -\gamma \widetilde{q}^{\pi^k}(s^{\prime}_i,\pi^{k})\right)\Big|\bigg\} + \lambda\xi_{n}\big(\widetilde{\tau}^{k}(s_i,a_i)\big).
\$
Follow the proof of Lemma \ref{pik_return_lemma}, we have 
\$
& \sup_{\tau}\bigg\{\frac{c^{*}}{(1-\gamma)n} \Big| \sum^{n}_{i=1} \tau(s_i,a_i)\left(\widetilde{q}^{\pi^k}(s_i,a_i)-r_i -\gamma \widetilde{q}^{\pi^k}(s^{\prime}_i,\pi^{k})\right)\Big|\bigg\} \\ 
 \leq &  \frac{c^{*}}{1-\gamma}\Bigg\{\big(3\mathcal{U}^{\star}_{2}\bar{V}+2\lambda \|\mathbb{D}(\tau(s,a))\|^{\text{prime}}_{L_2(\mu)}\big)\sqrt{\frac{2 \ln \frac{\operatorname{Vol}({\Theta}^{\dagger})}{\delta}}{n}} \notag\\
 &+ \frac{2\big(3 \mathcal{U}_{\infty}\mathcal{U}^{\star}_{2}\bar{V}+2\lambda \|\mathbb{D}(\tau(s,a))\|^{\text{prime}}_{L_\infty}\big) \ln \frac{\operatorname{Vol}({\Theta}^{\dagger})}{\delta}}{3 n} +  \mathcal{U}_{\infty}\mathcal{U}^{\star}_{2}\sqrt{\varepsilon_{\mathcal{Q}}}\Bigg\} := \varepsilon_{3,n},
\$
for $\operatorname{Vol}({\Theta}^{\dagger}) = (e^{D}\max\{D_{\Omega},D_{\mathcal{Q}},D_{\Pi}\}+1)^3(\{1 \vee L\}{\mathcal{U}}^{\tau}_{2})^{2D}$ for $D=D_{\Omega}+D_{\mathcal{Q}}+D_{\Pi}$. And by inequality \eqref{reg_asymp}, we have
$
\lambda\xi_{n}\big(\widetilde{\tau}^{k}(s_i,a_i)\big) \lesssim \mathcal{O}\left(\lambda\mathcal{U}_{\infty}\mathcal{U}^{\star}_{2}\sqrt{\varepsilon_{\mathcal{Q}}} \right) + \mathcal{O}(\frac{1}{\sqrt{n}})
$.
Then, we conclude that 
\$
& q^{k}(s^0,\pi^{k}) + \sup_{\tau}\bigg\{\frac{c^{*}}{(1-\gamma)n} \Big| \sum^{n}_{i=1} \tau(s_i,a_i)\left(q^{k}(s_i,a_i)-r_i -\gamma q^{k}(s^{\prime}_i,\pi^{k})\right)\Big| \bigg\} \\
\lesssim & \varepsilon_{3,n} +  \mathcal{O}\left(\lambda\mathcal{U}_{\infty}\mathcal{U}^{\star}_{2}\frac{\sqrt{\varepsilon_{\mathcal{Q}}}}{1-\gamma} \right) + \mathcal{O}(\frac{1}{\sqrt{n}}).
\$ 
With the boundedness condition on $q^{k}(s^0,\pi^{k}) \in [-\bar{V},\bar{V}]$, by some algebra, then 
\#
& \sup_{\tau}\bigg\{\frac{c^{*}}{(1-\gamma)n} \Big| \sum^{n}_{i=1} \tau(s_i,a_i)\left(q^{k}(s_i,a_i)-r_i -\gamma q^{k}(s^{\prime}_i,\pi^{k})\right)\Big| \bigg\} \notag \\
\leq & \varepsilon_{3,n} + \mathcal{O}\left(\lambda\mathcal{U}_{\infty}\mathcal{U}^{\star}_{2}\frac{\sqrt{\varepsilon_{\mathcal{Q}}}}{1-\gamma} \right) + \mathcal{O}(\frac{1}{\sqrt{n}}) + \bar{V}.
\label{emprical_bound_reg}
\#
Therefore, we can conclude that 
\#
& \sup_{\tau}\bigg\{\frac{1}{(1-\gamma)n} \Big| \sum^{n}_{i=1} \tau(s_i,a_i)\left(q^{k}(s_i,a_i)-r_i -\gamma q^{k}(s^{\prime}_i,\pi^{k})\right)\Big| \bigg\} \notag \\
\leq & \frac{1}{1-\gamma}\Bigg\{\big(3\mathcal{U}^{\star}_{2}\bar{V}+2\lambda \|\mathbb{D}(\tau(s,a))\|^{\text{prime}}_{L_2(\mu)}\big)\sqrt{\frac{2 \ln \frac{\operatorname{Vol}({\Theta}^{\dagger})}{\delta}}{n}} \notag\\
 &+ \frac{2\big(3 \mathcal{U}_{\infty}\mathcal{U}^{\star}_{2}\bar{V}+2\lambda \|\mathbb{D}(\tau(s,a))\|^{\text{prime}}_{L_\infty}\big) \ln \frac{\operatorname{Vol}({\Theta}^{\dagger})}{\delta}}{3 n} +  \mathcal{U}_{\infty}\mathcal{U}^{\star}_{2}\sqrt{\varepsilon_{\mathcal{Q}}}\Bigg\} \\
 & + \mathcal{O}\left(\frac{\lambda\mathcal{U}_{\infty}\mathcal{U}^{\star}_{2}\frac{\sqrt{\varepsilon_{\mathcal{Q}}}}{1-\gamma}}{c^{*}} \right) + \mathcal{O}(\frac{1}{\sqrt{n}}) + 
 \frac{\bar{V}}{c^{*}} .
\label{bbd_emprical_lk}
\#
According to Lemma \ref{risk_bound_lm}, we have for any $\tau$, 
\$
& \frac{1}{(1-\gamma)} \Big| \mathbb{E}_{\mu} \left[\tau(s,a)\left(q^{k}(s,a)-r(s,a) -\gamma q^{k}(s^{\prime},\pi^{k})\right)\right]\Big|  \\
\leq & \frac{1}{(1-\gamma)n} \Big| \sum^{n}_{i=1} \tau(s_i,a_i)\left(q^{k}(s_i,a_i)-r_i -\gamma q^{k}(s^{\prime}_i,\pi^{k})\right)\Big|\\
& + \frac{1}{1-\gamma}\Bigg\{\mathcal{U}^{\star}_{2}\sqrt{\frac{2 \bar{V}^2\ln \frac{\operatorname{Vol}({\Theta}^{\dagger})}{\delta}}{n}}  + \frac{2\mathcal{U}_{\infty}\mathcal{U}^{\star}_{2} \bar{V} \ln \frac{\operatorname{Vol}({\Theta}^{\dagger})}{\delta}}{3 n}\Bigg\} \\
& + \mathcal{O}(\frac{\bar{V}\mathcal{U}_{\infty}\mathcal{U}^{\star}_{2}}{n}) + \mathcal{O}(\frac{\bar{V}\mathcal{U}_{\infty}\mathcal{U}^{\star}_{2}(1+\gamma)}{n}) + \mathcal{O}(\frac{\gamma \bar{V}}{n}) .
\$
Combine with the bound \eqref{bbd_emprical_lk}, we conclude that 
\$
& \sup_{\psi}\bigg\{\frac{1}{(1-\gamma)}  \mathbb{E}_{\mu} \left[\tau_{\psi}(s,a)\left(r(s,a) +\gamma q^{k}(s^{\prime},\pi^{k})-q^{k}(s,a)\right)\right]\bigg\}\\
\leq &
\sup_{\psi}\bigg\{\frac{1}{(1-\gamma)} \Big|\mathbb{E}_{\mu} \left[\tau_{\psi}(s,a)\left(q^{k}(s,a)-r(s,a) -\gamma q^{k}(s^{\prime},\pi^{k})\right)\right]\Big|\bigg\}\\
\leq & 
\frac{1}{1-\gamma}\Bigg\{2\big(2\mathcal{U}^{\star}_{2}\bar{V}+2\lambda \|\mathbb{D}(\tau(s,a))\|^{\text{prime}}_{L_2(\mu)}\big)\sqrt{\frac{2 \ln \frac{\operatorname{Vol}({\Theta}^{\dagger})}{\delta}}{n}} \notag\\
 &+ \frac{4\big(2\mathcal{U}_{\infty}\mathcal{U}^{\star}_{2}\bar{V}+2\lambda \|\mathbb{D}(\tau(s,a))\|^{\text{prime}}_{L_\infty}\big) \ln \frac{\operatorname{Vol}({\Theta}^{\dagger})}{\delta}}{3 n} +  \mathcal{U}_{\infty}\mathcal{U}^{\star}_{2}\sqrt{\varepsilon_{\mathcal{Q}}}\Bigg\}\notag \\
    & + \mathcal{O}\left(\frac{\lambda\mathcal{U}_{\infty}\mathcal{U}^{\star}_{2}\frac{\sqrt{\varepsilon_{\mathcal{Q}}}}{1-\gamma}}{c^{*}} \right)
    + \frac{\bar{V}}{c^{*}} + \mathcal{O}(\frac{\bar{V}\mathcal{U}_{\infty}\mathcal{U}^{\star}_{2}}{n}) + \mathcal{O}(\frac{\bar{V}\mathcal{U}_{\infty}\mathcal{U}^{\star}_{2}(1+\gamma)}{n}) + \mathcal{O}(\frac{\gamma \bar{V}}{n}) +\mathcal{O}\left(\frac{1}{\sqrt{n}}\right).
\$
This completes the proof
\end{proof}

\subsection{Proof of Theorem 5.1}

\begin{proof}
Let the policy $\widehat{\pi}$ be the output of the penalized adversarial in Algorithm 1 of maintext. In this proof, we aim to bound the regret 
\$
J(\pi) - J(\widehat{\pi}).
\$
First, we note that $\widehat{\pi}$ is a mixed policy over  $\{\pi^{k}\}^{\bar{K}}_{k=1}$, then we follow Theorem 1 in \citep{wu2021offline} to decompose the discounted return of $\widehat{\pi}$, i.e., $J(\widehat{\pi})$, that is, 
$
J(\widehat{\pi}) = \frac{1}{\bar{K}}\sum^{\bar{K}}_{k=1} J(\pi^{k})
$. Based on this, it suffices to bound
\#
J(\pi) -  \frac{1}{\bar{K}}\sum^{\bar{K}}_{k=1} J(\pi^{k})
= \frac{1}{\bar{K}}\sum^{\bar{K}}_{k=1} \left(J(\pi) - J(\pi^{k}) \right).
\label{prac_regret_dec1}
\#
By Lemma \ref{iden_mdp} for that  the learned $q$-function at the $k$-th iteration $q^{k}$ is the true action-value function under the policy $\pi^{k}$ in the iterative MDP $\left\{\mathcal{S}, \mathcal{A}, \mathds{P}_{k}, \gamma, r_{k}, s^{0}\right\}$. The regret \eqref{prac_regret_dec1} can be further decomposed as follows: 
\#
\frac{1}{\bar{K}}\sum^{\bar{K}}_{k=1} \left(J(\pi) - J(\pi^{k}) \right) = &  \underbrace{\frac{1}{\bar{K}}\sum^{\bar{K}}_{k=1}\left(J(\pi^{k};\left\{\mathcal{S}, \mathcal{A}, \mathds{P}_{k}, \gamma, r_{k}, s^{0}\right\}) - J(\pi^{k})\right)}_{\Delta_{1}} \notag \\
& + \underbrace{\frac{1}{\bar{K}}\sum^{\bar{K}}_{k=1}\left(
J(\pi;\left\{\mathcal{S}, \mathcal{A}, \mathds{P}_{k}, \gamma, r_{k}, s^{0}\right\}) - J(\pi^{k};\left\{\mathcal{S}, \mathcal{A}, \mathds{P}_{k}, \gamma, r_{k}, s^{0}\right\}) \right)}_{\Delta_{2}} \notag \\
& + \underbrace{\frac{1}{\bar{K}}\sum^{\bar{K}}_{k=1}\left( 
J(\pi) - J(\pi;\left\{\mathcal{S}, \mathcal{A}, \mathds{P}_{k}, \gamma, r_{k}, s^{0}\right\})
\right)}_{\Delta_{3}} 
\label{regret_dec_reg}
\#
Based on this error decomposition, it suffices to upper-bound the above three terms. In analysis, first, 
$\Delta_1$ is the regret over true MDP and iterative MDP for policy $\pi^{k}$. Second, $\Delta_2$ is the regret over policy $\pi^{k}$ and $\pi$ under iterative MDP. Third, $\Delta_3$ is the regret over true MDP and iterative MDP for policy $\pi$. In the following, we bound each term subsequently.

\textbf{Bounding $\Delta_1$.} According to Lemma \ref{pik_return_lemma}, we have $\Delta_1$ is upper bounded by 
\#
\Delta_1 \leq &  \frac{c^{*}}{1-\gamma}\Bigg\{\big(3\mathcal{U}^{\star}_{2}\bar{V}+2\lambda \|\mathbb{D}(\tau(s,a))\|^{\text{prime}}_{L_2(\mu)}\big)\sqrt{\frac{2 \ln \frac{\operatorname{Vol}({\Theta}^{\dagger})}{\delta}}{n}} \notag\\
 &+ \frac{2\big(3 \mathcal{U}_{\infty}\mathcal{U}^{\star}_{2}\bar{V}+2\lambda \|\mathbb{D}(\tau(s,a))\|^{\text{prime}}_{L_\infty}\big) \ln \frac{\operatorname{Vol}({\Theta}^{\dagger})}{\delta}}{3 n} +  \mathcal{U}_{\infty}\mathcal{U}^{\star}_{2}\sqrt{\varepsilon_{\mathcal{Q}}}\Bigg\} \\
&+  \mathcal{O}\left((1+\lambda\mathcal{U}_{\infty}\mathcal{U}^{\star}_{2})\frac{\sqrt{\varepsilon_{\mathcal{Q}}}}{1-\gamma} \right) + \mathcal{O}(\frac{1}{\sqrt{n}}).
    \label{Delta1_bd_reg}
\#

\textbf{Bounding $\Delta_2$.} For this term, it is concerned with the optimization error. According to Lemma \ref{iden_mdpret}, our algorithm achieves a no-regret oracle, and the optimization error can be minimized by increasing the rounds of optimization, i.e., increasing $bar{K}$.
\#
\Delta_2  \leq \frac{2\sqrt{2\bar{V}\log|\mathcal{A}|}}{\sqrt{\bar{K}}(1-\gamma)}.
    \label{Delta2_bd_reg}
\#

\textbf{Bounding $\Delta_3$.} To bound $\Delta_3$, it suffices to bound 
$
J(\pi) - J(\pi;\left\{\mathcal{S}, \mathcal{A}, \mathds{P}_{k}, \gamma, r_{k}, s^{0}\right\},
$
for any $k \in [\bar{K}]$. we define admissible implicit exploratory distribution as $\rho_{k}$ that satisfies, which essentially can be induced and controlled via penalization on the detection function through $\lambda$ in Algorithm 1, i.e.,  
\#
& J(\pi) - J(\pi;\left\{\mathcal{S}, \mathcal{A}, \mathds{P}_{k}, \gamma, r_{k}, s^{0}\right\}) \\
= \, & q^{\pi}(s^0,\pi) - J(\pi;\left\{\mathcal{S}, \mathcal{A}, \mathds{P}_{k}, \gamma, r_{k}, s^{0}\right\}) \notag  =  \frac{\mathbb{E}_{d^{\pi}}\left[q^{\pi}(s,a)-r_{k}(s,a)-\gamma q^{\pi}(s,\pi) \right]}{1-\gamma} \notag  \\
= & \underbrace{\frac{1}{1-\gamma}\mathbb{E}_{\mu}\left[\frac{\rho_{k}(s,a)}{\mu(s,a)}\left[q^{\pi}(s,a)-r_{k}(s,a)-\gamma q^{\pi}(s,\pi)\right]\right]}_{\Delta_{31}} - \underbrace{\frac{1}{1-\gamma} \mathbb{E}_{\rho_{k}}\left[q^{\pi}(s,a)-r_{k}(s,a)-\gamma q^{\pi}(s,\pi)\right]}_{\Delta_{32}} \notag \notag  \\
&+ \underbrace{\frac{1}{1-\gamma} \mathbb{E}_{d^\pi}\left[q^{\pi}(s,a)-r_{k}(s,a)-\gamma q^{\pi}(s,\pi)\right]}_{\Delta_{33}} .
\label{err_decomp_reg_prac}
\#
Accordingly, we can make a mirror decomposition as in the proof of Theorem 4.1. The difference is, instead of controlling the uncertainty level through constrained set $\widetilde{\Omega}_{\widetilde{\sigma}_n}$, in this penalization adversarial algorithm, the uncertainty level is controlled via penalization. To proceed with the proof, we first study and connect the penalized uncertainty control to constrained uncertainty control. According to Lemma \ref{prime_bound_tau}, we have 
\#
\|\tau(s,a)\|_{L_2(\mu)} \leq  \frac{1}{\sqrt{M}}\bigg\{L\mathcal{U}^{\tau}_{2}\sqrt{\frac{2 \ln \frac{\operatorname{Vol}(\mathcal{G}^{\mathbb{D}})}{\delta}}{n}} +  2\sqrt{\frac{ L\mathcal{U}_{\infty}\mathcal{U}^{\tau}_{2}\ln \frac{\operatorname{Vol}(\mathcal{G}^{\mathbb{D}})}{\delta}}{3 n}}  + \sqrt{2\varepsilon^{\mathbb{D}}_{n}} + \sqrt{M}\bigg\}.
\label{tau_bound_reg_prac}
\#
where we can determine $\varepsilon^{\mathbb{D}}_{n}$ using Lemma \ref{reg_tau_l2_bound}. This implies that we can well control $\lambda$ even in the penalization adversarial estimation to control the uncertainty level in the form of $
\|\tau(s,a)\|_{L_{2}(\mu)}$ for $\tau \in \Omega$, i.e., $
\|\tau(s,a)\|_{L_{2}(\mu)} \leq \mathcal{U}^{\star}_{2}$ for  $\mathcal{U}^{\star}_{2}$ depending on $\lambda$. Throughout the rest of the proof, it is sufficient to proceed with the condition on  $
\|\tau(s,a)\|_{L_{2}(\mu)} \leq \mathcal{U}^{\star}_{2}$.

\textbf{Bounding $\Delta_{31}$.} We define the $\Omega$ sub-class that $\widetilde{\Omega} = \{\tau: \|\tau(s,a)\|_{L_{2}(\mu)} \leq \mathcal{U}^{\star}_{2}, \tau \in \Omega\}$, and define a importance-weight estimator over $\rho_{k}$:  
\$
\tau_{\rho_{k}/\mu}(s,a) := \argmin_{\tau \in \text{lr-hull}(\widetilde{\Omega})} \frac{1}{1-\gamma}\mathbb{E}_{\mu}\left[\left(\frac{\rho_{k}(s,a)}{\mu(s,a)} - \tau(s,a)\right)  \left[q^{\pi}(s,a)-r_{k}(s,a)-\gamma q^{\pi}(s,\pi)\right]\right],
\$
 Then we make the following error decomposition for $\Delta_{31}$ as
\$
\Delta_{31} =& \frac{1}{1-\gamma}\mathbb{E}_{\mu}\left[\left(\frac{\rho_{k}(s,a)}{\mu(s,a)} - \tau_{\rho_{k}/\mu}(s,a)\right)  \left[q^{\pi}(s,a)-r_{k}(s,a)-\gamma q^{\pi}(s,\pi)\right]\right] \\ 
&+ \frac{1}{1-\gamma}\mathbb{E}_{\mu}\left[\tau_{\rho_{k}/\mu}(s,a) \left[q^{\pi}(s,a)-r_{k}(s,a)-\gamma q^{\pi}(s,\pi)\right]\right] \\
\leq & \underbrace{\frac{\mathbb{E}_{\mu}\left[\left(\frac{\rho_{k}(s,a)}{\mu(s,a)} - \tau_{\rho_{k}/\mu}(s,a)\right) \left[q^{\pi}(s,a)-r_{k}(s,a)-\gamma q^{\pi}(s,\pi)\right]\right]}{1-\gamma}}_{\Delta_{311}} \\
&+ \underbrace{\sup_{\tau \in \text{lr-hull}(\Omega,C_{\tau,\rho_{k}},\mathcal{U}^{\star}_{2})}\frac{1}{1-\gamma}\mathbb{E}_{\mu}\left[\tau(s,a) \left[q^{\pi}(s,a)-r_{k}(s,a)-\gamma q^{\pi}(s,\pi)\right]\right]}_{\Delta_{312}}.
\$
\textbf{Bounding $\Delta_{311}$.} Follow the definition of $r_{k}$ in Lemma \ref{iden_mdp}, it observes that 
\#
 q^{\pi}(s,a)-r_{k}(s,a)-\gamma q^{\pi}(s,\pi)  
= & r(s,a)-r_{k}(s,a) \notag \\
= & r(s,a)- q^{k}(s,a) + \gamma \mathbb{E}_{s^{\prime} \sim \mathds{P}(\cdot|s,a)}\left[ \sum_{a^{\prime}\in \mathcal{A}} \pi^{k}(a^{\prime}|s^{\prime})q^{k}(s^{\prime},a^{\prime}) \right]  \notag \\
= &r(s,a)+ \gamma q^{k}(s^{\prime},\pi^{k}) - q^{k}(s,a).
\label{equv_rk}
\#
Then to bound $\Delta_{311}$ it suffices to bound 
\$
&
\frac{\mathbb{E}_{\mu}\left[\left(\frac{\rho_{k}(s,a)}{\mu(s,a)} - \tau_{\rho_{k}/\mu}(s,a)\right) \left[r(s,a)+ \gamma q^{k}(s^{\prime},\pi^{k}) - q^{k}(s,a) \right]\right] }{1-\gamma}.
\$
It observes that 
\$
& \sup_{\tau}\left\{\frac{\Big|\mathbb{E}_{\mu} \left[\tau(s,a)\left(q^{k}(s,a)-r(s,a) -\gamma q^{k}(s^{\prime},\pi^{k})\right)\right]\Big|}{1-\gamma} \right\} \\
 = & \sup_{\tau}\left\{\frac{\Big|-\mathbb{E}_{\mu} \left[\tau(s,a)\left(q^{k}(s,a)-r(s,a) -\gamma q^{k}(s^{\prime},\pi^{k})\right)\right]\Big|}{1-\gamma} \right\} \\
  = & \sup_{\tau}\left\{\frac{\Big|\mathbb{E}_{\mu} \left[\tau(s,a)\left(r(s,a)+\gamma q^{k}(s^{\prime},\pi^{k})-q^{k}(s,a)\right)\right]\Big|}{1-\gamma} \right\}.
\$
Then we apply Lemma \ref{iden_concen}, with control on $\tau$ as $\mathcal{U}^{\star}_{2}$ for $L_2$ boundedness and $\mathcal{U}_{\infty}\mathcal{U}^{\star}_{2}$ for $L_{\infty}$. Then, we have 
\$
\sup_{\tau}\bigg\{\frac{\mathbb{E}_{\mu} \left[\tau(s,a)\left( r(s,a) + \gamma q^{k}(s^{\prime},\pi^{k})-q^{k}(s,a)\right)\right]}{1-\gamma} \bigg\} \leq \varepsilon^{4}_n. 
\$
where $\varepsilon^{4}_n$ is the upper bound as in Lemma \ref{iden_concen}. Now, as $\|\tau(s,a)\|_{L_2(\mu)} \leq  \mathcal{U}^{\star}_{2}$. According to Lemma \ref{L_2MMD}, we have 
\#
\frac{\left\|r(s,a) + \gamma q^{k}(s^{\prime},\pi^{k})-q^{k}(s,a)\right\|_{L_2(\mu)}}{(1-\gamma)} \leq  \frac{\varepsilon^{4}_n}{\mathcal{U}^{\star}_{2}}.
\label{delta311_1st}
\#
Also, due to the non-negativity of $\frac{\rho_{k}(s,a)}{\mu(s,a)}$ and $\tau_{\rho_{k}/\mu}(s,a)$ for any $(s,a)$ over the support on $\mu$, we have
\#
\mathbb{E}_{\mu}\left[\left(\frac{\rho_{k}(s,a)}{\mu(s,a)} - \tau_{\rho_{k}/\mu}(s,a)\right)^2\right] \leq & 
\min\left\{\mathbb{E}_{\mu}\left[\left(\frac{\rho_{k}(s,a)}{\mu(s,a)}\right)^2\right], \mathcal{U}^{\star}_{2}  \right\}.
\label{delta311_2nd}
\# 
By Cauchy-schwarz inequality, and combine with inequalities \eqref{delta311_1st} and \eqref{delta311_2nd}, we conclude that 
\#
\Delta_{311} 
\leq &  \frac{\sqrt{\min\left\{\mathbb{E}_{\mu}\left[\left(\frac{\rho_{k}(s,a)}{\mu(s,a)}\right)^2\right], \mathcal{U}^{\star}_{2}  \right\}}}{1-\gamma} \frac{\varepsilon^{4}_n}{\mathcal{U}^{\star}_{2}}
:= \varepsilon^{4,\prime}_{n}.
\label{delta311_final_bd}
\#
We can plug-in $\varepsilon^{4}_{n}$ to complete the upper bound for $\Delta_{311}$. 

\textbf{Bounding $\Delta_{312}$.} According to \eqref{tau_bound_reg_prac}. The supermum boundedness condition for $\tau$ can be identified over the class $\text{lr-hull}(\widetilde{\Omega})$: 
\#
\|\tau(s,a)\|_{L_{\infty}} \leq &\mathcal{U}_{\infty}\mathcal{U}^{\star}_{2}.
\label{tau_bound_coeff}
\#
for some constant $\mathcal{U}_{\infty}$ as described in \citep{wainwright2019high}. It follows from Lemma \ref{iden_concen}, under the boundedness conditions on $\tau$ in \eqref{tau_bound_coeff}, we can conclude that 
\#
\Delta_{312} \leq \varepsilon^{5}_n.
\label{delta312_final_bd}
\#
where $\varepsilon^{5}_n$ is defined as the upper bound term in Lemma \ref{iden_concen}.  

Combine with the upper bounds on $\Delta_{311}$ and $\Delta_{312}$ in \eqref{delta311_final_bd} and \eqref{delta312_final_bd}, we have 
\$
\Delta_{31} \leq \varepsilon^{4,\prime}_{n} + \varepsilon^{5}_n.
\$
This completes upper bounding $\Delta_{31}$. 

\textbf{Bounding $\Delta{33}-\Delta{32}$.} According to the error decomposition \eqref{err_decomp_reg_prac}, it remains to bound $-\Delta{32}+\Delta{33}$. We have the upper bound on 
\$
& (1-\gamma)(\Delta_{33} - \Delta_{32}) \\
= & \underbrace{\mathbb{E}_{d^\pi}\left[q^{\pi}(s,a)-r_{k}(s,a)-\gamma q^{\pi}(s,\pi)\right]}_{\Delta_{33}}  - \underbrace{\mathbb{E}_{\rho_{k}}\left[q^{\pi}(s,a)-r_{k}(s,a)-\gamma q^{\pi}(s,\pi)\right]}_{\Delta_{32}} \\
= & \sum_{a\in\mathcal{A}, s\in\mathcal{S}}\mathds{1}_{\{d_{\pi}(s,a)-\rho_{k}(s,a) \geq 0 \} }[d_{\pi}(s,a)-\rho_{k}(s,a)]\left[q^{\pi}(s,a)-r_{k}(s,a)-\gamma q^{\pi}(s,\pi)\right] \\
& + \sum_{a\in\mathcal{A}, s\in\mathcal{S}}\mathds{1}_{\{d_{\pi}(s,a)-\rho_{k}(s,a) < 0\} }[d_{\pi}(s,a)-\rho_{k}(s,a)]\left[q^{\pi}(s,a)-r_{k}(s,a)-\gamma q^{\pi}(s,\pi)\right] \\
\leq &  \underbrace{\sum_{a\in\mathcal{A}, s\in\mathcal{S}}\mathds{1}_{\{d_{\pi}(s,a)-\rho_{k}(s,a) < 0\} }[\rho_{k}(s,a)-d_{\pi}(s,a)]\left[q^{\pi}(s,a)-r_{k}(s,a)-\gamma q^{\pi}(s,\pi)\right]}_{\Delta_{331}} \\ 
& + \underbrace{\sum_{a\in\mathcal{A}, s\in\mathcal{S}}\mathds{1}_{\{d_{\pi}(s,a)-\rho_{k}(s,a) \geq 0 \} }[d_{\pi}(s,a)-\rho_{k}(s,a)]\left[q^{\pi}(s,a)-r_{k}(s,a)-\gamma q^{\pi}(s,\pi)\right]}_{\Delta_{332}}.
\$

\textbf{Bounding $\Delta_{331}$.} First, we observe that 
\$
\Delta_{331} = \sum_{a\in\mathcal{A}, s\in\mathcal{S}}\mathds{1}_{\{\rho_{k}(s,a)-d_{\pi}(s,a) > 0\} }[\rho_{k}(s,a)-d_{\pi}(s,a)]\left[q^{\pi}(s,a)-r_{k}(s,a)-\gamma q^{\pi}(s,\pi)\right],
\$
which is equivalent to 
$
\sum_{a\in\mathcal{A}, s\in\mathcal{S}}\left(\rho_{k}(s,a)-d_{\pi}(s,a)\right)^{+}\left[q^{\pi}(s,a)-r_{k}(s,a)-\gamma q^{\pi}(s,\pi)\right]
$.
We apply change of measure for shifting to the distribution over $\mu$, i.e,
\#
&\sum_{a\in\mathcal{A}, s\in\mathcal{S}}\left(\rho_{k}(s,a)-d_{\pi}(s,a)\right)^{+}\left[q^{\pi}(s,a)-r_{k}(s,a)-\gamma q^{\pi}(s,\pi)\right]\notag \\ 
= & \sum_{a\in\mathcal{A}, s\in\mathcal{S}} \left\{\frac{\left(\rho_{k}(s,a)-d_{\pi}(s,a)\right)^{+}}{\mu(s,a)}\left[q^{\pi}(s,a)-r_{k}(s,a)-\gamma q^{\pi}(s,\pi)\right]\cdot \mu(s,a) \right\} \notag \\
\leq & \left\|\frac{\left(\rho_{k}(s,a)-d_{\pi}(s,a)\right)^{+}}{\mu(s,a)}\right\|_{L_2(\mu)} \|q^{\pi}(s,a)-r_{k}(s,a)-\gamma q^{\pi}(s,\pi) \|_{L_2(\mu)}
\label{imm_bd}
\#
As for any implicit exploratory distribution $\rho_{k}$, we have $\left(\rho_{k}(s,a)-d_{\pi}(s,a)\right)^{+} \leq \rho_{k}(s,a)$ for any $s,a$, thus we have  
\#
\left\|\frac{\left(\rho_{k}(s,a)-d_{\pi}(s,a)\right)^{+}}{\mu(s,a)}\right\|_{L_2(\mu)} \leq  \left\|\frac{\rho_{k}(s,a)}{\mu(s,a)}\right\|_{L_2(\mu)}
\leq  \min\{\mathcal{U}_{2,\rho_{k}}, \mathcal{U}^{\star}_{2}\},
\label{upper_diff_rk} 
\#
where $\mathcal{U}_{2,\rho_{k}} := \left\|\frac{\rho_{k}(s,a)}{\mu(s,a)}\right\|_{L_2(\mu)}$. 
Upon the inequalities \eqref{delta311_1st}, \eqref{equv_rk}, \eqref{imm_bd}, and the observation \eqref{upper_diff_rk}, we conclude that 
\$
\Delta_{331} \leq \frac{   \min\{\mathcal{U}_{2,\rho_{k}}, \mathcal{U}^{\star}_{2}\} \varepsilon^{4}_n}{\mathcal{U}^{\star}_{2}}.
\$

\textbf{Bounding $\Delta_{332}$.} With respect to the on-support and off-supprot region: $\mu(s,a)=0$ and $\mu(s,a)>0$, we have the following decomposition, 
\$
\Delta_{332} =& \underbrace{\sum_{a\in\mathcal{A}, s\in\mathcal{S}}\mathds{1}_{\mu(s,a)>0}\left(d_{\pi}(s,a)-\rho(s,a)\right)^{+}\left[q^{\pi}(s,a)-r_{k}(s,a)-\gamma q^{\pi}(s,\pi)\right]}_{\Delta_{3321}} \\
&+ \underbrace{\sum_{a\in\mathcal{A}, s\in\mathcal{S}}\mathds{1}_{\mu(s,a)=0}\left(d_{\pi}(s,a)-\rho(s,a)\right)^{+}\left[q^{\pi}(s,a)-r_{k}(s,a)-\gamma q^{\pi}(s,\pi)\right]}_{\Delta_{3322}} .
\$

\textbf{Bounding  $\Delta_{3321}$.} According to \eqref{equv_rk} and \eqref{delta311_1st}, we have 
\$
\sup_{\{(s,a) \in \mathcal{S} \times \mathcal{A}: \mu(s,a)>0\}} \frac{|r(s,a) + \gamma q^{k}(s^{\prime},\pi^{k})-q^{k}(s,a)|}{(1-\gamma)}  \leq \frac{\varepsilon^{4}_n}{\mathcal{U}^{\star}_{2}} .
\$
Then we conclude that 
\$
\Delta_{3321} \leq \sum_{a\in\mathcal{A}, s\in\mathcal{S}}\left(d_{\pi}(s,a)-\rho(s,a)\right)^{+} \frac{\varepsilon^{4}_n}{\mathcal{U}^{\star}_{2}} .
\$
The term $\Delta_{3322}$ is the off-support extrapolation error. Therefore, we conclude that 
\$
\Delta_{332} \leq &  \sum_{a\in\mathcal{A}, s\in\mathcal{S}}\left(d_{\pi}(s,a)-\rho(s,a)\right)^{+} \frac{\varepsilon^{4}_n}{\mathcal{U}^{\star}_{2}}  
+ \Delta_{3322}.
\$
In the following, we conclude that
\$
& \Delta_{33} - \Delta_{32} \\
\leq &   \frac{   \min\{\mathcal{U}_{2,\rho_{k}}, \mathcal{U}^{\star}_{2}\} \varepsilon^{4}_n}{\mathcal{U}^{\star}_{2}} +  \sum_{a\in\mathcal{A}, s\in\mathcal{S}}\left(d_{\pi}(s,a)-\rho(s,a)\right)^{+} \cdot \frac{\varepsilon^{4}_n}{\mathcal{U}^{\star}_{2}} \\
& + \sum_{a\in\mathcal{A}, s\in\mathcal{S}}\mathds{1}_{\mu(s,a)=0}\left(d_{\pi}(s,a)-\rho(s,a)\right)^{+}\left[q^{\pi}(s,a)-r_{k}(s,a)-\gamma q^{\pi}(s,\pi)\right]\\
= &  \frac{   \min\{\mathcal{U}_{2,\rho_{k}}, \mathcal{U}^{\star}_{2}\} \varepsilon^{4}_n}{\mathcal{U}^{\star}_{2}} +  \sum_{a\in\mathcal{A}, s\in\mathcal{S}}\left(d_{\pi}(s,a)-\rho(s,a)\right)^{+} \cdot \frac{\varepsilon^{4}_n}{\mathcal{U}^{\star}_{2}} \\
& + \sum_{a\in\mathcal{A}, s\in\mathcal{S}}\mathds{1}_{\mu(s,a)=0}d_{\pi}(s,a)\left[q^{\pi}(s,a)-r_{k}(s,a)-\gamma q^{\pi}(s,\pi)\right]
\$
Combine with the bound on $\Delta_{31}$, we have 
\#
\Delta_{3} \leq & \varepsilon^{4,\prime}_{n} + \varepsilon^{5}_n + \frac{   \min\{\mathcal{U}_{2,\rho_{k}}, \mathcal{U}^{\star}_{2}\} \varepsilon^{4}_{n}}{\mathcal{U}^{\star}_{2}} +  \sum_{a\in\mathcal{A}, s\in\mathcal{S}}\left(d_{\pi}(s,a)-\rho(s,a)\right)^{+} \cdot \frac{\varepsilon^{4}_n}{\mathcal{U}^{\star}_{2}} \notag \\
& + \sum_{a\in\mathcal{A}, s\in\mathcal{S}}\mathds{1}_{\mu(s,a)=0}\left(d_{\pi}(s,a)-\rho(s,a)\right)^{+}\left[q^{\pi}(s,a)-r_{k}(s,a)-\gamma q^{\pi}(s,\pi)\right]
\label{Delta3_bd_reg}
\#

According to the regret decomposition in \eqref{regret_dec_reg}, and the upper bound on $\Delta_{1}$ in \eqref{Delta1_bd_reg}, the upper bound on $\Delta_{2}$ in \eqref{Delta2_bd_reg}, and the upper bound on $\Delta_{3}$ in \eqref{Delta3_bd_reg}, by some algebra, we set $c^{*} =  \widetilde{\mathcal{O}}\big(\sqrt{n\bar{V}/(\lambda L{\mathcal{U}}^{\tau}_{2}\ln \{\operatorname{Vol}({\Theta}^{\dagger})/\delta\})}\big). 
$
we set $\lambda$ as the solution of $\lambda = 2\mathcal{E}_2/\mathcal{E}_1$  which depends on ${\mathcal{U}}^{\star}_{2}$ in order to ensure the $L_2(\mu)$ norm for uncertainty control. It follows from Lemma \ref{prime_bound_tau} and Lemma \ref{reg_tau_l2_bound}, we have 
$
\mathcal{E}_1 = \mathcal{O}\big((\sqrt{M}{\mathcal{U}}^{\star}_{2} +(\max\{{\mathcal{U}}^{\star}_{2},{\mathcal{U}}^{\tau}_{\infty}\}\bar{V} + L{\mathcal{U}}^{\star}_{2})\sqrt{\ln \{\operatorname{Vol}({\Theta}^{\dagger})/\delta\}/n})^2\big) 
$ and $
\mathcal{E}_2 = \mathcal{O}\big(\bar{V} +(1-\gamma)^{-1/2}\sqrt{{\mathcal{U}}^{\star}_{2}(\bar{V}^2 + \lambda L\bar{V})}\sqrt[4]{\ln \{\operatorname{Vol}({\Theta}^{\dagger})/\delta\}/n} + \sqrt{{\mathcal{U}}^{\tau}_{\infty}(\bar{V}^3 + \lambda L\bar{V})}\sqrt{\ln \{\operatorname{Vol}({\Theta}^{\dagger})/\delta\}/n} + (((\lambda \mathcal{U}^{\tau}_{\infty}\varepsilon_{\mathcal{Q}}+\bar{V}\mathcal{U}^{\tau}_{\infty}(1-\gamma)\varepsilon^{0.5}_{\mathcal{Q}}))/(1-\gamma)^2)^{0.5}\big)
$. Plug-in the choice of $\lambda$ and $c^{*}$, by some algebra, 
if we further ignore the high-order fast terms using a big-Oh notation $\widetilde{\mathcal{O}}$ and set $\varepsilon_{\mathcal{Q}}=0$, we conclude that 
\$
& J(\pi) - J(\widehat{\pi}) \leq \; \frac{1}{1-\gamma}\widetilde{\mathcal{O}}\Bigg(\sqrt[\leftroot{-1}\uproot{2}\scriptstyle 4]{\frac{({\mathcal{U}}^{\star}_{2})^2\mathfrak{C}^{1}_{\bar{V},\lambda, L}\ln\{\operatorname{Vol}({\Theta}^{\dagger})/\delta\}}{n}} + \sqrt{\frac{\bar{V}\log|\mathcal{A}|}{\bar{K}}} \\
&  +  \frac{1}{\bar{K}}\sum^{\bar{K}}_{k=1}\min_{\rho_k \in \Delta_{{\mathcal{U}}^{\star}_{2}}} \mathbb{E}_{\left(d_{\pi}-\rho_k\right)^{+}}\bigg[\mathds{1}_{\mu=0}\left(\mathcal{B}^{\pi^{k}}q^{k}(s,a) -q^{k}(s,a)\right) + \mathds{1}_{\mu>0}\sqrt{\frac{\mathfrak{C}^{2}_{\bar{V},\lambda, L}\ln\{\operatorname{Vol}({\Theta}^{\dagger})/\delta\}}{n}}\bigg]\Bigg),
\$
where $\Delta_{{\mathcal{U}}^{\star}_{2}}:=\{\rho_k: \|\frac{\rho_k}{\mu}\|_{L_2(\mu)} < {\mathcal{U}}^{\star}_{2}\}$, $\mathfrak{C}^{1}_{\bar{V},\lambda, L}, \mathfrak{C}^{2}_{\bar{V},\lambda, L}$ are some constant terms, and the function class complexity $\operatorname{Vol}({\Theta}^{\dagger}) = (e^{D}\max\{D_{\Omega},D_{\mathcal{Q}},D_{\Pi}\}+1)^3(\{1 \vee L\}{\mathcal{U}}^{\tau}_{2})^{2D}$ for $D=D_{\Omega}+D_{\mathcal{Q}}+D_{\Pi}$.
This completes the proof. 
\end{proof}

\section{Proof of Theorem 5.2}

\subsection{Proof of Lemma \ref{cover_linear_lemma}}

\begin{lemma}[Covering number for $\mathcal{Q}_{\theta}$, $\Pi_{\omega}$ and $\Omega_{\psi}$] \label{cover_linear_lemma}
For any $\varepsilon \in (0,1]$, the covering number for $\mathcal{Q}_{\theta}$, $\Pi_{\omega}$ and $\Omega_{\psi}$ satisfy the following conditions: 
\$
\mathcal{N}\left(\varepsilon/\bar{V} ; \mathcal{Q}_{\theta}\left(\text{diam}_{\theta}\right) ,\|\cdot\|_{L_2}\right)  \leq &\left(\frac{2\bar{V}}{\varepsilon}+1\right)^d, \\
\mathcal{N}\left(\varepsilon ; \Omega_{\psi}\left(\text{diam}_{\psi}\right),\|\cdot\|_{L_2}\right) \leq &\left(\frac{2\text{diam}_{\psi}}{\varepsilon}+1\right)^d ,\\
\mathcal{N}\left(\varepsilon ; \Pi_{\omega}\left(\text{diam}_\omega\right),\|\cdot\|_{L_2}\right) \leq & \left(\frac{4e\text{diam}_\omega}{\varepsilon}+1\right)^d.
\$
\end{lemma}

\begin{proof}
In this proof, we calculate the covering number over the class $\mathcal{Q}_{\theta}$, $\Pi_{\omega}$ and $\Omega_{\psi}$. 

\textbf{For $\mathcal{Q}_{\theta}$.} It follows the definition of $\mathcal{Q}_{\theta}$ 
   $
\mathcal{Q}_{\theta}\left(\text{diam}_{\theta}\right) \stackrel{\text { def }}{=}\left\{(s, a) \mapsto\langle\phi(s, a), \theta\rangle \right\}.
$
with $\|\theta\|_{L_2} \leq \bar{V} $. Thus 
$\mathcal{Q}_{\theta}\left(\text{diam}_{\theta}\right)$ is a Euclidean  ball with radis $\bar{V}$. As $\phi(s,a)$ is a $d$-dimensional feature space and $\theta \in \mathbb{R}^{d}$, it follows Lemma 5.7 in \citep{wainwright2019high}, we have, for any $\varepsilon >0$,
\$
\mathcal{N}\left(\varepsilon/\bar{V} ; \mathcal{Q}_{\theta}\left(\text{diam}_{\theta}\right) ,\|\cdot\|^{\prime}_{L_2}\right) \leq \frac{\left(\frac{2}{\varepsilon}+1\right)^d \operatorname{Vol}(\mathcal{Q}_{\theta}\left(\text{diam}_{\theta}\right))}{\operatorname{Vol}\left(\mathcal{Q}_{\theta}^{\prime}\left(\text{diam}_{\theta}\right)\right)},
\$
where $\|\cdot\|^{\prime}_{L_2}$ is the pair norm of $\|\cdot\|_{L_2}$ and $\mathcal{Q}_{\theta}^{\prime}\left(\text{diam}_{\theta}\right)$ is the corresponding ball in $\|\cdot\|^{\prime}_{L_2}$ norm. We take the balls $\mathcal{Q}_{\theta}^{\prime}\left(\text{diam}_{\theta}\right) = \mathcal{Q}_{\theta}\left(\text{diam}_{\theta}\right)$, then we obtain
\$
\mathcal{N}\left(\varepsilon ; \mathcal{Q}_{\theta}\left(\text{diam}_{\theta}\right),\|\cdot\|_{L_2}\right) \leq \left(\frac{2\bar{V}}{\varepsilon}+1\right)^d. 
\$
\textbf{For $\Omega_{\psi}$.} It follows the definition 
$
\Omega_{\psi}\left(\text{diam}_{\psi}\right) \stackrel{\text { def }}{=}\left\{(s, a) \mapsto\langle\phi(s, a), \psi\rangle \mid\|\psi\|_{L_2} \leq\text{diam}_{\psi}\right\}
$
and a similar argument as in the calculation on $\Omega_{\psi}$, we have 
\$
\mathcal{N}\left(\varepsilon ; \Omega_{\psi}\left(\text{diam}_{\psi}\right),\|\cdot\|_{L_2}\right) \leq \left(\frac{2\text{diam}_{\psi}}{\varepsilon}+1\right)^d. 
\$
 \textbf{For $\Pi_{\theta}$.} To apply the standard results in a Euclidean ball, we need bound, for any $\omega_1, \omega_2$ where $\| \omega_1 - \omega_2 \|_{L_2} \leq 0.5$, 
 $
\|\pi_{\omega_1}-\pi_{\omega_2}\|_{L_{2}} := \sqrt{\int_{\mathcal{S}}|\pi_1(\cdot|s) - \pi_2(\cdot|s)|^2 dP(\mathcal{S})}.
 $
with respect to some probability measure $P$. First, we observe, for any $(s,a) \in \mathcal{S} \times \mathcal{A}$: 
\$
 \pi_{\omega_1}(a|s)-\pi_{\omega_2}(a|s) =  \exp(\log(\pi_{\omega_1}(a|s)-\pi_{\omega_2}(a|s))) = \exp(\log(\pi_{\omega_1}(a|s)/\pi_{\omega_2}(a|s))).
\$
Follow the definition of the policy class $\Pi_{\omega}$, we denote
$
\iota(a, s, \omega) =  \exp({\langle\phi(s, a), \omega\rangle})$ and $\iota(\cdot, s, \omega) =  \int_{a\in \mathcal{A}} \exp({\langle\phi(s, a), \omega\rangle}) 
$,
for any $s,a \in \mathcal{S} \times \mathcal{A}$. For any $s,a$, as $\|\phi(s,a)\|_{L_2} \leq 1$, by Cauchy-Schewarz inequality, then we have 
\$
& \exp(\log(\pi_{\omega_1}(a|s)) - \log(\pi_{\omega_2}(a|s))) \\
= & \exp(\log(\pi_{\omega_1}(a|s)/\pi_{\omega_2}(a|s))) \\
=&  \exp(\log(\frac{\iota(a, s, \omega_1)}{\iota(\cdot, s, \omega_1)}(\frac{ \iota(\cdot, s, \omega_2)}{\iota(a, s, \omega_2)})) \\ 
= & \exp\left(\log\left(\iota(a, s, \omega_1 - \omega_2)   \int_{\widetilde{a} \in \mathcal{A}} \left(\iota(\widetilde{a}, s, \omega_1-\omega_2)  \frac{\iota(\cdot, s, \omega_1)}{\iota(\cdot, s, \omega_2)}\right)\right)\right) \\ 
 = & \exp\left(\log\left(\iota(a, s, \omega_1 - \omega_2)  \int_{\widetilde{a} \in \mathcal{A}}\{\pi_{\omega_2}(\widetilde{a}|s)\iota(\widetilde{a}, s, \omega_1-\omega_2) \}\right)\right) \\
\leq & \exp\left(\log\left( \exp(\|\omega_1 - \omega_2\|_{L_2})\int_{\widetilde{a} \in \mathcal{A}}\{\pi_{\omega_2}(\widetilde{a}|s)\exp(\|\omega_1 - \omega_2\|_{L_2}) \}    \right)\right) \\
= &\exp\left(\log\left( \exp(\|\omega_1 - \omega_2\|_{L_2})\exp(\|\omega_1 - \omega_2\|_{L_2}) \}    \right)\right).
\$
Now, it suffices to bound $\exp\left(\log\left( \exp(\|\omega_1 - \omega_2\|_{L_2})\exp(\|\omega_1 - \omega_2\|_{L_2}) \}    \right)\right)$, and by the exponential inequality, we have for $\| \omega_1 - \omega_2\|_{L_2}\leq 0.5$,
\$
& \exp\left(\log\left( \exp(\|\omega_1 - \omega_2\|_{L_2})\exp(\|\omega_1 - \omega_2\|_{L_2}) \}    \right)\right) \\
\leq & \exp\left(\log\left( \exp(\|\omega_1 - \omega_2\|_{L_2})\left(1+(e/2)\|\omega_1 - \omega_2\|_{L_2}\right) \}    \right)\right) \\
= & \|\omega_1 - \omega_2\|_{L_2})\left(1+(e/2)\|\omega_1 - \omega_2\|_{L_2}\right) .
\$
This directly implies that 
$
\pi_{\omega_1}(a|s) - \pi_{\omega_2}(a|s) \leq e\|\omega_1 - \omega_2 \|_{L_2}\pi_{\omega_2}(\widetilde{a}|s)
$
Shuffle $\omega_1$ and $\omega_2$, we have 
$
\pi_{\omega_2}(a|s) - \pi_{\omega_1}(a|s) \leq e\|\omega_2 - \omega_1 \|_{L_2} \pi_{\omega_1}(\widetilde{a}|s)
$
and therefore we obtain
\$
\|\pi_{\omega_1}-\pi_{\omega_2}\|_{L_{2}(\mu)} \leq 
\sup_{s,a \in \mathcal{S} \times {A}}|\pi_{\omega_1}(a|s) - \pi_{\omega_2}(a|s)| \leq 2e\|\omega_1 - \omega_2 \|_{L_2}.
\$
as $\pi_{\omega_1}, \pi_{\omega_2}$ are probability density function with integration $1$. This completes the proof. Now, we apply the standard covering number arguments in Lemma 7 of \citep{wainwright2019high}  over Euclidean ball of $\omega$, we have 
\$
\mathcal{N}\left(\varepsilon ; \Pi_{\omega}\left(\text{diam}_\omega\right),\|\cdot\|_{L_2}\right) \leq \left(\frac{4e\text{diam}_\omega}{\varepsilon}+1\right)^d. 
\$
This completes the proof. 
\end{proof}

\subsection{Proof of Theorem 5.2}

\begin{proof}
We follow the error decomposition as in the proof of Theorem 5.1, according to identical MDP Lemma \ref{iden_mdp}, we have the error decomposition 
\$
& \frac{1}{\bar{K}}\sum^{\bar{K}}_{k=1} \left(J(\pi) - J(\pi^{k}) \right) \\
 = &  \underbrace{\frac{1}{\bar{K}}\sum^{\bar{K}}_{k=1}\left(J(\pi^{k};\left\{\mathcal{S}, \mathcal{A}, \mathds{P}_{k}, \gamma, r_{k}(s,a), s^{0}\right\}) - J(\pi^{k})\right)}_{\text{err}_1}  \\
&\underbrace{\frac{1}{\bar{K}}\sum^{\bar{K}}_{k=1}\left(
J(\pi;\left\{\mathcal{S}, \mathcal{A}, \mathds{P}_{k}, \gamma, r_{k}(s,a), s^{0}\right\}) - J(\pi^{k};\left\{\mathcal{S}, \mathcal{A}, \mathds{P}_{k}, \gamma, r_{k}(s,a), s^{0}\right\}) \right)}_{\text{err}_2} \\
&\underbrace{\frac{1}{\bar{K}}\sum^{\bar{K}}_{k=1}\left( 
J(\pi) - J(\pi;\left\{\mathcal{S}, \mathcal{A}, \mathds{P}_{k}, \gamma, r_{k}(s,a), s^{0}\right\})
\right)}_{\text{err}_3}. 
\$
As in the analysis in the proof of Theorem 5.1,  we can well control $\lambda$ even in the penalization adversarial estimation to control the uncertainty level in the form of $
\|\phi(s,a)^{^{\top}}\psi\|_{L_{2}(\mu)} \leq \mathcal{U}^{\text{lr}}_{2}$ for $\tau_{\psi} \in \Omega_{\psi}$ for some constant $\mathcal{U}^{\text{lr}}_{2}$. Also, we define the produce space $\ddot{\mathcal{G}} = \Pi_{\omega} \times \mathcal{Q}_{\theta} \times \Omega_{\psi}$, so that 
\$
g(s,a,r,s^{\prime}) = {\tau}_{\psi}(s,a)(r(s,a)+\gamma {q}_{\theta} (s^{\prime},{\pi_{\omega}})-{q}_{\theta}(s, a))-\lambda \mathbb{D}(\tau_{\psi}(s,a))
\$ 
for any $g \in \ddot{\mathcal{G}}$. It follows the steps on calculating the complexity of the product space, e.g., \eqref{cover_num_func_2} in the proof of Lemma \ref{risk_bound_alpha}, we  plug-in the covering number in Lemma \ref{cover_linear_lemma} and apply Corollary 2 in \citep{haussler1995sphere}, by some algebra we have 
\$
\mathcal{N}({\varepsilon}, \ddot{\mathcal{G}}, \|\cdot\|_{L_2(\mu)}) \lesssim \left(1+\frac{48e\widetilde{C}^2\bar{V}\text{diam}_{\psi}\text{diam}_\omega}{\varepsilon}\right)^d.
\$
We set $\varepsilon = \mathcal{O}(\widetilde{C}^2/\sqrt{n})$ for some $\widetilde{C} \geq 0$, and obtain 
\#
\mathcal{N}(\widetilde{C}^2/\sqrt{n}, \bar{\mathcal{G}}, \|\cdot\|_{L_2(\mu)}) \lesssim \left(1+e\sqrt{n}\bar{V}\text{diam}_{\psi}\text{diam}_\omega\right)^d.
\label{small_cover}
\#
which implies that 
\#
\mathcal{N}(\varepsilon, \ddot{\mathcal{G}}, \|\cdot\|_{L_2(\mu)}) \lesssim \left(1+e\sqrt{n}(1\vee L)\bar{V}\text{diam}_{\psi}\text{diam}_\omega\right)^d.
\label{small_err1_cover}
\#

\textbf{Bounding $\text{err}_1$.} According to Lemma \ref{pik_return_lemma} and follow the \eqref{Delta1_bd_reg}, we obtain
\#
\text{err}_1 \leq &  \frac{c^{*}}{1-\gamma}\Bigg\{\big(3\mathcal{U}^{\text{lr}}_{2}\bar{V}+2\lambda \|\mathbb{D}(\tau(s,a))\|^{\mathcal{U}^{\text{lr}}_{2}}_{L_2(\mu)}\big)\sqrt{\frac{2 \ln \frac{8\mathcal{N}\left(\epsilon, 
\ddot{\mathcal{G}}, \|\cdot\|_{L_2(\mu)}\right)}{\delta}}{n}} \notag\\
 &+ \frac{2\big(3 d\text{diam}_{\psi}\bar{V}+2\lambda \|\mathbb{D}(\tau(s,a))\|^{\text{diam}_{\psi}}_{\infty}\big) \ln \frac{8\mathcal{N}\left(\epsilon, 
\ddot{\mathcal{G}}, \|\cdot\|_{L_2(\mu)}\right)}{\delta}}{3 n}, 
    \label{lr_err1_reg_step1}
\#
where $\|\mathbb{D}(\tau(s,a))\|^{\mathcal{U}^{\text{lr}}_{2}}_{L_2(\mu)} = \sup_{\tau: \|\tau(s,a)\|_{L_{2}(\mu)} \leq \mathcal{U}^{\text{lr}}_{2}} \|\mathbb{D}(\tau(s,a))\|_{L_2(\mu)}$, and $\|\mathbb{D}(\tau(s,a))\|^{\text{diam}_{\psi}}_{L_\infty} = \sup_{\tau: \|\tau(s,a)\|_{L_{2}(\mu)} \leq \mathcal{U}^{\text{lr}}_{2}} \|\mathbb{D}(\tau(s,a))\|_{L_{\infty}}$. As $\mathbb{D}$ is $M$-strongly convex function and thus locally Lipschitz with a bounded Lipschitz constant $L \leq \infty$, then we have 
\#
\text{err}_1 \lesssim &   \frac{c^{*}}{1-\gamma}\bigg(\big(\mathcal{U}^{\text{lr}}_{2}\bar{V}+\lambda L\mathcal{U}^{\text{lr}}_{2} \big)\sqrt{\frac{\ln \frac{8\mathcal{N}\left(\epsilon, 
\ddot{\mathcal{G}}, \|\cdot\|_{L_2(\mu)}\right)}{\delta}}{n}} \notag\\
 &+ \frac{2\big( d\text{diam}_{\psi}\bar{V}+\lambda Ld\text{diam}_{\psi}\big) \ln\frac{8\mathcal{N}\left(\epsilon, 
\ddot{\mathcal{G}}, \|\cdot\|_{L_2(\mu)}\right)}{\delta}}{3 n}\bigg). 
    \label{lr_err1_reg_step2}
\#
Plug-in the covering number in \eqref{small_err1_cover}, we conclude 
\#
\text{err}_1 \lesssim &   \frac{c^{*}}{1-\gamma}\bigg(\big(\mathcal{U}^{\text{lr}}_{2}\bar{V}+\lambda L\mathcal{U}^{\text{lr}}_{2} \big)\sqrt{\frac{\ln \{8 \left(1+e\sqrt{n}(1\vee L)\bar{V}\text{diam}_{\psi}\text{diam}_\omega\right)^d/\delta\}}{n}} \notag\\
 &+ \frac{2\big( \text{diam}_{\psi}d\bar{V}+\lambda Ld\text{diam}_{\psi}\big) \ln \{8 \left(1+e\sqrt{n}(1\vee L)\bar{V}\text{diam}_{\psi}\text{diam}_\omega\right)^d/\delta\}}{3 n}\bigg). 
    \label{lr_err1_reg}
\#

\textbf{Bounding $\text{err}_2$.} According to Lemma \ref{iden_mdpret}, we achieve no-regret policy optimization oracle, and thus
\#
\text{err}_2 \leq &
\frac{2\sqrt{2\bar{V}\log|\mathcal{A}|}}{\sqrt{\bar{K}}(1-\gamma)}. 
\label{lr_err2_reg}
\#

\textbf{Bounding $\text{err}_3$.} For any $\pi \in \Pi_{\omega}$, following the definition of $\text{err}_3$ we have
 \#
\text{err}_3 =& \frac{1}{\bar{K}}\sum^{\bar{K}}_{k=1}\left( 
J(\pi) - J(\pi;\left\{\mathcal{S}, \mathcal{A}, \mathds{P}_{k}, \gamma, r_{k}(s,a), s^{0}\right\})
\right) \notag \\
= &\frac{1}{\bar{K}}\sum^{\bar{K}}_{k=1} \left(q^{\pi}(s^0,\pi) - J(\pi;\left\{\mathcal{S}, \mathcal{A}, \mathds{P}_{k}, \gamma, r_{k}, s^{0}\right\})\right) \notag\\
= & \frac{1}{\bar{K}}\sum^{\bar{K}}_{k=1}\frac{\mathbb{E}_{d^{\pi}}\left[q^{\pi}(s,\pi)-r_{k}(s,a)-q^{\pi}(s,\pi) \right]}{1-\gamma} \notag \\
=  & \frac{1}{\bar{K}}\sum^{\bar{K}}_{k=1}\frac{\mathbb{E}_{d^{\pi}}\left[
\phi(s,a)^{^\top}\theta_{k} - \mathds{P}^{\pi^{k}}\phi(s,a)^{^\top}\theta_{k}\right]}{1-\gamma},
\label{linear_dpi_decom}
\#
where the last equality comes from a similar derivation as in \eqref{equv_rk}. Therefore, it suffices to bound
\#
\frac{\mathbb{E}_{d^{\pi}}\left|\phi(s,a)^{^\top}\theta_{k} - \mathds{P}^{\pi^{k}}\phi(s,a)^{^\top}\theta_{k} \right|}{1-\gamma},
\label{bound_aim_reg}
\#
for any $k \in [\bar{K}]$. According to Lemma \ref{iden_concen}, and the covering number Lemma \ref{cover_linear_lemma}, it immediately obtains
\$
& \sup_{\psi}\frac{1}{1-\gamma}\mathbb{E}_{\mu}\left[\phi(s,a) ^{^\top}\psi\left(r(s,a)+ \gamma \phi(s^{\prime},\pi^{k})^{^\top}\theta_{k} - \phi(s,a)^{^\top}\theta_{k}\right)\right] \\
\leq & 
\frac{1}{1-\gamma}\bigg(\big(\mathcal{U}^{\text{lr}}_{2}\bar{V}+\lambda L\mathcal{U}^{\text{lr}}_{2} \big)\sqrt{\frac{\ln \{8 \left(1+e\sqrt{n}(1\vee L)\bar{V}\text{diam}_{\psi}\text{diam}_\omega\right)^d/\delta\}}{n}} \notag\\
 &+ \frac{2\big( d\text{diam}_{\psi}\bar{V}+\lambda Ld\text{diam}_{\psi}\big) \ln \{8 \left(1+e\sqrt{n}(1\vee L)\bar{V}\text{diam}_{\psi}\text{diam}_\omega\right)^d/\delta\}}{3 n}\bigg) + \frac{\bar{V}}{c^{*}}. 
\$
As $\|\phi(s,a)^{\top}\psi\|_{L_2(\mu)} \leq \mathcal{U}^{\text{lr}}_{2} < \infty$, we apply Lemma \ref{L_2MMD}, then 
\$
& \frac{\mathbb{E}_{\mu}\left|\phi(s,a)^{^\top}\theta_{k} - \mathds{P}^{\pi^{k}}\phi(s,a)^{^\top}\theta_{k} \right|}{1-\gamma} \leq   \frac{1}{1-\gamma}\bigg(\big(\bar{V}+\lambda L\big)\sqrt{\frac{\ln \{8 \left(1+e\sqrt{n}(1\vee L)\bar{V}\text{diam}_{\psi}\text{diam}_\omega\right)^d/\delta\}}{n}} \notag\\
 & \qquad \qquad + \frac{\text{diam}_{\psi}}{\mathcal{U}^{\text{lr}}_{2}}\frac{2\big(\bar{V}+\lambda L\big) \ln \{8 \left(1+e\sqrt{n}(1\vee L)\bar{V}\text{diam}_{\psi}\text{diam}_\omega\right)^d/\delta\}}{3 n}\bigg) + \frac{\bar{V}}{c^{*}\mathcal{U}^{\text{lr}}_{2}} := \mathcal{E}(c^{*},\lambda,n).
\$
According to Lemma \ref{risk_bound_lm}  and \eqref{emprical_bound_reg} with covering number arguments in Lemma \ref{cover_linear_lemma}, it observes that
\$
& \frac{\sum^{n}_{i=1} \left[\phi(s_i,a_i)^{^\top}\theta_{k} - \mathds{P}^{\pi^{k}}\phi(s_i,a_i)^{^\top}\theta_{k} \right]}{n(1-\gamma)} \leq  \mathcal{E}(c^{*},\lambda,n) \\
 \iff& \frac{\sum^{n}_{i=1} \left[\phi(s_i,a_i)^{^\top}\left(\theta_{k} - \theta^{\prime} -\gamma \sum_{s^{\prime}, a^{\prime}} \varphi\left(s^{\prime}\right) \pi^{k}\left(a^{\prime} \mid s^{\prime}\right) \phi\left(s^{\prime}, a^{\prime}\right)^{\top} \theta_{k}\right)\right]}{n(1-\gamma)} \leq  \mathcal{E}(c^{*},\lambda,n) \\
 := & \frac{\sum^{n}_{i=1} \left[\phi(s_i,a_i)^{^\top}\mathbb{G}^{\pi^k}\right]}{n(1-\gamma)} \leq \mathcal{E}(c^{*},\lambda, n).
\$
where $\theta^{\prime}$ is the coefficients for linear representation of $r(s,a)$, and $\varphi(s^{\prime})$ is the low rank decomposition for transition kernel \citep{jin2020provably}. This implies 
\$
\left\|\sqrt{\frac{1}{n}\sum^{n}_{i=1}[\phi(s_i,a_i)\phi(s_i,a_i)^{^\top}]}\mathbb{G}^{\pi^k}\right\|_{L_2} \leq  (1-\gamma)\mathcal{E}(c^{*},\lambda, n).
\$
Then by Cauchy-Schwarz inequality, for any $s,a$, we have 
\$
|\phi(s,a)^{^\top}\mathbb{G}^{\pi^k}| \leq & \|\phi(s,a)^{^\top}\mathbb{G}^{\pi^k}\|_{L_2} \\
= & \|\phi(s,a)^{^\top}\sqrt{(\frac{1}{n}\sum^{n}_{i=1}[\phi(s_i,a_i)\phi(s_i,a_i)^{^\top}])^{-1}} \sqrt{\frac{1}{n}\sum^{n}_{i=1}[\phi(s_i,a_i)\phi(s_i,a_i)^{^\top}]}\mathbb{G}^{\pi^k}\|_{L_2} \\
\leq & \|\phi(s,a)^{^\top}\sqrt{(\frac{1}{n}\sum^{n}_{i=1}[\phi(s_i,a_i)\phi(s_i,a_i)^{^\top}])^{-1}}\|_{L_2} \| \sqrt{\frac{1}{n}\sum^{n}_{i=1}[\phi(s_i,a_i)\phi(s_i,a_i)^{^\top}]}\mathbb{G}^{\pi^k}\|_{L_2} \\
\leq & \|\phi(s,a)^{^\top}\sqrt{(\frac{1}{n}\sum^{n}_{i=1}[\phi(s_i,a_i)\phi(s_i,a_i)^{^\top}])^{-1}}\|_{L_2} (1-\gamma)\mathcal{E}(c^{*},\lambda, n).
\$
Then we have an upper bound for \eqref{bound_aim_reg}, 
\$
& \frac{\mathbb{E}_{d^{\pi}}\left|\phi(s,a)^{^\top}\theta_{k} - \mathds{P}^{\pi^{k}}\phi(s,a)^{^\top}\theta_{k} \right|}{1-\gamma} \\
\leq & \mathbb{E}_{d^{\pi}}[\|\phi(s,a)^{^\top}\sqrt{(\frac{1}{n}\sum^{n}_{i=1}[\phi(s_i,a_i)\phi(s_i,a_i)^{^\top}])^{-1}}\|_{L_2}] \mathcal{E}(c^{*},\lambda, n) \\
\leq & \mathbb{E}_{d^{\pi}}\left[\sqrt{\phi(s,a)^{^\top}(\frac{1}{n}\sum^{n}_{i=1}[\phi(s_i,a_i)\phi(s_i,a_i)^{^\top}])^{-1}\phi(s,a)}\right]\mathcal{E}(c^{*},\lambda, n).
\$
To facilitate the proof, we use the notation  $\|x\|_{\Sigma} \stackrel{\text { def }}{=} \sqrt{x^{\top}\left(\Sigma\right)^{-1} x}$.  According to Lemma 32 in \citep{zanette2021cautiously}, we obtain the upper bound for $ \mathbb{E}_{d^{\pi}}\left[ \|\phi(s,a)\|_{\Sigma^{-1}_{n}}\right]$ as follows:
\#
 \mathbb{E}_{d^{\pi}}\left[ \|\phi(s,a)\|_{\Sigma^{-1}_{n}}\right] = & \mathbb{E}_{d^{\pi}}\left[\sqrt{\phi(s,a)^{^\top}(\frac{1}{n}\sum^{n}_{i=1}[\phi(s_i,a_i)\phi(s_i,a_i)^{^\top}])^{-1}\phi(s,a)}\right] \notag \\
 \leq & \sqrt{\mathbb{E}_{d^{\pi}}\left[\phi(s,a)^{^\top}(\frac{1}{n}\sum^{n}_{i=1}[\phi(s_i,a_i)\phi(s_i,a_i)^{^\top}])^{-1}\phi(s,a)\right]} \notag  \\
 \leq & \sqrt{\text{trace}\left(\mathbb{E}_{d^{\pi}}\left[\phi(s,a)\phi(s,a)^{^\top}\right](\frac{1}{n}\sum^{n}_{i=1}[\phi(s_i,a_i)\phi(s_i,a_i)^{^\top}])^{-1}\right)} \notag  \\
 \leq & \sqrt{\iota(d_{\pi},\mu)\text{trace}(\frac{1}{n}\sum^{n}_{i=1}[\phi(s_i,a_i)\phi(s_i,a_i)^{^\top}] (\frac{1}{n}\sum^{n}_{i=1}[\phi(s_i,a_i)\phi(s_i,a_i)^{^\top}]))^{-1}} \notag  \\
 =& \sqrt{\iota(d_{\pi},\mu) d}.
 \label{relative_cond_num}
\#
where 
$
\iota(d_{\pi},\mu) = \sup _{x \in \mathbb{R}^d} \frac{x^T \mathbb{E}_{ d_{\pi}}\left[\phi(s, a) \phi(s, a)^{\top}\right] x}{x^{\top} \mathbb{E}_{\mu}\left[\phi(s, a) \phi(s, a)^{\top}\right] x}
$
Based on this, we conclude that 
\#
\text{err}_3 \leq \sqrt{\iota(d_{\pi},\mu)d} \mathcal{E}(c^{*},\lambda, n).
\label{lr_err3_reg}
\#
We combine the upper bounds in \eqref{lr_err1_reg}, \eqref{lr_err2_reg} and \eqref{lr_err3_reg}, and we set $c^{*} = \widetilde{\mathcal{O}}\big(\sqrt[4]{n/d\ln\{(1+e\sqrt{n}(1\vee L)\bar{V}c_{\psi}c_\omega)/\delta\}}\big)$ and set $\lambda = \lambda(c_{\psi}(\mathcal{U}^{\text{lr}}_{2}))$ for $
c_{\psi}\{\mathcal{U}^{\text{lr}}_{2}\} = \sup_{\{\psi: \|\phi(s,a)^{\top}\psi \|_{L_2(\mu)} = \mathcal{U}^{\text{lr}}_{2}\}}\|\psi\|_{L_\infty}$, by some algebra, we conclude that
\$
& J(\pi) - J(\widehat{\pi}^{\text{lr}}) 
\leq  \frac{1}{\sqrt{1-\gamma}}\sqrt[4]{\frac{\iota(d_{\pi},\mu)d(\bar{V}^2+\bar{V}\lambda L)^2\ln \{8 \left(1+e\sqrt{n}(1\vee L)\bar{V}\text{diam}_{\psi}\text{diam}_\omega\right)^d/\delta\}}{n}} \\
& + \frac{1}{\sqrt{1-\gamma}}\sqrt{\frac{(\iota(d_{\pi},\mu)d)^{0.5}\bar{V}^2\big( d\text{diam}_{\psi}\bar{V}+\lambda Ld\text{diam}_{\psi}\big)/\mathcal{U}^{\text{lr}}_{2}\ln \{8 \left(1+e\sqrt{n}(1\vee L)\bar{V}\text{diam}_{\psi}\text{diam}_\omega\right)^d/\delta\}}{n}}\\
 & + \frac{\sqrt{\iota(d_{\pi},\mu)d} }{1-\gamma} \frac{d\text{diam}_{\psi}}{\mathcal{U}^{\text{lr}}_{2}}\frac{2\big(\bar{V}+\lambda L\big) \ln \{8 \left(1+e\sqrt{n}(1\vee L)\bar{V}\text{diam}_{\psi}\text{diam}_\omega\right)^d/\delta\}}{3 n} + \frac{2\sqrt{2\bar{V}\log|\mathcal{A}|}}{\sqrt{\bar{K}}(1-\gamma)}. 
\$
If we ignoring the fast term and let $\bar{K} \gg \log|\mathcal{A}|$, we have 
\#
J(\pi) - J(\widehat{\pi}^{\text{lr}}) \lesssim & \frac{1}{\sqrt{1-\gamma}}\sqrt[4]{\frac{\iota(d_{\pi},\mu)d(\bar{V}^2+\bar{V}\lambda L)^2\ln \{8 \left(1+e\sqrt{n}(1\vee L)\bar{V}\text{diam}_{\psi}\text{diam}_\omega\right)^d/\delta\}}{n}} .
\label{one_bound}
\#
For $\psi \in \{\psi: \|\phi(s,a)^{\top}\psi \|_{L_2(\mu)} \leq \mathcal{U}^{\text{lr}}_{2}\}$, we can observe that 
\#
\| \phi(s,a)^{\top}\psi\|_{L_2(\mu)} \leq &\text{trace}(\mathbb{E}_{\mu}[\phi(s,a)\phi(s,a)^{\top}])\|\theta\|_{L_2} \\
\leq & \text{trace}(\mathbb{E}_{\mu}[\phi(s,a)\phi(s,a)^{\top}])d
c_{\psi}\{\mathcal{U}^{\text{lr}}_{2}\}. 
\label{last_lr}
\#
It combines with $\eqref{last_lr}$ and \eqref{one_bound}, we conclude that 
\$
& J(\pi) - J(\widehat{\pi}^{\text{lr}}) \\
\lesssim & \frac{\min\{\text{trace}(\mathbb{E}_{\mu}[\phi(s,a)\phi(s,a)^{\top}])d
c_{\psi}\{\mathcal{U}^{\text{lr}}_{2}\},\sqrt{\iota(d_{\pi},\mu)d} \}}{1-\gamma} \\
& \cdot \sqrt[4]{\frac{(\bar{V}^2+\bar{V}\lambda L)^2\ln \{(1+e\sqrt{n}(1\vee L)\bar{V}\text{diam}_{\psi}\text{diam}_\omega)^d/\delta\}}{n}} \\
\leq & \frac{\sqrt{\min\{\kappa^2
c_{\psi}^{2}\{\mathcal{U}^{\text{lr}}_{2}\}d^2,\iota(d_{\pi},\mu)d  \}}}{1-\gamma}  \sqrt[4]{\frac{(\bar{V}^2+\bar{V}\lambda L)^2\ln \{(1+e\sqrt{n}(1\vee L)\bar{V}\text{diam}_{\psi}\text{diam}_\omega)^d/\delta\}}{n}},
\$
where $\kappa = \text{trace}(\mathbb{E}_{\mu}[\phi(s,a)\phi(s,a)^{\top}])$. This completes the proof. 
\end{proof}

\section{Additional Related Works}

\textbf{Offline RL.} The domain approaches of offline RL include  fitted Q-iteration
(FQI; \cite{ernst2005tree,riedmiller2005neural} ), fitted policy iteration \citep{antos2007value, lagoudakis2003least},
Bellman Residual Minimization (BRM;
\cite{antos2008learning,farahmand2016regularized, dai2018sbeed}, and actor-critic \citep{konda1999actor,konda2003onactor,haarnoja2018softac}. We refer the reader to \cite{levine2020offline} for more comprehensive discussions on the topics of the offline RL. In the aforementioned mainstreams of works, ours is closely related to the actor-critic. Actor-critic methods are a
hybrid class of methods that mitigate some deficiencies of methods that are either purely policy
or purely value-based; in modern RL, they are widely used in practice \citep{wu2019behavior,wu2021uncertainty}. A standard framework in actor-critic methods is that actor supervises the policy to improve in order to maximize its values estimated by the critic, value function. From a high-level point of view, this connects our bi-level structured optimization to actor-critic methods. In our framework, the upper-level components make decisions, i.e., searching for a policy maximizing the pessimistic evaluation based on the lower-level outputs, i.e., the uncertainty-controlled confidence set of value estimates. Therefore, our works demonstrate the advantages of actor-critic-type methods in offline RL from a bi-level reformulation perspective. 

\textbf{Minimax learning.} 
In a seminal work, \cite{liu2018breaking} proposed the first minimax estimation procedure requiring two function approximators, one for modeling the marginalized importance-weight function, and the other for modeling the value function. This method becomes particularly efficient in estimating the discounted return when the offline data-generating distribution aptly encompasses the distribution invoked by the evaluation policy, thereby avoiding the significant issue of exponential variance in the horizon, a notable drawback of importance sampling \citep{precup2000eligibility,li2015toward}. The ripple effect of this method has led to a surge of interest within the RL community \citep{xie2019towards,uehara2020minimax,jiang2020minimax,nachum2019dualdice,liu2020understanding,shi2022minimax,zhou2022estimating}. Intriguingly, our bi-level policy optimization aligns with this trend of minimax learning, where we build a confidence set for policy evaluation using the marginalized importance-weight. In particular, algorithmically, our low-level component is most related to the value interval learning in \citep{jiang2020minimax}. They provide a minimax interval for quantifying the value bias involved in the discounted return under function approximation settings. However, they only handle the function approximation errors but do not quantify the statistical uncertainty, as well as no uncertainty control is performed. In contrast, our work yields a confidence interval that concurrently incorporates the bias introduced by function approximation and uncertainty stemming from sampling. This also provides a basis for the operations at the upper level in our bi-level structured optimization. 

\textbf{Conservative value estimation.} Following the principle of pessimism in the face of uncertainty, a significant portion of recently proposed offline RL methods rely on
on estimating conservative $q$-values for optimizing the target policy, with the constraint or regularizer serving to limit deviation from the behavior policy \citep{kumar2019stabilizing,kumar2020conservative,nachum2019algaedice,kostrikov2021offline,fujimoto2019off,lee2020batch,lee2021optidice}. For our work, we also following the pessimistic principle for value estimation. The major differences between ours and the existing works in this mainstream are two-fold. First, with uncertainty control through favoring the policy close to the behavior policy, our algorithm also ensures the consistency of the value estimates. This consistency guarantee plays a key role in our method to ensure no overly pessimistic reasoning. Second, from a high-level point of view, our algorithm has a bi-level structure, and more close to actor-critic-based methods. In contrast, the aforementioned works are more close to approximate dynamic programming \citep{kostrikov2021offline}.

\section{Statistical Learning Tools}

In this section, we introduce fundamental concepts from statistical learning theory, as outlined in \citep{anthony1999neural,vapnik2015uniform}. We begin with the concept of the covering number. This metric quantifies the number of spherical balls of a specified size required to encompass a designated space, allowing for potential overlaps.

\begin{definition}[Covering number]
\label{cover_def}
    Let $\left(\mathcal{C},\|\cdot\|\right)$ be a $\|\cdot\|$ normed space, and $\mathcal{H} \subseteq \mathcal{C}$. The set $\left\{b_1, b_2, \ldots, b_m\right\}$ is a $\epsilon$-covering over $\mathcal{H}$ if $\mathcal{H} \subseteq \cup_{i=1}^m \mathbb{B}\left(b_i, \varepsilon\right)$, where $\mathbb{B}\left(b_i, \varepsilon\right)$ is the sup-norm-ball centered at $b_i$ with radius $\varepsilon$. Then the covering number of $\mathcal{H}$ is defined as $\mathcal{N}\left(\epsilon, \mathcal{H},\|\cdot\|_{L_2}\right)=\min \{n: \exists \ \epsilon$-covering over $\mathcal{H}$ of size $m\}$.
\end{definition}

A widely recognized method for examining the generalization capability of statistical learning models involves the use of the \textit{VC-dimension}. This dimension not only characterizes uniform convergence, as detailed in \citep{vapnik2015uniform}, but also asymptotically dictates the sample complexity of PAC learning \citep{blumer1989learnability}.

\begin{definition}[growth function, VC-dimension, shattering] \label{vc_def}
Let $\mathcal{H}$ denote a class of functions from $\mathcal{X}$ to $\{0,1\}$. For any non-negative integer $m$, we define the growth function of $\mathcal{H}$ as
$$
\Pi_\mathcal{H}(m):=\max _{x_1, \ldots, x_m \in \mathcal{X}}\left|\left\{\left(h\left(x_1\right), \ldots, h\left(x_m\right)\right): h \in \mathcal{H}\right\}\right| .
$$
If $\left|\left\{\left(h\left(x_1\right), \ldots, h\left(x_m\right)\right): h \in \mathcal{H}\right\}\right|=2^m$, we say $H$ shatters the set $\left\{x_1, \ldots, x_m\right\}$. The Vapnik-Chervonenkis dimension of $\mathcal{H}$, denoted $\operatorname{VCdim}(\mathcal{H})$, is the size of the largest shattered set, i.e. the largest $m$ such that $\Pi_\mathcal{H}(m)=2^m$. If there is no largest $m$, we define $\operatorname{VCdim}(\mathcal{H})=\infty$
\end{definition}

For a set of real-valued functions, like those produced by neural networks, the \textit{pseudo dimension} serves as an intuitive measure of complexity. This dimension also suggests similar uniform convergence properties and was introduced by \cite{pollard1990empirical}. 

\begin{definition}[Pollard's pseudo dimension]
\label{pseudo_def}
    Let $\mathcal{F}$ be a class of functions from $\mathcal{X}$ to $\Re$. The pseudodimension of $\mathcal{F}$, written $D_{\mathcal{F}}$, is the largest integer $m$ for which there exists $\left(x_1, \ldots, x_m, \linebreak y_1, \ldots, y_m\right) \in \mathcal{X}^m \times \Re^m$ such that for any $\left(b_1, \ldots, b_m\right) \in\{0,1\}^m$ there exists $f \in \mathcal{F}$ such that
$$
\forall i: f\left(x_i\right)>y_i \Longleftrightarrow b_i=1
$$
\end{definition}

In the end, it's worth noting that the pseudo dimension extends the concept of the VC-dimension to real-valued functions.

\end{document}